\documentclass{article}

\PassOptionsToPackage{square,numbers}{natbib}
     \usepackage[final]{neurips_2021}

\usepackage[utf8]{inputenc} 
\usepackage[T1]{fontenc}    
\usepackage{xr-hyper}
\usepackage{hyperref}       
\usepackage{url}            
\usepackage{booktabs}       
\usepackage{amsfonts}       
\usepackage{nicefrac}       
\usepackage{microtype}      

\usepackage{amsmath}
\DeclareMathOperator*{\argmax}{arg\,max}
\DeclareMathOperator*{\argmin}{arg\,min}
\usepackage[linesnumbered,ruled,vlined,commentsnumbered]{algorithm2e}
\let\oldnl\nl
\newcommand{\nonl}{\renewcommand{\nl}{\let\nl\oldnl}}%
\usepackage{multicol}
\usepackage{changepage}
\usepackage[skins]{tcolorbox}
\usepackage{dsfont}
\usepackage{subcaption}
\usepackage{sidecap}
\usepackage{wrapfig}
\usepackage{xcolor}
\usepackage{amsthm}
\usepackage[capitalise]{cleveref}
\usepackage{acronym}
\usepackage{filecontents}

\newtheorem{theorem}{Theorem}
\newtheorem{corollary}[theorem]{Corollary}

\theoremstyle{definition}

\acrodef{lfd}[LfD]{learning from demonstration}
\acrodef{tom}[ToM]{theory of mind}

\title{Iterative Teacher-Aware Learning}

\author{%
  Luyao Yuan$^1$\\
  \texttt{yuanluyao@ucla.edu} \\
   \And
   Dongruo Zhou$^1$ \\
   \texttt{drzhou@cs.ucla.edu}\\
   \And
   Juhong Shen$^2$ \\
   \texttt{jhshen@ucla.edu} \\
   \And
   Jingdong Gao$^1$ \\
   \texttt{mxuan@ucla.edu}\\
   \And
   Jeffrey L. Chen$^1$ \\
   \texttt{jlchen0@ucla.edu}\\
   \And
   Quanquan Gu$^1$ \\
   \texttt{qgu@cs.ucla.edu}\\
   \And
   Ying Nian Wu$^3$ \\
   \texttt{ywu@stat.ucla.edu} \\
   \And
   Song-Chun Zhu$^{1, 3}$ \\
   \texttt{sczhu@stat.ucla.edu}\\
   \AND
    \normalfont{$^1$Department of Computer Science},
    \normalfont{$^2$Department of Mathematics},
    \normalfont{$^3$Department of Statistics}\\
    University of California, Los Angeles\\
    \normalfont{$^4$Beijing Institute for General Artificial Intelligence (BIGAI)}\\
}

\makeatletter
\newcommand*{\addFileDependency}[1]{
  \typeout{(#1)}
  \@addtofilelist{#1}
  \IfFileExists{#1}{}{\typeout{No file #1.}}
}
\makeatother

\newcommand*{\myexternaldocument}[1]{%
    \externaldocument{#1}%
    \addFileDependency{#1.tex}%
    \addFileDependency{#1.aux}%
}

\myexternaldocument{supplement}

\begin{document}

\maketitle
\vspace{-5pt}
\begin{abstract}
  In human pedagogy, teachers and students can interact adaptively to maximize communication efficiency. The teacher adjusts her teaching method for different students, and the student, after getting familiar with the teacher’s instruction mechanism, can infer the teacher’s intention to learn faster. Recently, the benefits of integrating this cooperative pedagogy into machine concept learning in discrete spaces have been proved by multiple works. However, how cooperative pedagogy can facilitate machine parameter learning hasn’t been thoroughly studied. In this paper, we propose a gradient optimization based teacher-aware learner who can incorporate teacher’s cooperative intention into the likelihood function and learn provably faster compared with the naive learning algorithms used in previous machine teaching works. We give theoretical proof that the iterative teacher-aware learning (ITAL) process leads to local and global improvements. We then validate our algorithms with extensive experiments on various tasks including regression, classification, and inverse reinforcement learning using synthetic and real data. We also show the advantage of modeling teacher-awareness when agents are learning from human teachers.

\end{abstract}

\section{Introduction}
Cooperative pedagogy is invoked across language, cognitive development, cultural anthropology, and robotics to explain people’s ability to effectively transmit information and accumulate knowledge~\citep{fan2018learning,wang2020mathematical}. As the usage of artificial intelligence and machine learning based systems ratchets up, it is foreseeable that extensive human-computer and agent-agent pedagogical scenarios will occur in the near future~\citep{peltola2019machine}. However, there is still a distance away from robots being able to directly teach or learn from humans as efficiently and effectively as humans do. One of the many difficulties is that machine learning and teaching are now usually studied in single-agent frameworks. Most of the prevailing machine learning methods focus on the improvement of \textbf{individual learners} and the explanations of how knowledge is obtained focus entirely on each learner’s unilateral experiences, either passive observations from a Markov decision process~\citep{mnih2013playing,silver2016mastering}, random samples from a data distribution~\citep{ILSVRC15,he2016deep}, responses of active queries provided by an oracle~\citep{angluin1988queries,settles2009active}, or demonstrations from an expert~\citep{argall2009survey}.


Such machine learning framework is diametrically distinctive from human education, in whose context, learning often occurs sequentially instead of in batch, and from intentional messages given by a pedagogical teacher rather than random data from a fixed sampling process~\citep{wang2020sequential}. Recently, the advantage of pedagogical teachers over randomly sampled data or optimal task completion trajectories from experts has been shown in machine teaching~\citep{cakmak2012algorithmic,zhu2013machine,zhu2015machine,liu2017iterative,eaves2016tractable,milli2017interpretable,chen2018near,chen2018understanding} and in \ac{lfd}~\citep{hadfield2016cooperative,ho2016showing}. Nonetheless, compared with human pedagogy, these works still lack a sophisticated student model that can accommodate the teacher's cooperation into his learning and acts differently than learning from passive data. Machine teaching algorithms model a cooperative teacher giving instructions in the format of data examples for continuous parameter~\citep{cakmak2012algorithmic,zhu2013machine,zhu2015machine,liu2017iterative}, Bayesian concept~\citep{eaves2016tractable,milli2017interpretable} or version space learning~\citep{chen2018near,chen2018understanding}, but seldom do they consider how learners may interpret differently between the data picked intentionally by the teacher and sampled randomly from the world. Standard \ac{lfd} takes in demonstrations from an (approximately-) optimal expert to learn the underlying reward function~\citep{argall2009survey}. \citet{hadfield2016cooperative,ho2016showing,ho2018effectively} shows the advantage of using pedagogical rather than optimal demonstrations, yet, in either case, the learners are not aware of the teacher . \citet{shafto2014rational,yang2018optimal,wang2020sequential,wang2020mathematical} move one step further and proposed recursive cooperative inference models having both the teacher and the student reasoning each other, an ability known as \ac{tom}~\citep{premack1978does,baker2017rational}. The first work modeled and predicted human behavior while the latter three managed to integrate \ac{tom} into machine learning.  Despite the theoretical contribution, their approach~\citep{yang2018optimal,wang2020sequential,wang2020mathematical} is confined to Bayesian concept learning with finite hypothesis space, in which the Sinkhorn scaling~\citep{sinkhorn1967concerning} is tractable. It is unclear how to apply their algorithms to settings involving continuous hypothesis spaces, such as learning neural networks.

In this paper, we study how to integrate the cooperative essence of pedagogy into machine parameter learning and propose a teacher-aware learner who learns significantly faster than a naive learner, given an iterative machine teacher~\citep{liu2017iterative,liu2018towards}. The learner estimates the teacher's data selection process with distribution and corrects his likelihood function with this estimation to accommodate the teacher's intention. Maximizing the new likelihood enables the learner to utilize both explicit information from the selected data and implicit information suggested by the pedagogical context. We theoretically proved the improvement brought by the learner's teacher awareness and justified our results with various experiments. We believe that our work can provide insights into both human-machine interactions, such as online education, and machine-machine communications, such as ad-hoc teamwork~\citep{barrett2017making}.

Our main contributions are \textbf{i)} modeling teacher-awareness for generic gradient optimization based parameter learning; \textbf{ii)} theoretically proving the improvement guaranteed by the teacher-aware learner over the naive learner under mild assumptions; \textbf{iii)} empirically illustrating the advantage of teacher-awareness learner when interacting with both machine and human teachers.

\section{Related Work}
\paragraph{Machine Teaching} Our work is related to machine teaching as we used an iterative machine teacher in our framework. Machine teaching~\citep{zhu2013machine,zhu2015machine,liu2017iterative,peltola2019machine} solves the problem of constructing an optimal (usually minimal) dataset according to a target concept such that a student model can learn the target concept based on this dataset. Most of the machine teaching algorithms consider a batch setting, where the teacher designs a minimal dataset and provides it to a learning algorithm in one shot~\citep{cakmak2012algorithmic,zhu2013machine,patil2014optimal,zhu2015machine,chen2018near}. Iterative machine teaching has also been studied, where the teacher gives out data iteratively and checks the learner's status before selecting the next data~\citep{johns2015becoming,bak2016adaptive,liu2017iterative,liu2018towards,melo2018interactive}, but previous works didn't consider teacher-aware learners. There are also works applying machine teaching to inverse reinforcement learning (IRL) and LfD~\citep{argall2009survey,cakmak2012algorithmic,hadfield2016cooperative,ho2016showing, ho2018effectively,haug2018teaching}. Our IRL experiments are different from those works as our data are provided iteratively and sequentially. Also, our learner is aware of the teacher's intention. \citet{ho2016showing,ho2018effectively} integrated Bayesian rule in LfD to model mutual reasoning between the teacher and the learner, but they mainly used their model to explain human data. A theoretical study of the teaching-complexity of the teacher-aware learners was presented in~\citep{zilles2011models,doliwa2014recursive} where the teacher and the learner are aware of their cooperation. Their analysis mainly attends to version-space learners which maintain all hypotheses consistent with the training data, and cannot be applied to hypothesis selection via optimization, such as learning parameters.

\citet{peltola2019machine} studied machine teaching for an active sequential multi-arm bandit learner. Although they also have a helpful teacher and a teacher-aware learner, their problem setting is different from ours. First, in every round of multi-arm bandit, the learner can actively choose an arm to pull, and then the teacher provides feedback for that arm, while in our cases, the data batch in each round is sampled randomly and independently from the learner's current status. Second, as the teacher can only give binary (success or not) feedback to the learner, the counterfactual reasoning required for pedagogy is significantly simplified. Besides, they required that the same feature representation for the arms is shared between the teacher and the learner. Also, the learner doesn't have to know the underlying parameter exactly to perform well in multi-arm bandit games, while in most of our cases, the learner tries to match the teacher's model exactly. \citet{fisac2020pragmatic} used similar formulation to model cooperative value alignment within a human-robot team. They assumed the human knows the robot's value function during interactions, and parameters to be aligned are sparse and low dimensional.

\textbf{Learning to Teach} Sharing the same goal as machine teaching, learning to teach (L2T) also seeks a teaching algorithm to improve the learning efficiency of AI. While machine teaching usually models the question as an optimization problem and solves for a closed-form teaching algorithm, works in L2T tend to train the teaching model in the process of the teacher-student interaction with gradient based optimization~\citep{wu2018learning} or reinforcement learning (RL) algorithms~\citep{fan2018learning}. Nevertheless, these works also focus mainly on the teacher algorithm and assume teacher-unaware learners. In some tasks, typically when the student aims to learn a concept in a discrete space, the teacher can track the learner's status by modeling his belief over the concept~\citep{shafto2014rational}. As the beliefs within the learner's mind are not usually known by the teacher, the teaching process can be modeled as a partially observable Markov decision process (POMDP)~\citep{monahan1982state}, solving the optimal policy of which returns a teaching algorithm~\citep{rafferty2011faster,whitehill2017approximately}. From the teacher's perspective, the unknown part of the environment is the learner's belief, a probability distribution over the concept space, so she has to form another belief over the learner's belief. The intractable modeling of belief over continuous variables usually requires approximation methods such as particle filters~\citep{rafferty2011faster,whitehill2017approximately} to solve, restricting the scope of these algorithms to naive learners and relatively simple learning tasks. Interactive POMDP~\citep{gmytrasiewicz2005framework,woodward2012learning}, a probabilistic multi-agent model, can also be used to model cooperative pedagogy with recursive teacher-learner reasoning. However, the nested belief over belief also suffers from intractability and is hard to scale up. If the concept space of the learner is continuous by itself, such as high dimensional continuous parameter spaces in our case, handling the belief over belief becomes difficult.

\textbf{Cooperative Bayesian Inference} \citet{shafto2014rational} studies human pedagogy with examples using Bayesian learning models. The cooperative inference~\citep{yang2018optimal,wang2019generalizing,wang2020mathematical} in machine learning also formalizes full recursive reasoning from the perspectives of both the teacher and the learner. Distinctive from their concept learning settings, in this paper, we focus on parameter learning, in which the student has  intractable posterior distribution and learns via gradient-based optimization. In addition, \citep{shafto2014rational,yang2018optimal,wang2019generalizing,wang2020mathematical} only consider the problem of a single interaction between the teacher and learner. \citet{wang2020sequential} proposed a sequential cooperative Bayesian inference algorithm and showed its performance advantage over naive Bayesian learner analytically and empirically. Nevertheless, they were still discussing concept learning with finite and usually small data and hypothesis sets. The Sinkhorn scaling~\citep{hershkowitz1988classifications,wang2020sequential,wang2020mathematical} becomes infeasible when dealing with continuous parameter learning.

\textbf{Pragmatics Reasoning} Our work is inspired by the study of pragmatics (how context contributes to the meaning)~\citep{grice1975logic,hadfield2016cooperative} and \ac{tom}~\citep{premack1978does,baker2017rational}. The rational speech act (RSA) model raised by~\citet{golland2010game} and developed in~\citep{frank2012predicting,shafto2014rational,goodman2016pragmatic,andreas2016reasoning} accommodates the idea of using not only the utterance but also the selection of the utterance to understand the speaker. Previous works in these areas are mainly from human action interpretation~\citep{frank2009informative,vogel2013emergence,khani2018planning}, language emergence~\citep{yuan2020emergence,kang2020Incorporating}, linguistics~\citep{jager2012game,andreas2016reasoning,cohn-gordon-etal-2019-incremental} and cognitive science~\citep{goodman2016pragmatic,baker2017rational} perspectives. To our knowledge, our work is the first to relate pedagogy and recursive reasoning to generic parameter learning and shows a provable improvement. \citet{shafto2014rational} proposed computational models for more diverse concept spaces, but mainly focus on modeling and predicting human behaviors.

\section{Background}\label{sec:background}
Finding the optimal way of teaching parameters has been a challenging problem because of the continuous state space and long horizon planning. One common framework is machine teaching~\cite{zhu2013machine,zhu2015machine}. Here, we adopt an iterative variation of machine teaching~\citep{liu2017iterative}, consisting of three entities: the learner, the teacher and the world. \textbf{The world} is defined as a parameter $\omega^*$, fixed and known only by the teacher. Given a model $y = h(x; \omega)$ parameterized by $\omega$, the world is defined as $ \omega^* = \argmin_{\omega\in\Omega}\mathbb{E}_{(x, y)\sim\mathcal{P}(x, y)}[l(h(x; \omega), y)]$, where $\mathcal{P}(x, y)$ is the data distribution in standard machine learning. Here, $l$ and $h$ can vary across tasks, eg. $l$ can be squared loss for regression, cross-entropy for classification, and negative log-likelihood for IRL~\citep{babes2011apprenticeship,macglashan2015between}. In this paper, we assume $l$ to be a convex function and $h(x; \omega) = h(\langle x, \omega\rangle)$. $h$ can be an identity function for linear regression and softmax function for classification. Thus, we can omit $h$ in the loss function and write $l(\langle x, \omega\rangle, y)$ for short. $l$ and $h$ are common knowledge of the teacher and the learner.

\textbf{Representation:} The teacher represents an example as $(x,y)$ while the student represents the same example as $(\Tilde{x},\Tilde{y})$ (typically $y=\Tilde{y}$ and we use $y$ when there is no ambiguity). The representation $x\in \mathcal{X}$ and $\Tilde{x}\in\Tilde{\mathcal{X}}$ can be different but deterministically related by an unknown mapping, $\Tilde{x}=\mathcal{G}(x)$. Suppose the teacher and the learner use model $h(\langle x, \omega\rangle)$ and $h(\langle \Tilde{x}, \nu\rangle)$ respectively, then $\omega^*$ and $\nu^*$ are very likely in different spaces too. This is a common scenario when the teacher and the learner are a human and a robot, or two robots from different factories. As the representation of examples can be complex, such as features extracted by deep neural networks~\citep{mnih2013playing,ILSVRC15,he2016deep}, using a linear model $h$ doesn't impinge the expressive power of the overall model. In the rest of the paper, we use $\omega$ for the teacher's parameter and $\nu$ for the learner's if they are from different spaces. Otherwise, we use $\omega$ for both of them. We use $x$ to refer to an example and its teacher representation. We use $\Tilde{x}$ for its learner representation. Also, we don't specify the choice of $\mathcal{G}$. Our only assumption about the teacher and the learner's representation will be discussed in \cref{thm:local}.

\textbf{Teacher:} In general, the teacher can only communicate with the learner via examples. This restriction doesn't impinge the generality of the machine teaching framework, as the format of the data can be generic, such as demonstration used in the IRL~\citep{ziebart2008maximum,babes2011apprenticeship,vroman2014maximum,macglashan2015between}. In this paper, data are provided iteratively. We use $x^t$ to denote the example used in the $t$-th iteration. The teacher aims to provide examples iteratively so that the student parameter $\nu$ converges to its optimum $\nu^*$ as fast as possible. Since the teacher doesn't know $\nu^t$ or $\nu^*$, we let the learner provide some feedback to her in each iteration so that she can track the pedagogy progress (details in \cref{sec:coop_teacher}).

\textbf{Learner:} The learner has an initial parameter $\nu^0$ before learning. At time $t$, he has learning rate $\eta_t$. The learning algorithms for teacher-unaware learners are often simple. For iterative gradient based optimization, the learner usually uses stochastic gradient descent~\citep{liu2017iterative,liu2018towards,fan2018learning,wu2018learning}. Suppose the learner receives $(x^t, y^t)$ from the teacher, his iterative update is:
\begin{align}
    \nu^t = \nu^{t-1}-\eta_t\frac{\partial l\big(\langle \Tilde{x}^t, \nu^{t-1}\rangle, y^t\big)}{\partial \nu^{t-1}}.\label{eq:gradient-descent}
\end{align}
\textbf{Mutual knowledge:} We limit the mutual knowledge between the teacher and the learner, otherwise, the mutual reasoning between the two can theoretically become an infinite recursion. In this paper, we consider a teacher who assumes a naive learner using \cref{eq:gradient-descent} to update his model. Meanwhile, the learner knows the teacher selects data deliberately instead of randomly (detailed in the next section). If we define a naive learner as having level-0 recursive reasoning, then the teacher and the teacher-aware learner have level-1 and level-2 recursive reasoning respectively. This level of recursion is very close to human cognitive capability~\citep{de2015higher,de2017negotiating} and was also adopted by~\cite{peltola2019machine}. 

To summarize, the loss function $l$, the model $h$, and the naive learner update function are common knowledge to the teacher and the learner. $\omega^*$ and the teaching mechanism are only known by the teacher, while $\nu^t$ and the learning mechanism, i.e. how to update $\nu^t$ given data, are only known by the learner. He knows the teacher intentionally selects helpful data according to her estimation of himself, and the teacher assumes that the learner learns following \cref{eq:gradient-descent}. For our teacher-aware learner, this assumption is inaccurate, but we'll show how the proposed learner can learn much faster than a naive learner.

\section{Teacher-Aware Learning}
\subsection{Cooperative Teacher}\label{sec:coop_teacher}
We first define the teacher whom the learner should be aware of. Let's consider a teacher using the same feature representation as the learner and knowing his parameter in each iteration. \cite{liu2017iterative} termed this kind of teacher as the omniscient teacher, who, in the $t-$th iteration, greedily chooses example from a data batch $D^t=\{(x_i, y_i)\sim\mathcal{P}(x, y)\}$:
\begin{align}
    (x^t, y^t) &= \argmin_{(x, y)\in D^t}\left\|\omega^{t-1} - \eta_t\frac{\partial l(\langle x, \omega^{t-1}\rangle, y)}{\partial \omega^{t-1}} - \omega^*\right\|_2^2\nonumber\\
    &= \argmax_{(x, y)\in D^t}\Big(-\eta^2_t\left\|\frac{\partial l(\langle x, \omega^{t-1}\rangle, y)}{\partial \omega^{t-1}}\right\|_2^2+2\eta_t\Big\langle \omega^{t-1} - \omega^*, \frac{\partial l(\langle x, \omega^{t-1}\rangle, y)}{\partial \omega^{t-1}}\Big\rangle\Big). \label{eq:teaching-volume}
\end{align}
The expression after $\argmax$ in \cref{eq:teaching-volume} is defined as the teaching volume $TV_{\omega^*}(x, y|\omega^{t})$, which represents the learner's progress in this iteration. It is a trade-off between the difficulty and the usefulness of an example~\citep[see][sec. 4.1]{liu2017iterative}. Notice that the teacher has no control over $D^t$, which is sampled from the data distribution $\mathcal{P}$ or from a large dataset. She only selects the best example from $D^t$. Given $D^t$ with a mild batch size, e.g. 20, the $\argmax$ in \cref{eq:teaching-volume} can be exactly calculated.

\citet{lessard2018optimal} has proved that, for an omniscient teacher, teaching greedily is sub-optimal. Yet, their findings cannot be directly applied to more practical teaching scenarios. Thus, we keep leveraging the greedy heuristic to model our cooperative teacher and generalize it to a non-omniscient teacher who doesn't fully know the learner in every iteration.

Suppose the teacher neither knows the learner's $\nu^{t-1}$ nor $\nu^*$ and they use different feature representations of the data. To teach cooperatively, she has to imitate the learner's model in her own feature space and use $\omega^*$ to guide the teaching. This can be done approximately if, in every round, the learner gives the inner products of $\nu^{t-1}$ and the data to the teacher as feedback. Given the convexity of the loss function $l$, we have:$\Big\langle\omega^{t-1}-\omega^*, \dfrac{\partial l(\langle x, \omega^{t-1}\rangle, y)}{\omega^{t-1}}\Big\rangle\ge l(\langle x, \omega^{t-1}\rangle, y) - l(\langle x, \omega^*\rangle, y).$ Now, \cref{eq:teaching-volume} can be approximated by inner products between the model parameter and the data~\citep{liu2017iterative}. Denote the learner's feedback as $\alpha_x = \langle \Tilde{x}, \nu^{t-1}\rangle, \forall (x, y) \in D^t, \Tilde{x} = \mathcal{G}(x)$, then the teacher will teach as following:
\begin{align}
    \argmax_{(x, y)\in D^t}\Big(-&\eta_t^2\left\|\frac{\partial l(\alpha_x, y)}{\partial \alpha_x}x\right\|_2^2+2\eta_t\big(l(\alpha_x, y)-l(\langle x, \omega^*\rangle, y)\big)\Big).\label{eq:imit-teacher}
\end{align}
It has been shown that cooperative teachers using \cref{eq:imit-teacher} can substantially speed up the learning of a standard SGD learner~\citep{liu2017iterative}. Nonetheless, only having a cooperative teacher doesn't provide us the most effective interaction between the two agents, as the learner doesn't exploit the fact that the data come from a helpful teacher~\citep{shafto2014rational}. In the next section, we introduce a teacher-aware learner.

\subsection{Teacher-Aware Learner}
Now we propose a learner who integrates teacher's pedagogy into his parameter updating process. Suppose we have a distribution $p(x, y|\nu^* = \nu) \propto\exp(-l(\langle \Tilde{x}, \nu\rangle, y))$, denoted as $p_{\nu}(x, y)$. Then, applying gradient descent to $l(\langle \Tilde{x}, \nu\rangle, y)$  is equivalent to maximizing $\log{p_{\nu}(x, y)}$ wrt. $\nu$. Hence, a learner updating parameters with \cref{eq:gradient-descent} can be considered as performing maximum likelihood estimation (MLE) when the data are randomly sampled from $\mathcal{P}(x, y)$.

Nonetheless, in the machine teaching framework, data are no longer randomly sampled from $\mathcal{P}(x, y)$. A teacher-aware learner should rectify his updating rule by considering teacher's helpfulness. Given the dataset $D^t$ at time $t$, the learner can postulate that the teacher is more likely to choose the example she thinks helpful following $p(x, y|\nu^* = \nu, \nu^{t-1}, D^t)$, denoted as $q_{\nu}(x, y|\nu^{t-1}, D^t)$ for short:
\begin{align}
    q_{\nu}(x, y|\nu^{t-1}, D^t&) = \frac{\exp(\beta_t \widehat{TV}_{\nu}(\Tilde{x}, y|\nu^{t-1}))}{\int_{(x', y')\in D^t}\exp(\beta_t \widehat{TV}_{\nu}(\Tilde{x}', y'|\nu^{t-1}))},\beta_t \ge 0,\label{eq:q-imit}\\
    \text{with }\widehat{TV}_{\nu}(\Tilde{x}, y&|\nu^{t-1}) = -\eta_t^2\left\|\frac{\partial l(\langle \Tilde{x}, \nu^{t-1}\rangle, y)}{\partial \nu^{t-1}}\right\|_2^2+2\eta_t\Big(l\big(\langle \Tilde{x}, \nu^{t-1}\rangle, y\big)-l\big(\langle \Tilde{x}, \nu\rangle, y\big)\Big).\label{eq:imit-teaching-volume}
\end{align}
The Boltzmann noisy rationality model~\citep{baker2014modeling} indicates that the teacher samples data according to the soft-max of their approximation of the teaching volumes, calculated wrt. her $\nu^*$ and the inner product feedback from the learner. Although, in practice, this estimation is usually different from the teacher's actual example selection distribution, which is a hard-max, corresponding to $\beta_t \rightarrow \infty$, 
maximizing it wrt. $\nu$ can still improve the learning. 

The learner now wants to learn a $\nu$, which not only makes $y^t$ more likely to be the correct label of $x^t$, but also $(x^t, y^t)$ more likely to be chosen from $D^t$. Intuitively, given all data in $D^t$ are coherent with the true distribution, the teacher gives $(x^t; y^t)$ but not other examples. With what $\nu$ can the probability of this selection be maximized? So, at every time $t$, the student maximizes $p_{\nu}(x^t, y^t)$ and $q_{\nu}(x^t, y^t|\nu^{t-1}, D^t)$ wrt. $\nu$ simultaneously. We can still use gradient descent. Omit $y$ when there is no confusion. Denote $g_x(\gamma) = \frac{\partial l(\langle\Tilde{x}, \gamma\rangle, y)}{\partial \gamma}$, then we have (derivation in supplementary \cref{sup:sec:proof-gradient}):
\begin{align}
    \nu^{t} \notag &= \nu^{t-1}-\eta_t\Big(\frac{\partial l(\langle \Tilde{x}^t, \nu^{t-1}\rangle, y)}{\partial \nu^{t-1}} -\frac{\partial \log q_{\nu}(\Tilde{x}^t, y^t)}{\partial \nu}\Big|_{\nu = \nu^{t-1}}\Big)\nonumber\\
    &= \nu^{t-1}-\eta_tg_{x^t}(\nu^{t-1})-2\beta_t\eta_t^2 \big(g_{x^t}(\nu^{t-1}) - \mathbb{E}_{x\sim q_{\nu^{t-1}}}[g_x(\nu^{t-1})]\big)\label{eq:imit-update}.
\end{align}
Notice that $\nu^{t-1}$ is a constant in $q_{\nu}(x^t, y^t|\nu^{t-1}, D^t)$ and the optimization is wrt. $\nu$, which is treated as $\nu^*$ in the calculation. The gradient is computed at $\nu = \nu^{t-1}$. This is equivalent to maximizing a new log-likelihood function $\log\big(p_{\nu}(x^t, y^t)q_{\nu}(x^t, y^t|\nu^{t-1}, D^t)\big)$, an approximation of the log probability that $(x^t, y^t)$ being sampled in $D^t$ and then being selected by the teacher given $\nu^* = \nu$. The product is an approximation of $p\big((x^t, y^t), D^t|\nu^{t-1}, \nu^* = \nu\big)$ because the sampling of data in $D^t$ except $(x^t, y^t)$ is regarded as deterministic. When $\beta_t = 0$, i.e. the learner thinks the teacher uniformly samples data, and \cref{eq:imit-update} becomes regular SGD.

An interpretation of the benefits brought by \cref{eq:imit-update} is that the learner not only learns from the literal meaning of the example selected by the teacher \textbf{(the second term)}, but also compares that example with ``also-rans'' in $D^t$ \textbf{(the third term)}, forming a context incorporating additional information. This is a prevailing phenomenon in human communication, as messages often convey both literal meanings and pragmatic (contributed by the context) meanings ~\citep{vogel2013emergence,smith2013learning,yuan2020emergence}. In other words, we can acquire not only explicit information from what others said but also implicit information from what others didn't say. When the message space is finite and known, exact computations of the implicit information becomes tractable. Therefore, in scenarios like human-robot interactions, where robots usually provide predefined user interfaces with a fixed choice of instructions, our algorithm can easily conduct counterfactual reasoning by comparing the user's selected instruction with the others and deliver faster learning than only using the selected one. In \cref{sec:exp-human}, our experiment with humans as the teacher illustrates such an advantage.

One nuance is that if we use $\nu^{t-1}$ as the $\nu$ in \cref{eq:q-imit}, the second term of the teaching volume will be 0. Thus, to better approximate $\nu^*$, in practice, we plug in $\nu^{t-1}-\eta_t\frac{\partial l(\langle \Tilde{x}, \nu^{t-1}\rangle, y)}{\partial \nu^{t-1}}$. That is, the learner first updates $\nu^{t-1}$ just like a naive learner. Then he calculates the gradient of $\log q_{\nu}$ wrt. the new $\nu$ and does an additional gradient descent corresponding to the last term in \cref{eq:imit-update}. 
Also, in supervised learning settings, the teacher needs to provide labels of the whole dataset for the learner to calculate the expectation. This is a mild requirement easy to be satisfied in practice. In the iterative process, $D^t$ is a mini-batch sampled from a large dataset with a small batch size, say 20 examples. Thus, $\mathbb{E}_{x\sim q_\nu}[g_x(\nu)]$ can be calculated exactly, and, compared with the standard mini-batch gradient descent, the only additional information needed from the teacher is the index of $(x_t, y_t)$. In fact, we can further relax this condition by letting the learner estimate $\mathbb{E}_{x\sim q_{\nu}}[g_x(\nu)]$ with only a subset $\widehat{D}^t\subseteq D^t$. In our experiments, we show that with only one random unchosen example provided, i.e. $|\widehat{D}^t|=1$, the teacher-aware learner outperforms the naive learner.  See \cref{alg:imit} for details.

\begin{algorithm}
\SetAlgoLined
\KwIn{Data distribution $\mathcal{D}$, teacher parameter $\omega^*$, learning rate $\eta_t$, teacher estimation scale $\beta_t$}
\KwResult{$\nu^{(T)}$}
 Randomly initialize student model $\nu^{(0)}\sim \text{Uniform}(N)$;
 Set $t = 1$ and $T$ as the maximum iteration number;\\
 \While{$t < T$}
 {
    Teacher gets data batch $D^t\sim \mathcal{D}$\\
    Learner reports $\alpha_x = \langle \nu^{(t-1)}, \Tilde{x}\rangle$ for all $x\in D^t$ to the teacher\\
    Teacher selects data for time $t$:\nonl\\
    $(x^t, y^t) = \argmax_{(x, y)\in D^t}\bigg(-\eta_t^2\left\|\frac{\partial l(\alpha_x, y)}{\partial \alpha_x}x\right\|^2 + 2\eta_t\Big(l\big(\alpha_x, y\big)-l\big(\langle \omega^*, x\rangle, y\big)\Big)\bigg)$\\
    Learner uses the selected data $(\Tilde{x}^t, y^t)$ and $D^t$ to calculate\\
    $\hat{\nu}^{(t)} = \nu^{(t-1)}-\eta_t\frac{\partial l(\langle \Tilde{x}^t, \nu^{(t-1)}\rangle, y^t)}{\partial \nu^{(t-1)}}$\\
    $\nu^{(t)} = \hat{\nu}^{(t)}-2\beta_t\eta_t^2\Big(\frac{\partial l(\langle \Tilde{x}^t, \hat{\nu}^{(t)}\rangle, y^t)}{\partial \hat{\nu}^{(t)}}-\mathbb{E}_{(\Tilde{x}, y)\sim q_{\hat{\nu}^{(t)}}(\Tilde{x}, y|\nu^{(t-1)}, D^t)}\left[\frac{\partial l(\langle \Tilde{x}, \hat{\nu}^{(t)}\rangle, y)}{\partial \hat{\nu}^{(t)}}\right]\Big)$\\
    where $q_{\hat{\nu}^{(t)}}(\Tilde{x}, y|\nu^{(t-1)}, D^t) = \frac{\exp\Big(\beta_t \widehat{TV}_{\hat{\nu}^{(t)}}\big(\Tilde{x}, y|\nu^{(t-1)}, D^t\big)\Big)}{\sum_{(x', y')\in D^t}\exp\Big(\beta_t \widehat{TV}_{\hat{\nu}^{(t)}}\big(\Tilde{x}', y'|\nu^{(t-1)}, D^t\big)\Big)}$\\
    with $\widehat{TV}_{\hat{\nu}^{(t)}}(\Tilde{x}, y|\nu^{(t-1)}, D^t)$ defined in \cref{eq:imit-teaching-volume}.\\
    $t = t + 1$\\
 }
 \caption{Iterative Teacher-Aware Learning}
 \label{alg:imit}
\end{algorithm}
We now prove the teacher-aware learner can always  perform better than a naive learner given proper conditions.
\begin{theorem}[Local Improvement]\label{thm:local}
Denote $\Tilde{\nu}^{t} = \nu^{t-1}-\eta_tg_{x^t}(\nu^{t-1})$. For a specific loss function $l$, given the same learning status $\nu^{t-1}$ and a teacher following \cref{eq:imit-teacher}, suppose $x^t$ satisfies that $x^t$ itself maximizes $\widehat{TV}_{\Tilde{\nu}^{t}}(x,y|\nu^{t-1})$. Denote $\hat x^t$ as the $x\in D^t$ which achieves the second largest $\widehat{TV}_{\Tilde{\nu}^{t}}(x,y|\nu^{t-1})$. Suppose that $\|g_x(\Tilde{\nu}^{t})\|_2 \leq G$ for any $x \in D^t$. If $\langle \Tilde\nu^t - \nu^*, g_{x^t}(\Tilde{\nu}^{t})-g_{\hat x^t}(\Tilde{\nu}^{t}) \rangle > 0$, then with large enough $\beta_t$, 
the teacher-aware learner using \cref{eq:imit-update} is guaranteed to make no smaller progress than a naive learner using \cref{eq:gradient-descent}.
\end{theorem}

One intuition for the assumption is that the best example selected by the teacher does bring more benefits to the learner than the other examples do. Suppose we have $\hat{\nu}^{t} = \Tilde{\nu}^{t}-\eta_t g_{\hat{x}^t}(\Tilde{\nu}^t)$, then moving from $\hat{\nu}^{t}$ to $\nu^{t}$ follows $\eta_t(g_{\hat{x}^t}(\Tilde{\nu}^t) - g_{x^t}(\Tilde{\nu}^t))$. The assumption $\langle \Tilde\nu^t - \nu^*, g_{x^t}(\Tilde{\nu}^t)-g_{\hat x^t}(\Tilde{\nu}^t) \rangle = \langle \nu^* - \Tilde{\nu}^t, g_{\hat{x}^t}(\Tilde{\nu}^t) - g_{x^t}(\Tilde{\nu}^t)\rangle> 0$ simply suggests that updating with $x^t$ gives the learner an advantage over updating with $\hat{x}^t$. The advantage points to $\nu^*$ (the two vectors $\nu^* - \Tilde{\nu}^t$ and $g_{\hat{x}^t}(\Tilde{\nu}^t) - g_{x^t}(\Tilde{\nu}^t)$ form an acute angle).

\begin{corollary}[Global Improvement]\label{thm:global}
Under the same condition of \cref{thm:local}, suppose that $\left\|\partial l(\langle \tilde x, \nu\rangle, y)/\partial \nu\right\|_2^2$ and $l( \langle \tilde x, \nu\rangle, y)$ are $L$-Lipschitz for $x$ with any $\nu$. Suppose the sample set $D^t$ satisfies that for any $x \in D^t$, there exists $x' \in D^t$ such that $\|x' - x\|_2 \leq \epsilon/(TL(\eta_t^2 + 4\eta_t))$ for any $t$, where $T$ is the total number of iterations. Then if the inequality
\begin{align}
    \|\nu_1 - \nu^*\|_2^2 - \max_{(x, y)\in D^t}\widehat{TV}_{\nu^*}(x, y|\nu_1)\le\|\nu_2 - \nu^*\|_2^2 - \max_{(x, y)\in D^t}\widehat{TV}_{\nu^*}(x, y|\nu_2)\label{eq:111}
\end{align}
holds for any $\nu_1, \nu_2$ that satisfy $\|\nu_1 - \nu^*\|_2^2 \le \|\nu_2 - \nu^*\|_2^2$, then with the same parameter initialization, learning rate and a teacher following \cref{eq:imit-teacher}, a teacher-aware learner can always converge not slower than a naive learner up to $\epsilon$ error.
\end{corollary}
To guarantee that $\|x' - x\|_2 \leq \epsilon/(TL(\eta_t^2 + 4\eta_t))$ for any $x \in D$, we need the subset $D^t \subseteq D$ to be `uniform distributed' on $D$. To achieve this goal, we can uniformly sample point $x \in D$ and let $D^t$ to be the set of these points. It is easy to verify that the `uniform distributed' property holds with high probability when $|D^t|$ is large enough. Meanwhile, \cref{eq:111} in \cref{thm:global} is defined as teaching monotonicity in~\citet{liu2017iterative}, and they proved that the squared loss satisfies teaching monotonicity given a dataset $D = \{x\in\mathbb{R}^d, \|x\| \le R\}$~\citep[see][proposition 3]{liu2017iterative}. The main difference between \cref{eq:111} and that in \citet{liu2017iterative} is that \cref{eq:111} works for the non-omniscient teacher setting, while \citet{liu2017iterative} focuses on the omniscient teacher setting. Detailed proofs of the theories can be found in \cref{sup:sec:proofs} of our supplementary.

\section{Experiments}\label{sec:exp}
\subsection{Machine Teacher}
\begin{figure}[t!]
\begin{subfigure}{0.55\textwidth}
    \begin{subfigure}{0.48\textwidth}
        \centering
       \includegraphics[width=\textwidth]{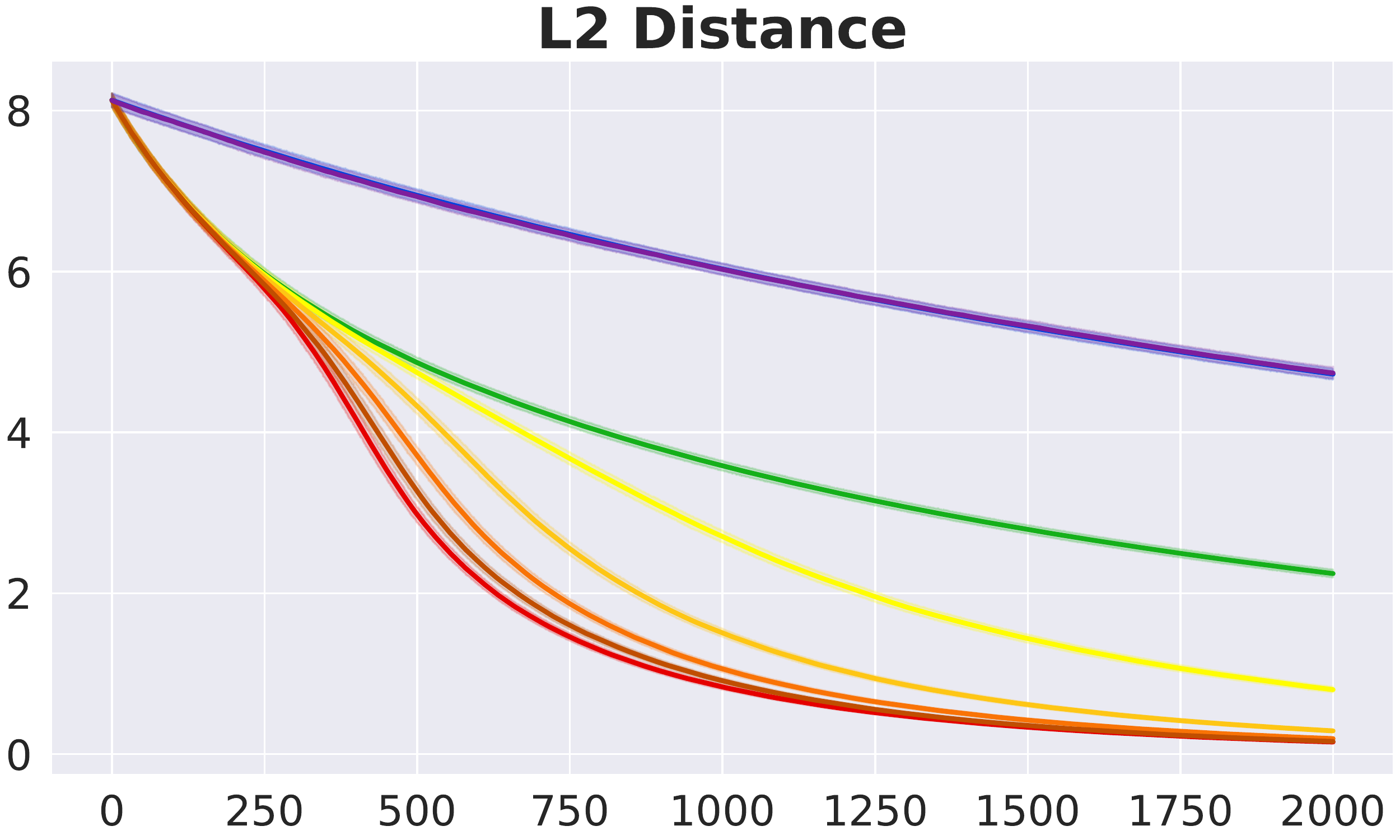}
        \caption{Linear regression.}
        \label{fig:regression-cop}
    \end{subfigure}%
    ~
    \begin{subfigure}{0.48\textwidth}
        \centering
        \includegraphics[width=\textwidth]{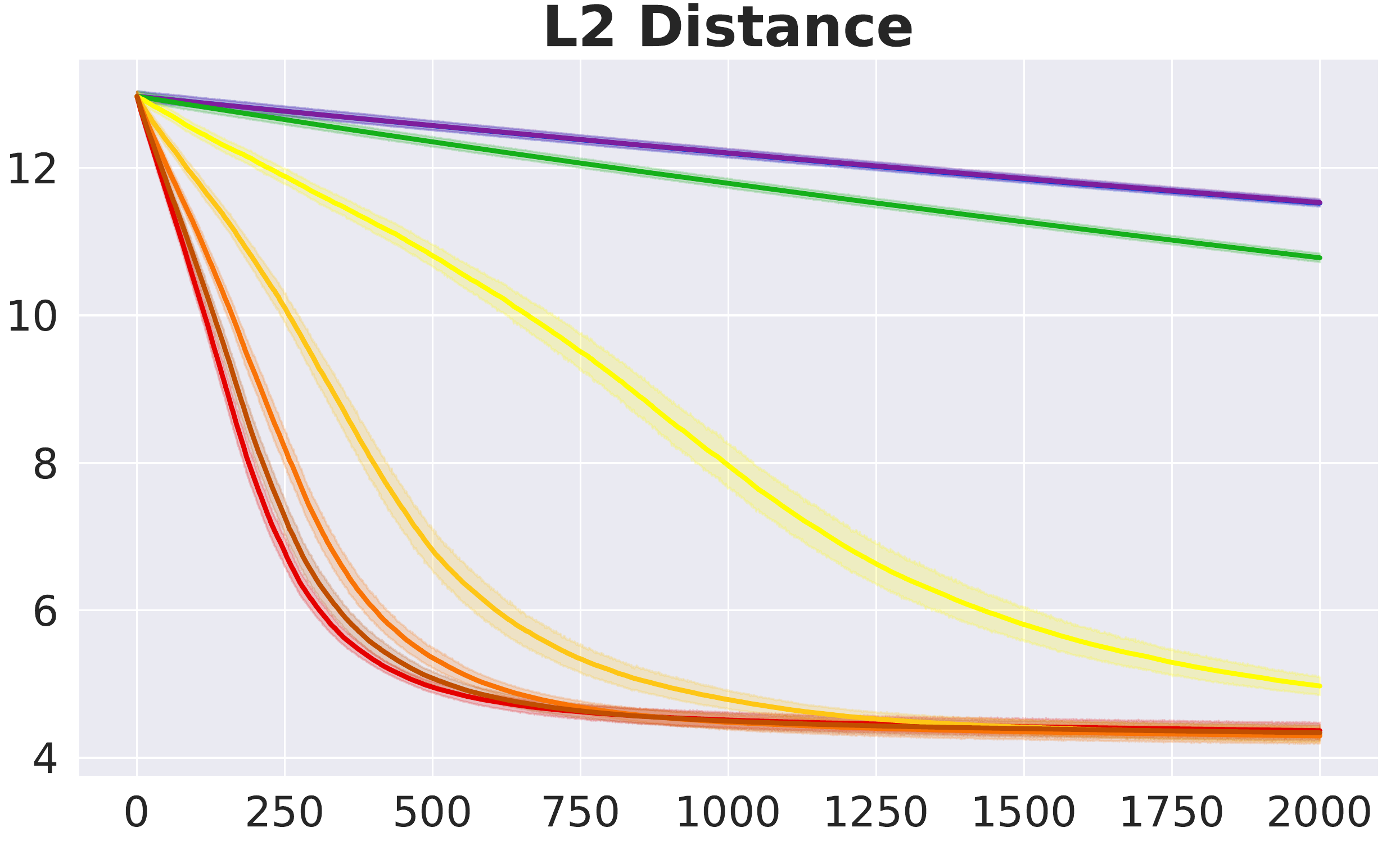}
        \caption{Gaussian data.}
        \label{fig:class10-cop}
    \end{subfigure}

    \begin{subfigure}{0.48\textwidth}
        \centering
        \includegraphics[width=\textwidth]{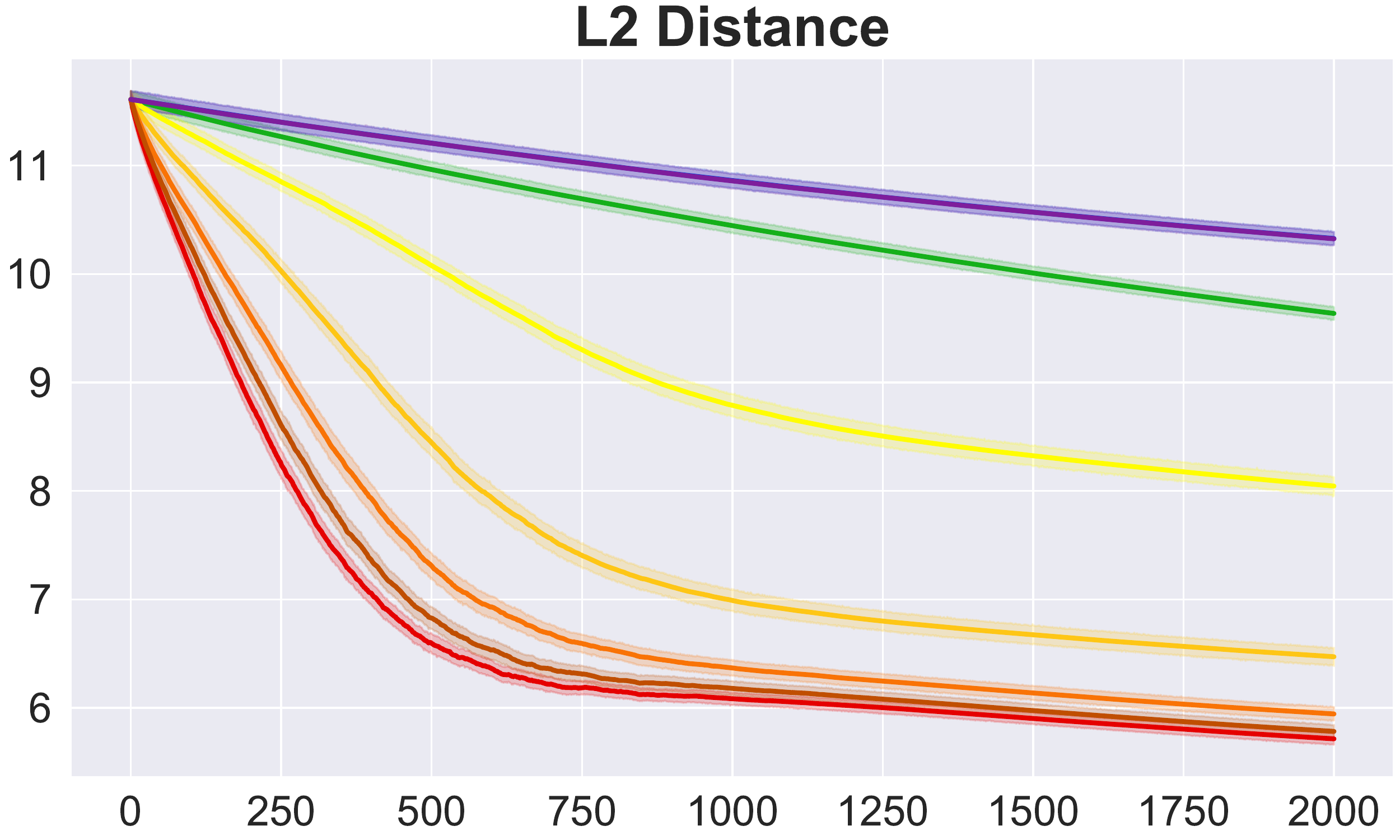}
        \caption{CIFAR-10.}
        \label{fig:CIFAR-10-cop}
    \end{subfigure}
    ~
    \begin{subfigure}{0.48\textwidth}
        \centering
        \includegraphics[width=\textwidth]{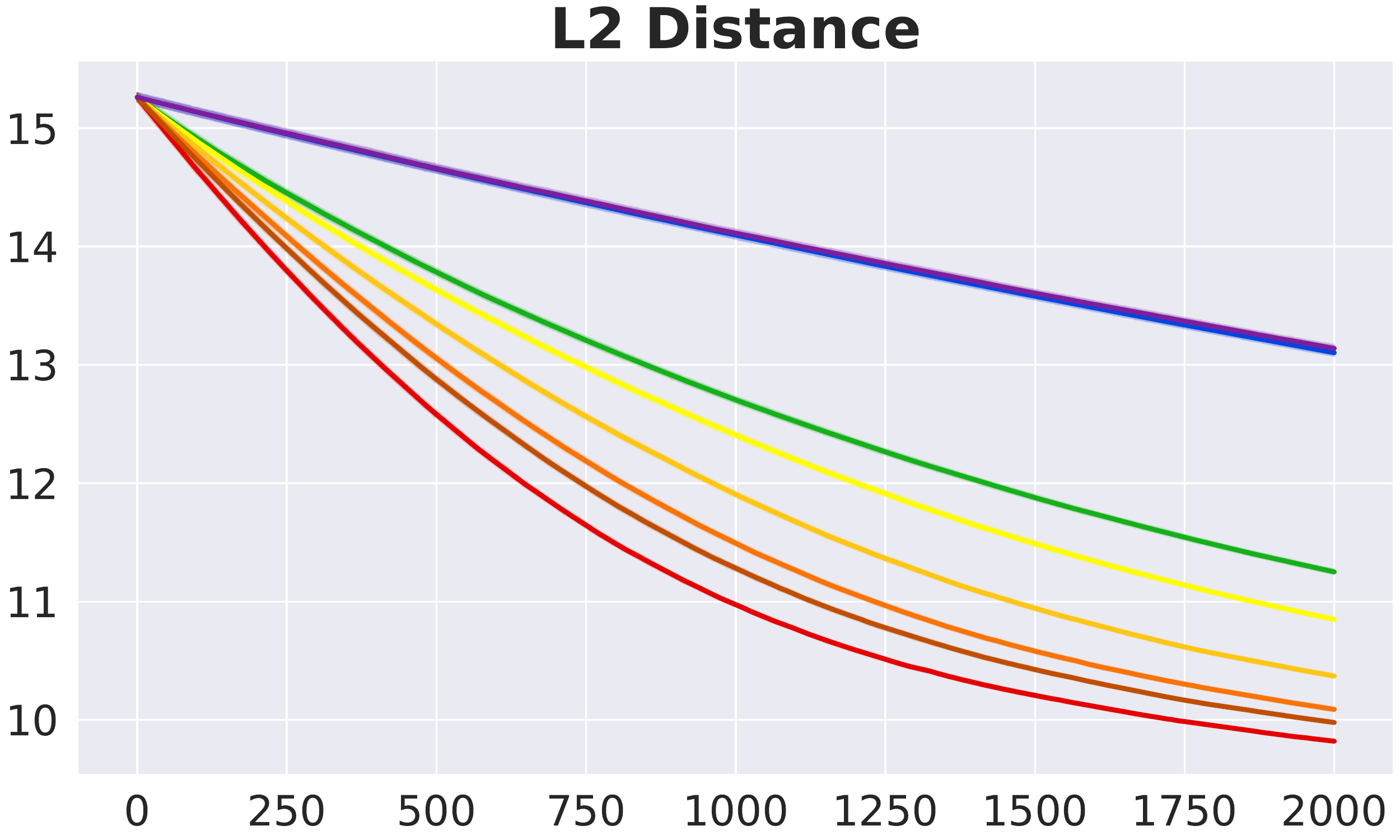}
        \caption{Tiny ImageNet.}
        \label{fig:ImageNet-13}
    \end{subfigure}%

    \begin{subfigure}{0.48\textwidth}
        \centering
        \includegraphics[width=\textwidth]{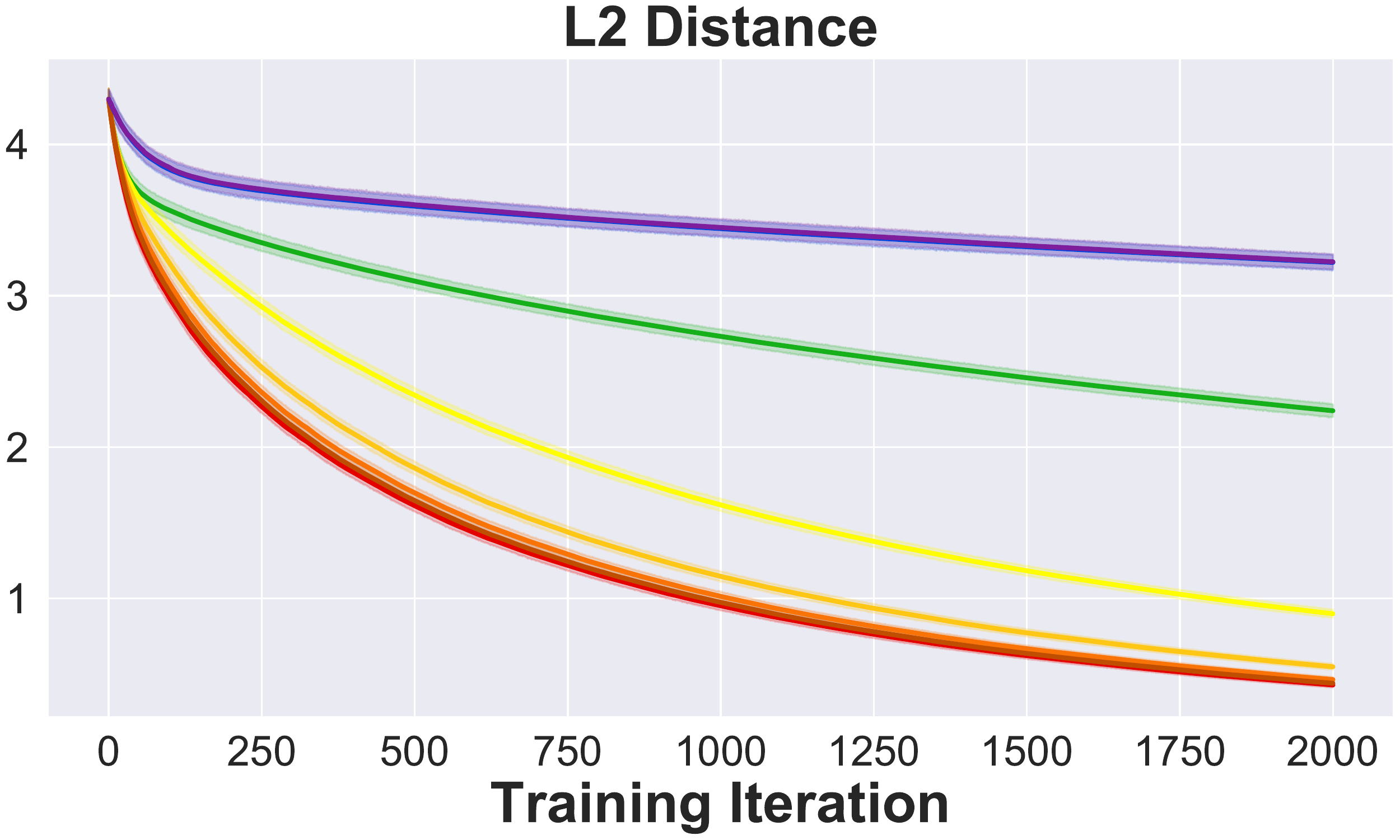}
        \caption{Equation.}
        \label{fig:Equation-cop}
    \end{subfigure}%
    ~
    \begin{subfigure}{0.48\textwidth}
        \centering
        \includegraphics[width=\textwidth]{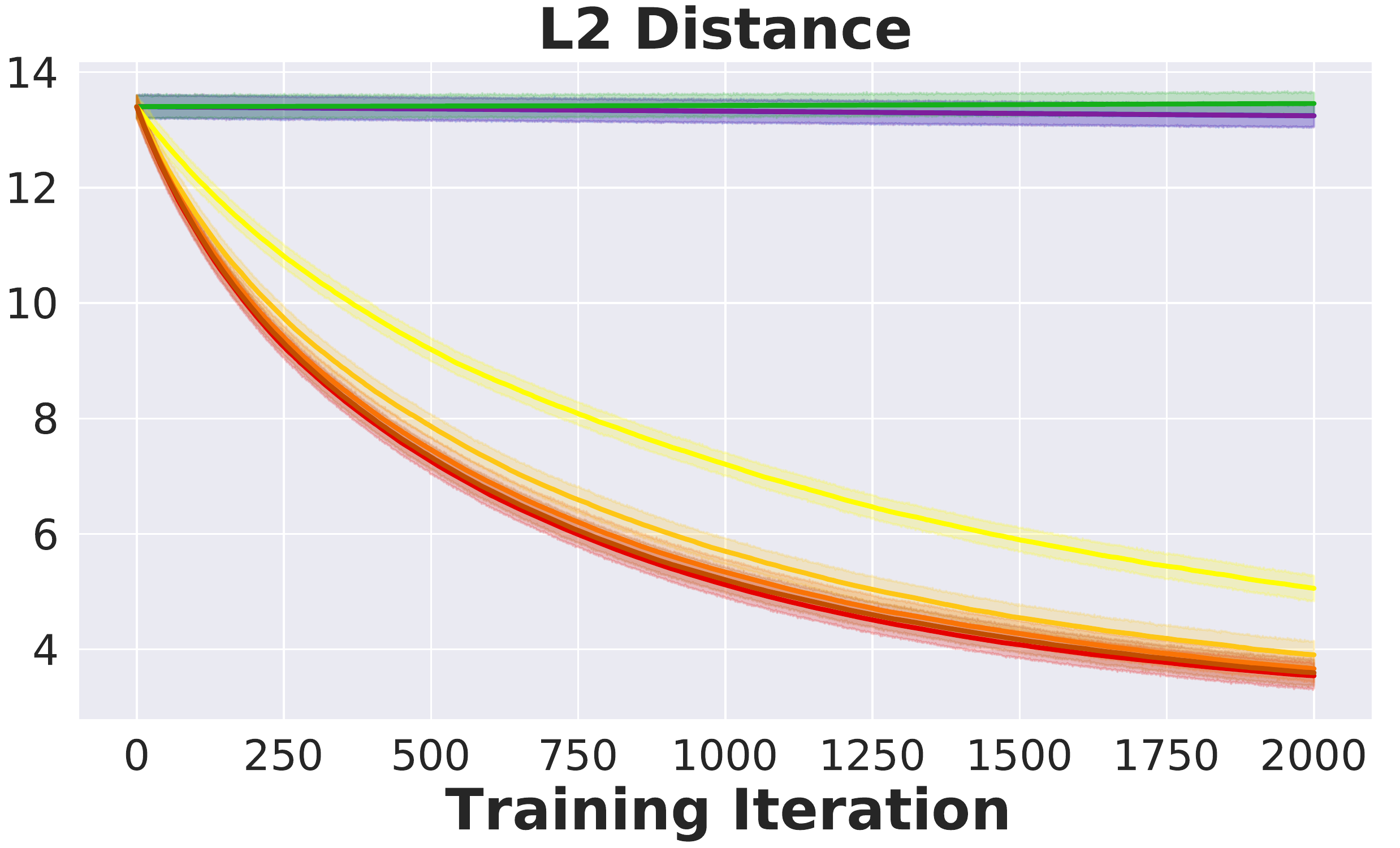}
        \caption{Online IRL.}
        \label{fig:OIRL-cop}
    \end{subfigure}
    
    \begin{subfigure}{0.48\textwidth}
        \centering
        \includegraphics[width=\textwidth]{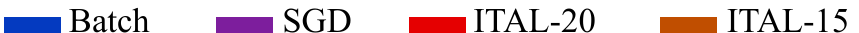}
    \end{subfigure}
    ~
    \begin{subfigure}{0.48\textwidth}
        \centering
        \includegraphics[width=\textwidth]{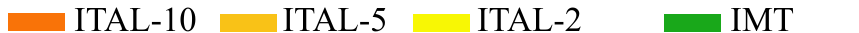}
    \end{subfigure}
\end{subfigure}
\begin{subfigure}{0.45\textwidth}
  \begin{subfigure}{\textwidth}
        \centering
        \includegraphics[width=\textwidth]{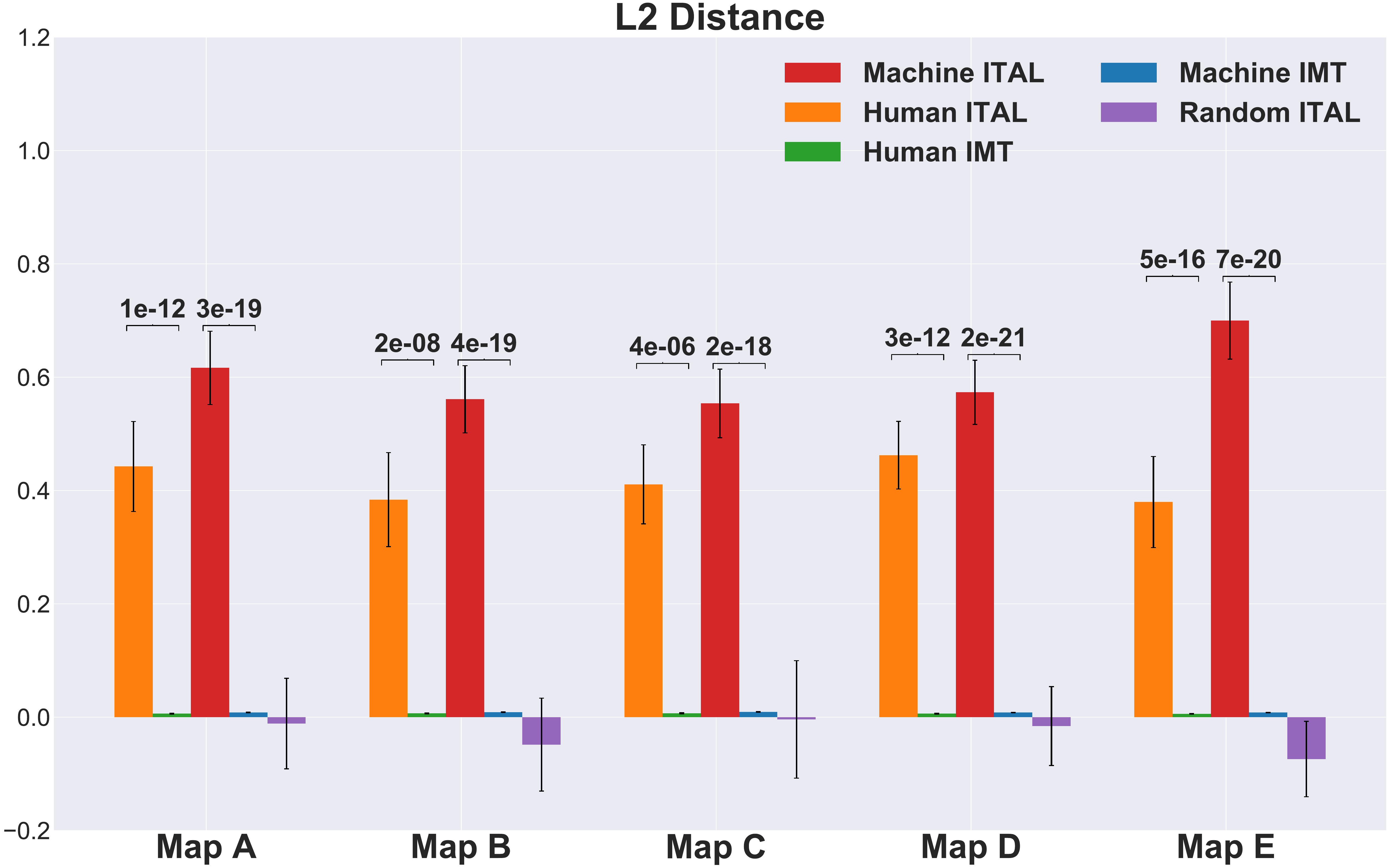}
    \end{subfigure}
    ~
    \begin{subfigure}{\textwidth}
        \centering
        \includegraphics[width=\textwidth]{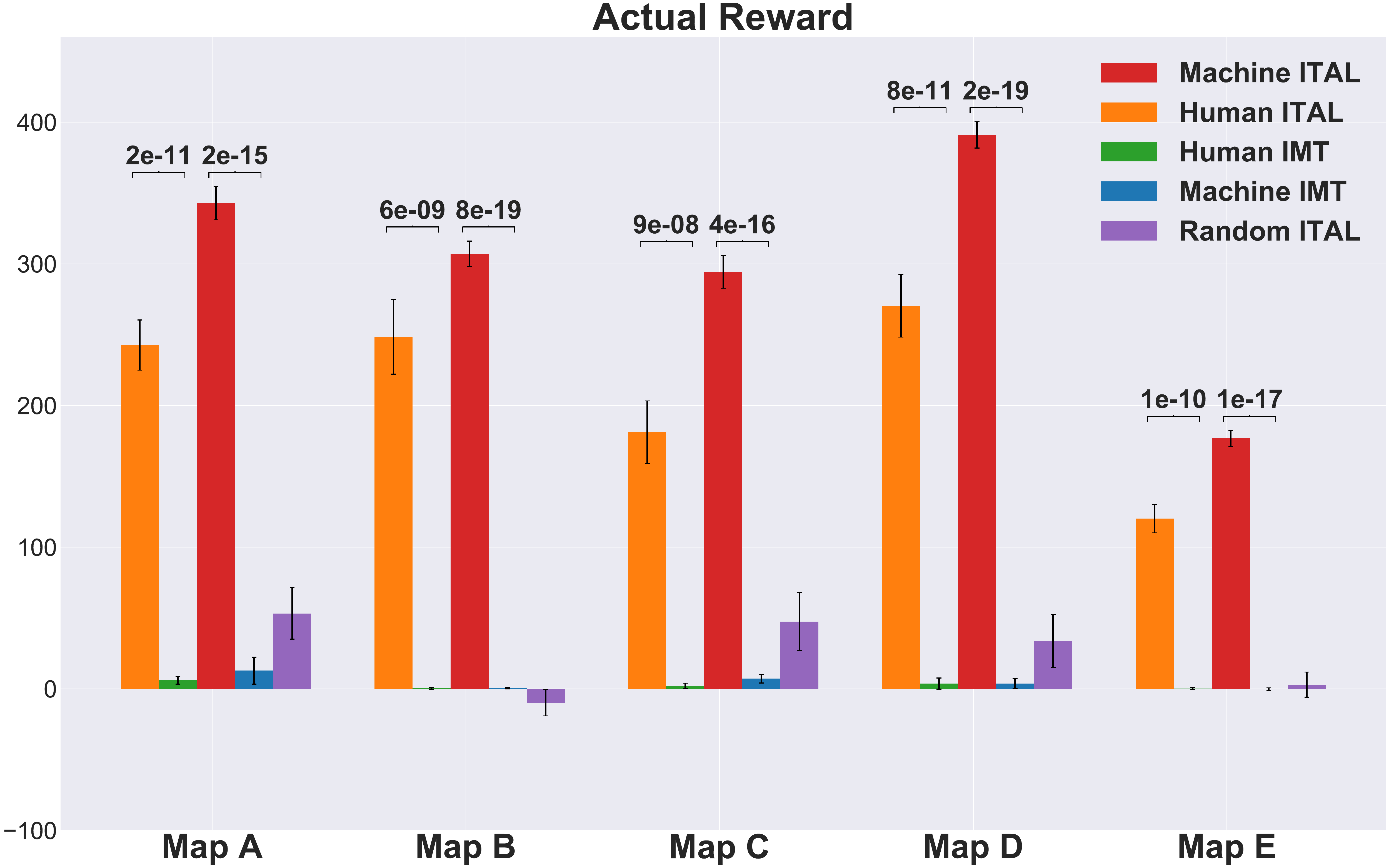}
    \end{subfigure}
  \caption{Human study results.}
  \label{fig:human-results}
\end{subfigure}
\caption{\textbf{\cref{fig:regression-cop}-\cref{fig:OIRL-cop}}: Cooperative teacher results. Our method always gives a substantial improvement over IMT, showing the effect of teacher-awareness. Within 2000 steps, ITAL already show convergence, while a naive learner only learns to a limited extent in most tasks. \textbf{\cref{fig:human-results}}: In the top plot, the height of each bar represents the \textbf{decrease} of the L2-distance between the learner's reward parameter and the ground-truth parameter. In the bottom, the height represents the accumulated reward. Paired t-tests were conducted between Human (Machine) ITAL \& IMT respectively.}
\label{fig:cooperative}
\vspace{-15pt}
\end{figure}
To justify the effectiveness of ITAL, we compared it with iterative machine teaching (IMT) with a naive learner on regression, classification, and IRL tasks. The coverage of squared loss, cross-entropy loss, and negative log-likelihood proves the robustness of our algorithm on various selections of $l$. For regression tasks, we measured the performance using the difference between $\|\omega^t - \omega^*\|_2$ and the mean squared loss of the test set. For the classification task, we measured the difference, the cross-entropy loss, and the classification accuracy of the test set. For online IRL problems, we measured the parameter difference, the total variance between the teacher's and the learner's policies, and the average rewards achieved by the learner. The feature dimension of the teachers can be different from that of the learners in some experiments. In \cref{fig:cooperative}, we show the results of the teacher having a smaller feature dimension than the learner does. We show the opposite in the supplementary \cref{sup:sec:exp}. \textbf{Batch} means the learner uses all the data in the mini-batch to calculate the mean gradient. \textbf{SGD} means the learner randomly selects an example in the mini-batch to calculate the gradient. \textbf{ITAL-$M$s} represent our algorithm with $M$ indicates $|\widehat{D}^t|$. The mini-batch $D^t$ is randomly sampled at every step with batch size 20. The learning rate is \texttt{1e-3} for all the experiments. $\beta_t$ is in the scale of \texttt{1e4}, varying for different settings. We grid search $\beta_t$ starting from $\texttt{1e4}$ and use the largest one inducing \cref{eq:q-imit} that is no longer a delta function. We ran each experiment with 20 different random seeds to calculate the mean and its standard error, shown in \cref{fig:cooperative}.

It can be seen that using the full mini-batch gives almost identical learning performance as using only one random sample from it. IMT has noticeable but limited improvements comparing with Batch and SGD, suggesting not necessarily substantial advantage brought by the helpful teacher. ITAL, on the other hand, significantly outperforms all other baselines, even with only 2 data points as the approximation of the full mini-batch. The learner modeled by these baselines only learns from the examples, but when the examples are no longer acquired randomly but from an intentional teacher, the example selection of the teacher also conveys a large amount of information. In particular, the teacher-aware learner can absorb information from not only the selected examples but also the unselected ones. As the learner has access to more unselected examples, he has a better approximation of the teaching process and learns more efficiently. Additional experimental details can be found in supplementary \cref{sup:sec:exp}.

\subsubsection{Supervised Learning}
\textbf{Linear Models on Synthetic Data:} In these experiments, we explored the convergence of our method in linear regression and multinomial logistic regression. For linear regression, we randomly generated a $M$-dimensional vector and a bias term as the $\omega^*$, and $X\in\mathbb{R}^{N\times (d + 1)}$ as the training set, with the last column being all 1s. The labels are $Y=X\omega^*$. For the classification task, we randomly generated $K$ points in the $d$-dimensional space, each of which is used as the mean of a normal distribution. Then we sampled $N/K$ points from each Gaussian distribution together as the training data. The labels are the indices of these distributions. With these data, we trained a logistic regression model using Scikit-learn~\citep{scikit-learn}, and used the coefficients as the teacher's $\omega^*$. We used a random orthogonal projection matrix to generate the teacher’s feature space from student’s. At every step, a subset of the training data is randomly selected as the mini-batch. The data points in that mini-batch along with their labels and the index of the data selected by the teacher are sent to the student. Details of the data generation can be found in supplementary \cref{sup:sec:exp-linear}.

\textbf{Linear Classifiers on Natural Image Datasets:} We further evaluated our teacher-aware learner on image datasets, CIFAR-10~\citep{krizhevsky2009learning} and Tiny ImageNet~\cite{tinyImageNet} (an adaptation of ImageNet~\citep{jia2009imageNet} used in Stanford 231n with 200 classes and 500 images in each class). In these experiments, the teacher tried to teach the parameters of the last fully connected (FC) layer in a convolutional neural network (CNN) trained on the dataset. We trained 3 baseline CNNs independently to do CIFAR-10 10-class and ImageNet 200-class classification. All of them achieved reasonable accuracy ($\ge  82\%$ for CIFAR-10 and $\ge 58\%$ top-1, $\ge 85\%$ top-5 for ImageNet). For CIFAR-10, we trained three different types of CNNs, CNN-6/9/12. For ImageNet we used VGG-13/16/19~\citep{Simonyan15vgg}. The features fed into the last FC layers are extracted to be the teaching dataset. The learner's feature is from CNN-9/VGG-16 and we set the teacher as either CNN-6/12 or VGG-13/19. Details about CIFAR-10 and Tiny ImageNet experiments are in ~\ref{sup:sec:exp-CIFAR} and \cref{sup:sec:exp-ImageNet} respectively.

\textbf{Linear Regression for Equation Simplification:} In this experiment, we learned a linear value function that can be used to guide action selections. Given polynomial equations with fraction coefficients and unmerged terms, we want to simplify them into cleaner forms with all the terms merged correctly, all the coefficients rescaled to integers without common factors larger than 1, and all the terms sorted by the descending power. For example, equation $-\frac{1}{2}x^2y + \frac{1}{3}xy = -\frac{1}5y^3 + \frac{1}{3}x^2y$ will be simplified to $-25x^2y+10xy+6y^3=0$. We defined a set of equation editing actions and a set of simplification rules. For a given equation, we applied the rules, recorded every editing action, and collected a simplifying trajectory. With all the trajectories of the training equations, we trained a value function by assuming that the value monotonically increases in each trajectory. Then the teacher tried to teach the student this value function. We used three different feature dimensions: 40D, 45D, and 50D. The learner always used 45D, and the teacher used 40D or 50D. Details can be found in \cref{sup:sec:exp-equation}.

\subsubsection{Online Inverse Reinforcement Learning}\label{sec:exp-oirl}
In this experiment, we changed from labeled data in standard supervised learning to demonstrations in IRL. The learner wanted to learn the parameter for a linear reward function $r(s, \omega^*)$ so that the likelihood of the demonstrations is maximized~\citep{babes2011apprenticeship,vroman2014maximum,macglashan2015between}. One challenge is that the $\max$ function in Bellman equations~\citep{sutton2018reinforcement} is non-differentiable. Thus, we approximated $\max$ with soft-max, namely: $\max(a_0, ..., a_n) \approx \frac{\log(\sum_{i=0}^n\exp{ka_i})}{n}$, with $k$ controlling the level of approximation and leveraged the online Bellman gradient iteration~\citep{li2017online}. The IRL environment is an 8$\times$8 map, with a randomly generated reward assigned to each grid. If we encode each grid using a one-hot vector, then the reward parameter is a 64D vector with the $i$-th entry corresponding to the reward of the $i$-th grid. The agent can go up, down, left, or right in each grid. All demonstrations are in the format of $(s, a)$, where $s$ indicates a grid and $a$ an action demonstrated in that grid. The teacher uses a shuffled map encoding. For instance, if the first grid is $[1, 0, ..., 0]$ to the learner, then it became $[0, ..., 0, 1, 0, ...]$ to the teacher. Details are included in \cref{sup:sec:exp-oirl}.

\subsubsection{Adversarial Teacher}\label{sec:exp-adversarial}
In addition to the cooperative setting that we assumed throughout the discussion above, we also explored if the learner can still learn given an adversarial teacher. An adversarial teacher doesn't mean that she gives fake data to the student, but she uses $\argmin$ in \cref{eq:q-imit} instead of using $\argmax$. That is, she always chooses the least helpful data for the learner. Hence, a learner, being aware of this unhelpful pedagogy, will adjust \cref{eq:imit-teacher} accordingly by using $\beta_t \le 0$. We redid all previous experiments with an adversarial teacher and showed that our learner can still learn effectively given an adversarial teacher, while a naive learner barely improves (see \cref{fig:adv} in \cref{sup:sec:exp-adv}). This experiment justifies the universal utility of modeling the teacher's intention regardless of the informativeness of the teaching examples.

\subsection{Human Teacher}\label{sec:exp-human}
In the previous section, we showed that teacher-awareness substantially accelerates learning, given a machine teacher. In this section, we further investigate if our teacher-aware learner can show an advantage in scenarios where humans play the role of the cooperative teacher. We hypothesize that despite the discrepancy between the pedagogical pattern of human and machine teachers, our learner can still benefit from his teacher-awareness modeled with \cref{eq:q-imit}.

We conducted a proof-of-concept human study with a similar but simplified version of the IRL experiments in \cref{sec:exp-oirl}. To better suit human participants, we first change the maps from $8\times 8$ to $5 \times 5$. Second, instead of assigning random continuous rewards to the grids, we color them with white, blue, and red, representing neutral, bad, and good tiles. In each teaching session, a participant is given one of the five reward maps as the ground truth and a randomly initialized learner to be taught. Then, the participant will be asked to teach the learner about the ground truth reward in each grid by providing $(s, a)$ examples as in \cref{sec:exp-oirl}. In each time step, we construct the examples by randomly sampling a set of 10 grids and drawing an arrow on each sampled grid indicating which direction the learner should go to if in that grid. The human teacher is asked to choose the most helpful arrow given the learner's current reward map. The map configurations are in \cref{sup:fig:map_configuration} in \cref{sup:sec:exp-human}. A similar map setup for reward teaching was used by~\citet{ho2016showing}. Every participant will teach both the naive and the teacher-aware learner about the same map. We then run a paired sample t-test to compare the learning effect of the two types of learners. We show the improvement of the L2-distance between the learner's reward parameters and the ground-truth reward parameters and the accumulated reward in \cref{fig:human-results}. Comparison of policy total variance and learning curves are included in \cref{sup:sec:exp-human} of the supplementary.

For all maps, the ITAL method has a significant ($p$-value $\approx$ 0) advantage over its IMT counterparts. The human ITALs all perform worse than machine ITALs. This is as expected as we directly reuse the machine teacher model to simulate humans. There is no guarantee that all the participants follow the same teaching pattern as the machine teacher, or even have a consistent teaching pattern at all. Yet, we still manage to grasp human cooperation to some extent. To illustrate the influence of the teacher model, we also teach the ITAL learner with a random teacher, who samples the example uniformly every time and is not cooperative at all. As shown in \cref{fig:human-results}, this combination doesn't benefit the learner, because the mismatch between the imagined cooperative teacher and the actual random teacher will very likely introduce over-interpretation of the examples. To summary, these results justify that human teachers do have cooperative (contrary to uniform) pedagogy patterns and the current teacher-aware model can take advantage of them. Finding a comprehensive and accurate human-robot communication model will be an open question for future works.

\vspace{-5pt}
\section{Discussion and Conclusions}\label{sec:conclusion}
\vspace{-5pt}
Pedagogy has a profound cognitive science background, but it hasn't received much attention in machine learning works until recently. In this paper, we integrate pedagogy with parameter learning and propose a teacher-aware learning algorithm. Our algorithm changes the model update step for the gradient learner to accommodate the intention of the teacher. We provide theoretical and empirical evidence to justify the advantage of the teacher-aware learner over the naive learner.

To be aware of the teacher, the learner needs an accurate estimation of the teaching model. In many cases, such a model is not directly accessible, e.g. when there is a human-in-the-loop. In this paper, we model the teacher in a heuristic manner. Our human study proved the generality of this model, especially when the learner only assumes a sub-optimal teacher with Boltzmann rationality. In future work, a more advanced teacher model should be investigated, acquired through task-specific data and/or interactions between the agents. Another limitation of our work is that, in our current setting, the learner's feedback is restricted to be inner products. A more generic message space can be leveraged to develop comprehensive learning as a bidirectional communication platform. We believe our work illustrates the promising benefits of accommodating human pedagogy into machine learning algorithms and approaching learning as a multi-agent problem.

\begin{ack}
The work was supported by DARPA XAI project N66001-17-2-4029, ONR MURI project N00014-16-1-2007 and NSF DMS-2015577. We would like to thank Yixin Zhu, Arjun Akula from the UCLA Department of Statistics and Prof. Hongjing Lu from the UCLA Psychology Department for their help with human study design, and three anonymous reviewers for their constructive comments.
\end{ack}

\bibliography{reference}
\bibliographystyle{IEEEtranSN}

\section*{Checklist}


\begin{enumerate}

\item For all authors...
\begin{enumerate}
  \item Do the main claims made in the abstract and introduction accurately reflect the paper's contributions and scope?
    \answerYes{}
  \item Did you describe the limitations of your work?
    \answerYes{We explicitly discussed the assumptions of our models in our problem setup and limitations in \cref{sec:conclusion}.}
  \item Did you discuss any potential negative societal impacts of your work?
    \answerNo{Our work aims to enable human-like AI behaviors and facilitate human-robot collaboration. At this moment, we cannot think of any negative societal impacts of our work.}
  \item Have you read the ethics review guidelines and ensured that your paper conforms to them?
    \answerYes
\end{enumerate}

\item If you are including theoretical results...
\begin{enumerate}
  \item Did you state the full set of assumptions of all theoretical results?
    \answerYes
	\item Did you include complete proofs of all theoretical results?
    \answerYes
\end{enumerate}

\item If you ran experiments...
\begin{enumerate}
  \item Did you include the code, data, and instructions needed to reproduce the main experimental results (either in the supplemental material or as a URL)?
    \answerYes{Details are in \cref{sec:exp}, \cref{sup:sec:exp} and supplementary codes.}
  \item Did you specify all the training details (e.g., data splits, hyperparameters, how they were chosen)?
    \answerYes{Details are in \cref{sec:exp}, \cref{sup:sec:exp} and supplementary codes.}
	\item Did you report error bars (e.g., with respect to the random seed after running experiments multiple times)?
    \answerYes{We use 20 random seeds and report the standard deviation (67\% confidence interval)}
	\item Did you include the total amount of compute and the type of resources used (e.g., type of GPUs, internal cluster, or cloud provider)?
    \answerYes{See supplementary \cref{sup:sec:exp}}
\end{enumerate}

\item If you are using existing assets (e.g., code, data, models) or curating/releasing new assets...
\begin{enumerate}
  \item If your work uses existing assets, did you cite the creators?
    \answerYes{}
  \item Did you mention the license of the assets?
    \answerYes{Cifar-10 has the MIT license and ImageNet data are downloaded from \href{https://www.image-net.org/download}{https://www.image-net.org/download}}
  \item Did you include any new assets either in the supplemental material or as a URL?
    \answerNo{But we discussed how our data are generated and included the source code in the supplementary code.}
  \item Did you discuss whether and how consent was obtained from people whose data you're using/curating?
    \answerNo{The data we used are open-source and our usage comply with their terms.}
  \item Did you discuss whether the data you are using/curating contains personally identifiable information or offensive content?
    \answerNo{The data we used doesn't include any personally identifiable information.}
\end{enumerate}

\item If you used crowdsourcing or conducted research with human subjects...
\begin{enumerate}
  \item Did you include the full text of instructions given to participants and screenshots, if applicable?
    \answerYes{See supplementary \cref{sup:sec:exp-human} for screenshots and supplementary Jupyter Notebook for instructions.}
  \item Did you describe any potential participant risks, with links to Institutional Review Board (IRB) approvals, if applicable?
    \answerYes
  \item Did you include the estimated hourly wage paid to participants and the total amount spent on participant compensation?
    \answerNo{Our data was collected from student volunteers.}
\end{enumerate}

\end{enumerate}




\makeatletter\@input{yy.tex}\makeatother
\end{document}


\maketitle
\tableofcontents
\section{Proofs and Derivations}\label{sup:sec:proofs}
\subsection{Gradient Derivation}\label{sup:sec:proof-gradient}

\begin{align*}
    \widehat{TV}_{\nu}(\Tilde{x}, y|\nu^{t-1}) &= -\eta_t^2\left\|\frac{\partial l(\langle \Tilde{x}, \nu^{t-1}\rangle, y)}{\partial \nu^{t-1}}\right\|_2^2+2\eta_t\Big(l\big(\langle \Tilde{x}, \nu^{t-1}\rangle, y\big)-l\big(\langle \Tilde{x}, \nu\rangle, y\big)\Big)\\
    \text{Denote } g_x(\nu) &= \frac{\partial l(\langle \Tilde{x}, \nu\rangle, y)}{\partial \nu}\\
    \frac{\partial \log q_{\nu}(\Tilde{x}^t, y^t|\nu^{t-1}, D)}{\partial \nu} &= \frac{\partial}{\partial \nu}\Big(\beta_t \widehat{TV}_{\nu}(x^t, y^t|\nu^{t-1}) - \log\int_{(x, y)\in D}\exp\big(\beta_t \widehat{TV}_{\nu}(x, y|\nu^{t-1})\big)\Big)\\
    &= -2\beta_t\eta_t g_{x^t}(\nu) +\frac{2\beta_t\eta_t\int_{(x, y)\sim D}g_x(\nu) \exp(-\beta_t \widehat{TV}_{\nu}(x, y|\nu^{t-1}))}{\int_{(x, y)\in D}\exp\big(-\beta_t \widehat{TV}_{\nu}(x, y|\nu^{t-1})\big)}\\
    &= -2\beta_t\eta_t (g_{x^t}(\nu) - \mathbb{E}_{x\sim q_{\nu}}[g_x(\nu)])
\end{align*}

\subsection{Proof of Theorem \ref{thm:local}}\label{sup:sec:proof-thm1}
For simplicity, let $g_x$ denote $g_x(\Tilde{\nu}^{t})$ in this proof. First we provide an intuition for the assumption. Suppose we have $\hat{\nu}^{t} = \Tilde{\nu}^{t}-\eta_t g_{\hat{x}^t}$, then moving from $\hat{\nu}^{t}$ to $\nu^{t}$ follows $\eta_t(g_{\hat{x}^t} - g_{x^t})$. The assumption $\langle \Tilde\nu^t - \nu^*, g_{x^t}-g_{\hat x^t} \rangle = \langle \nu^* - \Tilde{\nu}^t, g_{\hat{x}^t} - g_{x^t}\rangle> 0$ simply suggests that updating with $x^t$ gives the learner an advantage over updating with $\hat{x}^t$. The advantage points to $\nu^*$ (the two vectors $\nu^* - \Tilde{\nu}^t$ and $g_{\hat{x}^t} - g_{x^t}$ form an acute angle).

Next, we start the proof. We need the following lemma:
\begin{lemma}\label{lm:gradient}
Denote $\hat x^t$ as the $x$ which achieves the second largest $\widehat{TV}_{\nu^*}(\tilde x,y|\nu^{t-1})$. Suppose that $\langle \Tilde\nu^t - \nu^*, g_{x^t}-g_{\hat x^t} \rangle > 0$, then there exists $\alpha>0$ such that with large enough $\beta_t$, we have
\begin{align}
    \big\|\beta_t\eta_t^2(g_{x^t}-\mathbb{E}_{x~\sim q_{\Tilde{\nu}^{t}}}[g_x])\big\|_2 \leq \alpha\| \Tilde\nu^t - \nu^*\|_2\label{lm:main1}
\end{align}
and 
\begin{align}
    \langle \Tilde\nu^t - \nu^*,g_{x^t} - \mathbb{E}_{x~\sim q_{\Tilde{\nu}^{t}}}[g_x]\rangle \geq \alpha \|\Tilde\nu^t - \nu^*\|_2\big\|g_{x^t} - \mathbb{E}_{x~\sim q_{\Tilde{\nu}^{t}}}[g_x]\big\|_2.\label{lm:main2}
\end{align}
\end{lemma}
\begin{proof}[Proof of Lemma~\ref{lm:gradient}]
We set $\alpha = \langle \Tilde\nu^t - \nu^*, g_{x^t}-g_{\hat x^t} \rangle/(2\| \Tilde\nu^t - \nu^*\|_2\|g_{x^t}-g_{\hat x^t}\|_2)>0$. 

    First we show that $\big\|\beta_t\eta_t^2(g_{x^t}-\mathbb{E}_{x~\sim q_{\Tilde{\nu}^{t}}}[g_x])\big\|_2 \leq \alpha\| \Tilde\nu^t - \nu^*\|_2$. For simplicity we denote $s(x) = \widehat{TV}_{\Tilde{\nu}^{t}}(\tilde x,y|\nu^{t-1})$. Then by assumption on the selection of $x^t$ we have $x^t= \argmax_{x \in D^t} s(x)$ and
    \begin{align}
        &(g_{x^t} - \mathbb{E}_{x~\sim q_{\Tilde{\nu}^{t}}}[g_x])\int_{x'\in D}\exp(\beta_t s(x'))\notag \\
        & = \bigg(g_{x^t} - \frac{\int_{x'\in D}\exp(\beta_t s(x'))g_{x'}}{\int_{x'\in D}\exp(\beta_t s(x'))}\bigg)\int_{x'\in D}\exp(\beta_t s(x'))\notag \\
        & = \exp(\beta_ts(x^t))[g_{x^t} - g_{x^t}] + \exp(\beta_ts(\hat x^t))[g_{x^t}-g_{\hat x^t}] + \sum_{x \neq x^t, \hat x^t} \exp(\beta_ts(x))[g_{x^t}-g_{x}]\notag\\
        & = \exp(\beta_ts(\hat x^t))[g_{x^t}-g_{\hat x^t}] + \sum_{x \neq x^t, \hat x^t} \exp(\beta_ts(x))[g_{x^t}-g_{x}].\label{eq:lm_1}
    \end{align}
    Therefore, denote $\xi_{t-1} = s(\hat x^t) - s(x^t)<0$, we have that when $\beta_t \rightarrow \infty$,
    \begin{align}
        &\beta_t\big\|g_{x^t} - \mathbb{E}_{x~\sim q_{\Tilde{\nu}^{t}}}[g_x]\big\|_2\notag \\
        &\leq \beta_t\big\|g_{x^t} - \mathbb{E}_{x~\sim q_{\Tilde{\nu}^{t}}}[g_x]\big\|_2 \int_{x'\in D}\exp(\beta_t [s(x') - s(x^t)])\notag \\
        & =  \beta_t\exp(-\beta_t s(x^t))\bigg\|(g_{x^t} - \mathbb{E}_{x~\sim q_{\Tilde{\nu}^{t}}}[g_x]) \int_{x'\in D}\exp(\beta_t s(x'))\bigg\|_2\notag \\
        & = \beta_t\bigg\|\exp(\beta_t[s(\hat x^t) - s(x^t)])[g_{x^t}-g_{\hat x^t}] + \sum_{x \neq x^t, \hat x^t} \exp(\beta_t[s(x) - s(x^t)])[g_{x^t}-g_{x}]\bigg\|_2\notag \\
        & \leq \beta_t\sum_{x \neq x^t}\bigg\|\exp(\beta_t[s(x) - s(x^t)])[g_{x^t}-g_{x}] \bigg\|_2\notag \\
        & \leq \beta_t|D|\exp\big(\beta_t \xi_{t-1}\big)\max_{x' \in D}\|g_{x^t} - g_{x'}\|_2\notag \\
        & \leq \beta_t|D|\exp\big(\beta_t \xi_{t-1}\big)\max_{x' \in D}\big(\|g_{x^t}\|_2 +\|g_{x'}\|_2\big)\notag \\
        & = 2\beta_t|D|\exp\big(\beta_t \xi_{t-1}\big)G\notag \\
        &\rightarrow 0,\notag
    \end{align}
    where the first inequality holds due to the fact $s(x') < s(x^t)$ for any $x' \in D$, the second inequality holds due to triangle inequality, the third inequality holds due to the facts $\big(s(x) - s(x^t)\big) < \xi_{t-1}$ for any $x \neq x^t$, the fourth inequality holds due to the assumption that $\|g_x\|_2 \leq G$, the last line holds due to the fact that $x\exp(ax)\rightarrow 0$ for $a<0$ and $x \rightarrow \infty$. Therefore, taking large enough $\beta_t$, we have
    \begin{align}
        \big\|\beta_t\eta_t^2(g_{x^t}-\mathbb{E}_{x~\sim q_{\Tilde{\nu}^{t}}}[g_x])\big\|_2 \leq \alpha\| \Tilde\nu^t - \nu^*\|_2.\notag
    \end{align}
    Next we show that $\langle \Tilde\nu^t - \nu^*,g_{x^t} - \mathbb{E}_{x~\sim q_{\Tilde{\nu}^{t}}}[g_x]\rangle \geq \alpha \|\Tilde\nu^t - \nu^*\|_2\big\|g_{x^t} - \mathbb{E}_{x~\sim q_{\Tilde{\nu}^{t}}}[g_x]\big\|_2$. 
     From \eqref{eq:lm_1} we have
    \begin{align}
        &(g_{x^t} - \mathbb{E}_{x~\sim q_{\Tilde{\nu}^{t}}}[g_x])\exp(-\beta_ts(\hat x^t))\int_{x'\in D}\exp(\beta_t s(x'))\notag \\
        & = g_{x^t}-g_{\hat x^t} + \underbrace{\sum_{x \neq x^t, \hat x^t} \exp(\beta_t[s(x) - s(\hat x^t)])[g_{x^t}-g_{x}]}_{g(\beta_t)},\notag 
    \end{align}
    For $g(\beta_t)$, denote $\hat\xi_{t-1} = \min_{x \neq x^t, \hat x^t}[  s(x) - s(\hat x^t)] <0$. Then when $\beta_t \rightarrow \infty$, we have
    \begin{align}
        \|g(\beta_t)\|_2 &\leq \sum_{x \neq x^t, \hat x^t} \bigg\|\exp(\beta_t[s(x) - s(\hat x^t)])[g_{x^t}-g_{x}]\bigg\|_2 \notag \\
        &\leq |D|\exp(\beta_t \hat\xi_{t-1})\max_{x \in D}\|g_{x^t}-g_{x}\|_2\notag \\
        &\leq 2G|D|\exp(\beta_t \hat\xi_{t-1})\notag \\
        &\rightarrow 0, \label{eq:lm_10}
    \end{align}
    where the first inequality holds due to triangle inequality, the second inequality holds due to the fact $s(x) - s(\hat x^t) < \hat \xi_{t-1}$, the third inequality holds due to the assumption that $\|g_x\|_2 \leq G$, the last line holds because $\exp(-x)\rightarrow 0$ when $x \rightarrow \infty$. Thus when $\beta_t \rightarrow \infty$, we have
    \begin{align}
\frac{g_{x^t} - \mathbb{E}_{x~\sim q_{\Tilde{\nu}^{t}}}[g_x]}{\|g_{x^t} - \mathbb{E}_{x~\sim q_{\Tilde{\nu}^{t}}}[g_x]\|_2} &= \frac{(g_{x^t} - \mathbb{E}_{x~\sim q_{\Tilde{\nu}^{t}}}[g_x])\exp(-\beta_ts(\hat x^t))\int_{x'\in D}\exp(\beta_t s(x'))}{\big\|(g_{x^t} - \mathbb{E}_{x~\sim q_{\Tilde{\nu}^{t}}}[g_x])\exp(-\beta_ts(\hat x^t))\int_{x'\in D}\exp(\beta_t s(x'))\big\|_2}\notag \\
& = \frac{g_{x^t}-g_{\hat x^t}+g(\beta_t)}{\|g_{x^t}-g_{\hat x^t}+g(\beta_t)\|_2}\notag \\
& \rightarrow \frac{g_{x^t}-g_{\hat x^t}}{\|g_{x^t}-g_{\hat x^t}\|_2},\notag
    \end{align}
    where the last line holds due to $g(\beta_t) \rightarrow 0$ from \eqref{eq:lm_10}. 
    Therefore, we know that for large enough $\beta_t$, we have
    \begin{align}
        \bigg|\bigg\langle \frac{\Tilde\nu^t - \nu^*}{\|\Tilde\nu^t - \nu^*\|_2}, \frac{g_{x^t} - \mathbb{E}_{x~\sim q_{\Tilde{\nu}^{t}}}[g_x]}{\|g_{x^t} - \mathbb{E}_{x~\sim q_{\Tilde{\nu}^{t}}}[g_x]\|_2} \bigg\rangle - \bigg\langle \frac{\Tilde\nu^t - \nu^*}{\|\Tilde\nu^t - \nu^*\|_2}, \frac{g_{x^t}-g_{\hat x^t}}{\|g_{x^t}-g_{\hat x^t}\|_2} \bigg\rangle\bigg| \leq \alpha,\notag
    \end{align}
Finally, due to the fact $\langle \Tilde\nu^t - \nu^*, g_{x^t}-g_{\hat x^t} \rangle = 2\alpha \| \Tilde\nu^t - \nu^*\|_2\|g_{x^t}-g_{\hat x^t}\|_2$ from \eqref{lm:main2}, we have
\begin{align}
    \bigg\langle \frac{\Tilde\nu^t - \nu^*}{\|\Tilde\nu^t - \nu^*\|_2}, \frac{g_{x^t} - \mathbb{E}_{x~\sim q_{\Tilde{\nu}^{t}}}[g_x]}{\|g_{x^t} - \mathbb{E}_{x~\sim q_{\Tilde{\nu}^{t}}}[g_x]\|_2} \bigg\rangle \geq \bigg\langle \frac{\Tilde\nu^t - \nu^*}{\|\Tilde\nu^t - \nu^*\|_2}, \frac{g_{x^t}-g_{\hat x^t}}{\|g_{x^t}-g_{\hat x^t}\|_2} \bigg\rangle-\alpha = \alpha.\notag
\end{align}
\end{proof}
Now we prove the main theorem. 
\begin{proof}[Proof of Theorem~\ref{thm:local}]
    The naive learner and the teacher-aware learner, after receiving $(x^t, y^t)$, will update their model to $\Tilde{\nu}^t =(\nu^{t-1} - \eta_tg_{x^t})$ and $\nu^t = \big(\nu^{t-1} - \eta_tg_{x^t}-2\beta_t \eta_t^2 (g_{x^t} - \mathbb{E}_{x\sim q_{\Tilde{\nu}^{t}}}[g_x])\big)$ respectively. Then with large enough $\beta_t$, we have
    \begin{align}
        &\|\nu^t - \nu^*\|_2^2\notag \\
        & = \big\|\Tilde\nu^t -\nu^*- 2\beta_t\eta_t^2(g_{x^t}-\mathbb{E}_{x~\sim q_{\Tilde{\nu}^{t}}}[g_x])\big\|_2^2\notag \\
        & = \|\Tilde\nu^t - \nu^*\|_2^2 - 4\langle \Tilde\nu^t - \nu^* ,  \beta_t\eta_t^2(g_{x^t}-\mathbb{E}_{x~\sim q_{\Tilde{\nu}^{t}}}[g_x])\rangle + 4\big\|\beta_t\eta_t^2(g_{x^t}-\mathbb{E}_{x~\sim q_{\Tilde{\nu}^{t}}}[g_x])\big\|_2^2\notag \\
        & \leq \|\Tilde\nu^t - \nu^*\|_2^2 - 4\alpha\| \Tilde\nu^t - \nu^*\|_2  \big\|\beta_t\eta_t^2(g_{x^t}-\mathbb{E}_{x~\sim q_{\Tilde{\nu}^{t}}}[g_x])\big\|_2 + 4\big\|\beta_t\eta_t^2(g_{x^t}-\mathbb{E}_{x~\sim q_{\Tilde{\nu}^{t}}}[g_x])\big\|_2^2\notag \\
        & \leq \|\Tilde\nu^t - \nu^*\|_2^2 - 4\alpha\| \Tilde\nu^t - \nu^*\|_2  \big\|\beta_t\eta_t^2(g_{x^t}-\mathbb{E}_{x~\sim q_{\Tilde{\nu}^{t}}}[g_x])\big\|_2 \notag \\
        &\qquad + 4\alpha\| \Tilde\nu^t - \nu^*\|_2  \big\|\beta_t\eta_t^2(g_{x^t}-\mathbb{E}_{x~\sim q_{\Tilde{\nu}^{t}}}[g_x])\big\|_2\notag \\
        & = \|\Tilde\nu^t - \nu^*\|_2^2,\notag
    \end{align}
    where the first inequality holds due to \eqref{lm:main2} in Lemma \ref{lm:gradient}, the second inequality holds due to \eqref{lm:main1} in Lemma \ref{lm:gradient}. 
\end{proof}

\subsection{Proof of Corollary \ref{thm:global}}
\begin{proof}
    Let $\nu_a$ and $\nu_b$ be the model parameter of the naive learner and the teacher-aware learner. Denote $\widehat{TV}_{\nu^*}^E(\nu) = \max_{x \in E}\widehat{TV}_{\nu^*}(x, y|\nu)$, 
    $E$ is some dataset. Let $x^*$ denote the argmax of $\max_{x \in D}\widehat{TV}_{\nu^*}(x, y|\nu)$, then by the assumption on $D^t$, we know that there exists $x'$ such that $\|x' - x^*\|_2 \leq \epsilon/(TL(\eta_t^2 + 4\eta_t))$. Then we have
\begin{align}
       &\widehat{TV}_{\nu^*}^D(\nu) -  \widehat{TV}^{D^t}_{\nu^*}(\nu) \notag \\
       &= \max_{x \in D}(-\eta_t^2 g_x(\nu) +2\eta_t(l(\nu, x) - l(\nu^*, x)) - \max_{x \in D^t}(-\eta_t^2 g_x(\nu) +2\eta_t(l(\nu, x) - l(\nu^*, x))\notag \\
       & \leq (-\eta_t^2 g_{x^*}(\nu) +2\eta_t(l(\nu, x^*) - l(\nu^*, x^*)) - (-\eta_t^2 g_{x'}(\nu) +2\eta_t(l(\nu, x') - l(\nu^*, x'))\notag \\
       & \leq L(\eta_t^2\|x^* - x'\|_2 + 4\eta_t\|x^* - x'\|_2)&\text{($L$-Lipschitz)}\notag \\
       & \leq \epsilon/T.
       \end{align}
    Now we prove that $\|\nu_b^t - \nu^*\|_2^2 \le \|\nu_a^t - \nu^*\|_2^2 + t/T\epsilon$ for all $1 \leq t \leq T$. As $\nu^0$ is the same for both learners, knowing theorem \ref{thm:local}, we have $\|\nu_b^1 - \nu^*\|_2^2 \le \|\nu_a^1 - \nu^*\|_2^2 + 1/T\epsilon$. Suppose $\|\nu_b^t - \nu^*\|_2^2 \le \|\nu_a^t - \nu^*\|_2^2 + t/T\epsilon$, then we have
    \begin{align*}
        &\|\nu^{t+1}_b-\nu^*\|\notag \\
        &\leq\|\nu_b^t - \nu^*\|_2^2-TV^{D^t}_{\nu^*}(\nu_b^t)\notag &\text{(Theorem \ref{thm:local})}\\
        &\leq\|\nu_b^t - \nu^*\|_2^2-\widehat{TV}^{D^t}_{\nu^*}(\nu_b^t)\notag &\text{(convexity of $l$)}\\
        & \leq \|\nu_b^t - \nu^*\|_2^2-\widehat{TV}^D_{\nu^*}(\nu_b^t) + \epsilon/T &\widehat{TV}_{\nu^*}^D(\nu) -  \widehat{TV}^{D^t}_{\nu^*}(\nu) \leq \epsilon\notag \\
        &\leq\|\nu_a^t - \nu^*\|_2^2 -\widehat{TV}^D_{\nu^*}(\nu_a^t) + (t+1)/T\epsilon&\text{(condition)}\notag \\
        & \leq \|\nu_a^t - \nu^*\|_2^2 -\widehat{TV}^{D^t}_{\nu^*}(\nu_a^t) + (t+1)/T\epsilon&\widehat{TV}_{\nu^*}^D(\nu) -  \widehat{TV}^{D^t}_{\nu^*}(\nu) \geq 0\notag \\
        &= 
        \|\nu_a^{t+1} - \nu^*\|_2^2 + (t+1)/T\epsilon
    \end{align*}
Therefore, we have $\|\nu_b^t - \nu^*\|_2^2 \le \|\nu_a^t - \nu^*\|_2^2 + t/T\epsilon$, which suggests that the teacher-aware learner can always converge no slower than the naive learner up to an $\epsilon$ factor. 
\end{proof}

\section{Detailed Experiment Settings}\label{sup:sec:exp}
We used two types of loss functions in all the experiment. For regression tasks, our loss function is
\begin{align*}
    \min_{\omega\in\mathbb{R}^d, b\in\mathbb{R}}\frac{1}{n}\sum_{i=1}^n\frac{1}{2}\big(\omega^Tx_i + b - y_i\big)^2+\frac{\lambda}{2}\|\omega\|_2^2
\end{align*}
For classification tasks, our loss function is
\begin{align*}
    \min_{\omega\in\mathbb{R}^d, b\in\mathbb{R}}\frac{1}{n}\sum_{i=1}^n\sum_{k=1}^K-\mathbf{1}(y_i=k)\log{p_{ik}}+\frac{\lambda}{2}\|\omega\|_2^2\\
    p_{ik} = \frac{\exp(\omega_k^Tx_i + b_k)}{\sum_{k'=1}^K\exp(\omega_{k'}^Tx_i+b_k')}
\end{align*}
where $\omega\in\mathbb{R}^{K\times d}$ and $\omega_k$ is the $k$-th row of $\omega$, $b\in\mathbb{R}^K$ and $b_k$ is the $k$-th element of $b$. The norm is Frobenius norm. In both regression and classification tasks, we refer to $[\omega, b]$ as $\omega^*$. In all the following experiments, we used a constant learning rate $10^{-3}$ for all the algorithms. The size of the minibatch was set to 20. As the gradient scale is different in different experiments, we used different $\beta$s. We chose the hyperparameter $\beta$ so that at the beginning of the learning, the data in the mini-batch with the smallest teaching volume has above $80\%$ probability of being selected. In the supplementary material, we show additional results of our experiments. All plots are consistent with the results in the main text. All of our experiments were run on machines with 16 I9-9900K cores and 64GiB RAM. The longest setting is the Tiny ImageNet classification, takes about 12 hours to finish (2000 iterations for 8 methods and 20 random seeds).

\subsection{Linear Models on Synthesized Data}\label{sup:sec:exp-linear}
\begin{figure}
    
    \centering
    \includegraphics[width=0.6\textwidth]{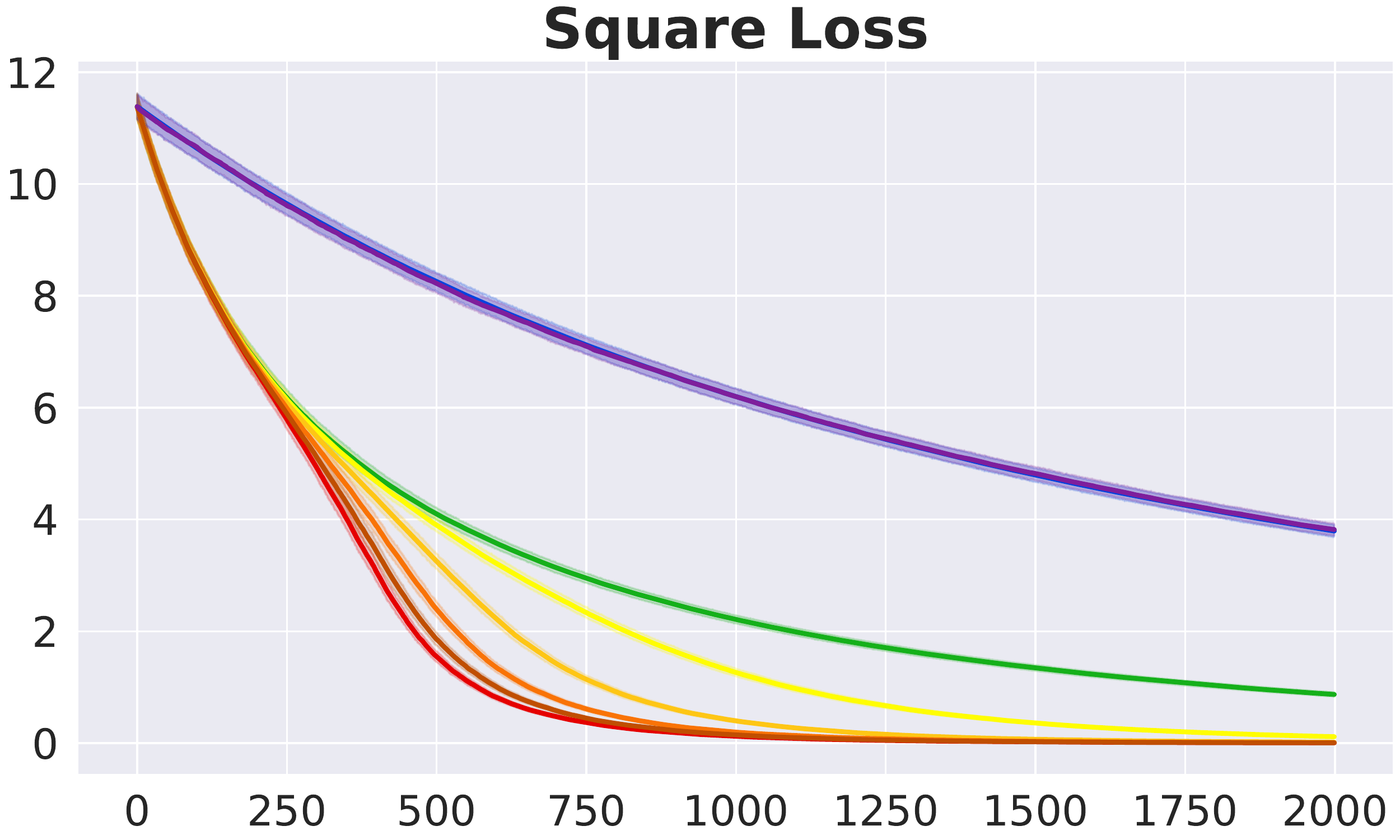}
    \caption{Square loss of the linear regression.}
    \label{sup:fig:regression}
\end{figure}
\begin{figure}[ht]
    \begin{subfigure}{0.49\textwidth}
        \centering
       \includegraphics[width=\textwidth]{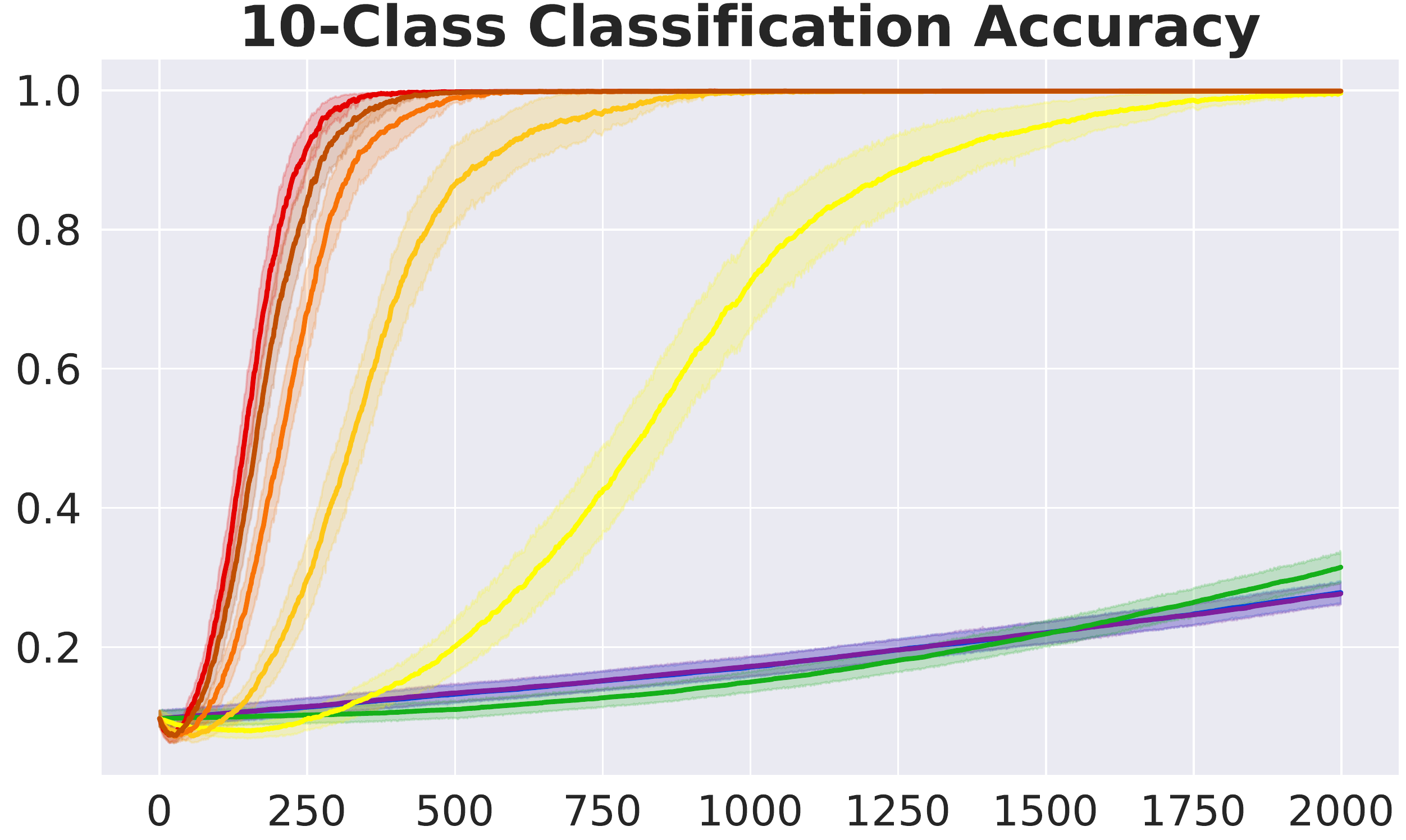}
    \end{subfigure}%
    ~
    \begin{subfigure}{0.49\textwidth}
        \centering
       \includegraphics[width=\textwidth]{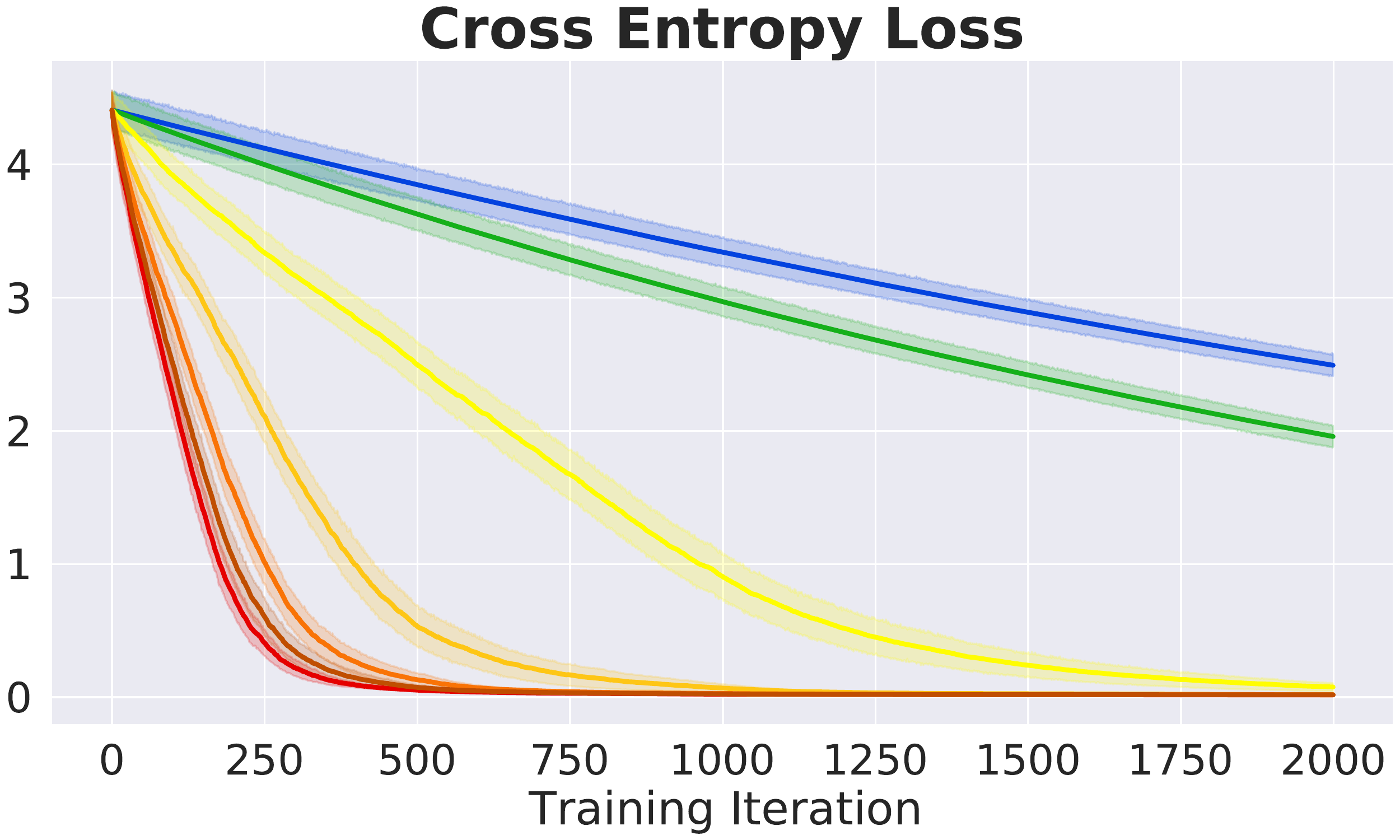}
    \end{subfigure}%
    
    \begin{subfigure}{\textwidth}
        \centering
        \includegraphics[width=\textwidth]{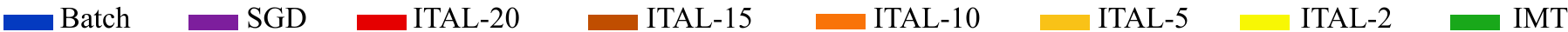}
    \end{subfigure}
    \caption{Classification accuracy and Cross Entropy loss of the 10-class Gaussian data classification.}
    \label{sup:fig:classification}
\end{figure}
For the regression task, both $\omega^*$ and $X$ are randomly generated from a uniform distribution, namely $\omega_i, b, X_{ij}\sim U[-1, 1]$. $Y = X\omega^*$. The data points have dimension 100, and $\beta$ is chosen to be 2000. For the classification task, we first randomly generate $K$ points as the center of each class from $U[-1, 1]$. Then, we use these points as the centers of Normal distributions with $\Sigma = 0.5I_{(d+1)}$. $N/K$ points are sampled from each distribution as the data. We get $\omega^*$ using the logistic regression model in Scikit-learn~\citep{scikit-learn}. For classification task with 30D data, we use $\beta = 60000$. We used $\lambda = 0$ for both tasks. For the scenario of different feature spaces, we use a random orthogonal projection matrix to generate the teacher’s feature space from the student’s. $\omega^*$ and $\nu^*$ are multiplied with the inverse of the projection matrix to preserve the inner product. Figure~\ref{sup:fig:regression} shows the square loss of the linear regression task. Figure \ref{sup:fig:classification} shows the classification accuracy and the Cross Entropy loss of 10-class classification tasks.

\subsection{Linear Classifiers on MNIST Dataset}\label{sup:sec:exp-MNIST}
\begin{table}[ht]
\centering
\caption{MNIST CNN structure}
\begin{tabular}{|c|c|c|c|}
\hline
       & 20-Dim CNN       & 24-Dim CNN      & 30-Dim CNN     \\ \hline
Conv 1 & \multicolumn{3}{c|}{1 layer, 64 {[}3$\times$3{]} filters, leaky ReLU}  \\ \hline
Pool   & \multicolumn{3}{c|}{2$\times$2 Max with Stride 2}          \\ \hline
Conv 2 & \multicolumn{3}{c|}{1 layers, 32 {[}3$\times$3{]} filters, leaky ReLU} \\ \hline
Pool   & \multicolumn{3}{c|}{2$\times$2 Max with Stride 2}          \\ \hline
Conv 3 & \multicolumn{3}{c|}{1 layer, 32 {[}3$\times$3{]} filters, leaky ReLU}  \\ \hline
FC     & 20, tanh               & 24, tanh              & 30, tanh             \\ \hline
\end{tabular}
\label{sup:tab:mnist-struct}
\end{table}
\begin{figure}
    \begin{subfigure}{0.32\textwidth}
        \centering
        \includegraphics[width=\textwidth]{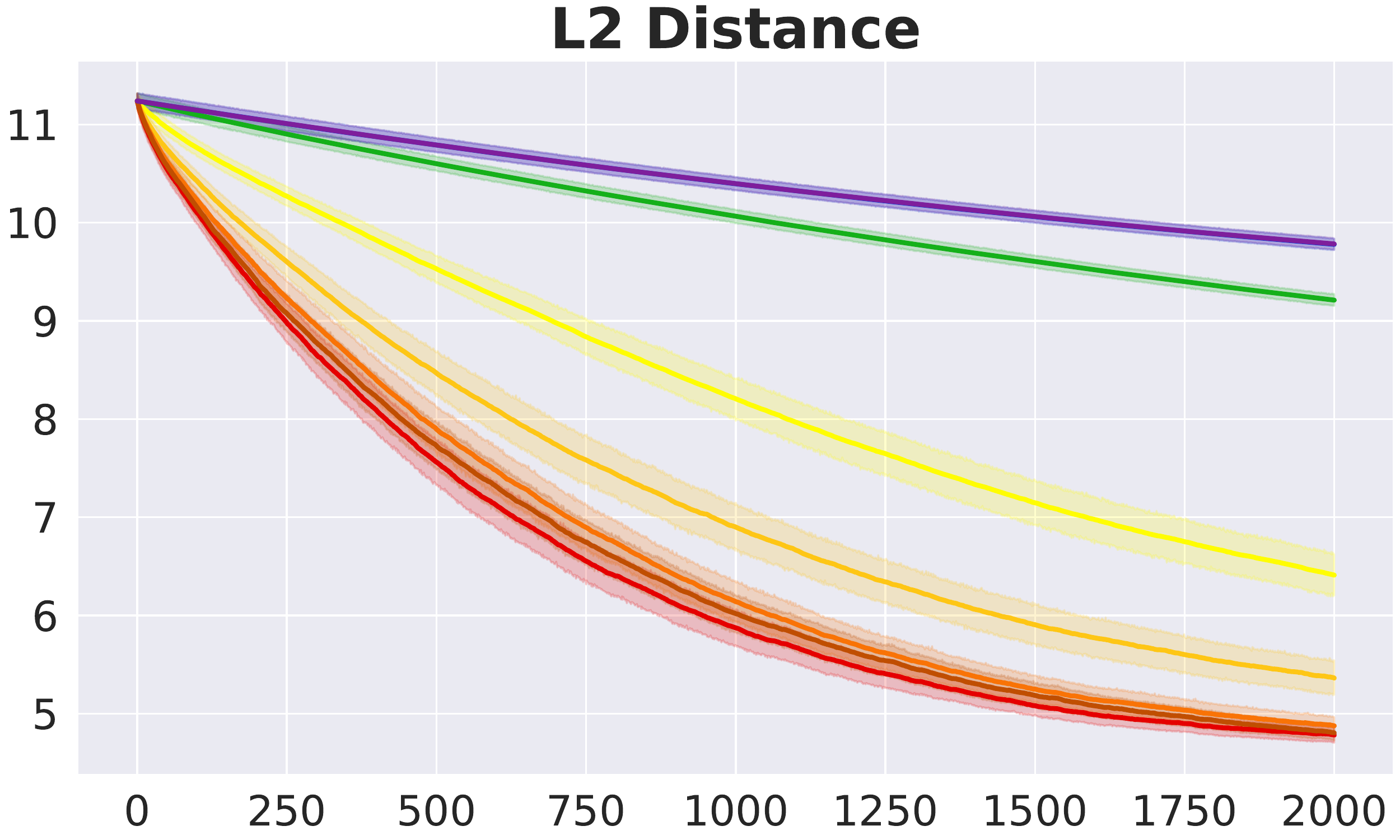}
    \end{subfigure}%
    ~
    \begin{subfigure}{0.32\textwidth}
        \centering
     \includegraphics[width=\textwidth]{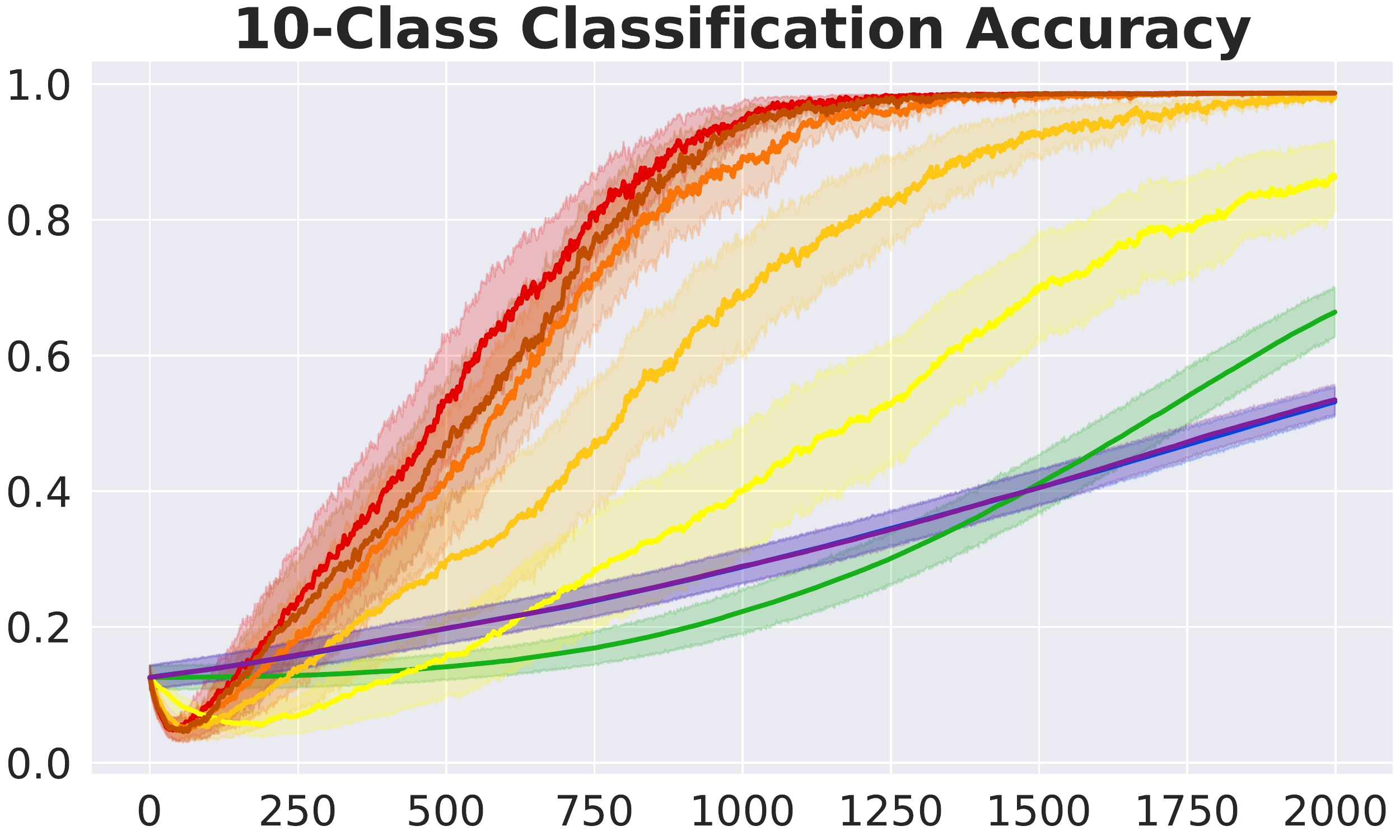}
    \end{subfigure}%
    ~
    \begin{subfigure}{0.32\textwidth}
        \centering
        \includegraphics[width=\textwidth]{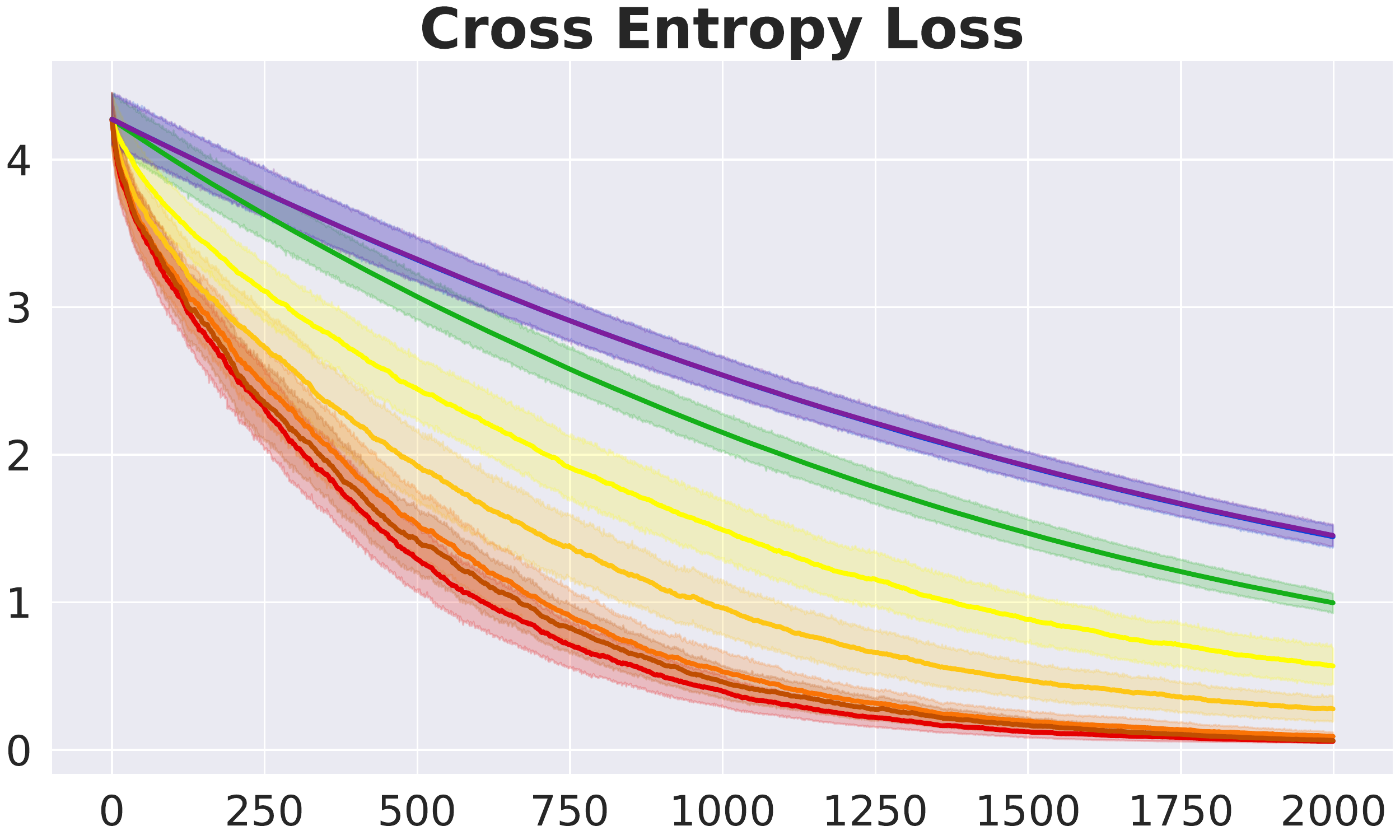}
    \end{subfigure}
    \begin{subfigure}{\textwidth}
        \centering
        \includegraphics[width=\textwidth]{Figures/imitate_legend.pdf}
    \end{subfigure}
    \caption{L2 distance, accuracy and Cross Entropy loss of the 10-class MNIST classification, in which the teacher uses 20D features.}
    \label{sup:fig:mnist-20}
\end{figure}
\begin{figure}[ht]
    \begin{subfigure}{0.32\textwidth}
        \centering
       \includegraphics[width=\textwidth]{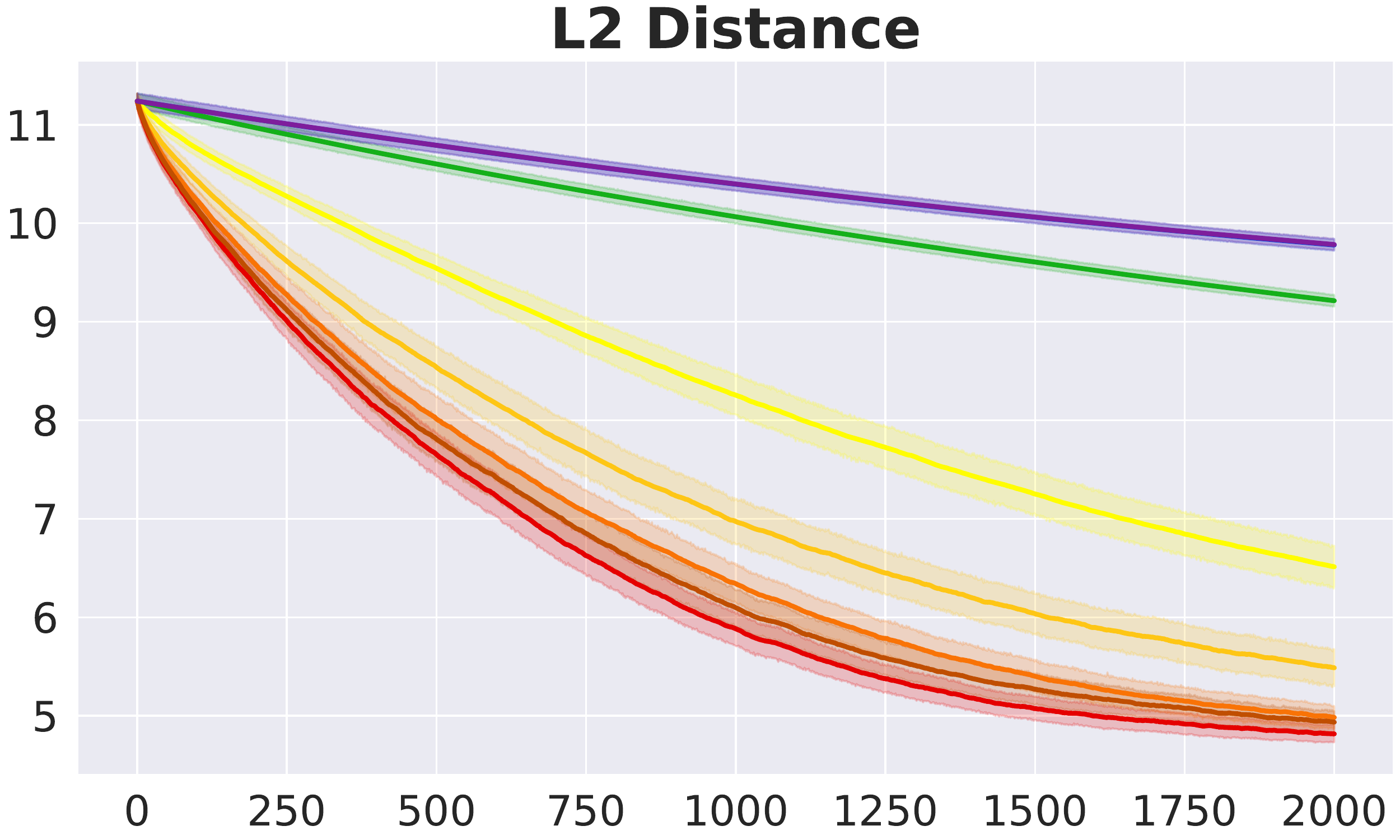}
    \end{subfigure}%
    ~
    \begin{subfigure}{0.32\textwidth}
        \centering
       \includegraphics[width=\textwidth]{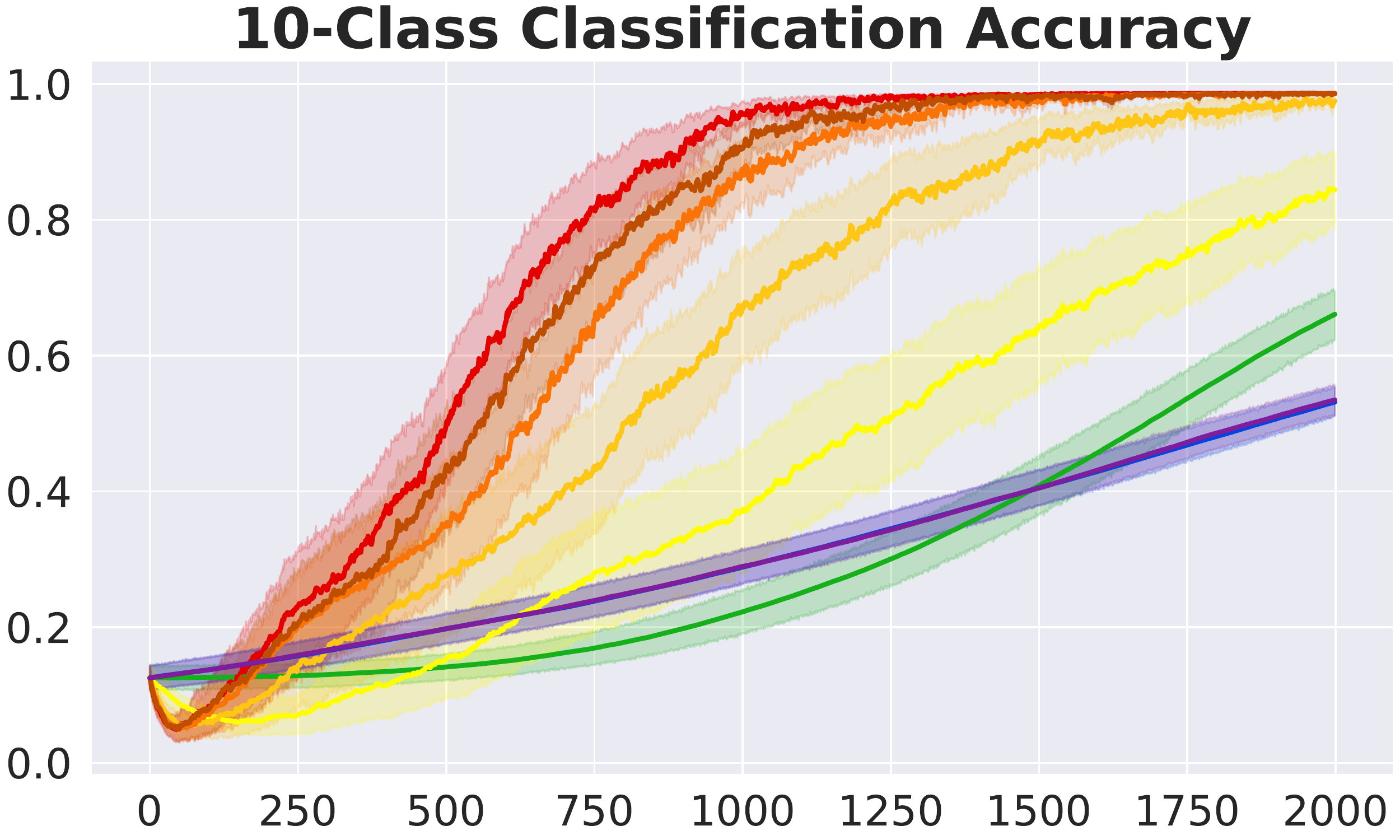}
    \end{subfigure}%
    ~
    \begin{subfigure}{0.32\textwidth}
        \centering
        \includegraphics[width=\textwidth]{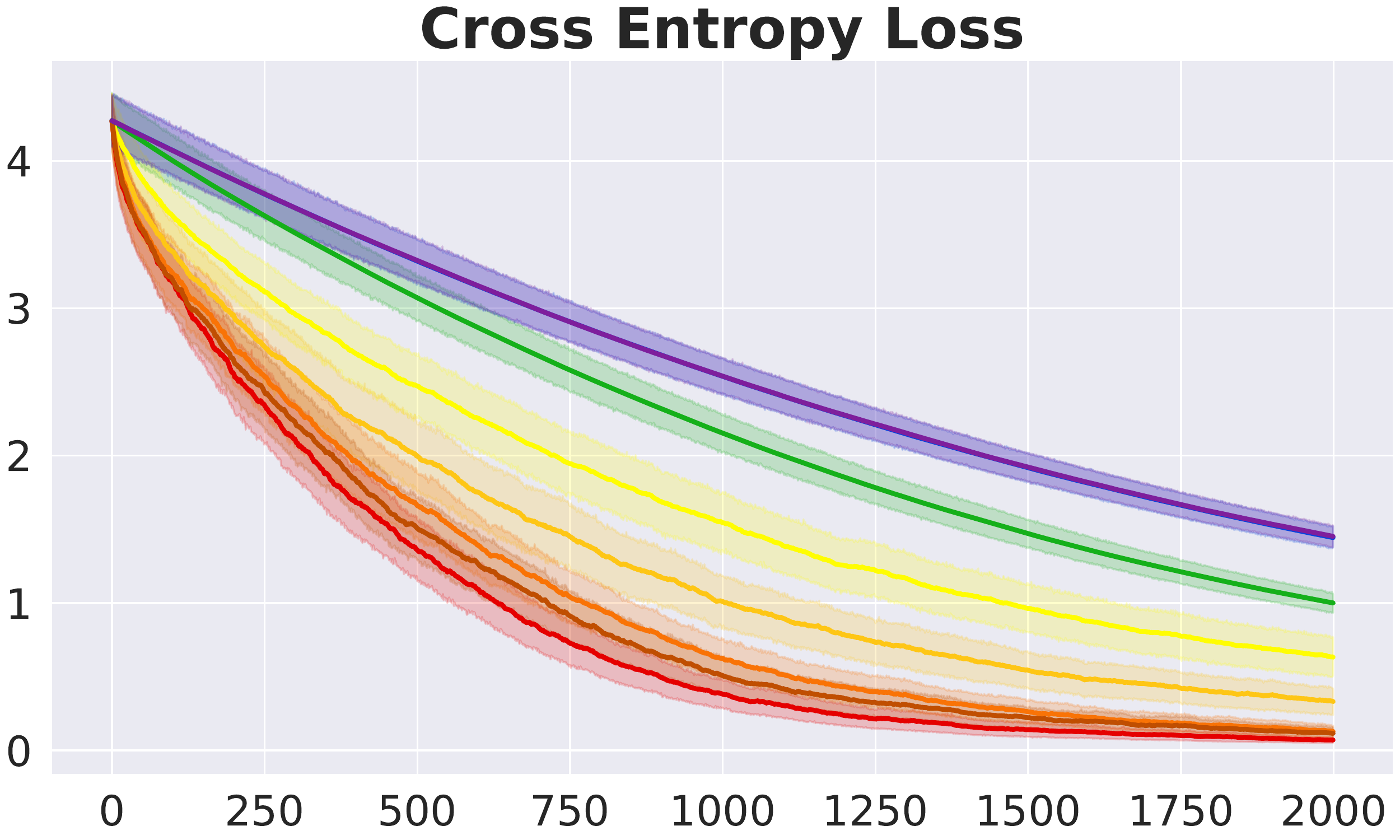}

    \end{subfigure}
    
    \begin{subfigure}{\textwidth}
        \centering
        \includegraphics[width=\textwidth]{Figures/imitate_legend.pdf}
    \end{subfigure}
    \caption{L2 distance, classification accuracy and Cross Entropy loss of the 10-class MNIST classification, in which the teacher uses 30D features.}
    \label{sup:fig:mnist-30}
\end{figure}
We also did experiment with the MNIST dataset, which didn't mentioned in the main text for space sake. For our 10-class MNIST experiment, we trained 3 different CNNs with the similar architecture, only differing in the number of units in the last fully connected (FC) layer. The structure is summarized in table~\ref{sup:tab:mnist-struct}. All three CNNs were able to achieve above 97$\%$ test accuracy. To test our ITAL method, we had the teacher teach the parameters of the FC layer to the student. The input of this layer is used as feature vectors of the images. The learner always used features with 24D, but the teacher varied with 20D and 30D, results presented in figure. In both settings, $\beta$ is set to 30000. The FC layer weights trained with supervise learning were used as $\nu^*$. Figure \ref{sup:fig:mnist-20} and \ref{sup:fig:mnist-30} show the classification accuracy and Cross Entropy loss of the training.

\subsection{Linear Classifiers on CIFAR-10}\label{sup:sec:exp-CIFAR}
\begin{figure}
    \begin{subfigure}{0.49\textwidth}
        \centering
        \includegraphics[width=\textwidth]{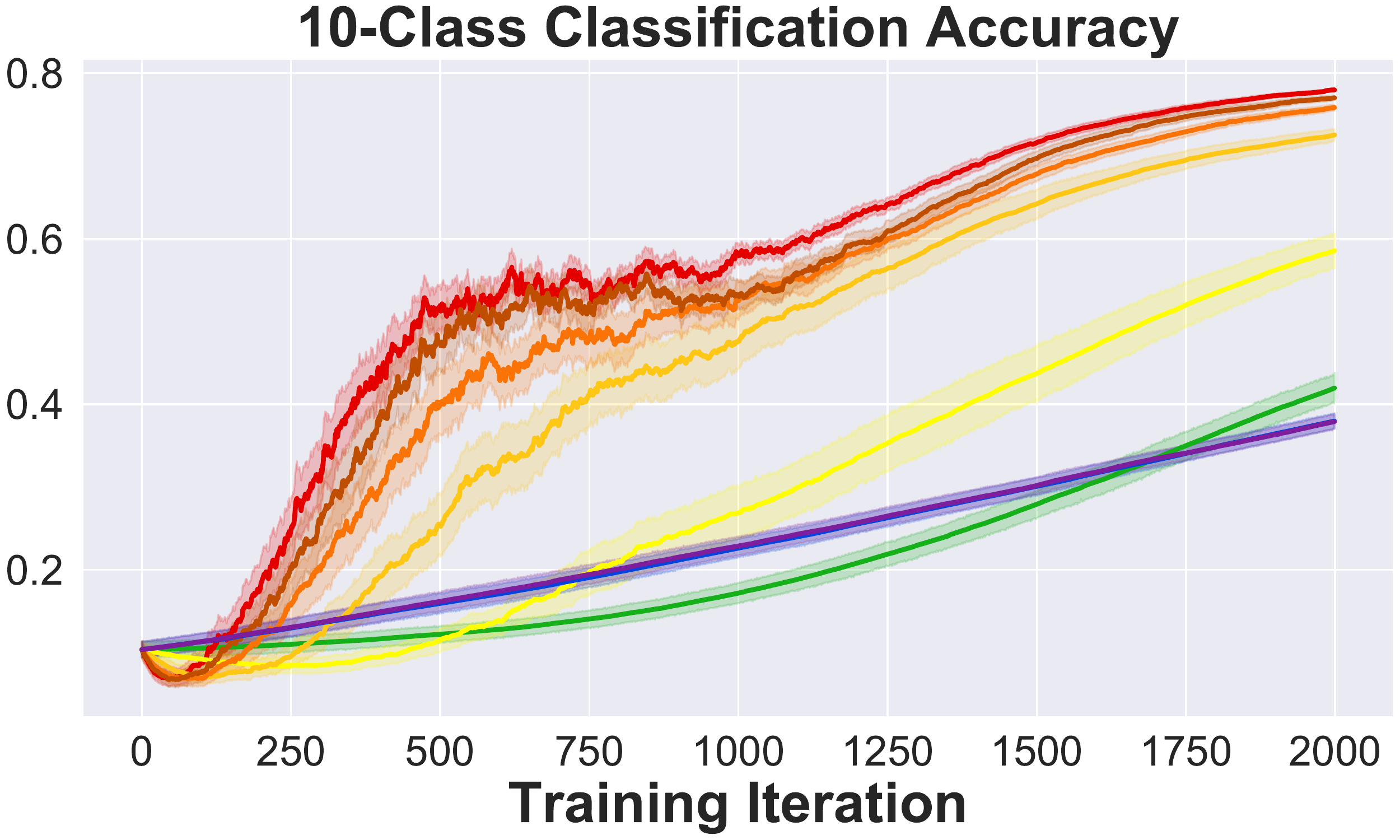}
    \end{subfigure}%
    ~
    \begin{subfigure}{0.49\textwidth}
        \centering
        \includegraphics[width=\textwidth]{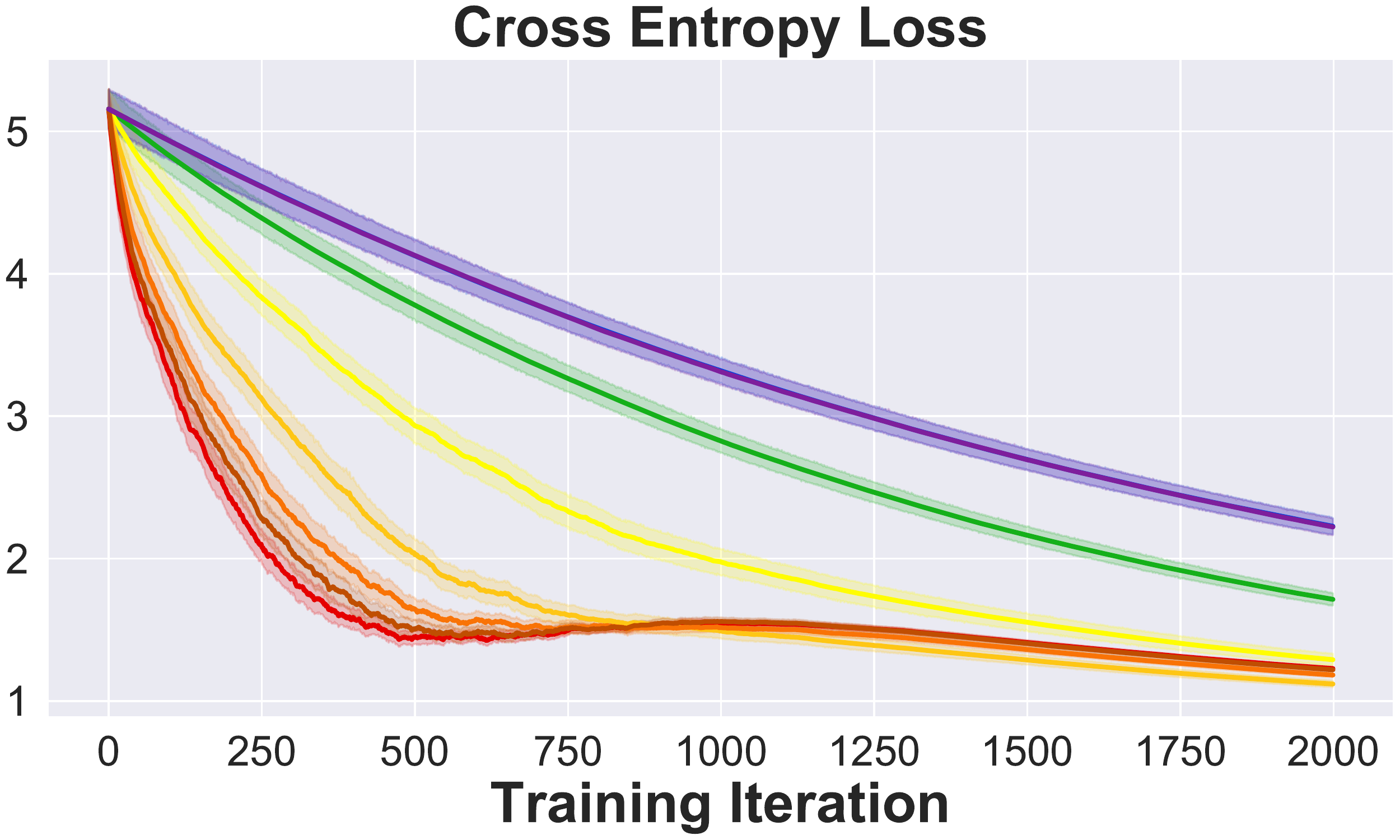}
    \end{subfigure}
    
    \begin{subfigure}{\textwidth}
        \centering
        \includegraphics[width=\textwidth]{Figures/imitate_legend.pdf}
    \end{subfigure}
    \caption{Accuracy and Cross Entropy loss of the 10-class CIFAR-10 classification, in which the teacher uses features extracted from CNN-6 detailed in table~\ref{sup:tab:cifar-struct}. The L2 loss curves we included in the main text section~\ref{sec:exp} figure~\ref{fig:CIFAR-10-cop}  was from this setting.}
    \label{sup:fig:cifar10-6}
\end{figure}

\begin{figure}[ht]
    \begin{subfigure}{0.32\textwidth}
        \centering
        \includegraphics[width=\textwidth]{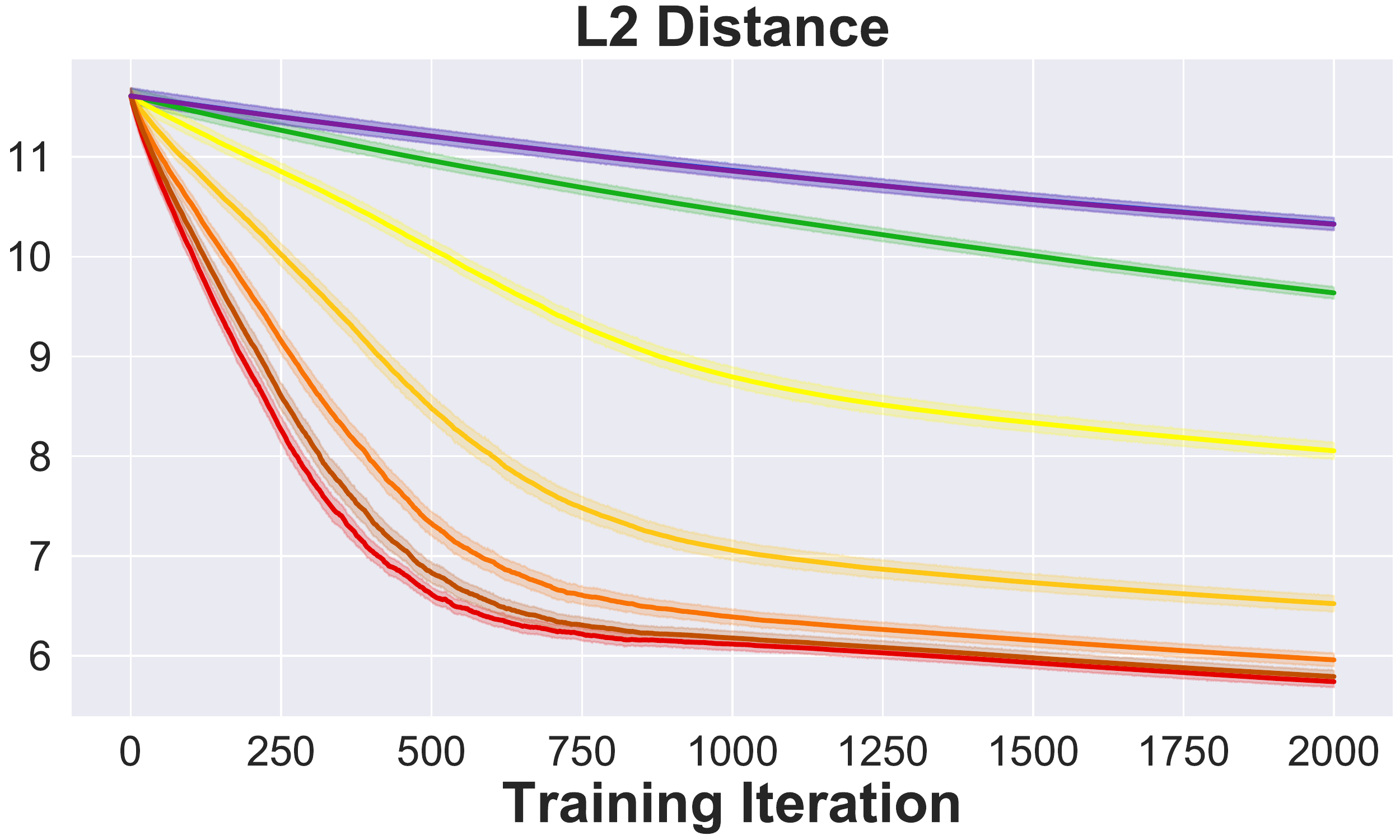}
    \end{subfigure}%
    ~
    \begin{subfigure}{0.32\textwidth}
        \centering
        \includegraphics[width=\textwidth]{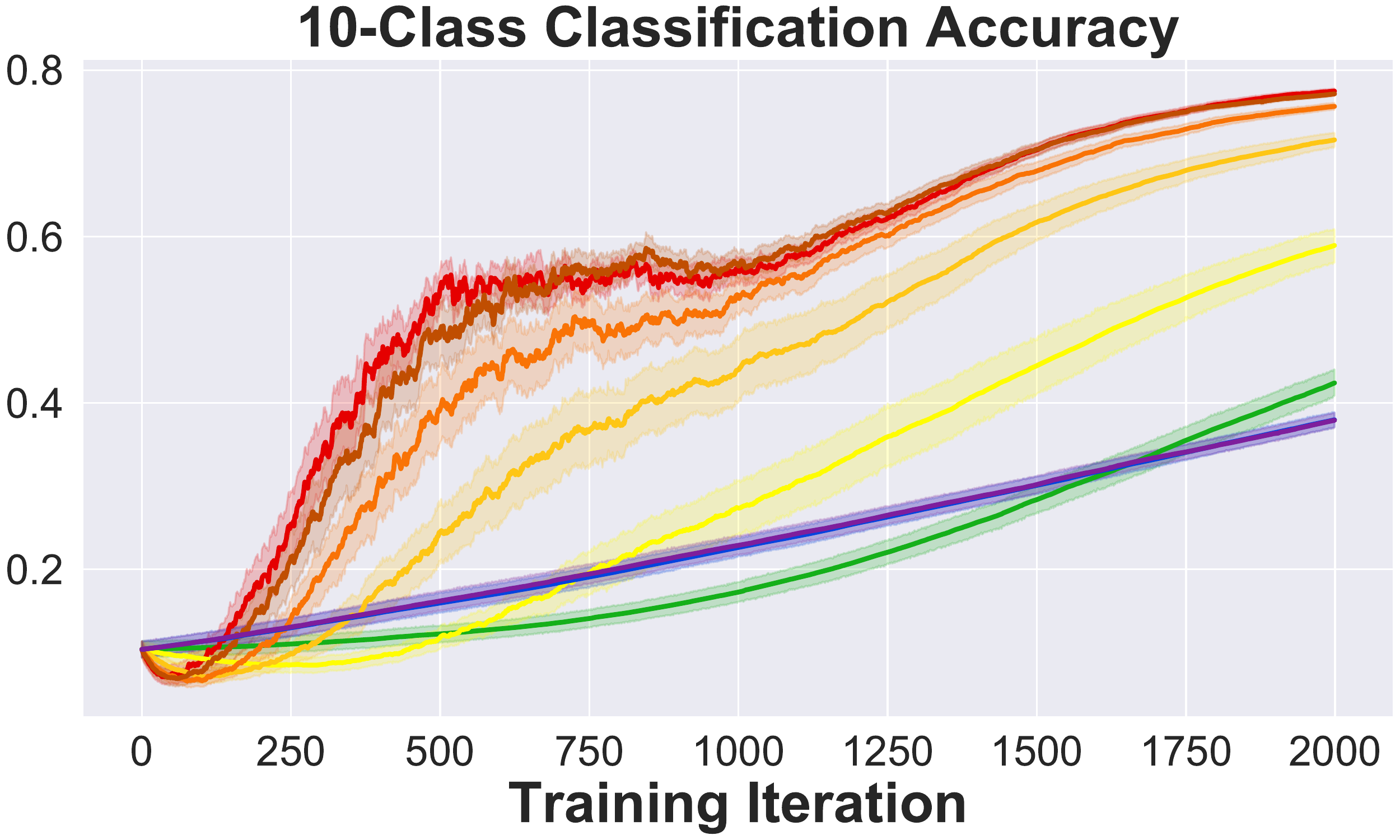}
    \end{subfigure}%
    ~
    \begin{subfigure}{0.32\textwidth}
        \centering
        \includegraphics[width=\textwidth]{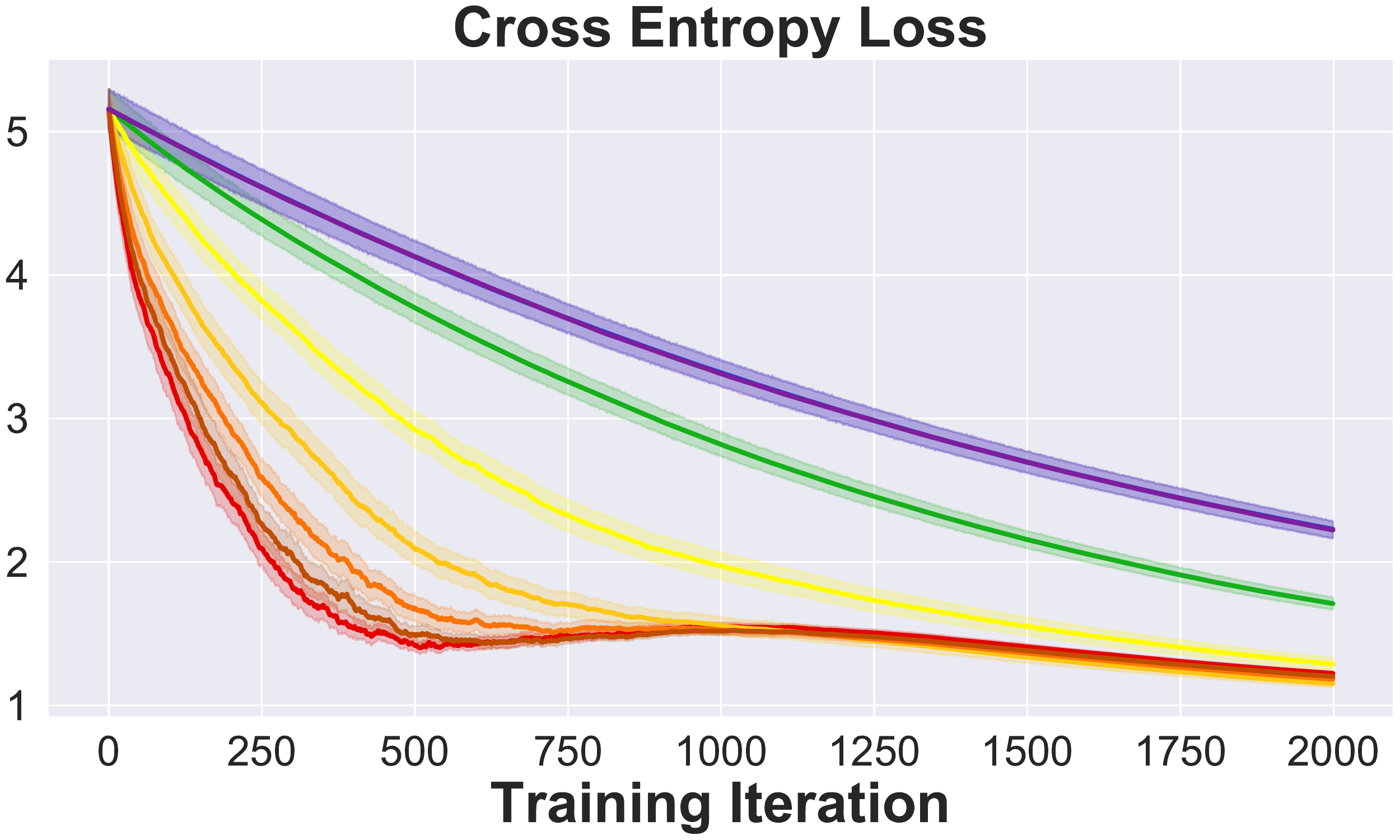}
    \end{subfigure}
    
    \begin{subfigure}{\textwidth}
        \centering
        \includegraphics[width=\textwidth]{Figures/imitate_legend.pdf}
    \end{subfigure}
    \caption{L2 distance, classification accuracy and Cross Entropy loss of the 10-class CIFAR-10 classification, in which the teacher uses features extracted from the CNN-12 detailed in table~\ref{sup:tab:cifar-struct}.}
    \label{sup:fig:cifar10-12}
\end{figure}
The overall design of this experiment resembles the MNIST classification. We used CIFAR-10, a dataset with more enriched and complicated natural images. We trained three different CNNs with 6, 9, and 12 convoluted layers on an augmented CIFAR dataset. With 40 epochs and an adaptive learning rate, we were able to achieve about 82 percent test accuracy for all three architectures. Table~\ref{sup:tab:cifar-struct} summarizes the CNN structure we used. To stabilize training, we used an exponential decaying $\beta$, $\beta_t = 50000(1-5e^{-6})^t$. We think that 
because the feature representation is quite different between the teacher and the student, as the iterative learning goes, the approximation error might accumulate. Thus, the learner's estimation of the teacher's data selection will be less accurate towards the end of the learning, especially for the most ambiguous examples (images prone to mistakes). At this time, using a large $\beta$ can be unstable. In other words, it is hard for the learner to reason about the teacher at the end of the learning, so he should be less confident using the pragmatic information suggested by the teacher's intention. Figure~\ref{sup:fig:cifar10-6} and~\ref{sup:fig:cifar10-12} show the accuracy and Cross Entropy loss of this task.

\begin{table}
\caption{CIFAR-10 CNN structures.}
\centering
\begin{tabular}{|c|c|c|c|}
\hline
                     & CNN-6                            & CNN-9                            & CNN-12                           \\ \hline
Conv 1               & 2 layers of 16 {[}3$\times$3{]} filters & 3 layers of 16 {[}3$\times$3{]} filters & 4 layers of 16 {[}3$\times$3{]} filters \\ \hline
Pool                 & \multicolumn{3}{c|}{2$\times$2 Max with Stride 2}                                                             \\ \hline
Conv 2               & 2 layers of 32 {[}3$\times$3{]} filters & 3 layers of 32 {[}3$\times$3{]} filters & 4 layers of 32 {[}3$\times$3{]} filters \\ \hline
Pool                 & \multicolumn{3}{c|}{2$\times$2 Max with Stride 2}                                                             \\ \hline
Conv 3               & 2 layers of 64 {[}3$\times$3{]} filters & 3 layers of 64 {[}3$\times$3{]} filters & 4 layers of 64 {[}3$\times$3{]} filters \\ \hline
Pool                 & \multicolumn{3}{c|}{2$\times$2 Max with Stride 2}                                                             \\ \hline
FC                   & 32                              & 32                               & 32                               \\ \hline
\end{tabular}
\label{sup:tab:cifar-struct}
\end{table}

\subsection{Linear Classifiers on Tiny ImageNet}\label{sup:sec:exp-ImageNet}
\begin{figure}
    \begin{subfigure}{0.49\textwidth}
        \centering
        \includegraphics[width=\textwidth]{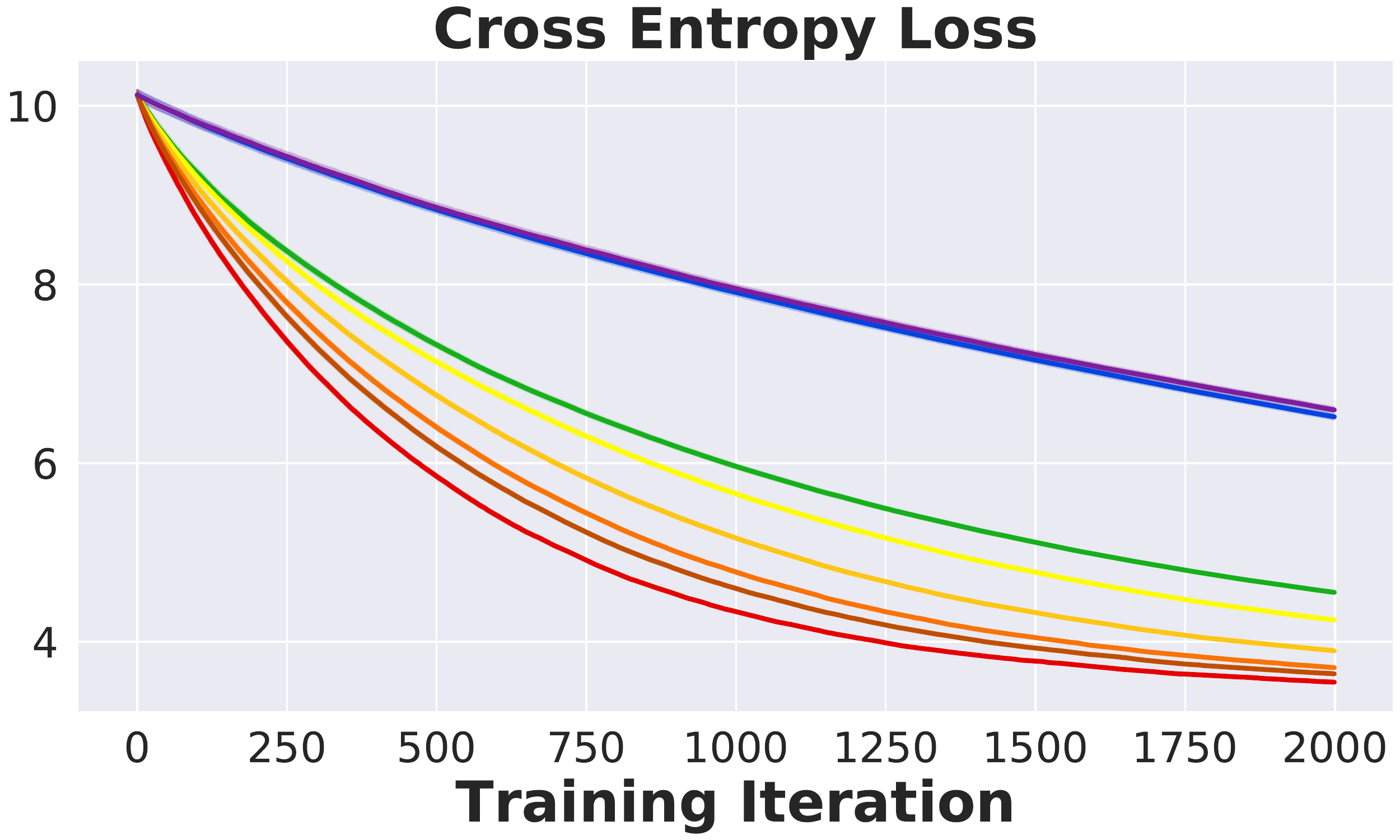}
    \end{subfigure}%
    ~
    \begin{subfigure}{0.49\textwidth}
        \centering
        \includegraphics[width=\textwidth]{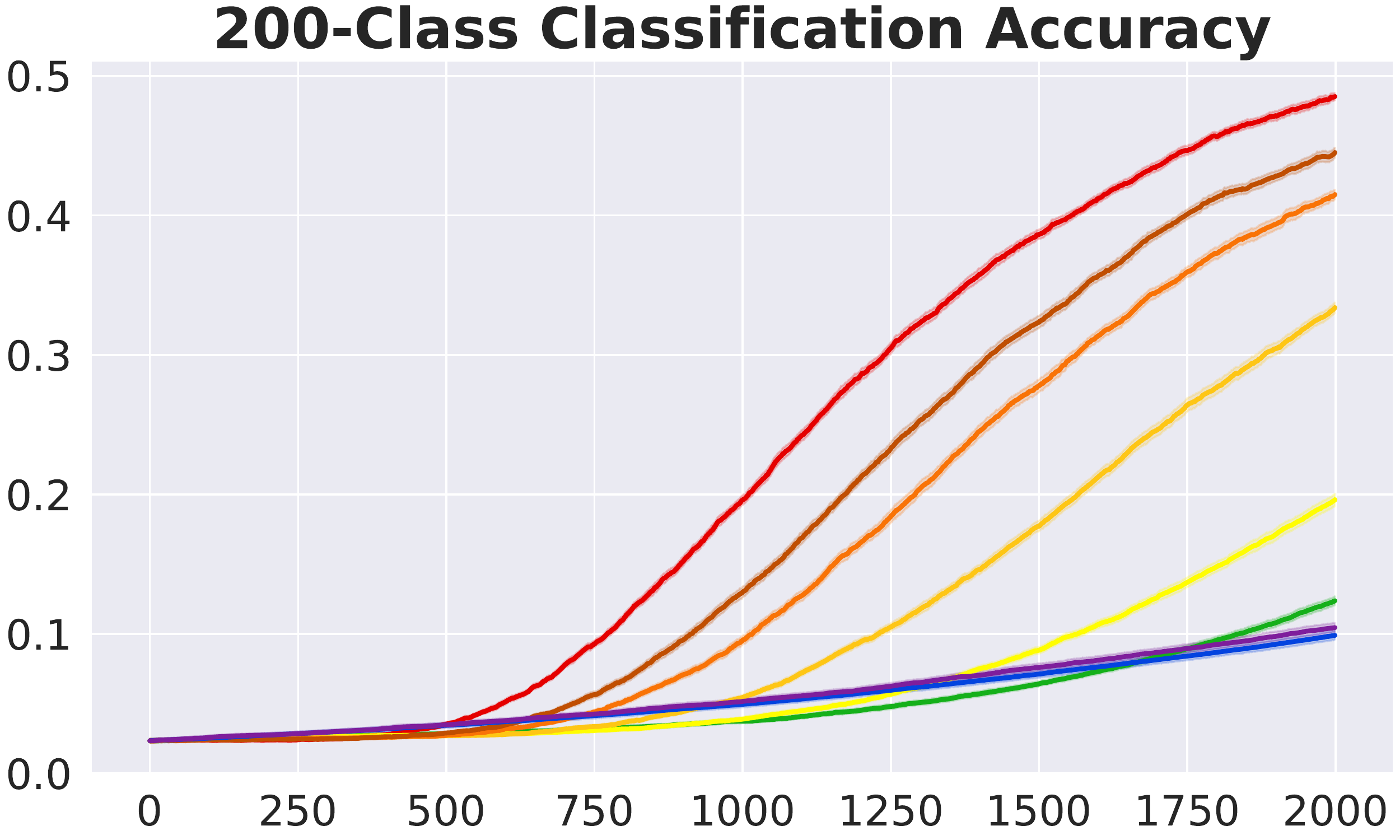}
    \end{subfigure}
    
    \begin{subfigure}{\textwidth}
        \centering
        \includegraphics[width=\textwidth]{Figures/imitate_legend.pdf}
    \end{subfigure}
    \caption{Top-5 accuracy and Cross Entropy loss of the 200-class Tiny ImageNet classification, in which the teacher uses features extracted from VGG-13. The L2 loss curves we included in the main text \cref{sec:exp} \cref{fig:ImageNet-13}  was from this setting.}
    \label{sup:fig:ImageNet-13}
\end{figure}

\begin{figure}[ht]
    \begin{subfigure}{0.32\textwidth}
        \centering
        \includegraphics[width=\textwidth]{Figures/ImageNet/imgnt_coop13_l2.pdf}
    \end{subfigure}%
    ~
    \begin{subfigure}{0.32\textwidth}
        \centering
        \includegraphics[width=\textwidth]{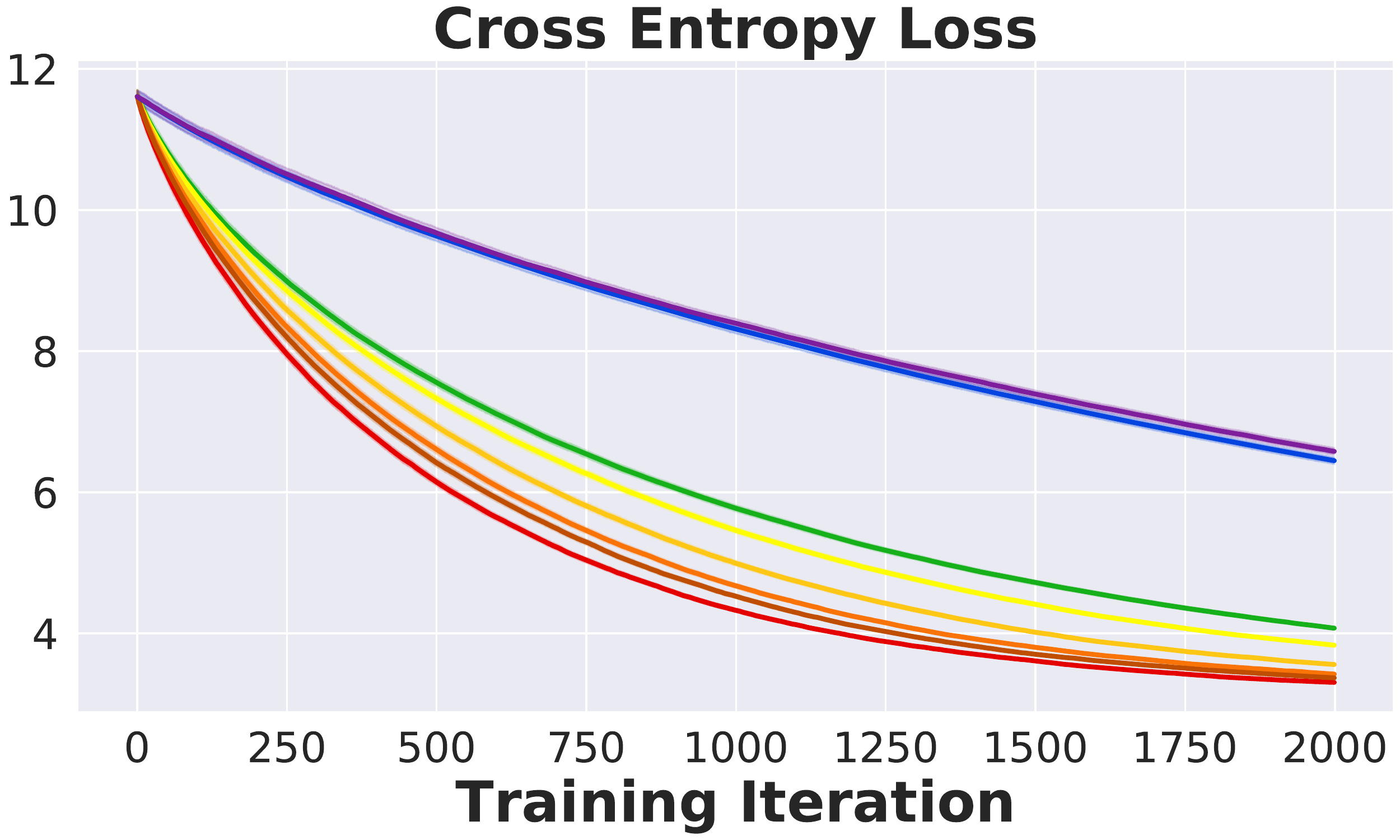}
    \end{subfigure}%
    ~
    \begin{subfigure}{0.32\textwidth}
        \centering
        \includegraphics[width=\textwidth]{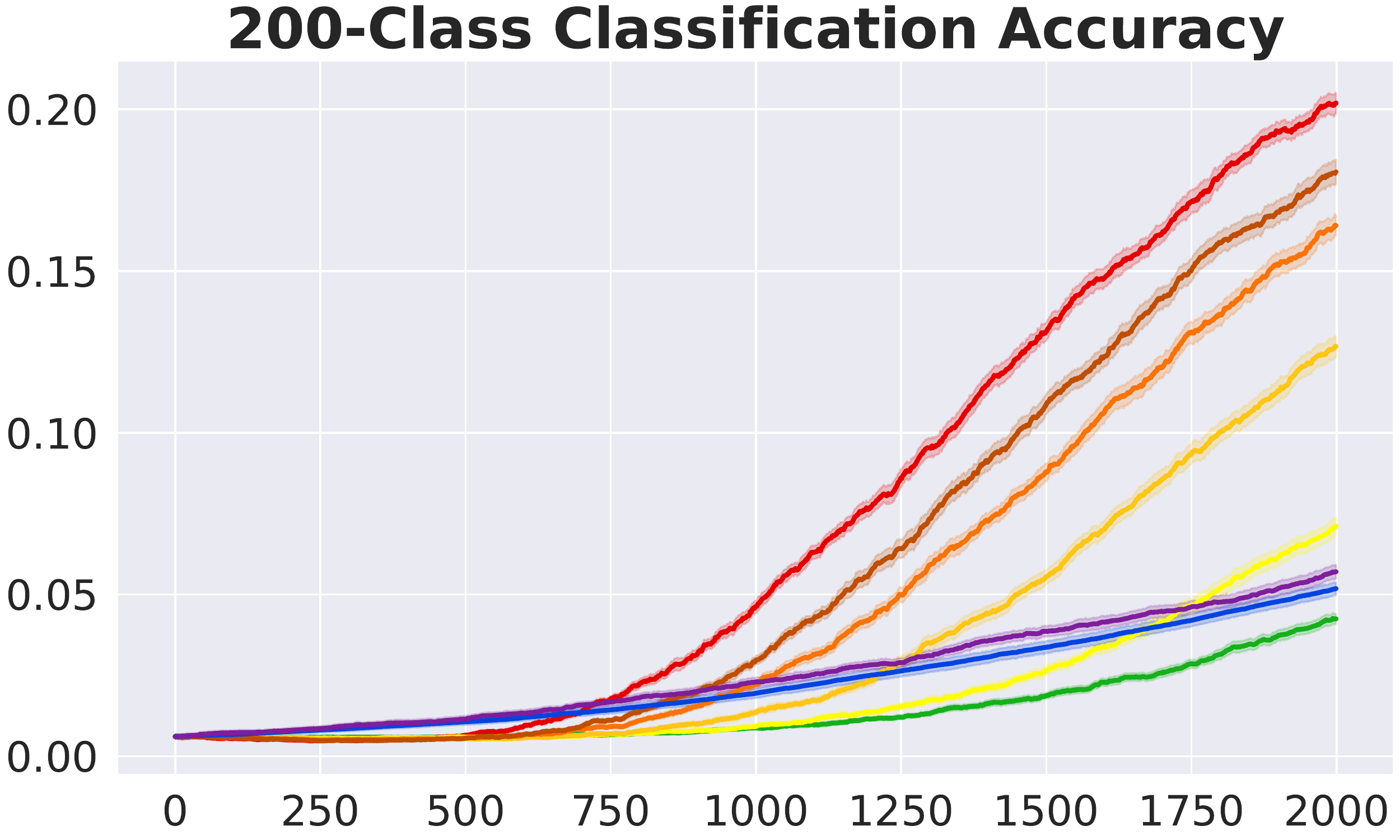}
    \end{subfigure}
    
    \begin{subfigure}{\textwidth}
        \centering
        \includegraphics[width=\textwidth]{Figures/imitate_legend.pdf}
    \end{subfigure}
    \caption{L2 distance, top-5 accuracy and Cross Entropy loss of the 200-class Tiny ImageNet classification, in which the teacher uses features extracted from the VGG-19.}
    \label{sup:fig:ImageNet-19}
\end{figure}
The overall design of this experiment resembles the MNIST and CIFAR-10 classification. We used Tiny ImageNet, a large scale dataset with natural images. We first extracted 2048D features from VGG-13/16/19 without finetuning and then downsampled the features to 10D with a multilayer perceptron with three FC-ReLU-layers (500, 250, 10) trained with Cross Entropy loss. Figure~\ref{sup:fig:ImageNet-13} and~\ref{sup:fig:ImageNet-19} show the accuracy and Cross Entropy loss of this task.

\subsection{Linear Regression for Equation Simplification}\label{sup:sec:exp-equation}
\begin{figure}[ht]
    \begin{subfigure}{0.49\textwidth}
        \centering
        \includegraphics[width=\textwidth]{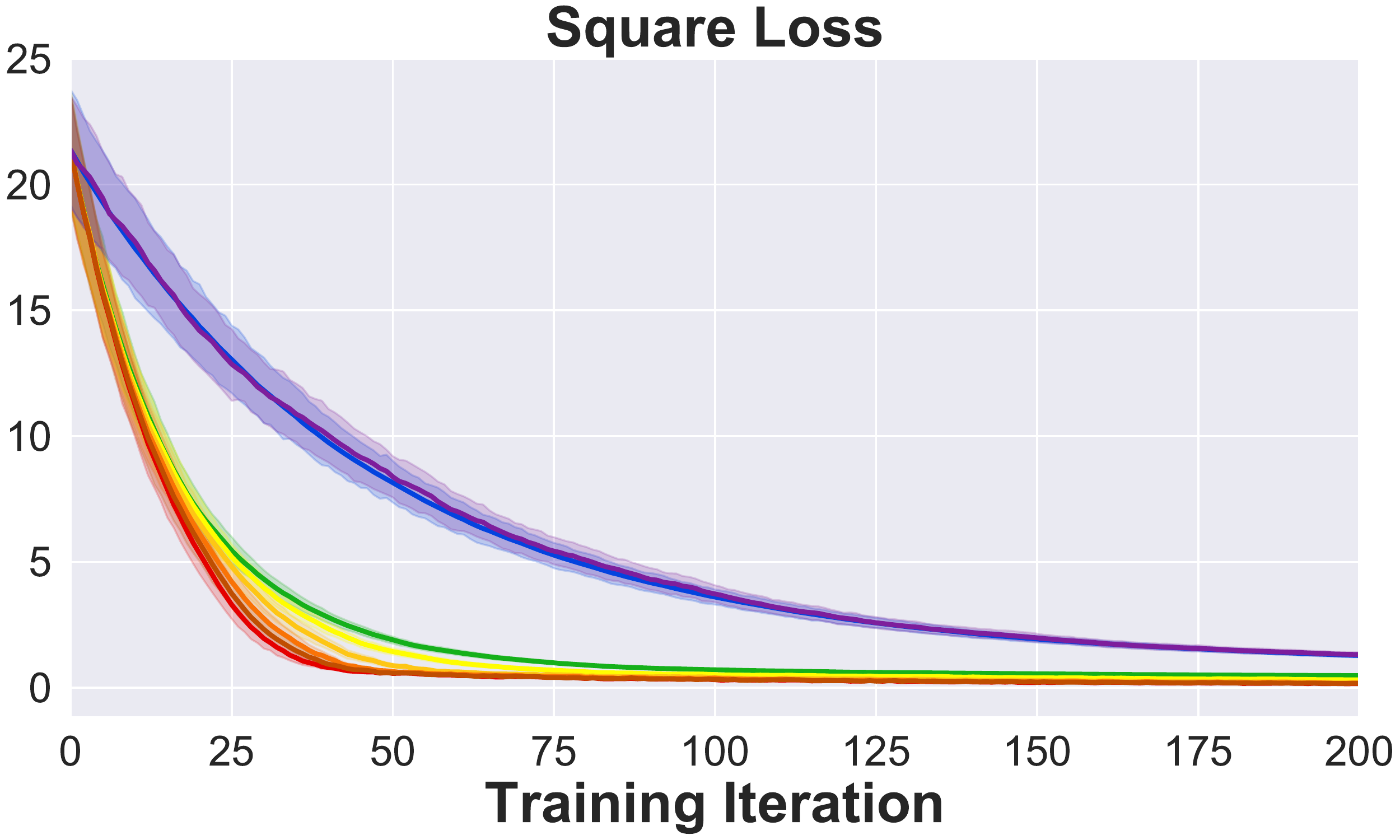}
    \end{subfigure}%
    ~
    \begin{subfigure}{0.49\textwidth}
        \centering
        \includegraphics[width=\textwidth]{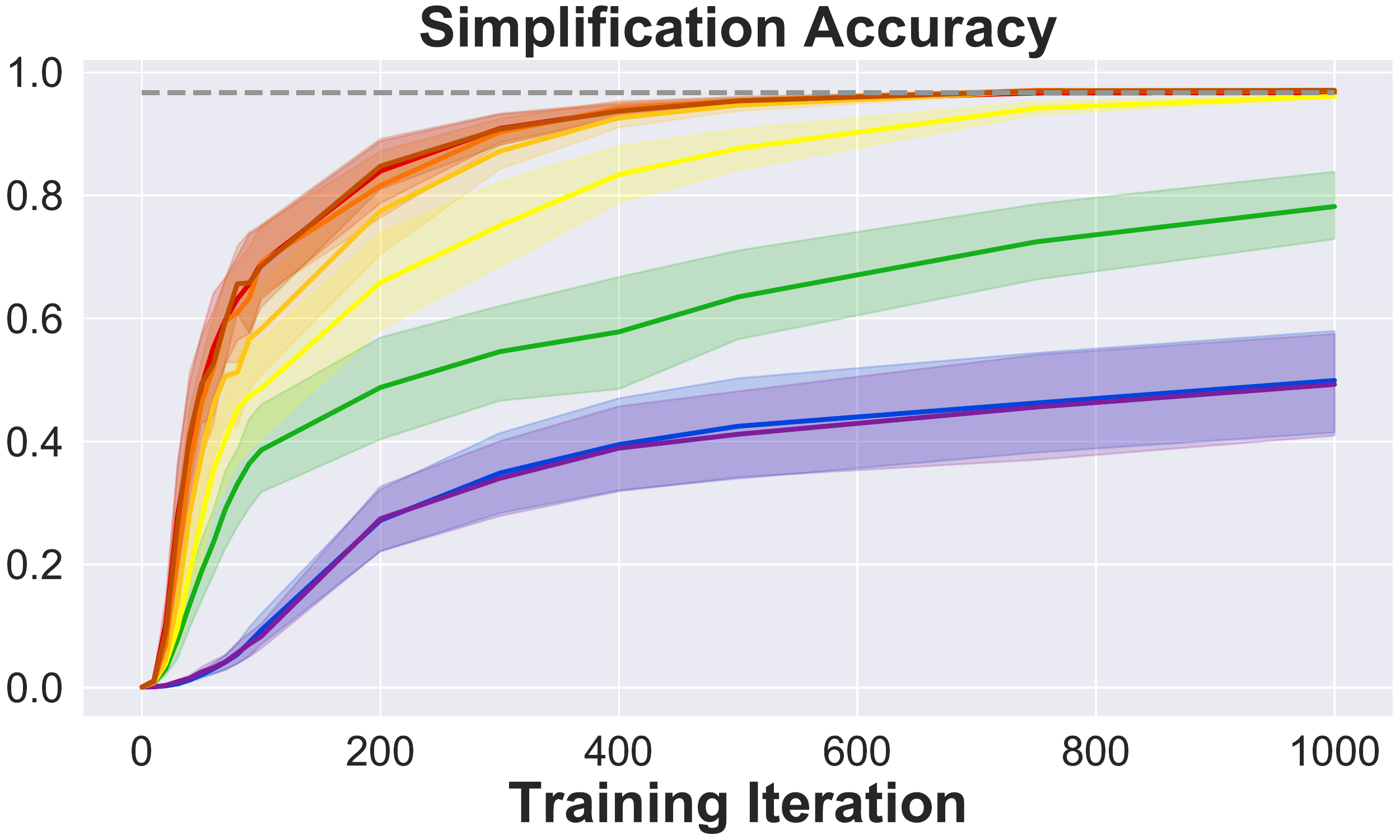}
    \end{subfigure}
    
    \begin{subfigure}{\textwidth}
        \centering
        \includegraphics[width=\textwidth]{Figures/imitate_legend.pdf}
    \end{subfigure}
    \caption{Square loss and simplification accuracy using the learned value function. We compared the last equation in the trajectories generated by the predefined rules and the greedy search results guided by the learned value function. Given the same teacher, teacher-aware learning algorithm outperforms naive learners in terms of accuracy and convergence rate. The gray horizontal dash line represents test accuracy using the ground truth parameter of 45D. For these results, the teacher uses 40D features, same as the L2-loss in section~\ref{sec:exp} figure~\ref{fig:Equation-cop}.}
    \label{sup:fig:equation-40}
\end{figure}
\begin{figure}
    \begin{subfigure}{0.32\textwidth}
        \centering
        \includegraphics[width=\textwidth]{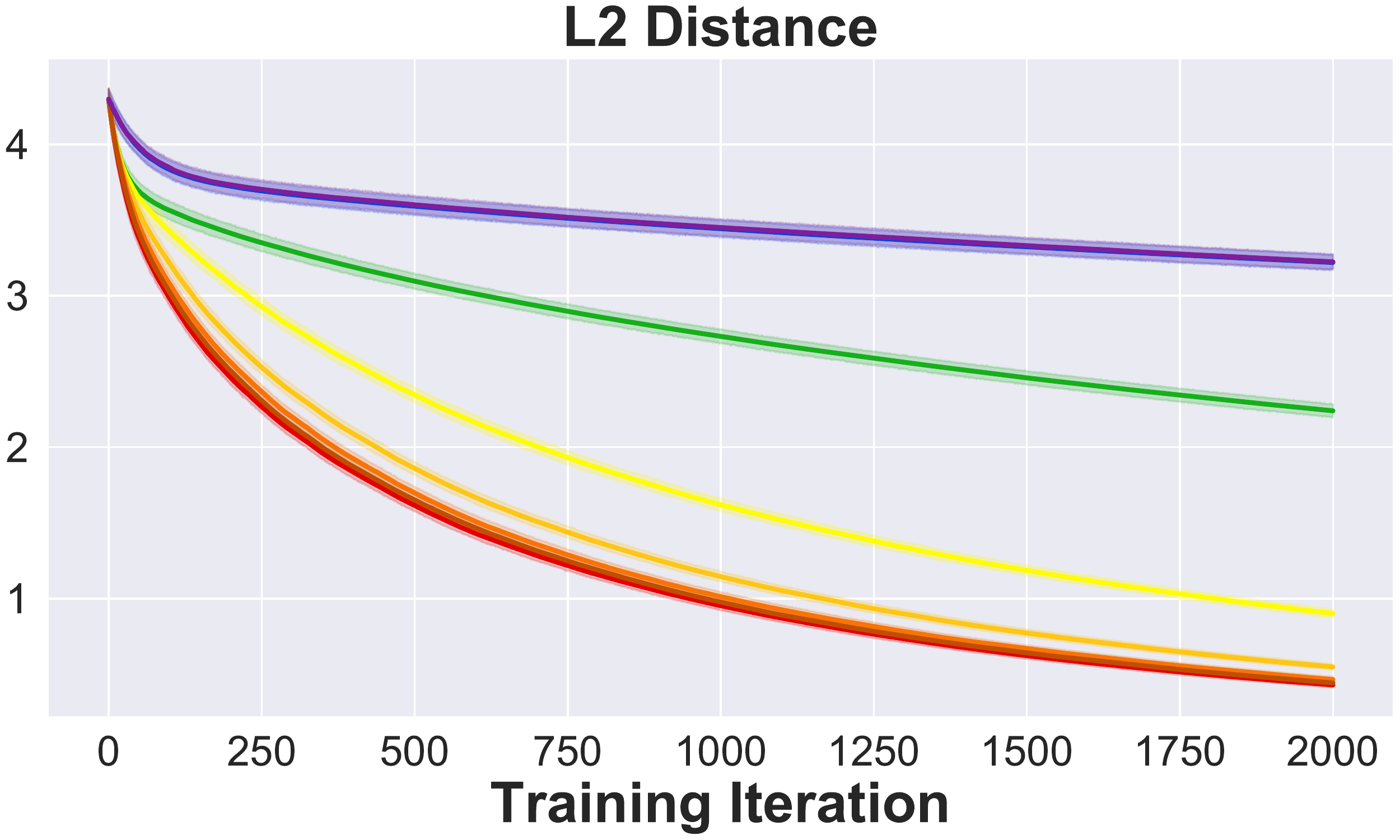}
    \end{subfigure}%
    ~
    \begin{subfigure}{0.32\textwidth}
        \centering
        \includegraphics[width=\textwidth]{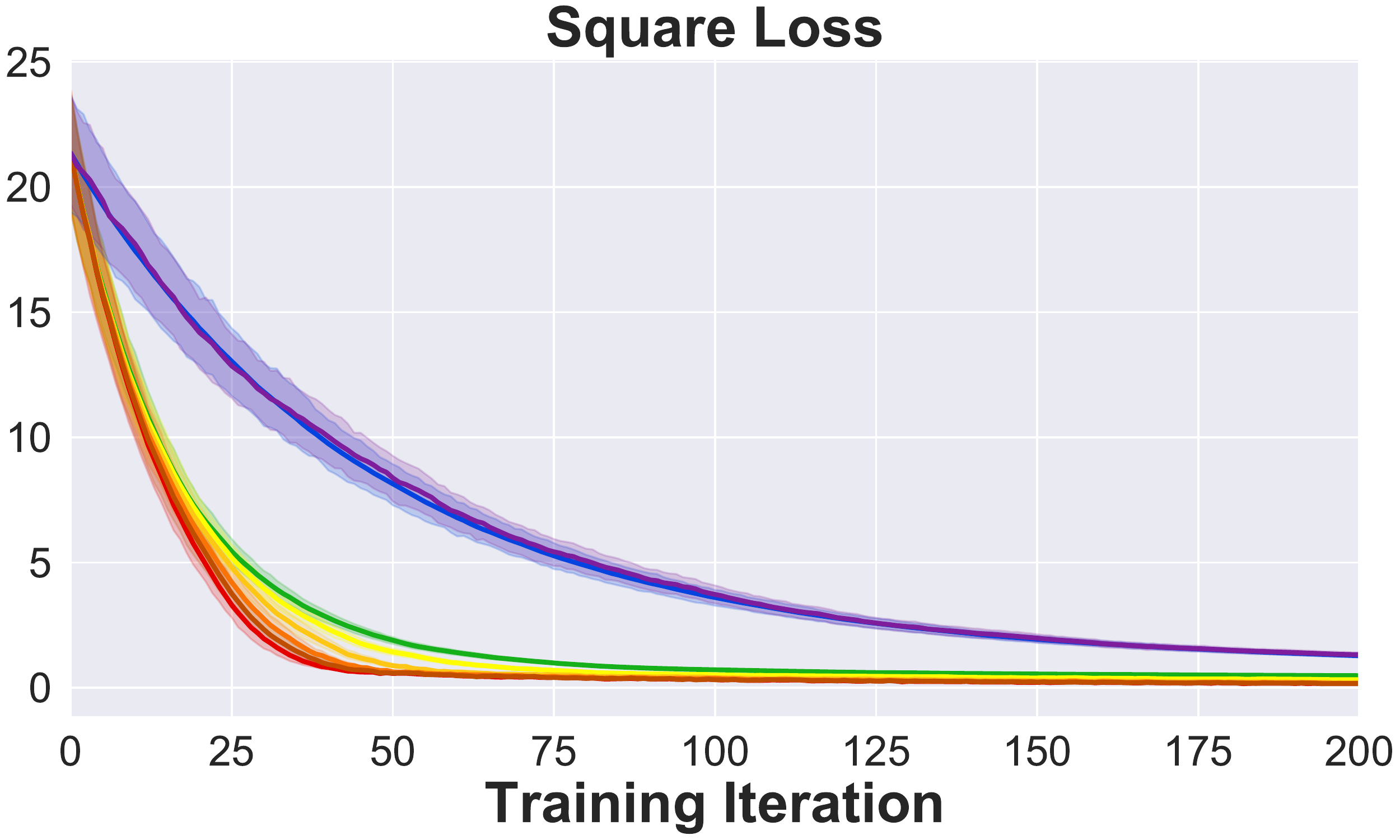}
    \end{subfigure}%
    ~
    \begin{subfigure}{0.32\textwidth}
        \centering
        \includegraphics[width=\textwidth]{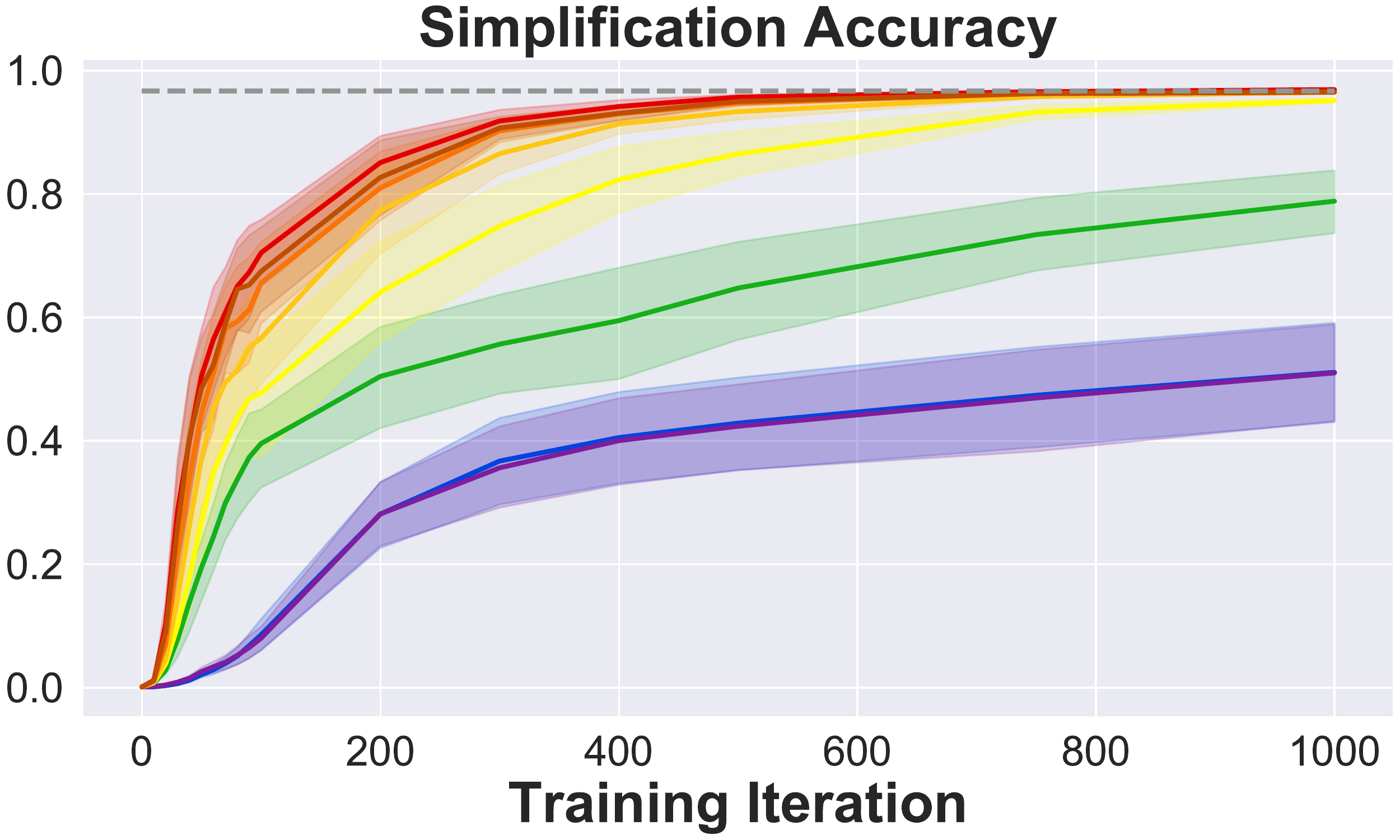}
    \end{subfigure}
    
    \begin{subfigure}{\textwidth}
        \centering
        \includegraphics[width=\textwidth]{Figures/imitate_legend.pdf}
    \end{subfigure}
    \caption{L2 distance, square loss and simplification accuracy using the learned value function. The teacher uses 50D features for these results.}
    \label{sup:fig:equation-50}
\end{figure}
In this experiment, we let the teacher teach a value function to the student so that he can use this value function to simplify polynomials given predefined operations. We first created an Equation Simplification dataset. we randomly generate two fourth-degree polynomials with three variables $x, y, z$ as the left- and right-hand sides of the equations. The coefficients of the polynomials are either fractions or integers. The range of magnitude is 20 and 5 for the nominator and denominator, respectively. We define a set of operations that can be performed on an equation.
\begin{itemize}
\item \textbf{Scale}: scale a term by an integer factor.
\item \textbf{Reduction}: reduce the fraction coefficient of a term to the simplest form.
\item \textbf{Cancel common factors}: divide coefficients of all terms by the greatest common factor of integer coefficients and the nominators of fractional coefficients.
\item \textbf{Move}: move a term to a specified position in the equation.
\item \textbf{Merge}: merge two terms that contain the same denominators and variables with the same degrees.
\item \textbf{Cancel denominators}: multiply all terms by the least common multiple of the denominators of all coefficients.
\end{itemize}
To simplify an equation, we apply operations in the following way:
\begin{enumerate}
    \item Canceling common factors
    \item Merging terms with the same denominators
    \item Merging terms with different denominators by scaling the terms with the least common multiple and then applying the merge operation
    \item Removing fractions in the coefficients
    \item Rearranging the terms by descending degrees of $x, y, z$ with the move operation
\end{enumerate}

At each step, only one operation is performed on a single term (two terms for merging), and we do not move on to the next operation until the present one is no longer applicable to the current equation. After the simplification process, all the remaining terms are on the left-hand side, while the right-hand side is simply 0. We record the series of equations generated as a simplification trajectory. Some example simplification trajectories would be:

\begin{flalign*}
&\text{Example equation } 1:&\frac{1}{4}xyz -\frac{3}{2}xyz = -14y^3 +\frac{1}{5}&\\
&&-\frac{5}{4}xyz = -14y^3 +\frac{1}{5}\quad&&\text{(merge)}\\
&& -25xyz = -280y^3 +4\quad&&\text{(cancel\;denominators)}\\
&&-25xyz +280y^3 = 4\quad&&\text{(move)}\\
&&-25xyz +280y^3 -4 = 0\quad&&\text{(move)}
\end{flalign*}

\begin{flalign*}
&\text{Example equation } 2:&5x^3y +\frac{8}{3}z = -\frac{14}{3}z +6xy^2z +\frac{11}{3}xz^2 +6yz^2 &\\
&&5x^3y +\frac{22}{3}z = 6xy^2z +\frac{11}{3}xz^2 +6yz^2\;&&\text{(merge)}\\
&&15x^3y +22z = 18xy^2z + 11xz^2 +18yz^2\;&&\text{(cancel\;denominators)}\\
&&15x^3y -18xy^2z +22z = 11xz^2 +18yz^2\;&&\text{(move)}\\
&&15x^3y -18xy^2z -11xz^2 +22z = 18yz^2\;&&\text{(move)}\\
&&15x^3y -18xy^2z -11xz^2 -18yz^2 +22z = 0\;&&\text{(move)}
\end{flalign*}

\begin{table}[ht]
    \centering
    \caption{Equation CNN structure}
    \begin{tabular}{|c|c|c|c|}
    \hline
           & 40-Dim CNN       & 45-Dim CNN      & 50-Dim CNN     \\ \hline
    Conv 1 & \multicolumn{3}{c|}{1 layer, 64 {[}5$\times$5{]} filters, leaky ReLU}  \\ \hline
    Conv 2 & \multicolumn{3}{c|}{1 layers, 64 {[}5$\times$5{]} filters, leaky ReLU} \\ \hline
    Pool   & \multicolumn{3}{c|}{2$\times$2 Max with Stride 2}          \\ \hline
    Conv 3 & \multicolumn{3}{c|}{1 layer, 32 {[}3$\times$3{]} filters, leaky ReLU}  \\ \hline
    Conv 4 & \multicolumn{3}{c|}{1 layer, 32 {[}3$\times$3{]} filters, leaky ReLU}  \\ \hline
    Pool   & \multicolumn{3}{c|}{2$\times$2 Max with Stride 2}          \\ \hline
    Conv 5 & \multicolumn{3}{c|}{1 layer, 32 {[}3$\times$3{]} filters, leaky ReLU}  \\ \hline
    Conv 6 & \multicolumn{3}{c|}{1 layer, 32 {[}3$\times$3{]} filters, leaky ReLU}  \\ \hline
    Pool   & \multicolumn{3}{c|}{2$\times$2 Max with Stride 2}          \\ \hline
    FC     & 40, tanh           & 45, tanh              & 50, tanh             \\ \hline
    \end{tabular}
\label{sup:tab:equation-struct}
\end{table}

We applied CNN $\phi_{\theta}$ to learn the features of the generated equations and a linear value function wrt. these features. We first encode equations using a codebook which maps each character to a trainable vector embedding. Thus, each equation can be encoded as a matrix. Then, we treat each equation as a 3D tensor with size $1\times W\times C$, where $W$ is the number of characters in the equation and $C$ is the length of embedding. We set $C = 30$, and $W$ ranges from 6 to 173. During training, we padded 0 to make sure all equations in one batch form a regular tensor. We fed the encoded equations to the CNN and used the output as their feature vectors. The structure of the CNN is summarized in table \ref{sup:tab:equation-struct}. The value of a given equation is the inner product of its feature vector and the parameter $\omega$. The loss function is based on contrastive loss and seeks to maximize the difference between the values of the simpler equations and the complicated equations. During training, we learn the network parameters and the weight vector simultaneously with:
\begingroup
\allowdisplaybreaks[0]
\begin{align*}
   \mathcal{L}(\omega, \theta)=&\frac{1}{|S_+|}\sum_{(E_i, E_j)\in S_+}\max\Big(1 - \big(\phi_{\theta}(E_i)-\phi_{\theta}(E_j)\big)^T\omega, 0\Big)+\\
   &\frac{1}{|S_-|}\sum_{(E_i, E_j)\in S_-}\max\Big(1 - \big(\phi_{\theta}(E_j)-\phi_{\theta}(E_i)\big)^T\omega, 0\Big) + \frac{\lambda}{2}\|\omega\|_2^2
\end{align*}
\endgroup
where $S_+$ and $S_-$ hold positive and negative pairs respectively. The positive data are pairs of equations from the same simplification trajectory, where the first equation in the pair is generated after the second one. That is, the first equation is simplified from the second equation, hence having a higher value. For the negative data, we randomly select an equation from a simplification trajectory excluding the simplification result and randomly apply an operation to that equation. If the result of the operation is different from the next equation in the trajectory, we add the result-equation pair to $S_-$. Otherwise, we randomly choose a different operation until the result of the operation is not the next equation in the trajectory, and then add the pair to $S_-$. This way, we acquire pairs whose first equations have lower values than the second ones'. We train 3 sets of value functions, with the different feature dimensions, 40D, 45D, and 50D.

After we learned a value function, we utilized $\omega^*$ as the ground truth parameter. The teacher and the learner represent the equations with the learned features. The learner always used features with 45D, and the teacher used 40D or 50D, corresponding to figure~\ref{sup:fig:equation-40} and~\ref{sup:fig:equation-50}. In all settings, $\beta$ is set to 5000. We tested the learned parameters with equations not included in the training set. Specifically, to simplify an equation, we applied all possible operations to it and obtained the outcome equation values. Then we used the greedy search to select candidates according to their values. The search ends when all the outcome equations have a lower value than the current equation. If the final equation generated by the learned value function matches with the simplification generated by our rules, we count this simplification as correct. In figure \ref{sup:fig:equation-40} and \ref{sup:fig:equation-50}, we provide the square loss and the accuracy for the simplification task.

\subsection{Online Inverse Reinforcement Learning}\label{sup:sec:exp-oirl}
\begin{figure}
    \begin{subfigure}{0.49\textwidth}
        \centering
        \includegraphics[width=\textwidth]{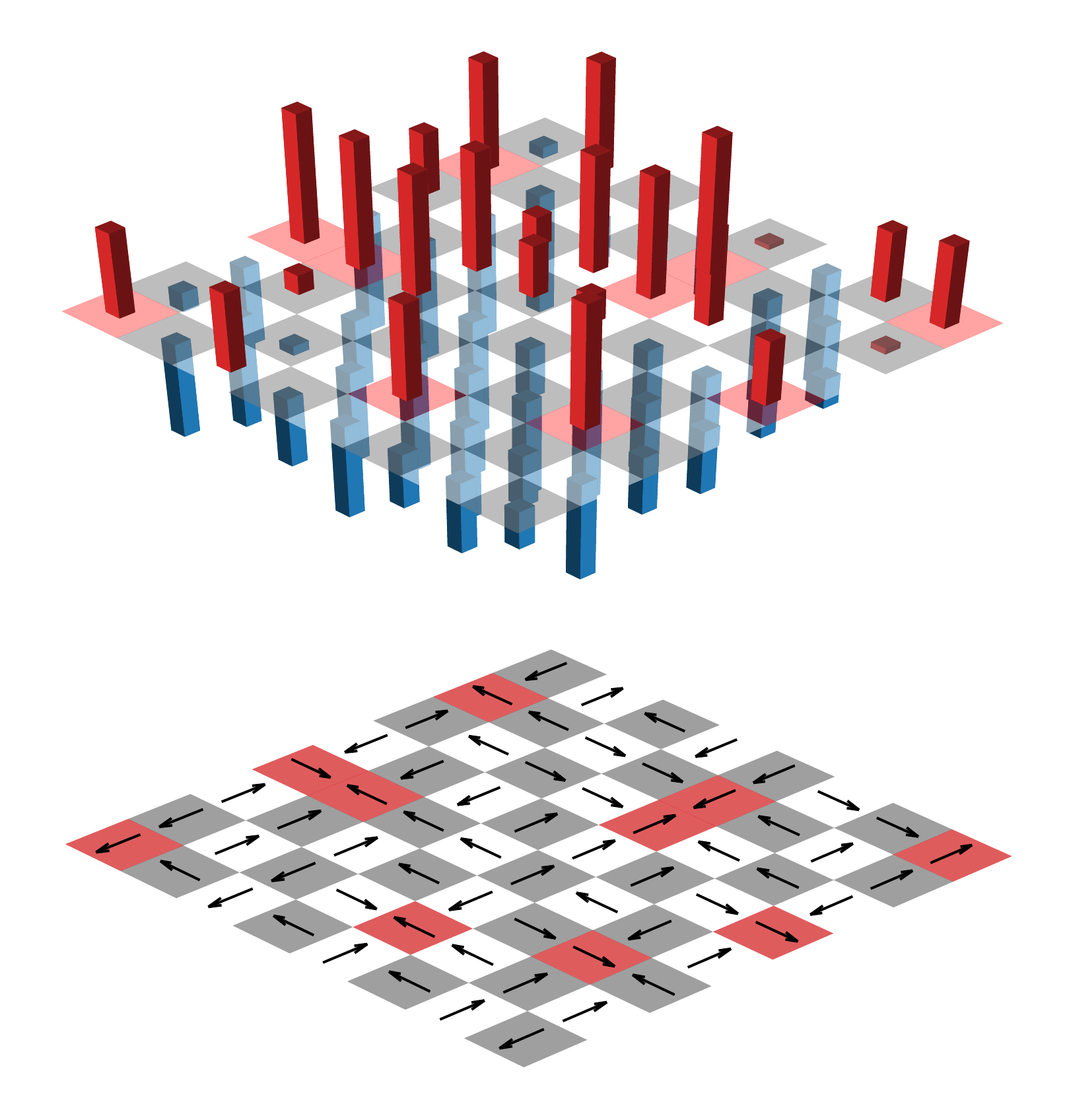}
    \end{subfigure}
    ~
    \begin{subfigure}{0.49\textwidth}
        \centering
        \includegraphics[width=\textwidth]{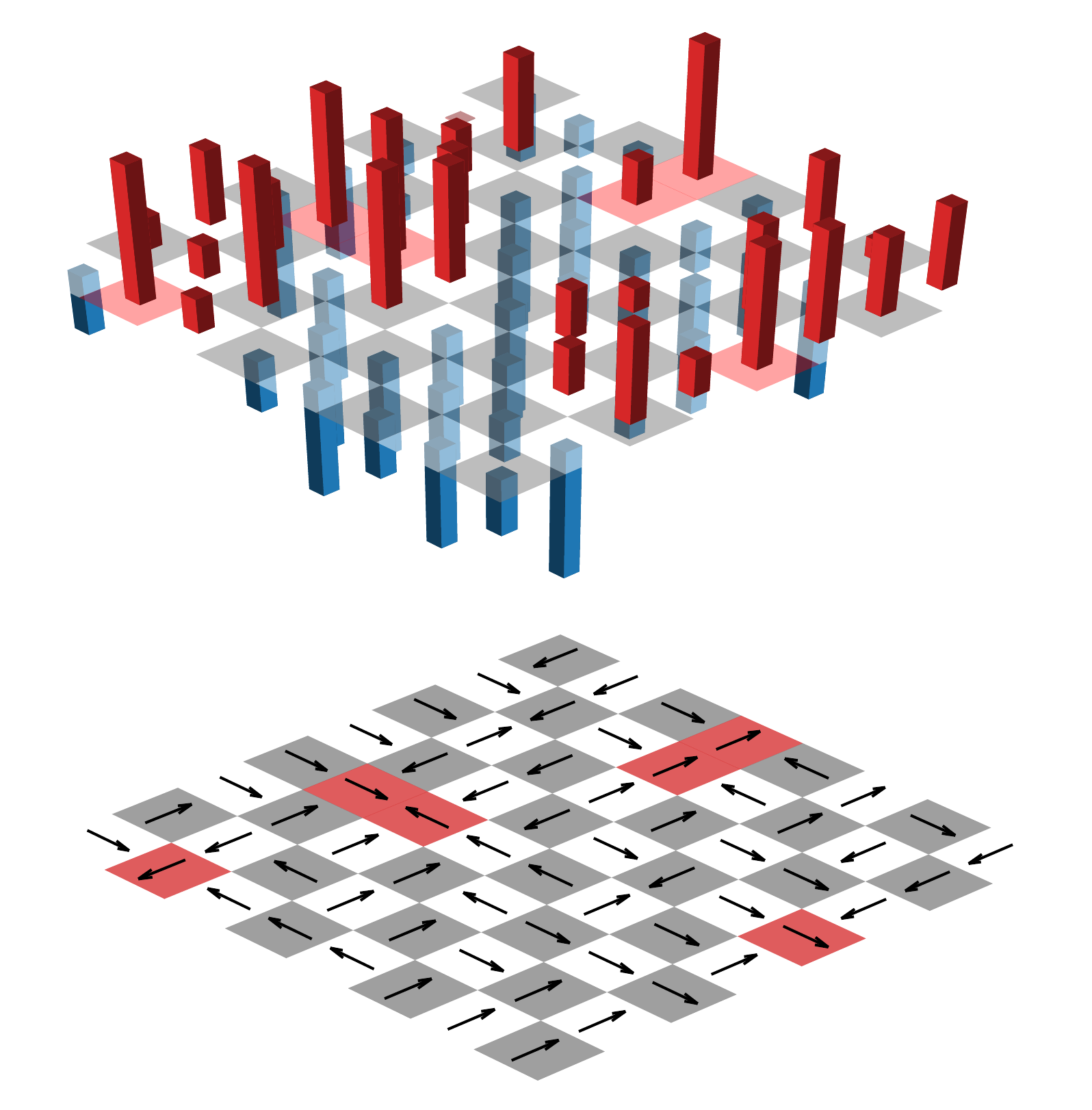}
    \end{subfigure}
    
    \caption{IRL map examples. Each map has $8\times 8$ grids. Every grid contains a reward. Maps in the first row plots the \textbf{ground truth} rewards in each grid. Red bars represent positive rewards and blue bars represent negative rewards. The learner tries to learn a policy to walk in the map and collect the most accumulative rewards. The arrows below indicate the most probable action taken by the learner after he learned the reward function. The red grids are targets of all their neighbors.}
    \label{sup:fig:irl-reward}
\end{figure}
\begin{figure}
    \begin{subfigure}{0.49\textwidth}
        \centering
        \includegraphics[width=\textwidth]{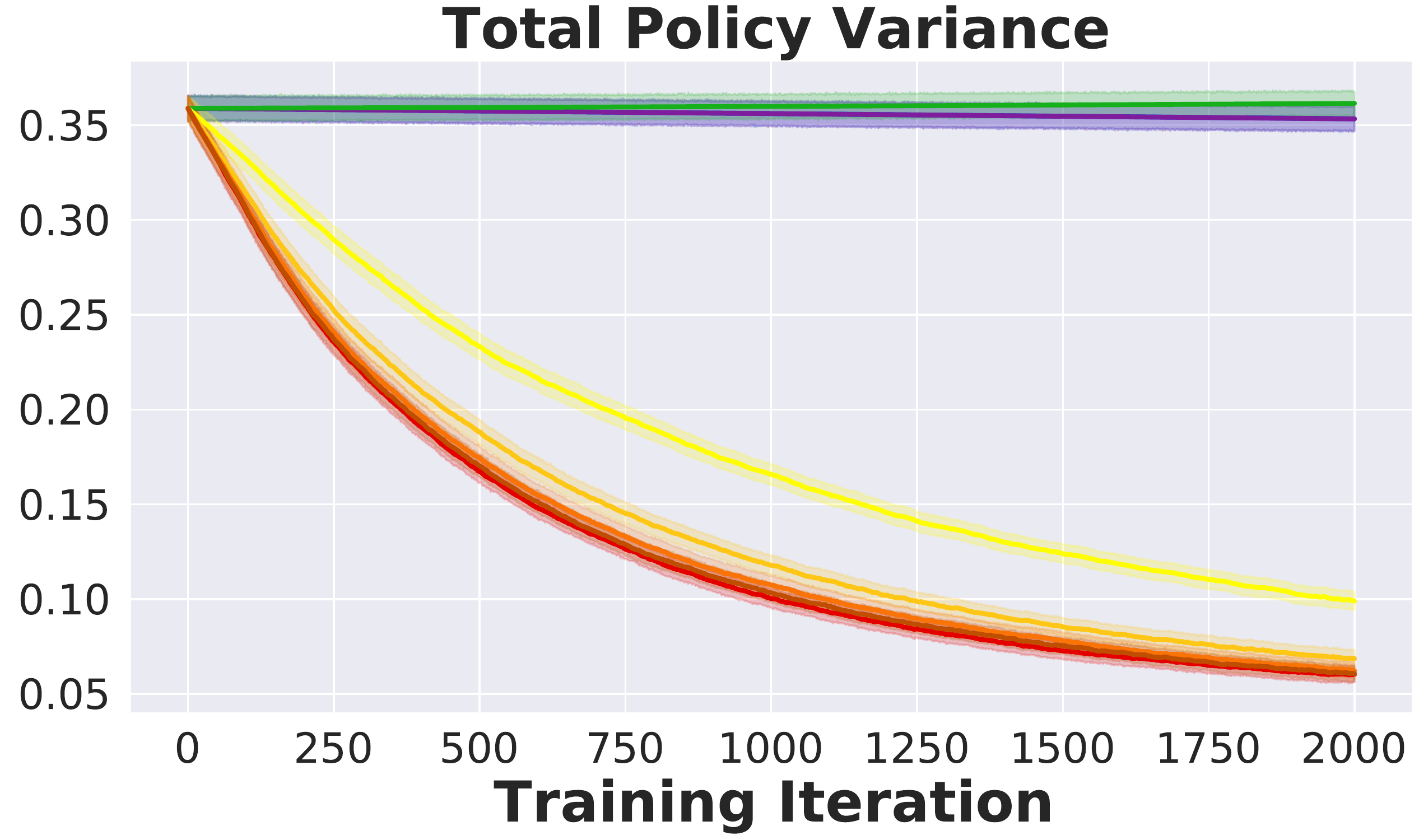}
    \end{subfigure}%
    ~
    \begin{subfigure}{0.49\textwidth}
        \centering
        \includegraphics[width=\textwidth]{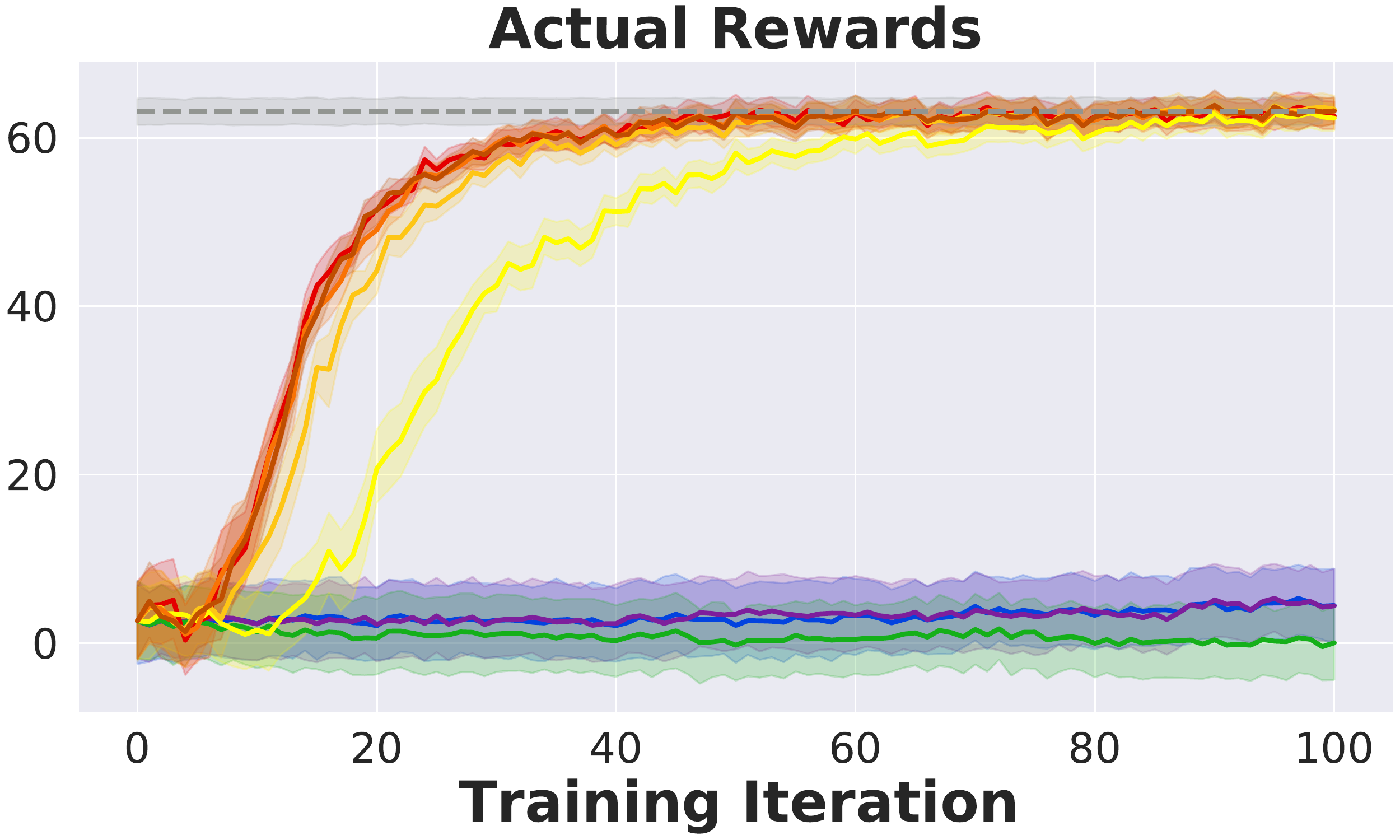}
    \end{subfigure}
    
    \begin{subfigure}{\textwidth}
        \centering
        \includegraphics[width=\textwidth]{Figures/imitate_legend.pdf}
    \end{subfigure}
    \caption{Total variance between the learner's policy and the teacher's policy and the actual gain of the learner during the learning process. The gray horizontal dash line represents teacher's expected accumulative reward.}
    \label{sup:fig:irl}
\end{figure}
\begin{figure}[ht]
    \begin{subfigure}{0.32\textwidth}
        \centering
        \includegraphics[width=\textwidth]{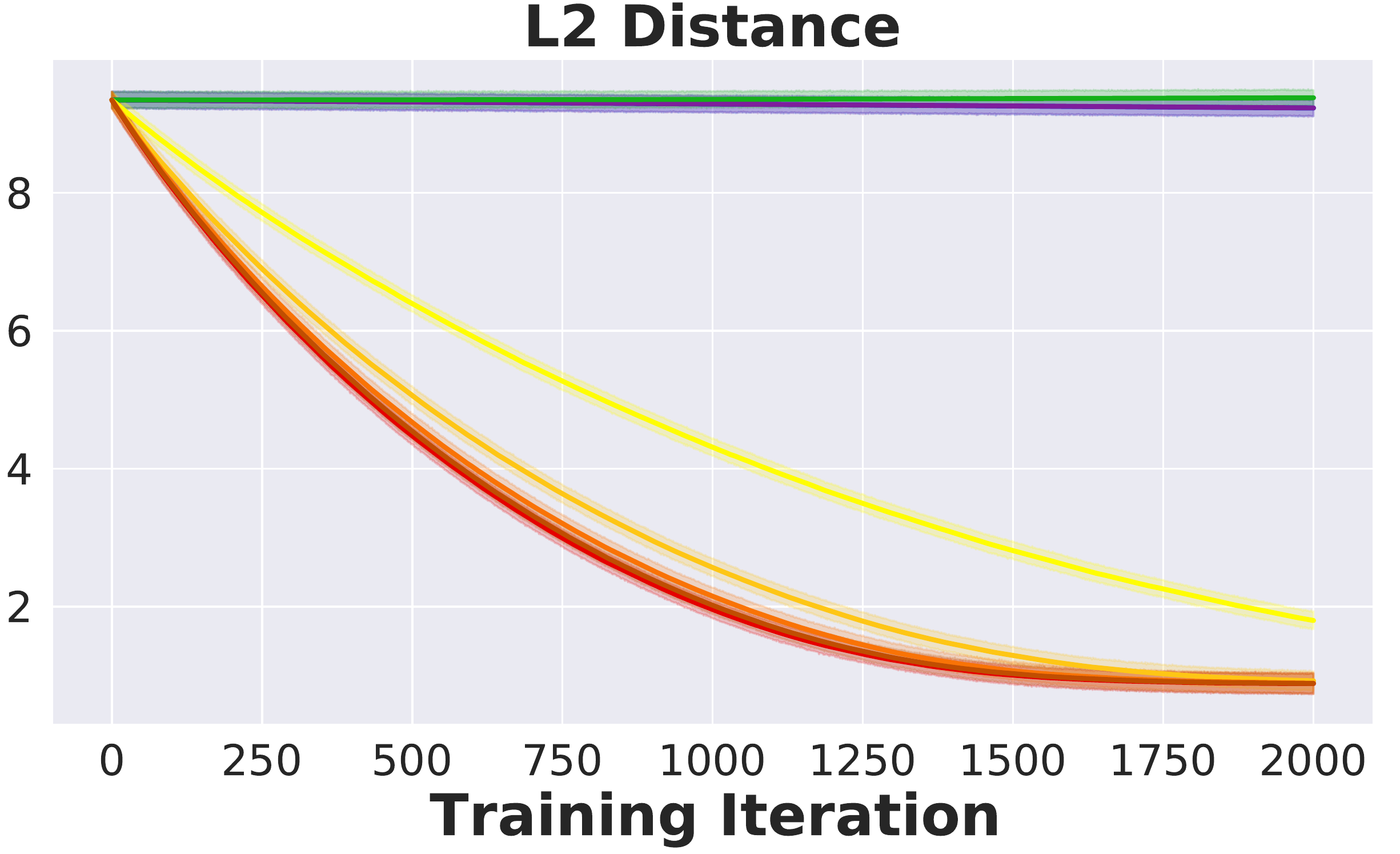}
    \end{subfigure}%
    ~
    \begin{subfigure}{0.32\textwidth}
        \centering
        \includegraphics[width=\textwidth]{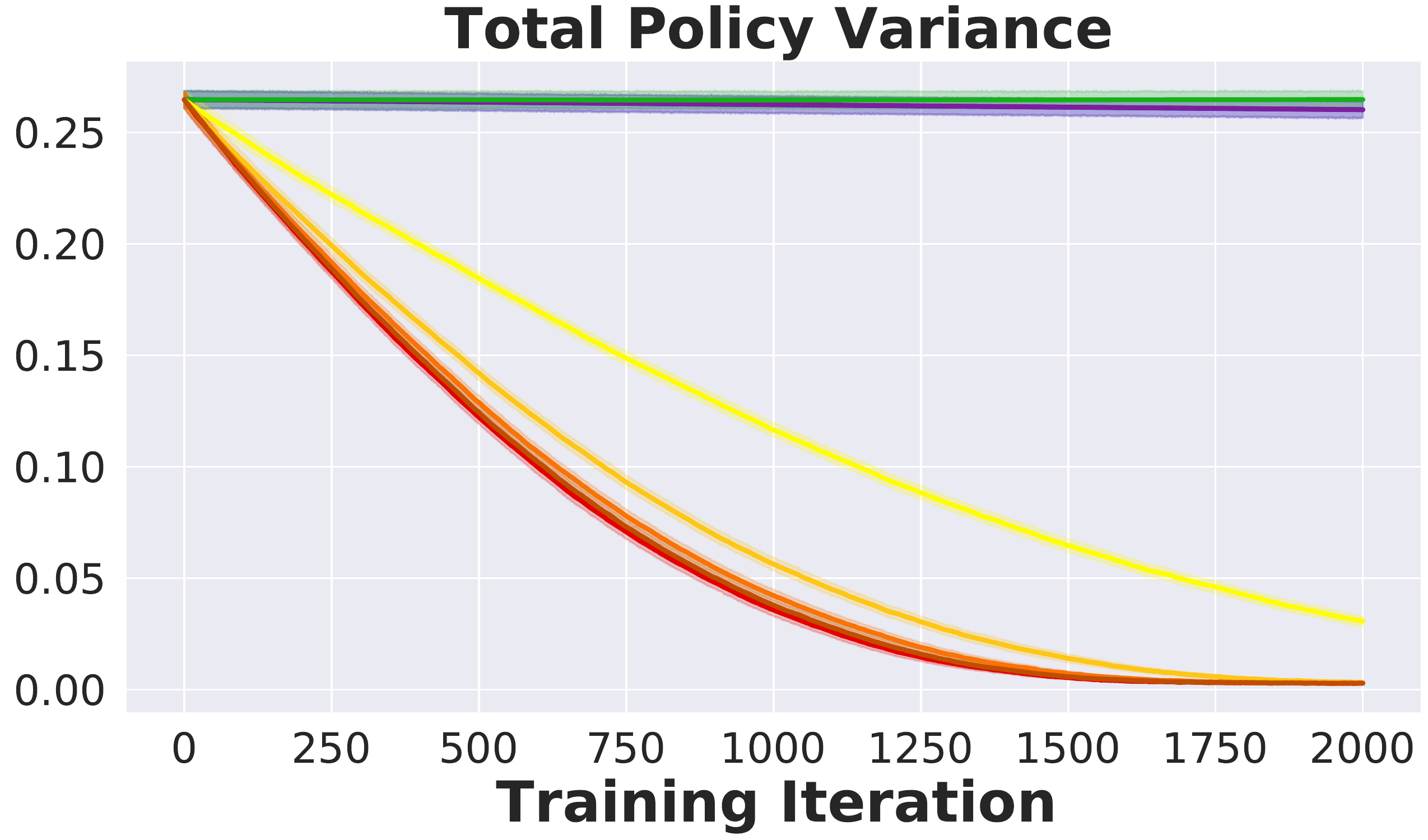}
    \end{subfigure}
    ~
    \begin{subfigure}{0.32\textwidth}
        \centering
        \includegraphics[width=\textwidth]{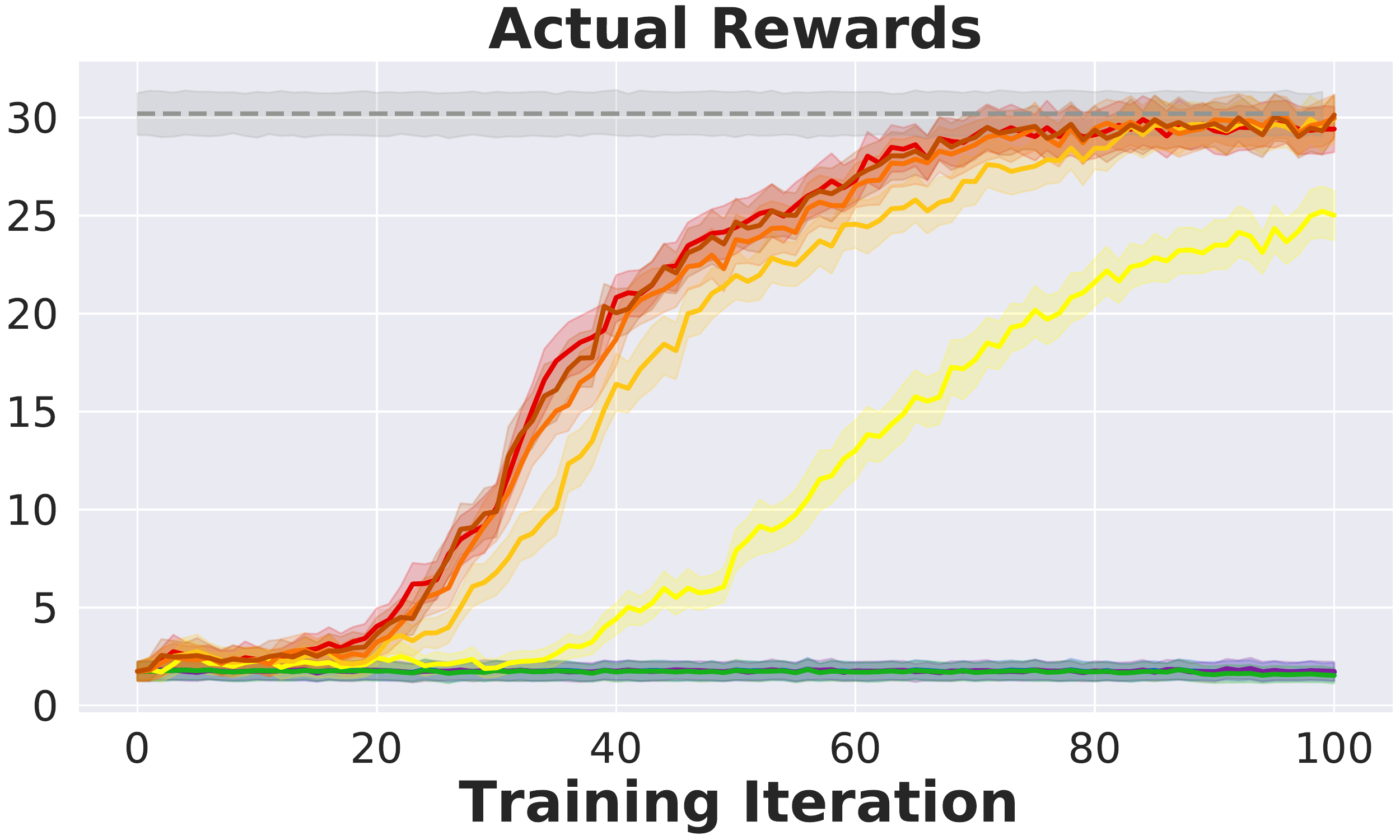}
    \end{subfigure}%
    
    \begin{subfigure}{\textwidth}
        \centering
        \includegraphics[width=\textwidth]{Figures/imitate_legend.pdf}
    \end{subfigure}
    \caption{Learning results of sparse reward maps. All grids in the 8$\times$8 map have 0 reward but 3 grids with reward 1. In every experiment, we randomly selected the 3 grids. Curves drawn with results using 20 different random seeds. The gray horizontal dash line represents teacher's expected accumulative reward.}
    \label{sup:fig:irl-easy}
\end{figure}
In this experiment, we want to learn a reward function $r(s, \omega^*)$. We can define a Markov Decision Process $\langle S, A, r, P, \gamma\rangle$, where $S$ is the state space, $A$ is the action space, $r:S\rightarrow R$ is a reward function mapping from state to a real number as the reward. $P_{ss'}^a$ is the transition model that state $s$ becomes $s'$ after the agent conducts action $a$. $\gamma$ is a discount factor that ensures the convergence of the
MDP over an infinite horizon. Given a reward function, using Bellman equation we have
\begin{align}
    V^*(s)&= \max_{a\in A}\sum_{s'|s, a}P_{ss'}^a\big[r(s') + \gamma V^*(s')\big]\\
    Q^*(s, a) &= \max_{a\in A}\sum_{s'|s, a}P_{ss'}^a\big[r(s') + \gamma \max_{a'\in A}Q^*(s', a')\big]
\end{align}
Suppose an agent behaves by following Boltzman rationality:
\begin{align}
    \pi(a^t|s^t; \omega) = \frac{\exp{(\alpha Q^*(s^t, a^t; \omega))}}{\sum_{a'\in A}\exp{(\alpha Q^*(s^t, a'; \omega))}}
\end{align}
Take log-likelihood of this function we can have an objective function that the learner can optimize to learn $\omega^*$.
\begin{align}
    l(s^t, a^t; \omega^{t-1}) &= \alpha Q^*(s^t, a^t; \omega^{t-1}) - \log\sum_{a'\in A}\alpha Q^*(s^t, a'; \omega^{t-1})\\
    \omega^{t} &= \omega^{t-1}+\eta_t\frac{\partial l(s^t, a^t; \omega^{t-1})}{\partial \omega^{t-1}}\label{sup:eq:irl_loss}
\end{align}
Then, the online IRL process can be accomodated by our learning framework. One issue is that the $\max$ operation in $Q$ is not deferentiable. Thus, we approximated $\max$ with soft-max, namely: $\max(a_0, ..., a_n) \approx \frac{\log(\sum_{i=0}^n\exp{ka_i})}{n}$, with $k$ controlling the level of approximation and leveraged the online Bellman gradient iteration~\citep{li2017online} to calculate the gradient for each step.
\begin{align}
    \frac{\partial V_{g, k}(s; \omega^t)}{\partial \omega^t} &= \sum_{a\in A}\frac{\exp{(kQ_{g, k}(s, a; \omega^t))}}{\sum_{a'\in A}\exp{(kQ_{g, k}(s, a'; \omega^t))}}\frac{\partial Q_{g, k}(s, a; \omega^t)}{\partial \omega^t}\\
    \frac{\partial Q_{g, k}(s, a; \omega^t)}{\partial \omega^t} &= \sum_{s'|s, a}P_{ss'}^a\big(\frac{\partial r(s'; \omega^t)}{\partial \omega^t}+\gamma\frac{\partial V_{g, k}(s'; \omega^t)}{\partial \omega^t}\big)
\end{align}
In every round, we randomly sample 20 $(s, a)$ pairs from $|S|\times |A|$ state-action pairs as our minibatch. Then the teacher will conduct Bellman gradient iteration. The learner will return his reward estimation for each grid to the teacher.

We used an $8\times 8$ grid map as the environment, and the action space $A$ includes four actions up, down, left, right. See figure \ref{sup:fig:irl-reward} for map examples. $80\%$ of the time, the agent goes to its target, $18\%$ of the time ends up in another random neighbor grid and $2\%$ of the time dies abruptly (game ends). We set $\gamma = 0.5$. The reward in each grid is randomly sampled from a uniform distribution, $U[-2, 2]$. If we encode each grid with a one-hot vector, then the reward parameter is a 64D vector with the $i$-th entry corresponding to the reward of the $i$-th grid. The teacher uses a shuffled map encoding as the student's. For instance, if the first grid is $[1, 0, ..., 0]$ to the learner, then it becomes $[0, ..., 0, 1, 0, ...]$ to the teacher. See figure \ref{sup:fig:irl} for the actual accumulative reward acquired by the agent during learning. 

In addition to the environment with random dense rewards, we tested the teacher-aware learner in a sparse reward environment. Each time, we only pick 3 grids randomly to assign non-zero reward. Our algorithm still shows robust performance. Results in figure \ref{sup:fig:irl-easy}.

\subsection{Adversarial Teacher}\label{sup:sec:exp-adv}
\begin{table}[ht]
    \centering
    \caption{Selection of $\beta$s in the adversarial teacher experiments. For cooperative teachers, the absolute values of the $\beta$s are the same, only the signs are flipped.}
    \begin{tabular}{|c|c|}
    \hline
         Experiment & Value of $\beta$ \\\hline
         Linear Classifiers on Synthesized Data&  -60000\\\hline
         Linear Regression on Synthesized Data& -5000 \\\hline
         Linear Classifiers on MNIST Dataset& -30000 \\\hline
         Linear Classifiers on CIFAR Dataset& $-50000(1-5e^{-6})^t$ \\\hline
         Linear Classifiers on Tiny ImageNet Dataset& $-1000$ \\\hline
         Linear Regression for Equation Simplification& -5000 \\\hline
         Online Inverse Reinforcement Learning (Random Rewards)& -25000 \\\hline
         Online Inverse Reinforcement Learning (Sparse Rewards)& -30000 \\\hline
    \end{tabular}
    \label{sup:tab:adv-beta}
\end{table}
\begin{figure*}
    \centering
    \begin{subfigure}{\textwidth}
    \hspace{0.17\textwidth}
    \begin{subfigure}{0.32\textwidth}
        \centering
       \includegraphics[width=\textwidth]{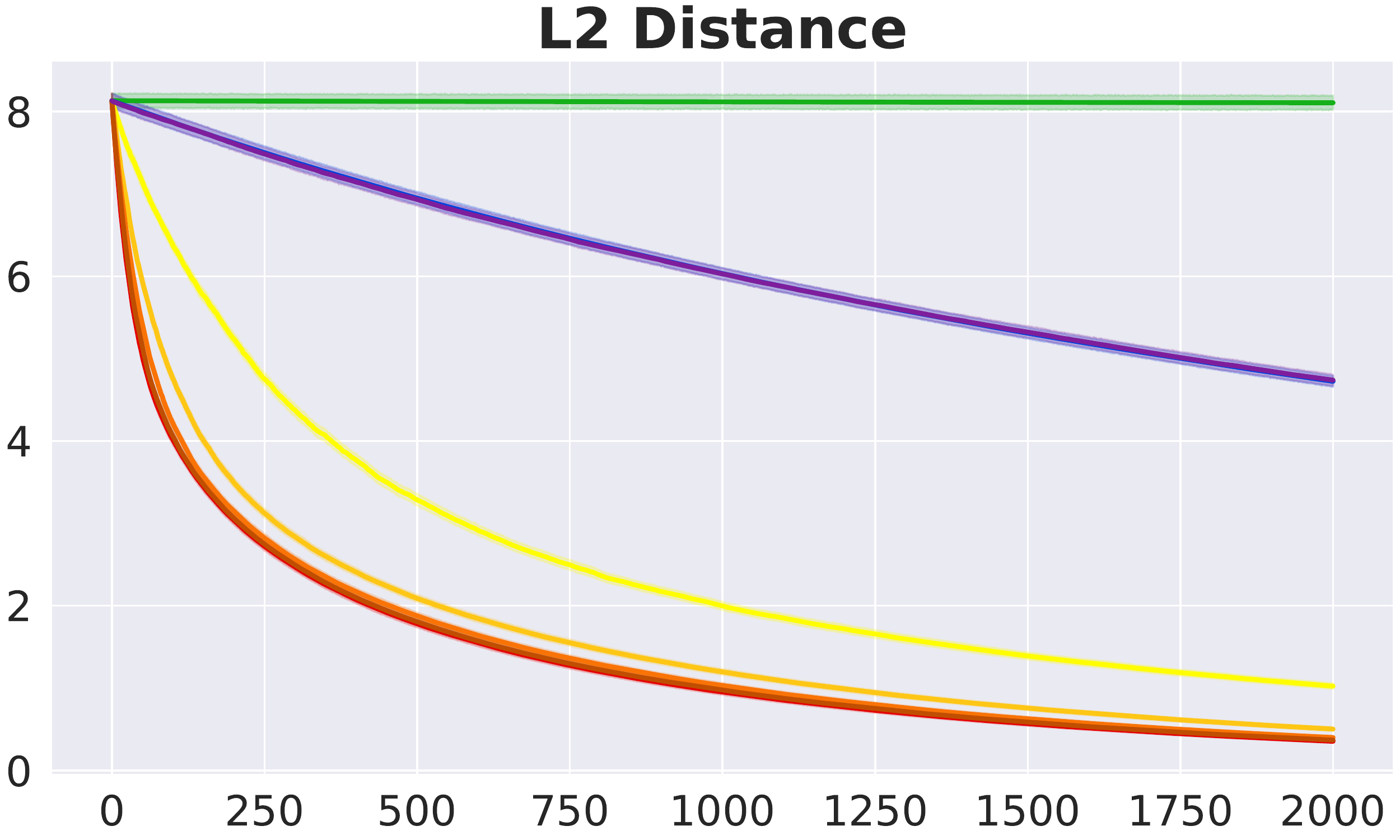}
    \end{subfigure}%
    ~
    \begin{subfigure}{0.32\textwidth}
        \centering
        \includegraphics[width=\textwidth]{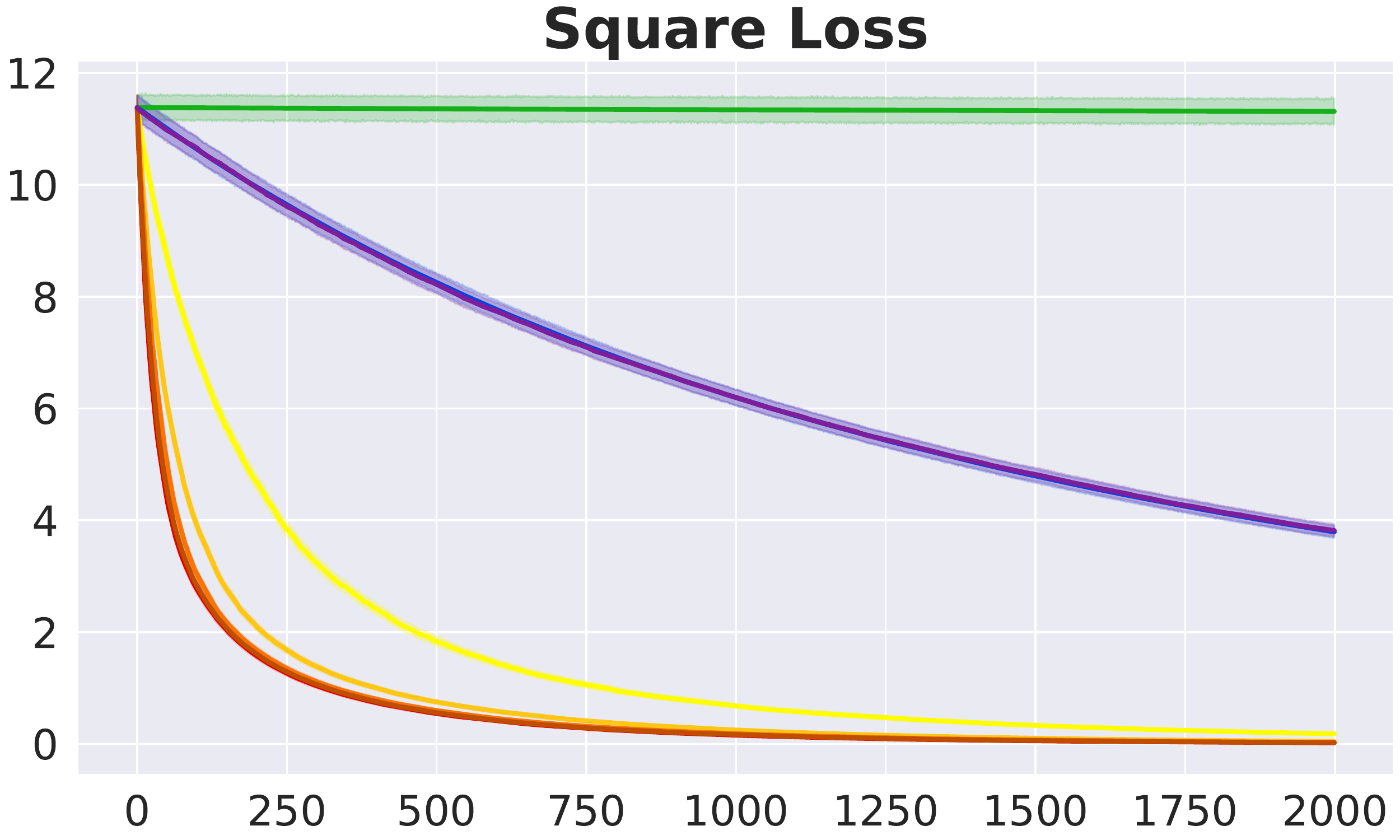}
    \end{subfigure}%
    \vspace{-5pt}
    \label{fig:regression-adv}
    \caption{Linear Regression}
    \end{subfigure}
    
    \begin{subfigure}{\textwidth}
    \begin{subfigure}{0.32\textwidth}
        \centering
        \includegraphics[width=\textwidth]{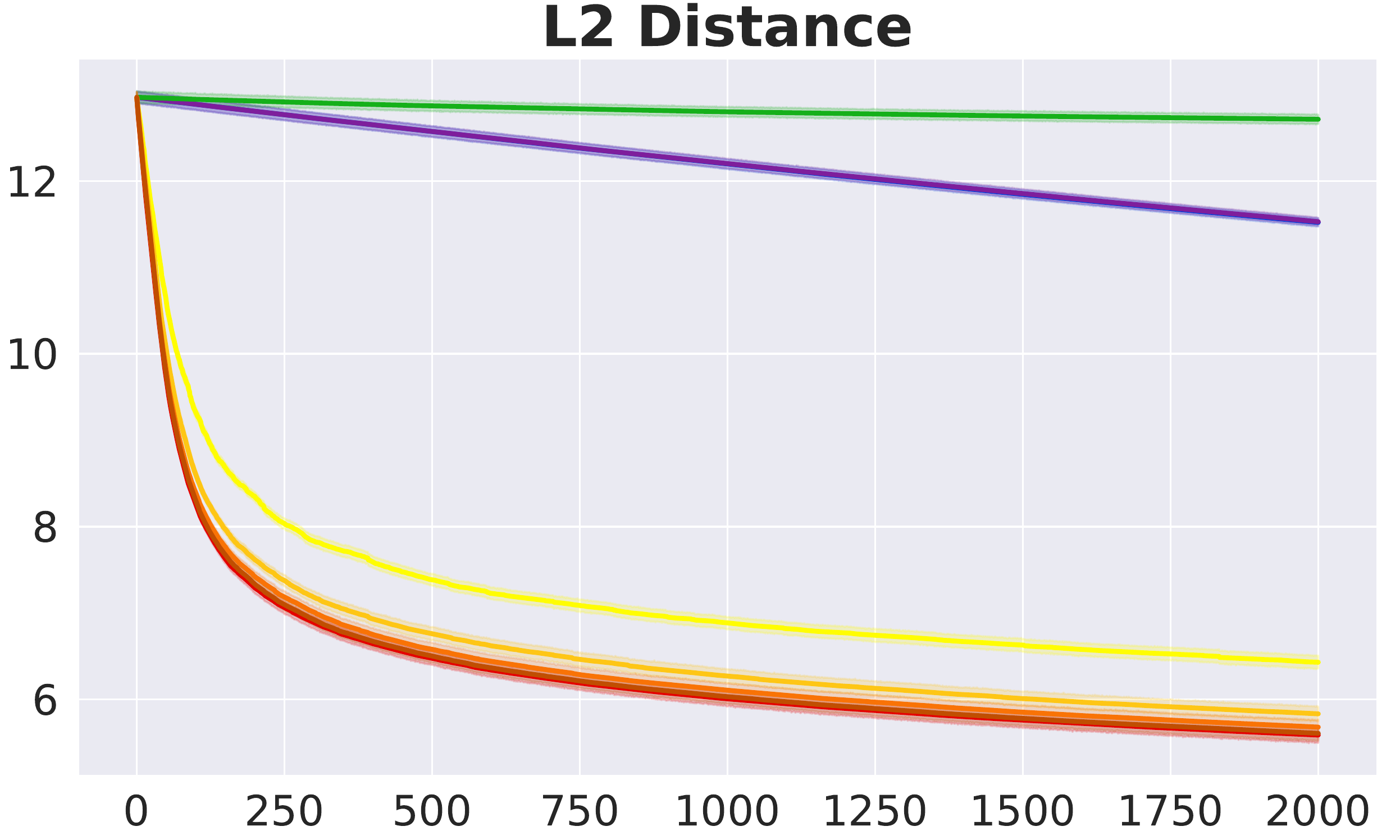}
    \end{subfigure}
    ~
    \begin{subfigure}{0.32\textwidth}
        \centering
        \includegraphics[width=\textwidth]{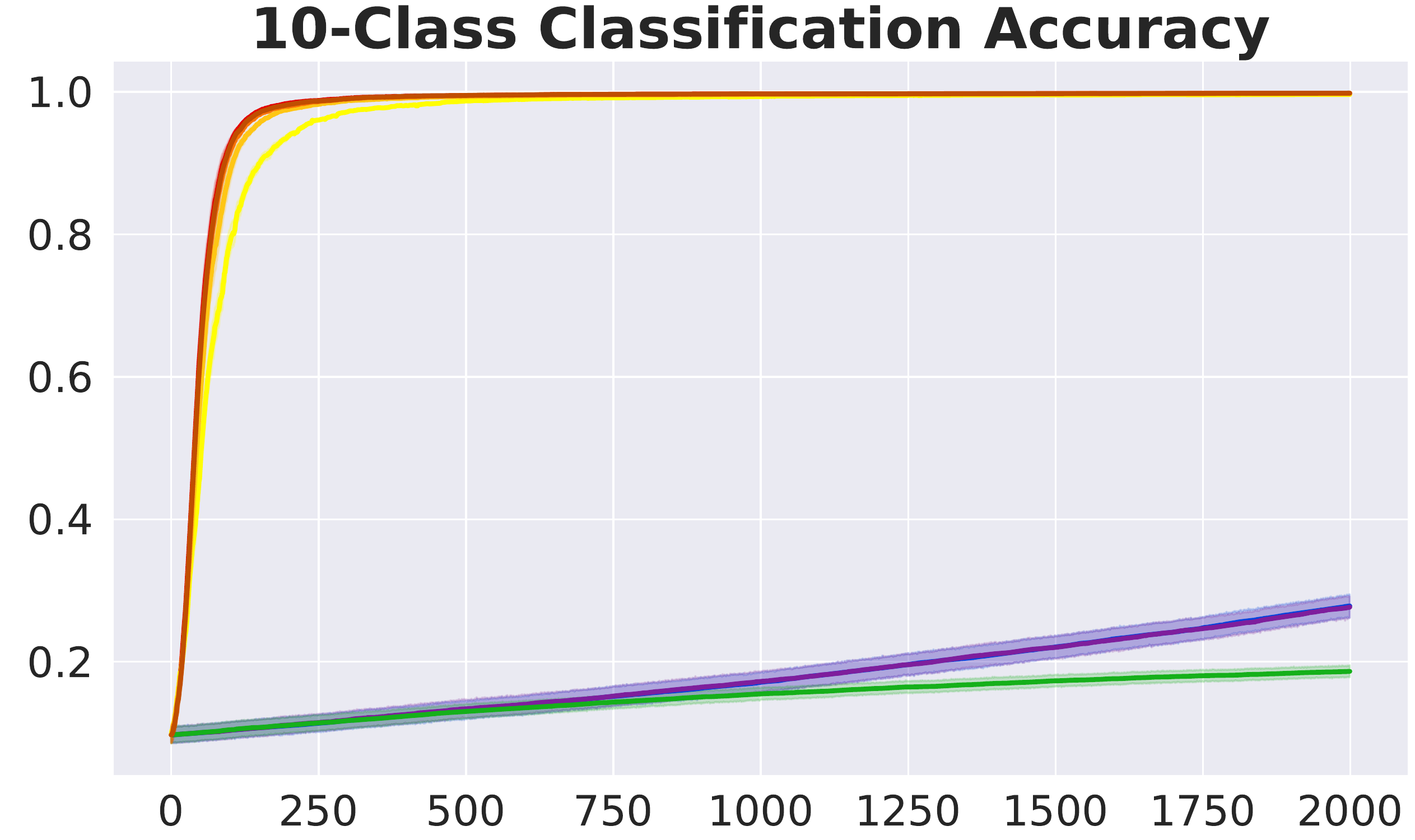}
    \end{subfigure}
    ~
    \begin{subfigure}{0.32\textwidth}
        \centering
       \includegraphics[width=\textwidth]{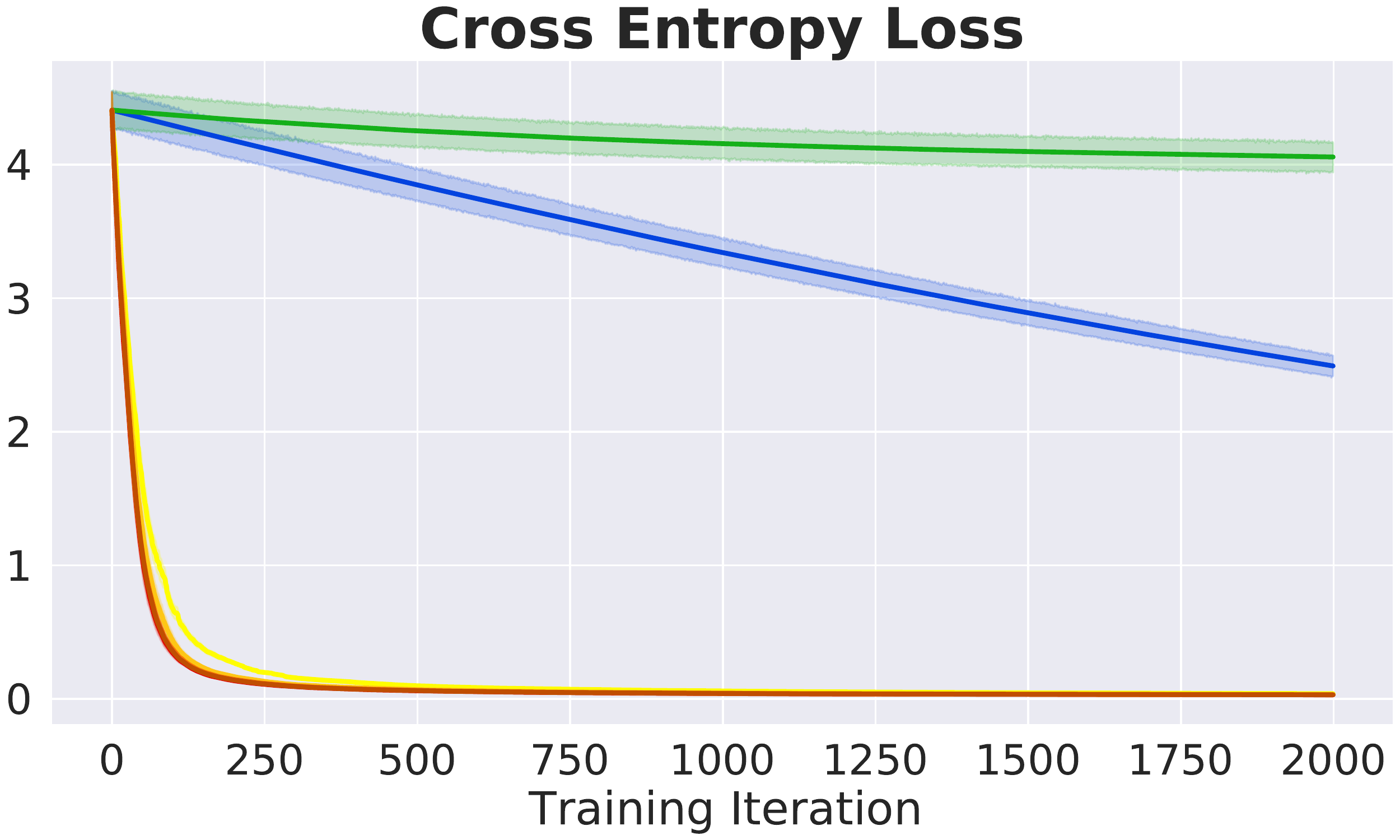}
    \end{subfigure}
    \caption{Gaussian data 10-class classification.}
    \label{fig:class10-adv}
    \end{subfigure}
    
    \begin{subfigure}{\textwidth}
    \begin{subfigure}{0.32\textwidth}
        \centering
        \includegraphics[width=\textwidth]{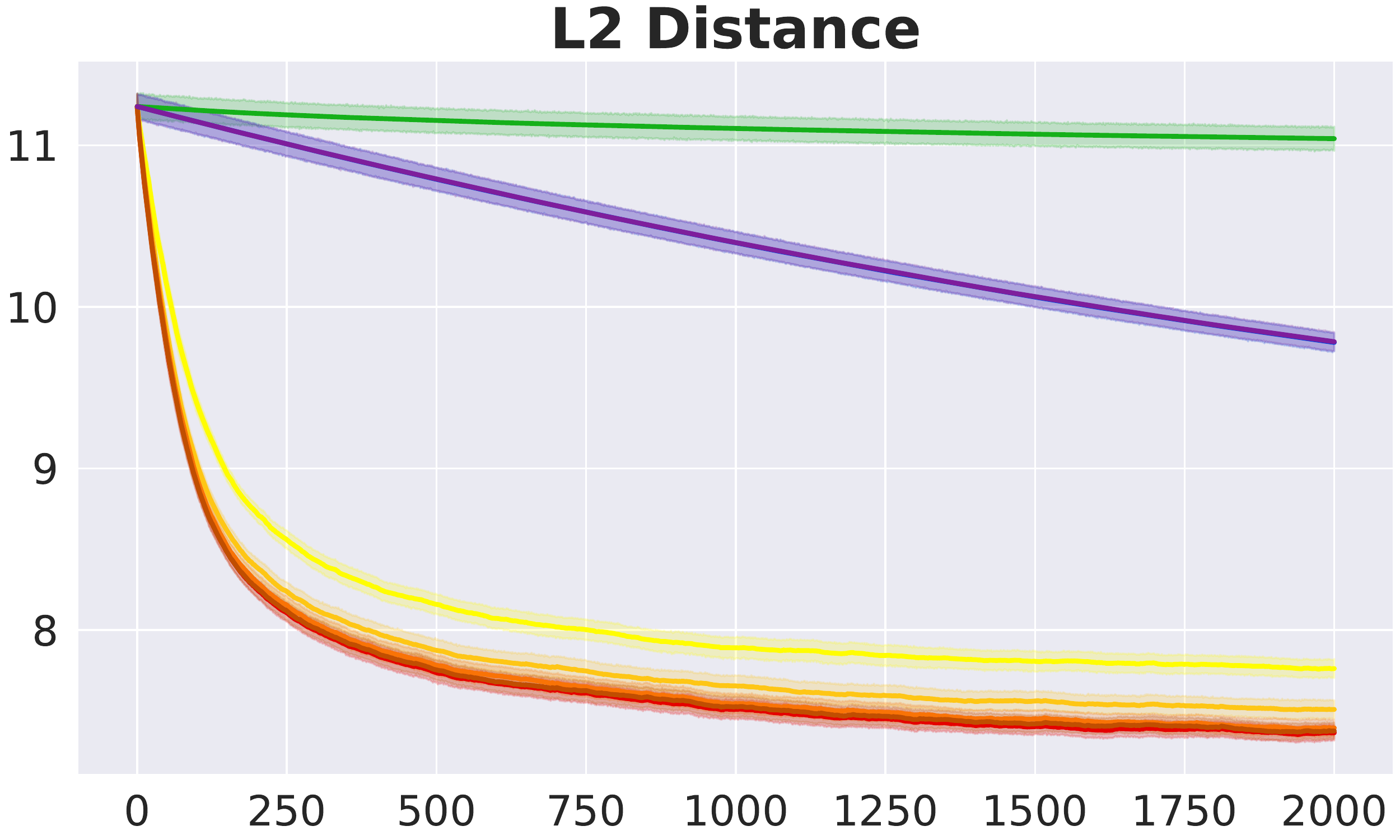}
    \end{subfigure}%
    ~
    \begin{subfigure}{0.32\textwidth}
        \centering
        \includegraphics[width=\textwidth]{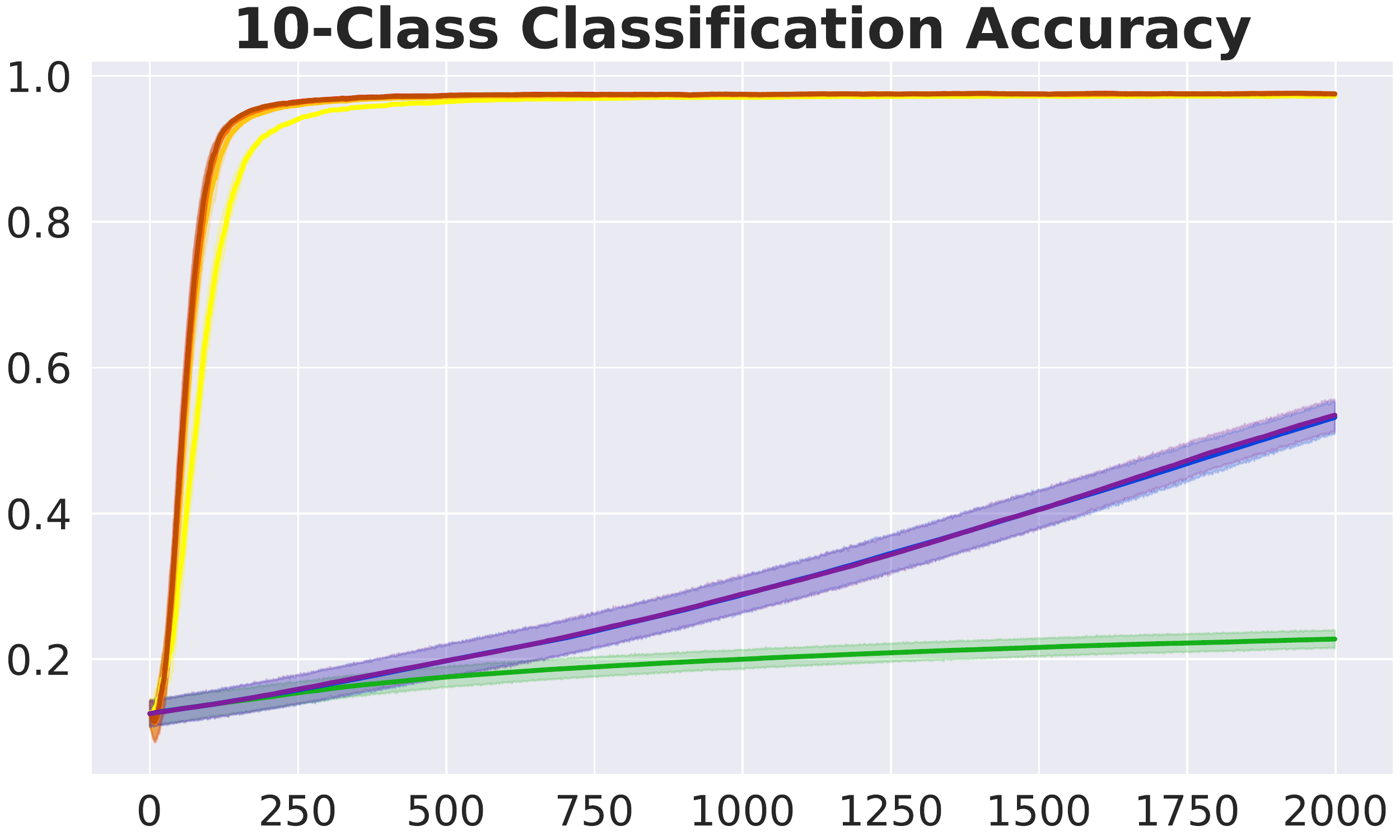}
    \end{subfigure}%
    ~
    \begin{subfigure}{0.32\textwidth}
        \centering
        \includegraphics[width=\textwidth]{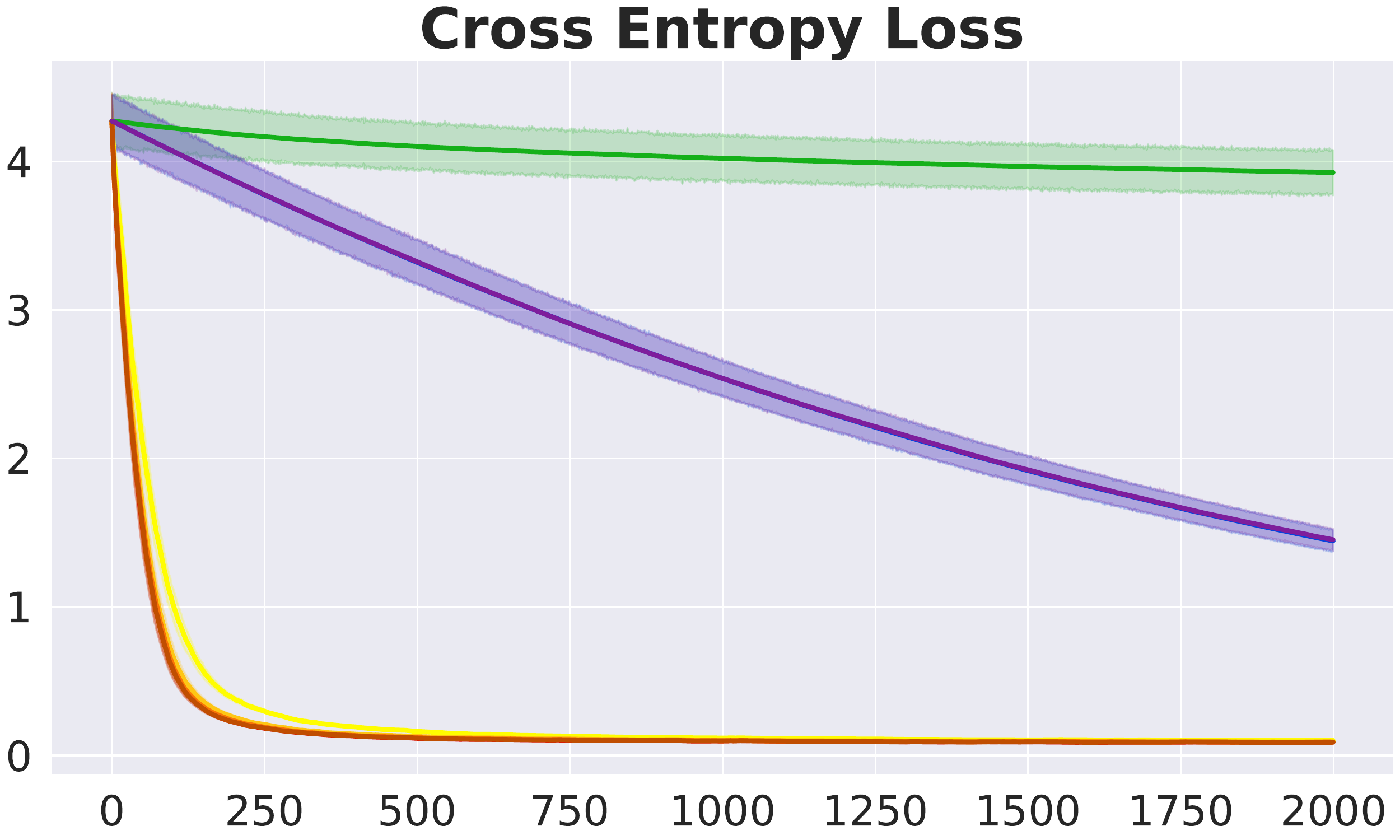}
    \end{subfigure}%
    \caption{MNIST 10-class,  20D teacher features}
    \label{fig:MNIST-adv20}
    \end{subfigure}
    
    \begin{subfigure}{\textwidth}
    \begin{subfigure}{0.32\textwidth}
        \centering
        \includegraphics[width=\textwidth]{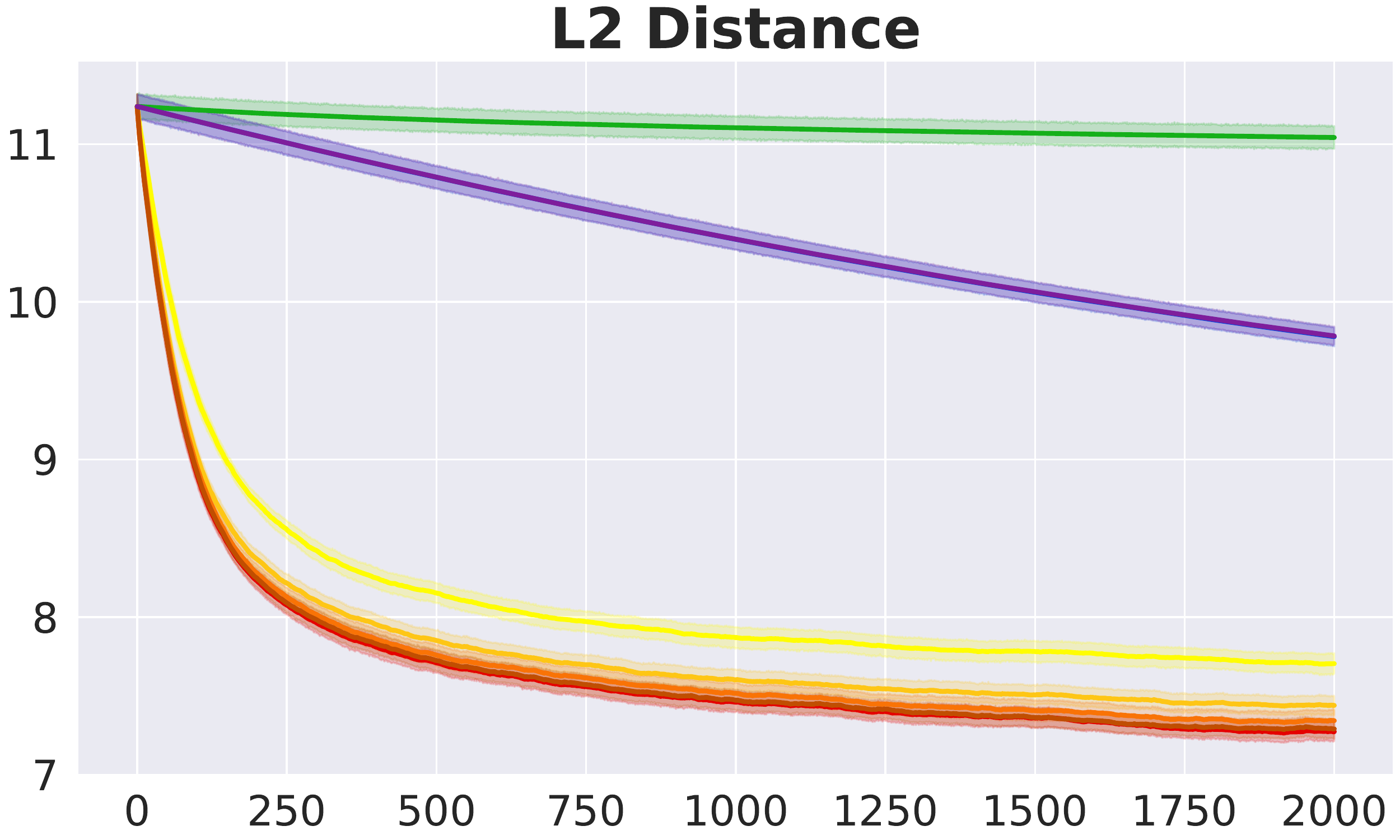}
    \end{subfigure}%
    ~
    \begin{subfigure}{0.32\textwidth}
        \centering
        \includegraphics[width=\textwidth]{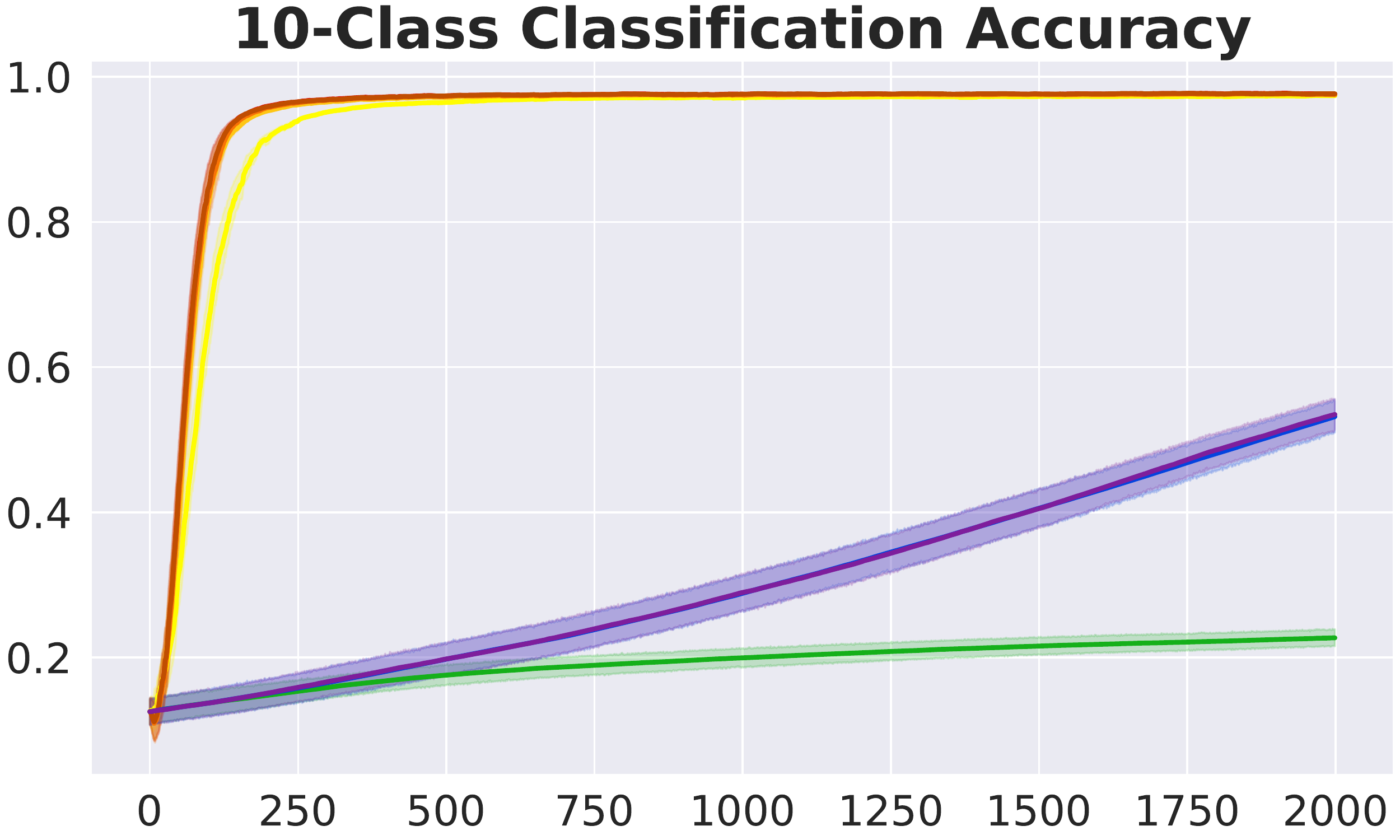}
    \end{subfigure}%
    ~
    \begin{subfigure}{0.32\textwidth}
        \centering
        \includegraphics[width=\textwidth]{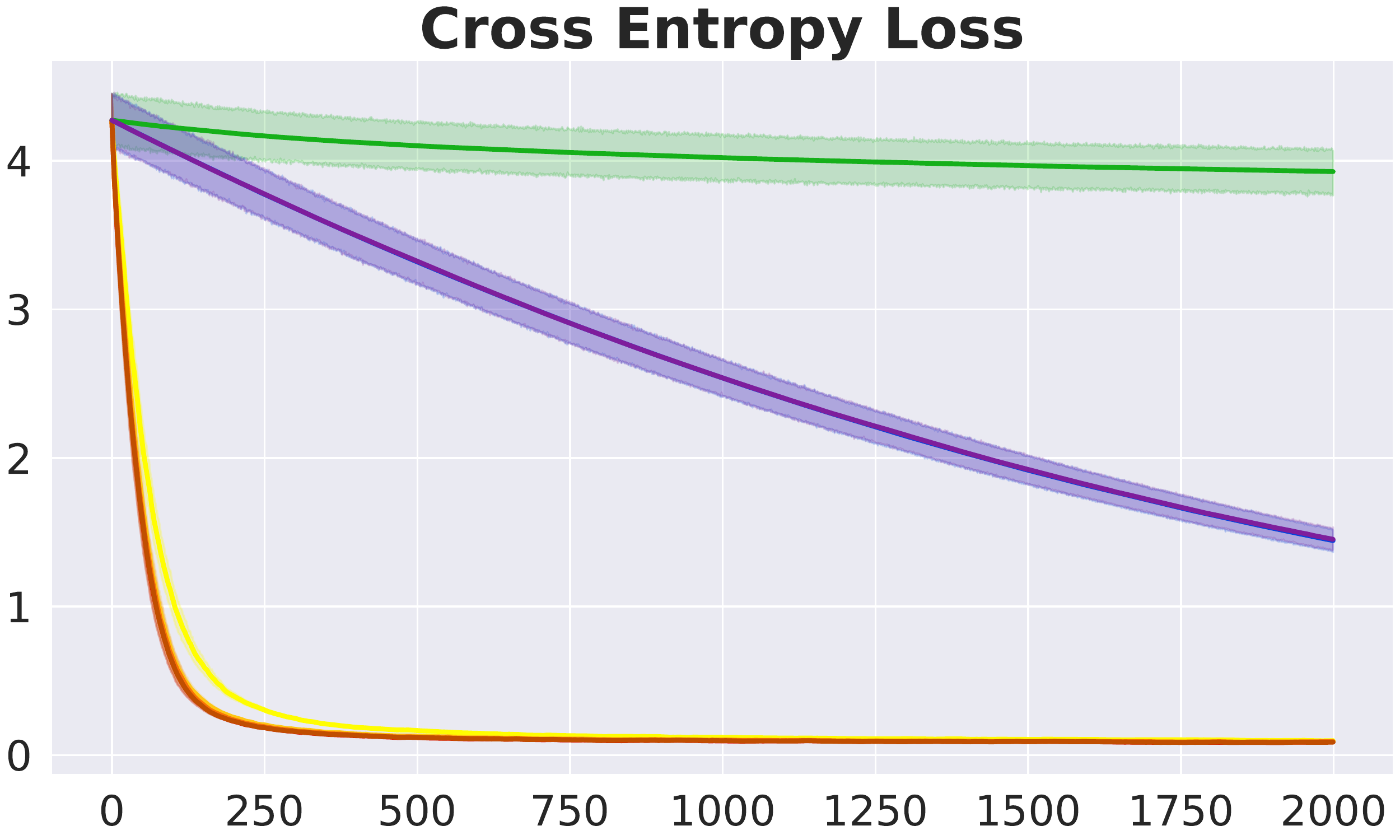}
    \end{subfigure}%
    \caption{MNIST 10-class, 30D teacher features}
    \label{fig:MNIST-adv30}
    \end{subfigure}
    
    \begin{subfigure}{\textwidth}
    \begin{subfigure}{0.32\textwidth}
        \centering
       \includegraphics[width=\textwidth]{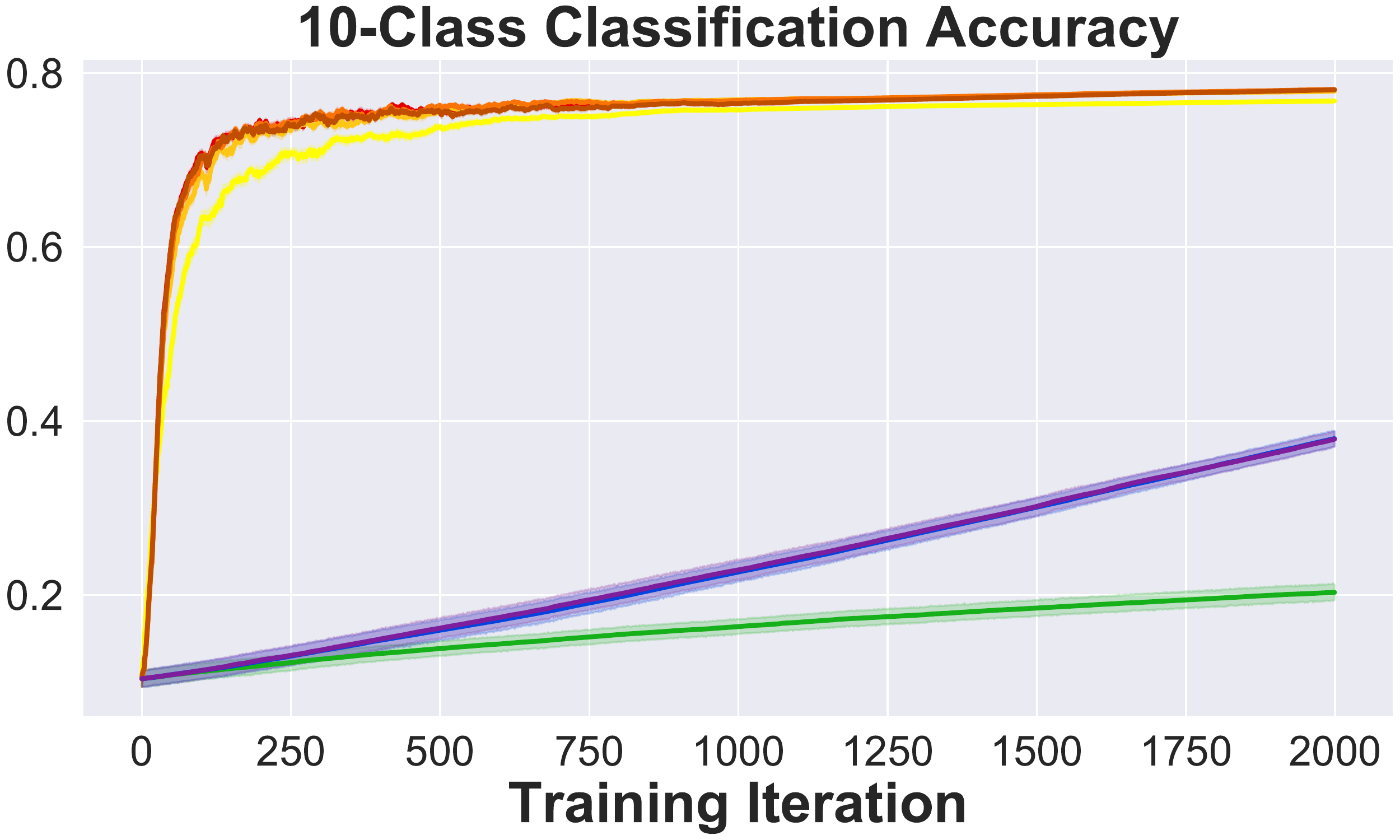}
    \end{subfigure}
    ~
    \begin{subfigure}{0.32\textwidth}
        \centering
        \includegraphics[width=\textwidth]{Figures/CIFAR-10/cifar10_adv9_accuracy.pdf}
    \end{subfigure}
    ~
    \begin{subfigure}{0.32\textwidth}
        \centering
        \includegraphics[width=\textwidth]{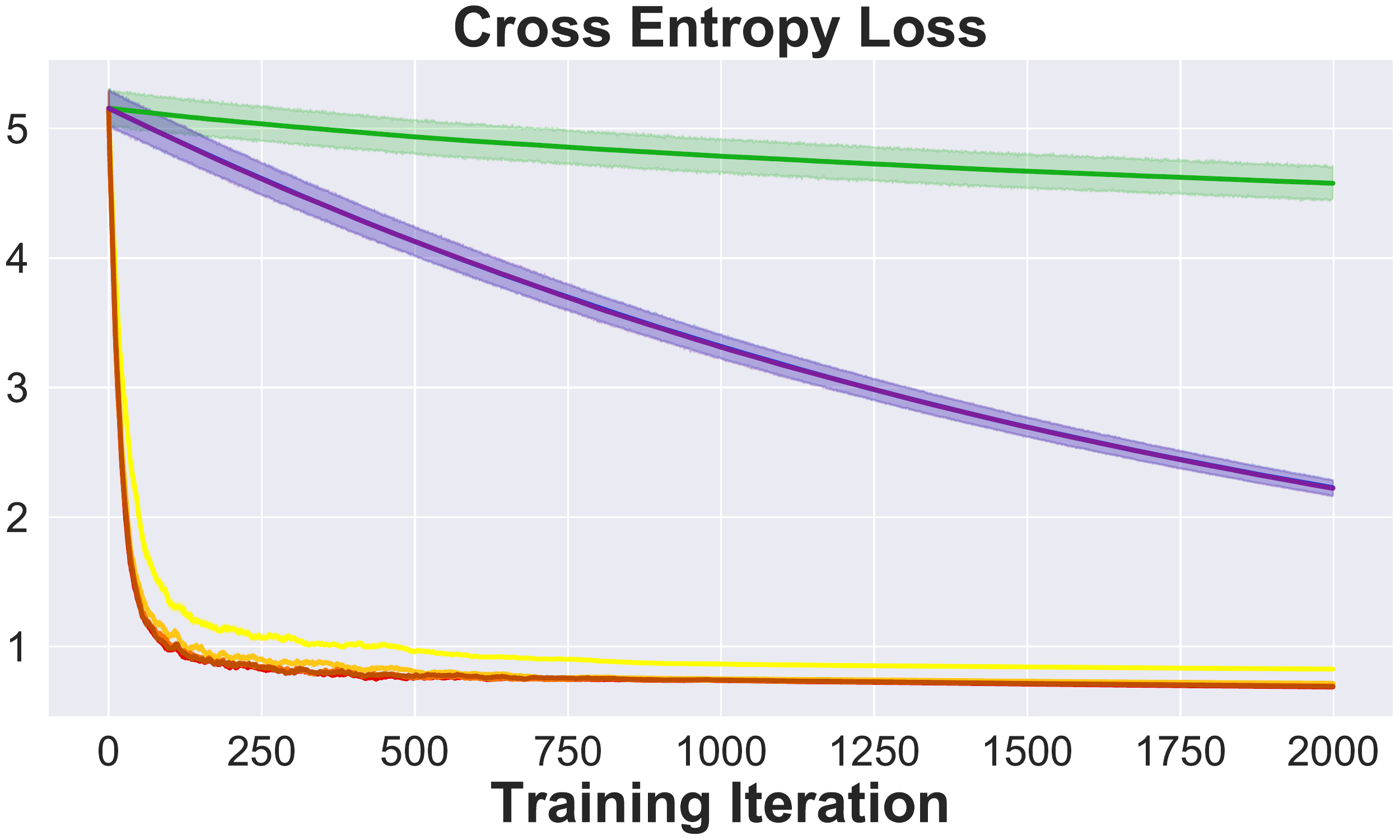}
    \end{subfigure}
    \caption{CIFAR-10, teacher feature from CNN-9.}
    \label{fig:CIFAR10-adv9}
    \end{subfigure}
    
    \begin{subfigure}{\textwidth}
    \begin{subfigure}{0.32\textwidth}
        \centering
       \includegraphics[width=\textwidth]{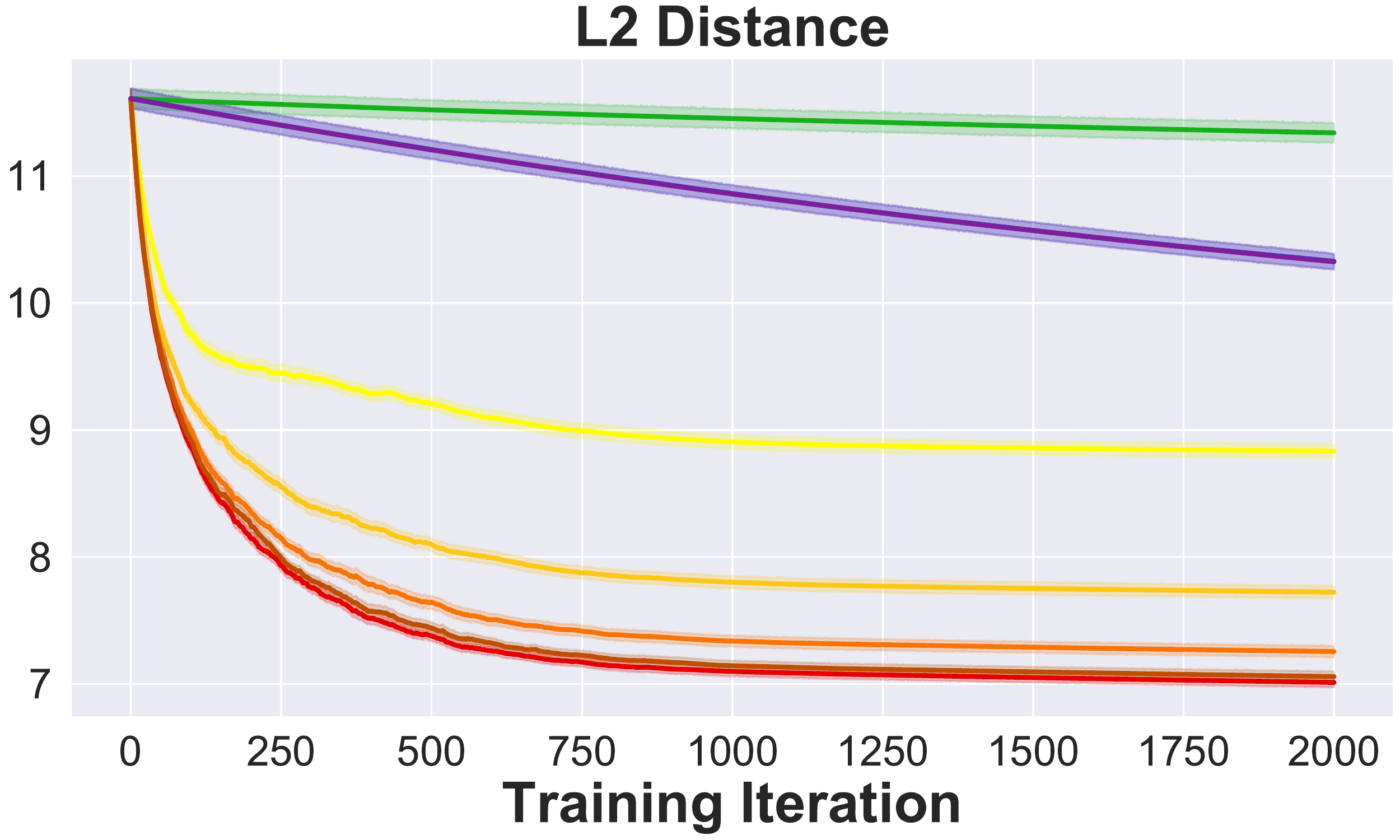}
    \end{subfigure}
    ~
    \begin{subfigure}{0.32\textwidth}
        \centering
        \includegraphics[width=\textwidth]{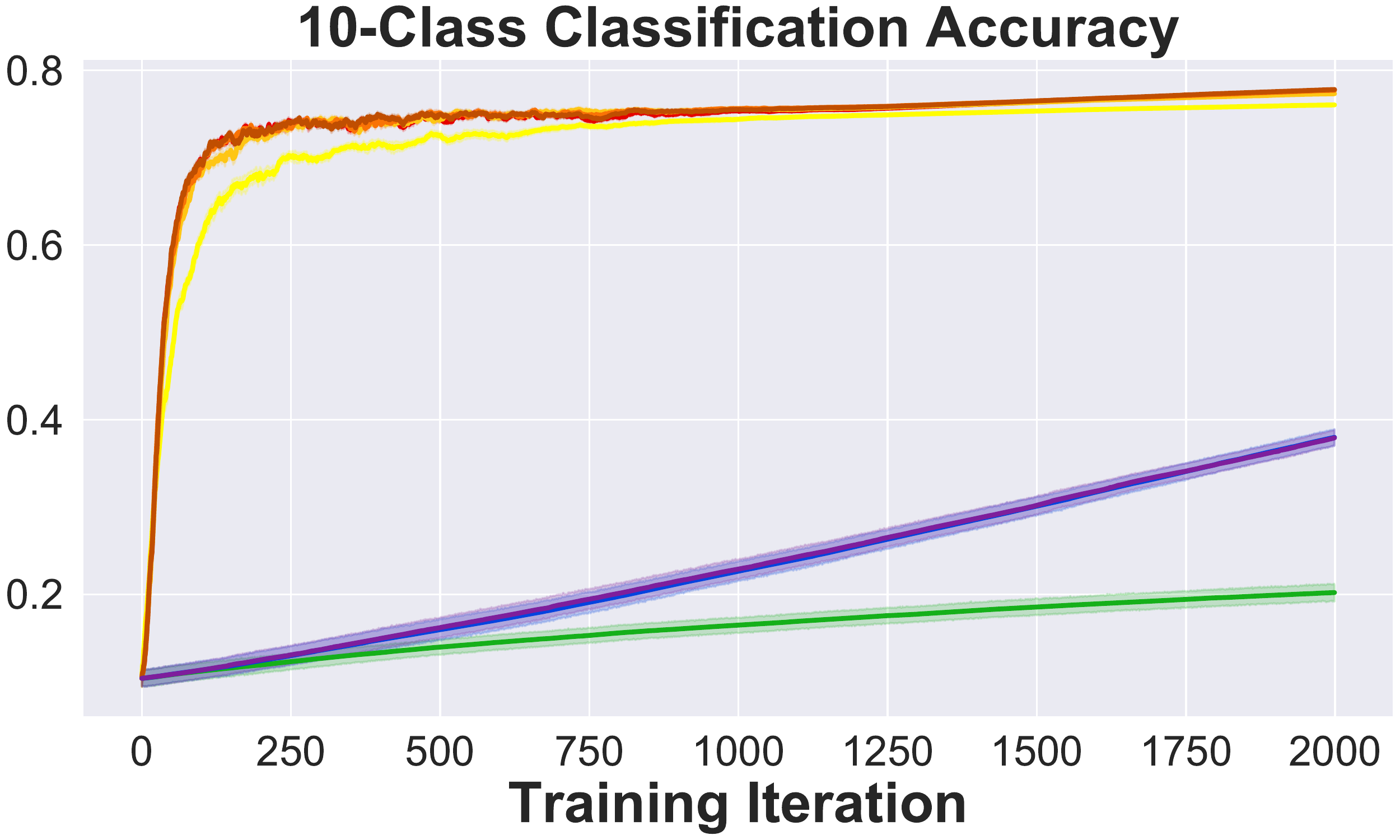}
    \end{subfigure}
    ~
    \begin{subfigure}{0.32\textwidth}
        \centering
        \includegraphics[width=\textwidth]{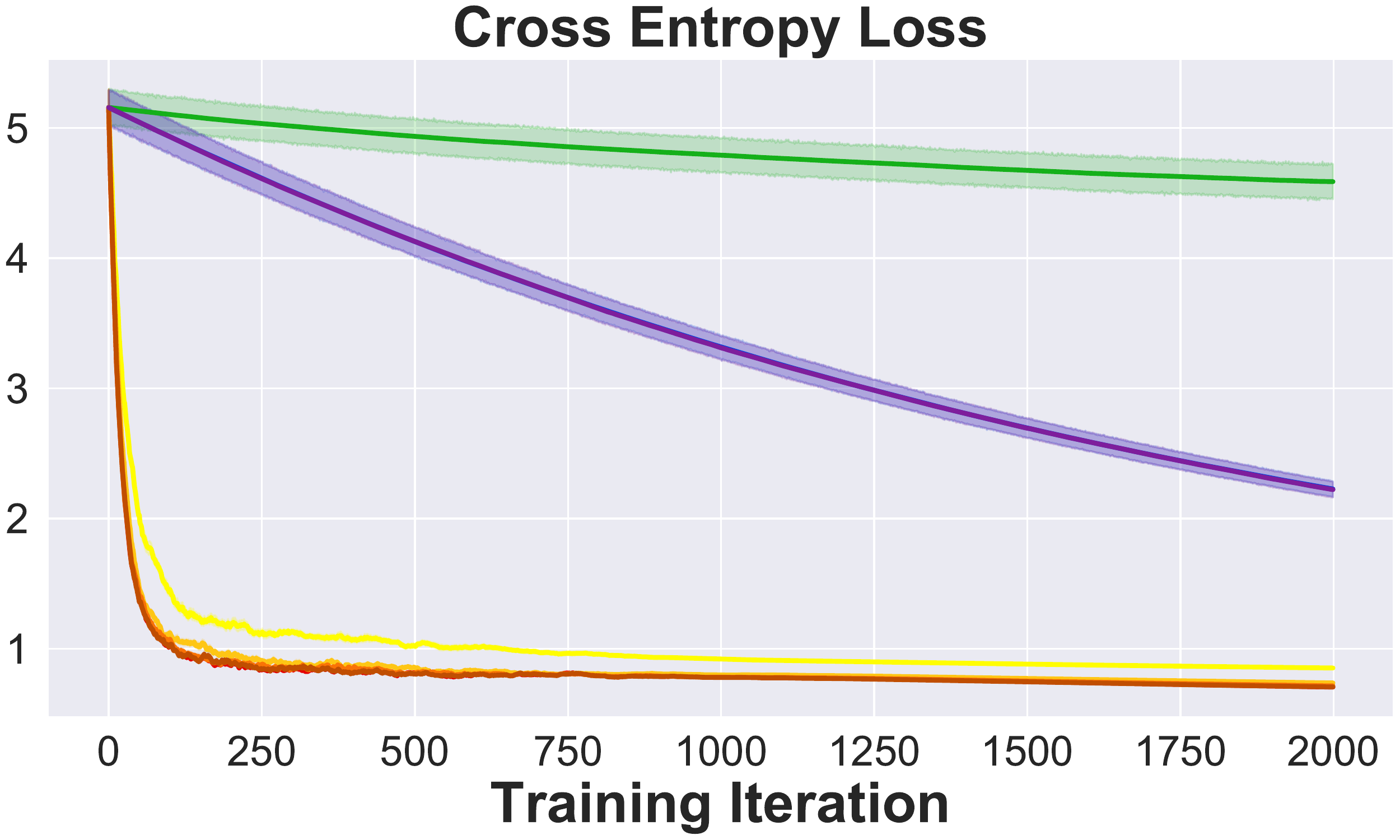}
    \end{subfigure}
    \caption{CIFAR-10, teacher feature from CNN-12.}
    \label{fig:CIFAR10-adv12}
    \end{subfigure}
\end{figure*}

\begin{figure*}\ContinuedFloat    
    \begin{subfigure}{\textwidth}
    \begin{subfigure}{0.32\textwidth}
        \centering
       \includegraphics[width=\textwidth]{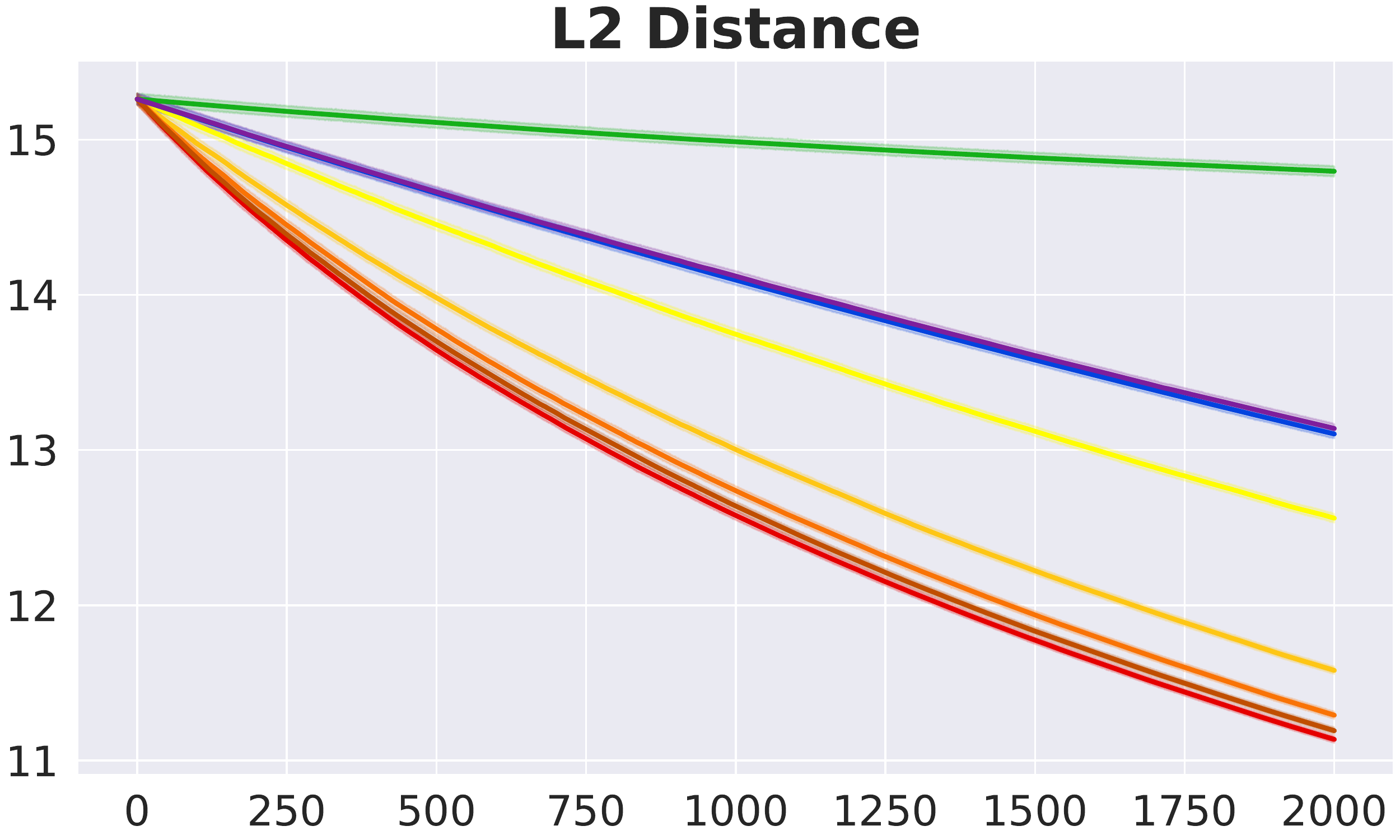}
    \end{subfigure}
    ~
    \begin{subfigure}{0.32\textwidth}
        \centering
        \includegraphics[width=\textwidth]{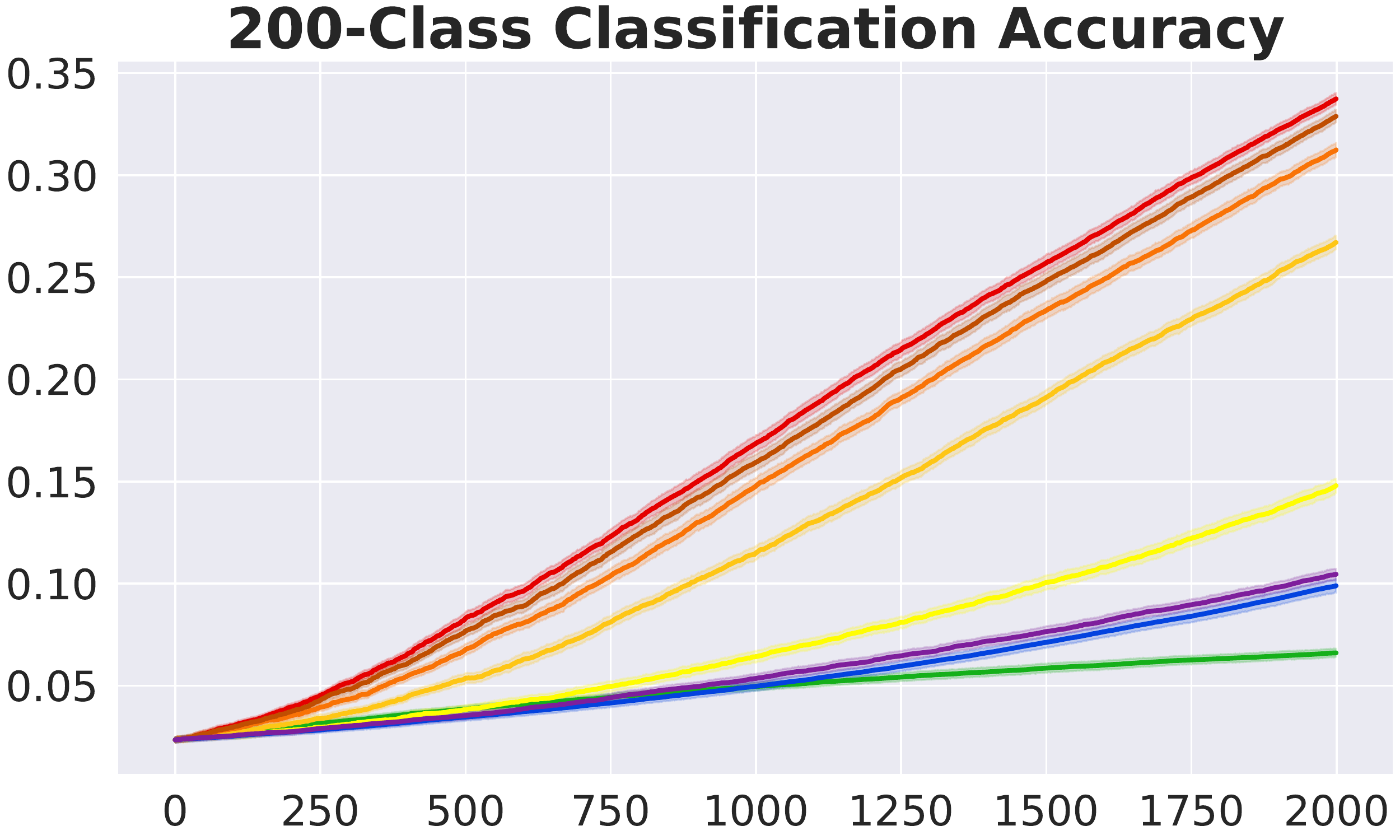}
    \end{subfigure}
    ~
    \begin{subfigure}{0.32\textwidth}
        \centering
        \includegraphics[width=\textwidth]{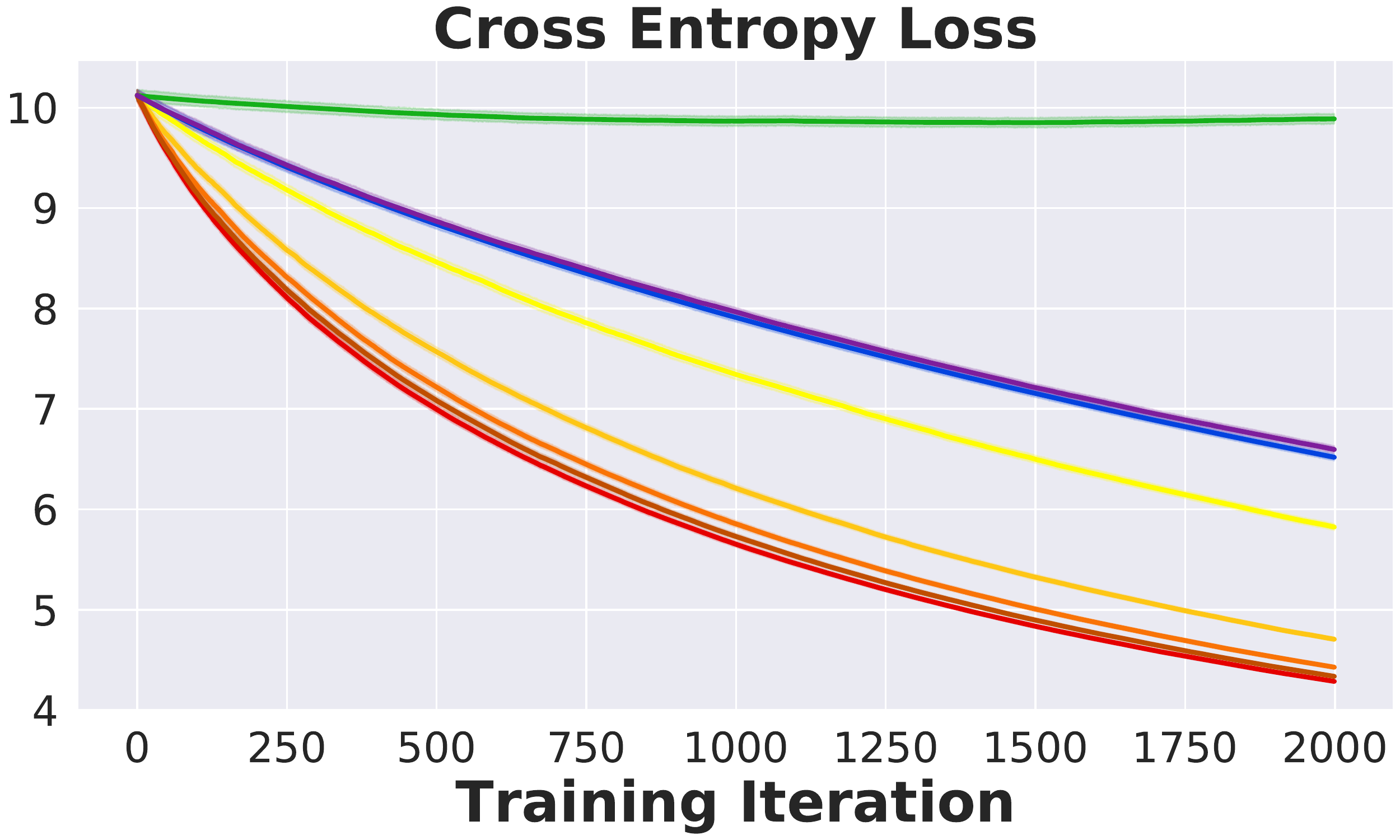}
    \end{subfigure}
    \caption{Tiny ImageNet, teacher feature from VGG-13, showing top-5 accuracy.}
    \label{fig:ImageNet-adv13}
    \end{subfigure}
    
    \begin{subfigure}{\textwidth}
    \begin{subfigure}{0.32\textwidth}
        \centering
       \includegraphics[width=\textwidth]{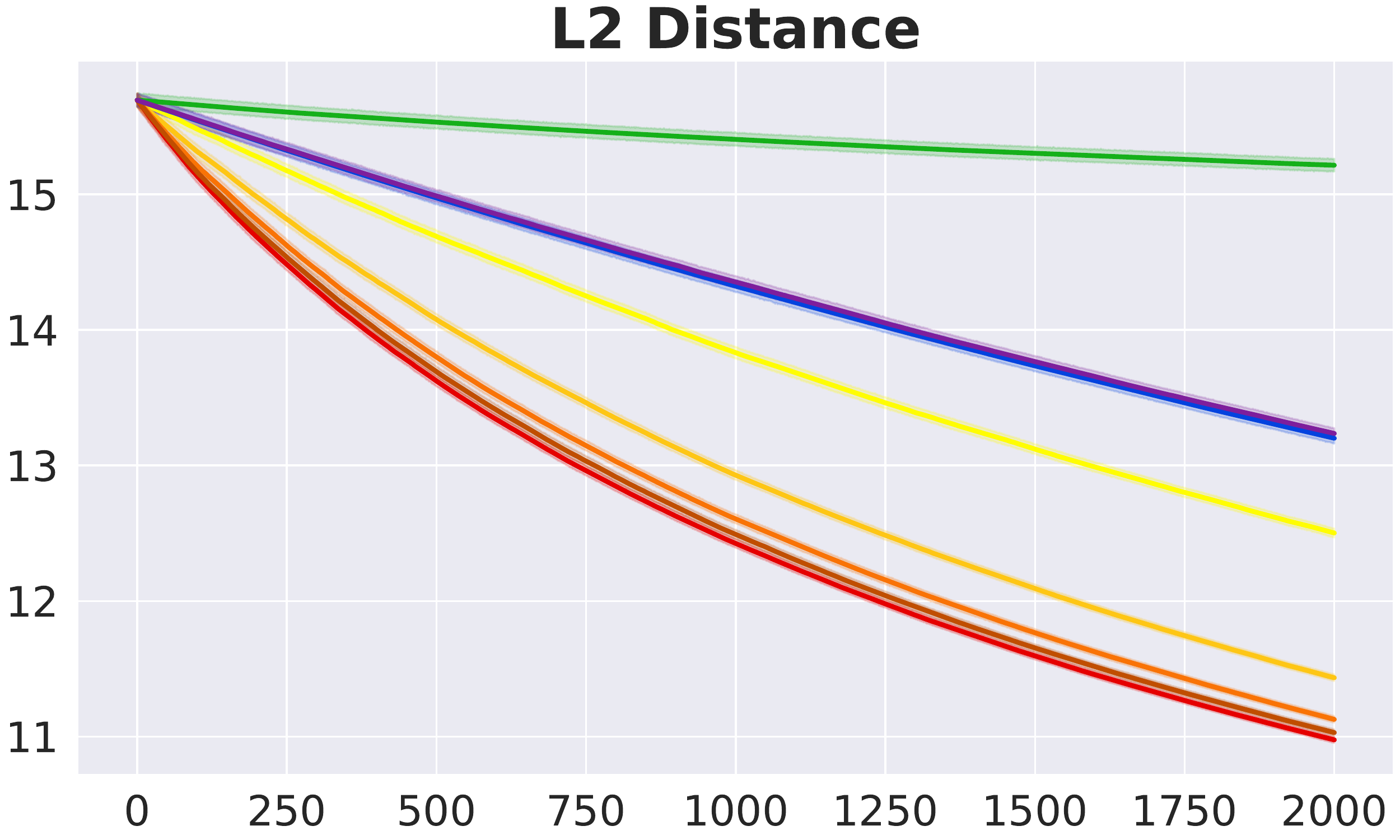}
    \end{subfigure}
    ~
    \begin{subfigure}{0.32\textwidth}
        \centering
        \includegraphics[width=\textwidth]{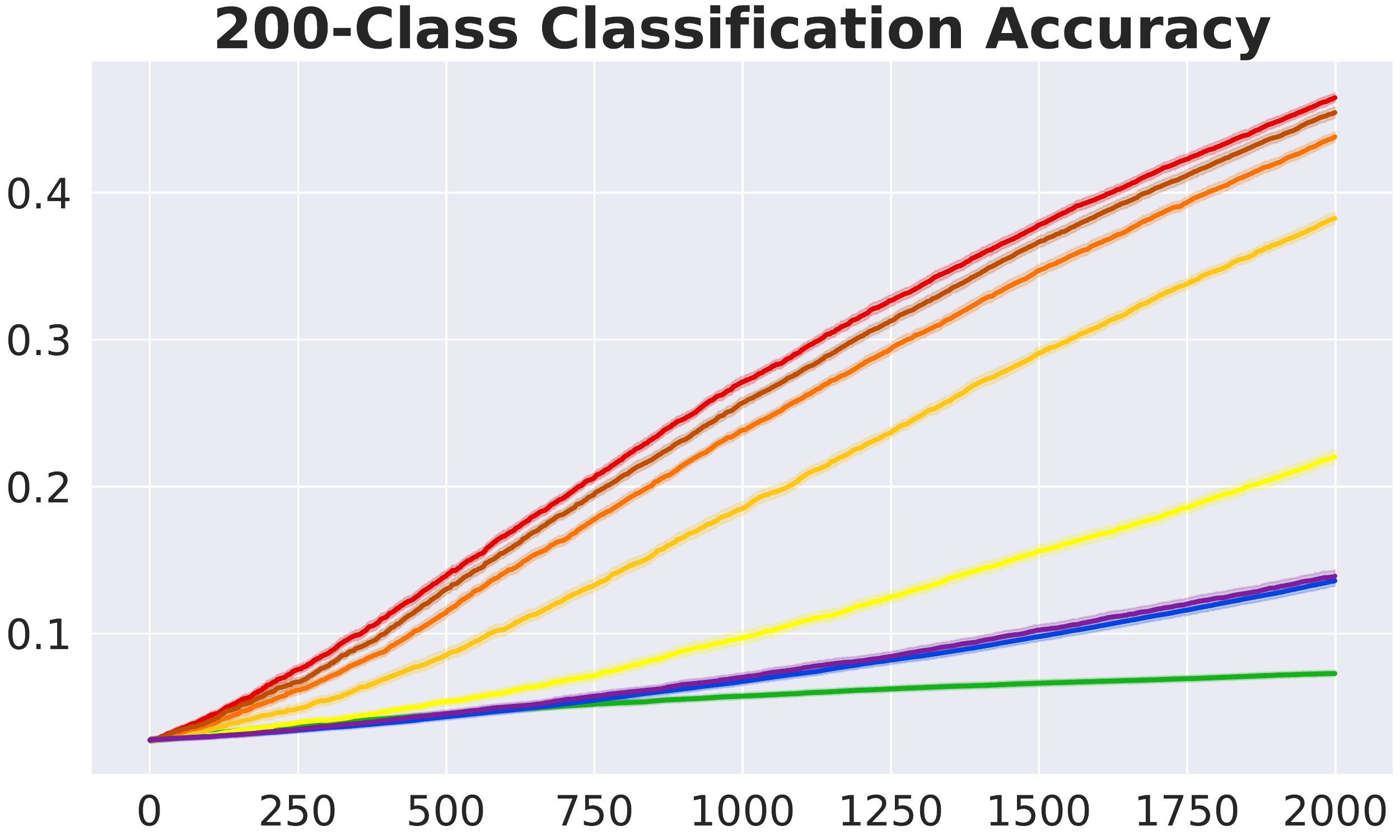}
    \end{subfigure}
    ~
    \begin{subfigure}{0.32\textwidth}
        \centering
        \includegraphics[width=\textwidth]{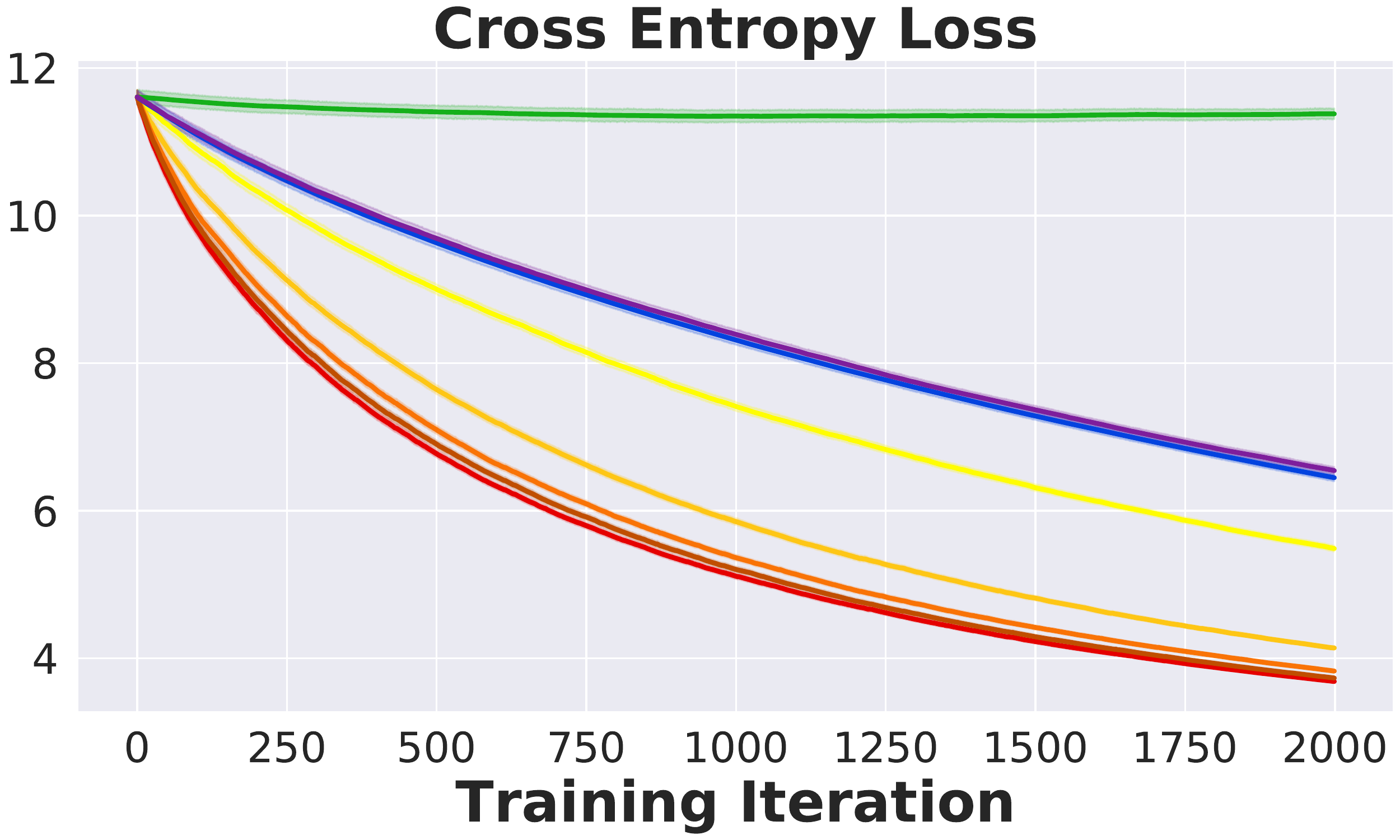}
    \end{subfigure}
    \caption{Tiny ImageNet, teacher feature from VGG-19, showing top-5 accuracy.}
    \label{fig:ImageNet-adv19}
    \end{subfigure}
    
    \begin{subfigure}{\textwidth}
    \begin{subfigure}{0.32\textwidth}
        \centering
        \includegraphics[width=\textwidth]{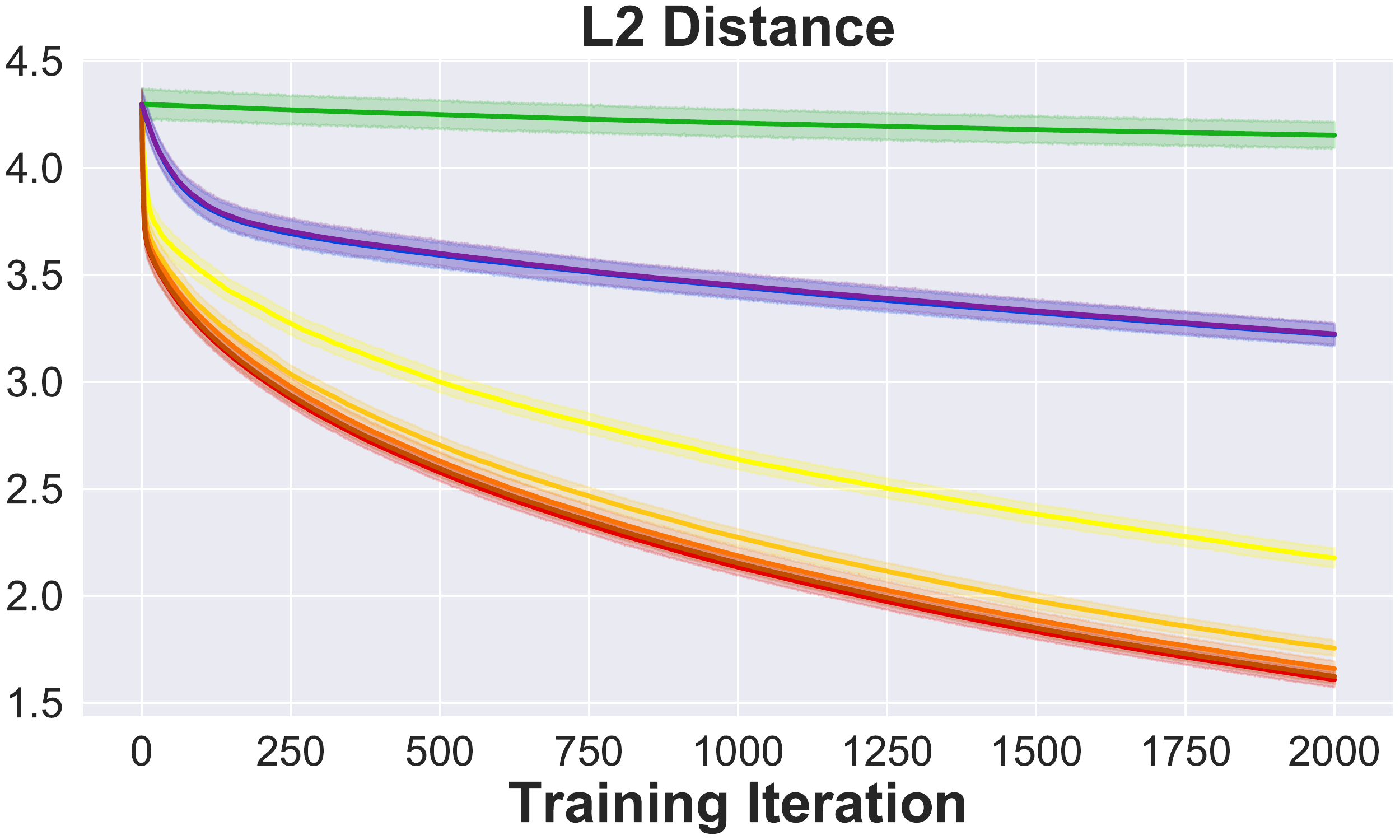}
    \end{subfigure}%
    ~
    \begin{subfigure}{0.32\textwidth}
        \centering
        \includegraphics[width=\textwidth]{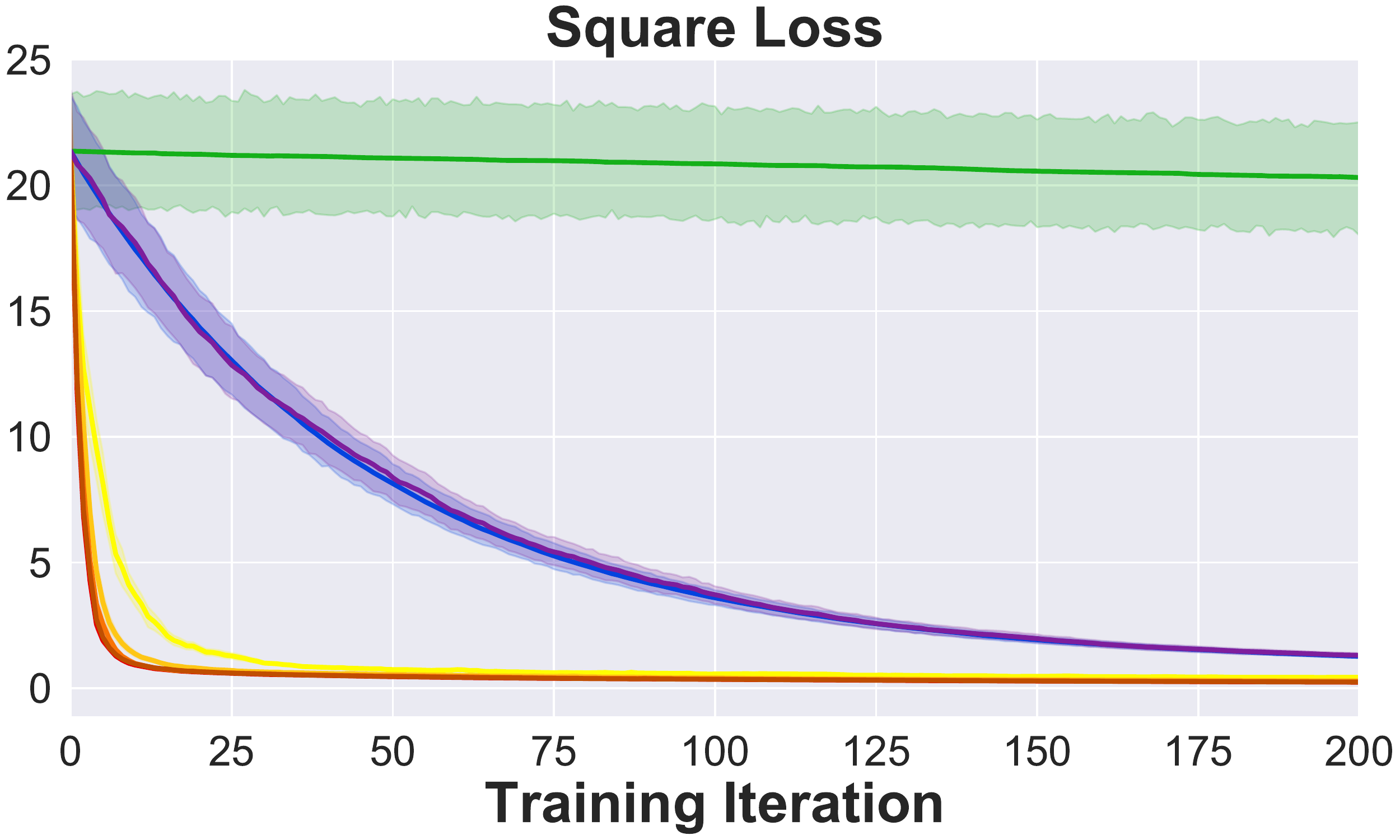}
    \end{subfigure}%
    ~
    \begin{subfigure}{0.32\textwidth}
        \centering
        \includegraphics[width=\textwidth]{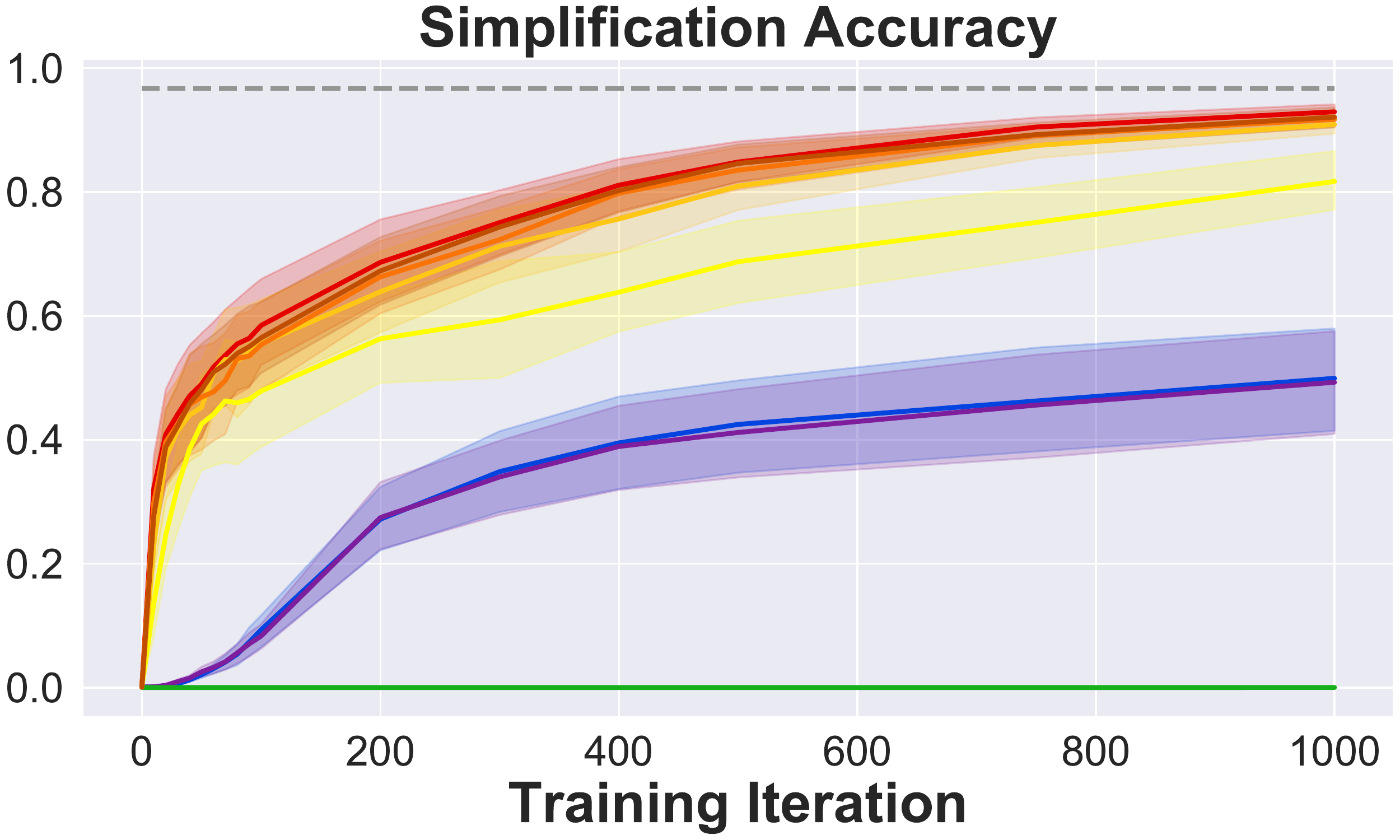}
    \end{subfigure}%
    \caption{Equation simplification, 40D teacher features}
    \label{fig:Equation-adv40}
    \end{subfigure}
    
    \begin{subfigure}{\textwidth}
    \begin{subfigure}{0.32\textwidth}
        \centering
        \includegraphics[width=\textwidth]{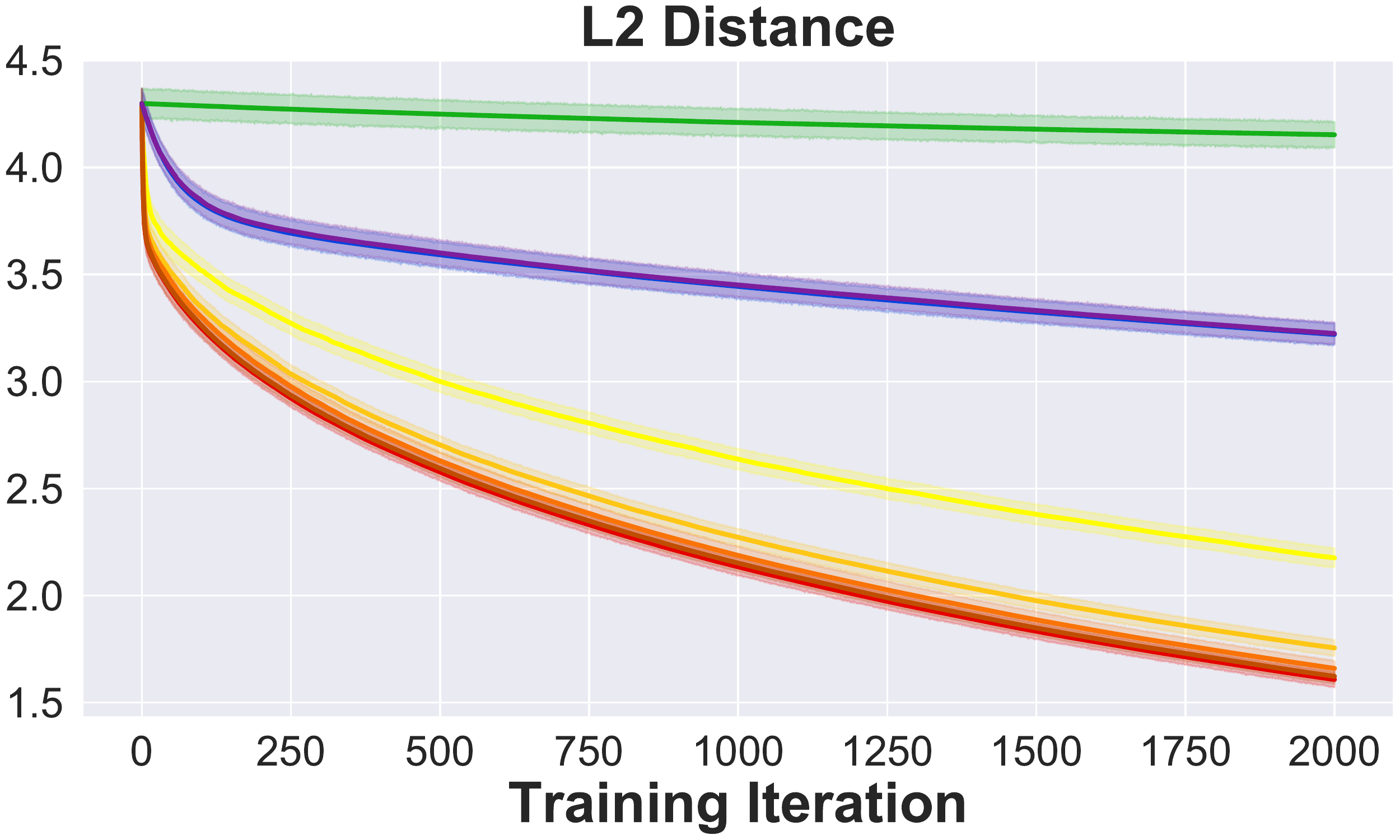}
    \end{subfigure}%
    ~
    \begin{subfigure}{0.32\textwidth}
        \centering
        \includegraphics[width=\textwidth]{Figures/Equation/equation_adv40_squareLoss.pdf}
    \end{subfigure}%
    ~
    \begin{subfigure}{0.32\textwidth}
        \centering
        \includegraphics[width=\textwidth]{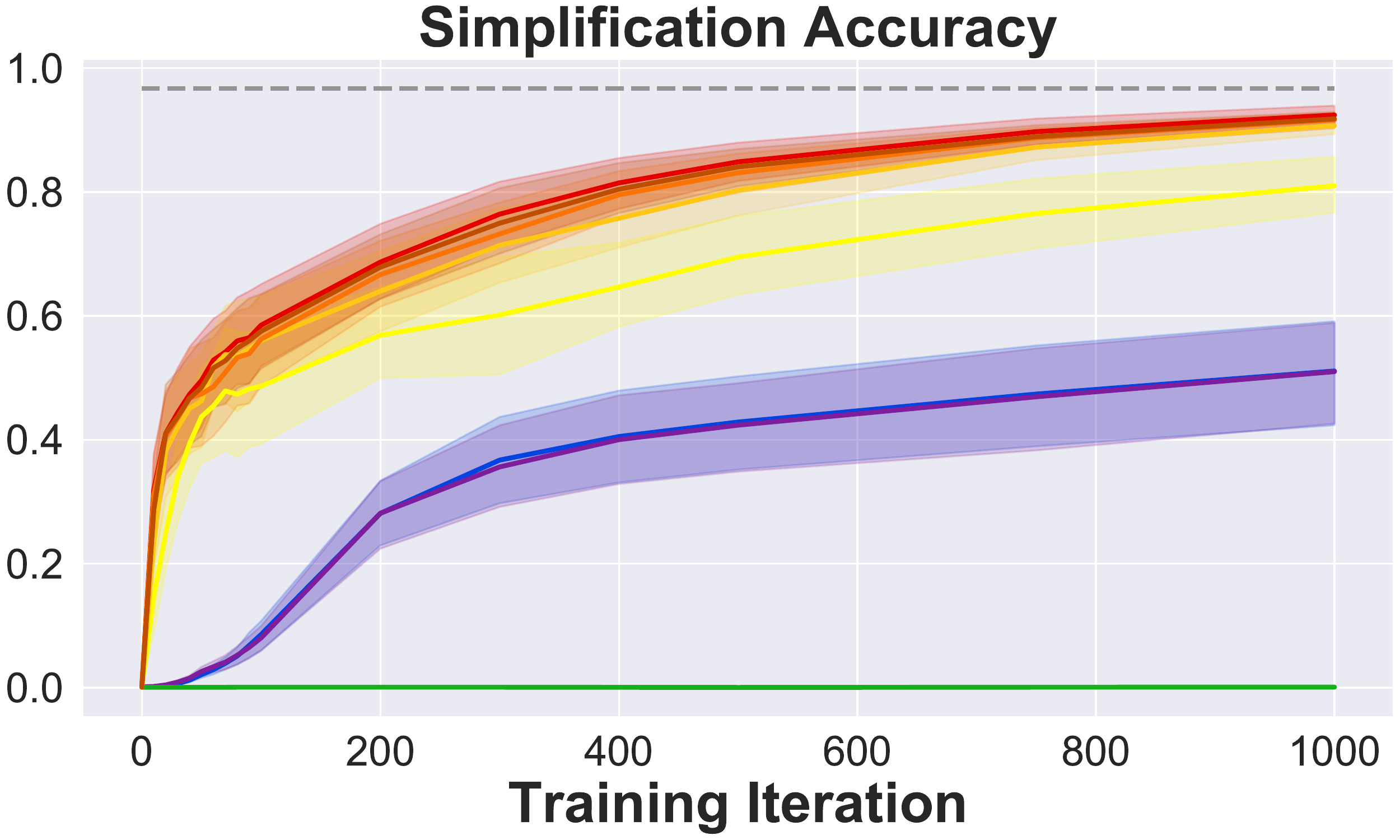}
    \end{subfigure}%
    \caption{Equation simplification, 50D teacher features}
    \label{fig:Equation-adv50}
    \end{subfigure}
    
    \begin{subfigure}{\textwidth}
    \begin{subfigure}{0.32\textwidth}
        \centering
        \includegraphics[width=\textwidth]{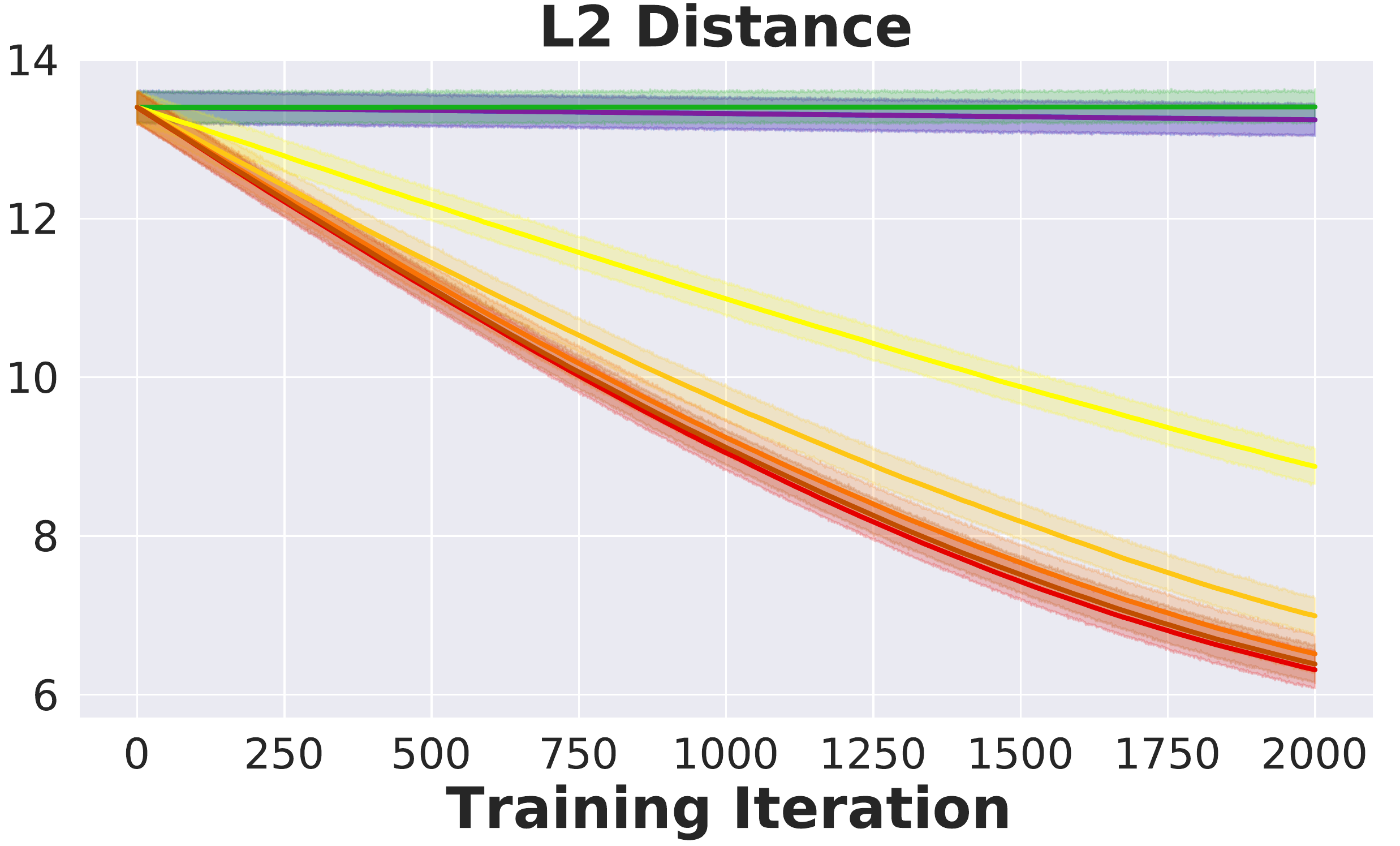}
    \end{subfigure}
    ~
    \begin{subfigure}{0.32\textwidth}
        \centering
        \includegraphics[width=\textwidth]{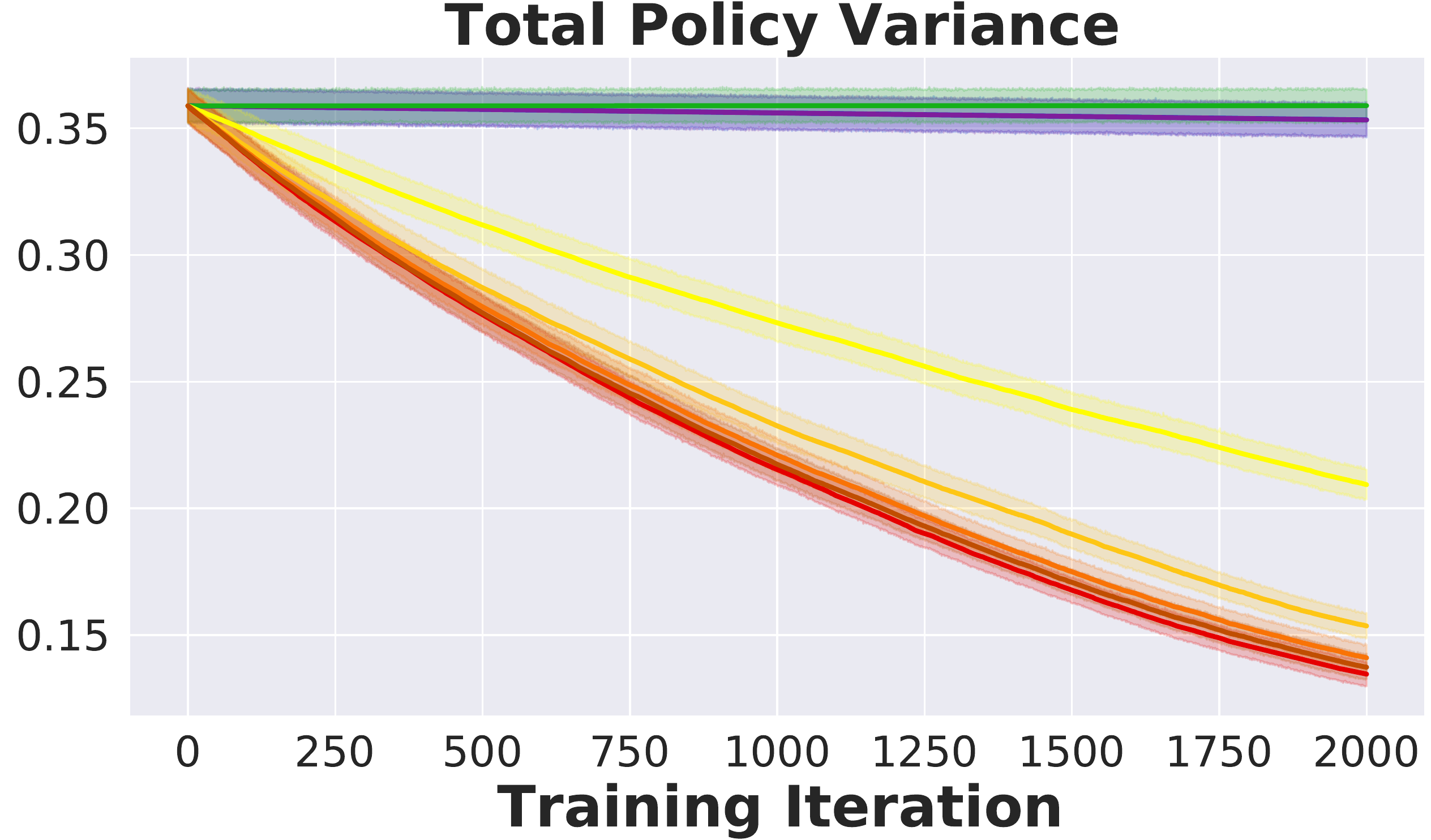}
    \end{subfigure}
    ~
    \begin{subfigure}{0.32\textwidth}
        \centering
        \includegraphics[width=\textwidth]{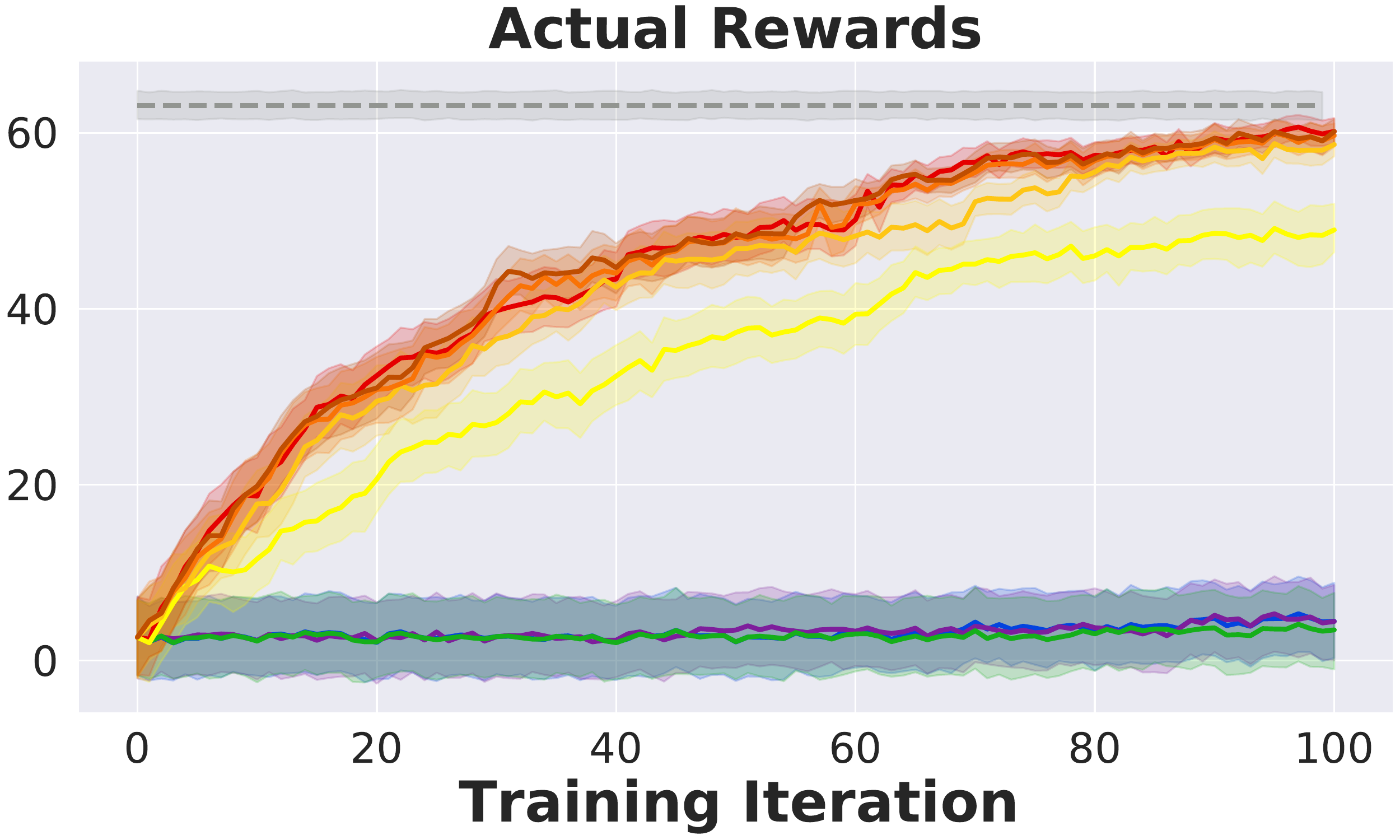}
    \end{subfigure}
    \caption{Online inverse reinforcement learning}
    \label{fig:OIRLH-adv}
    \end{subfigure}
    
    \begin{subfigure}{\textwidth}
    \begin{subfigure}{0.32\textwidth}
        \centering
        \includegraphics[width=\textwidth]{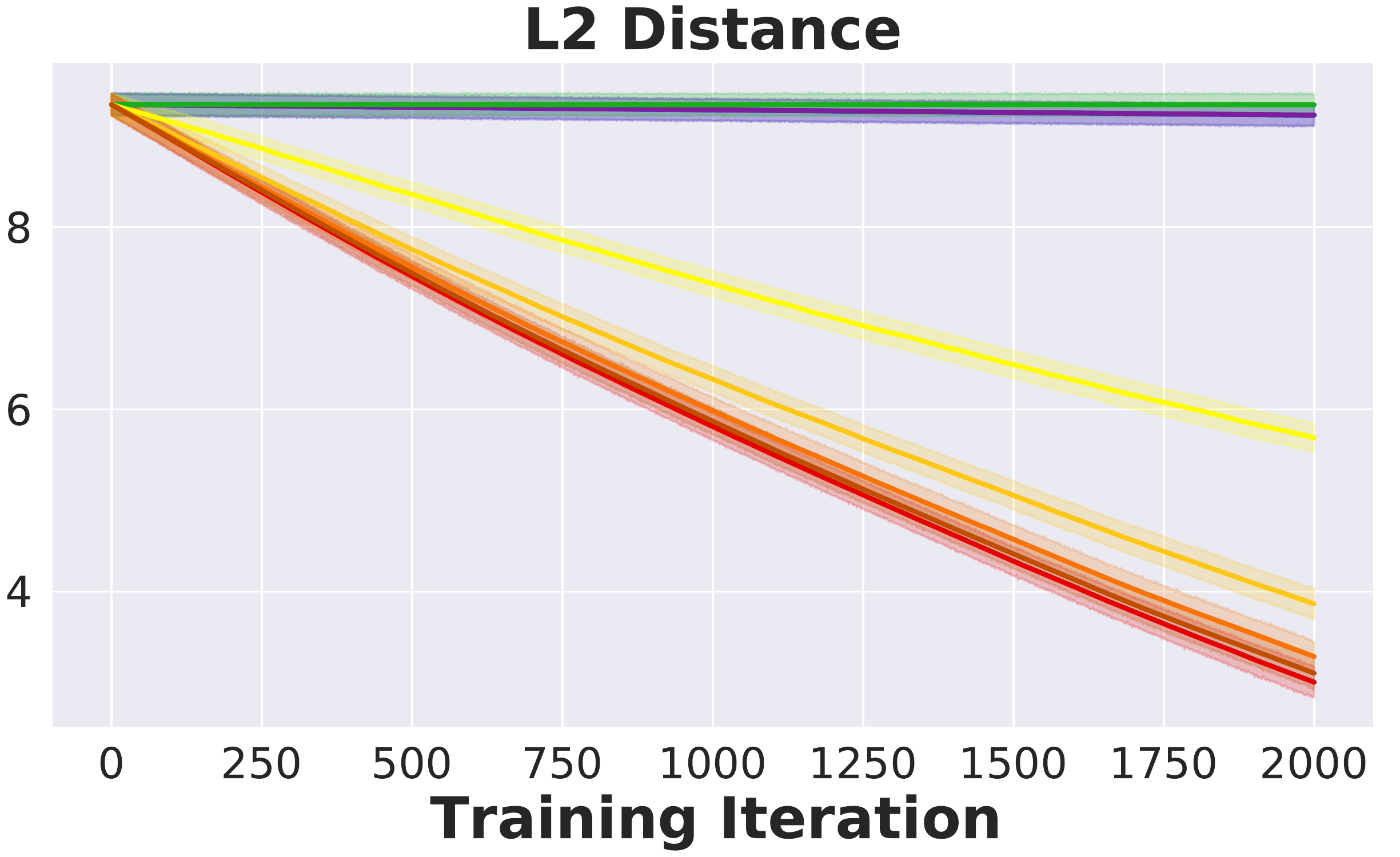}
    \end{subfigure}
    ~
    \begin{subfigure}{0.32\textwidth}
        \centering
        \includegraphics[width=\textwidth]{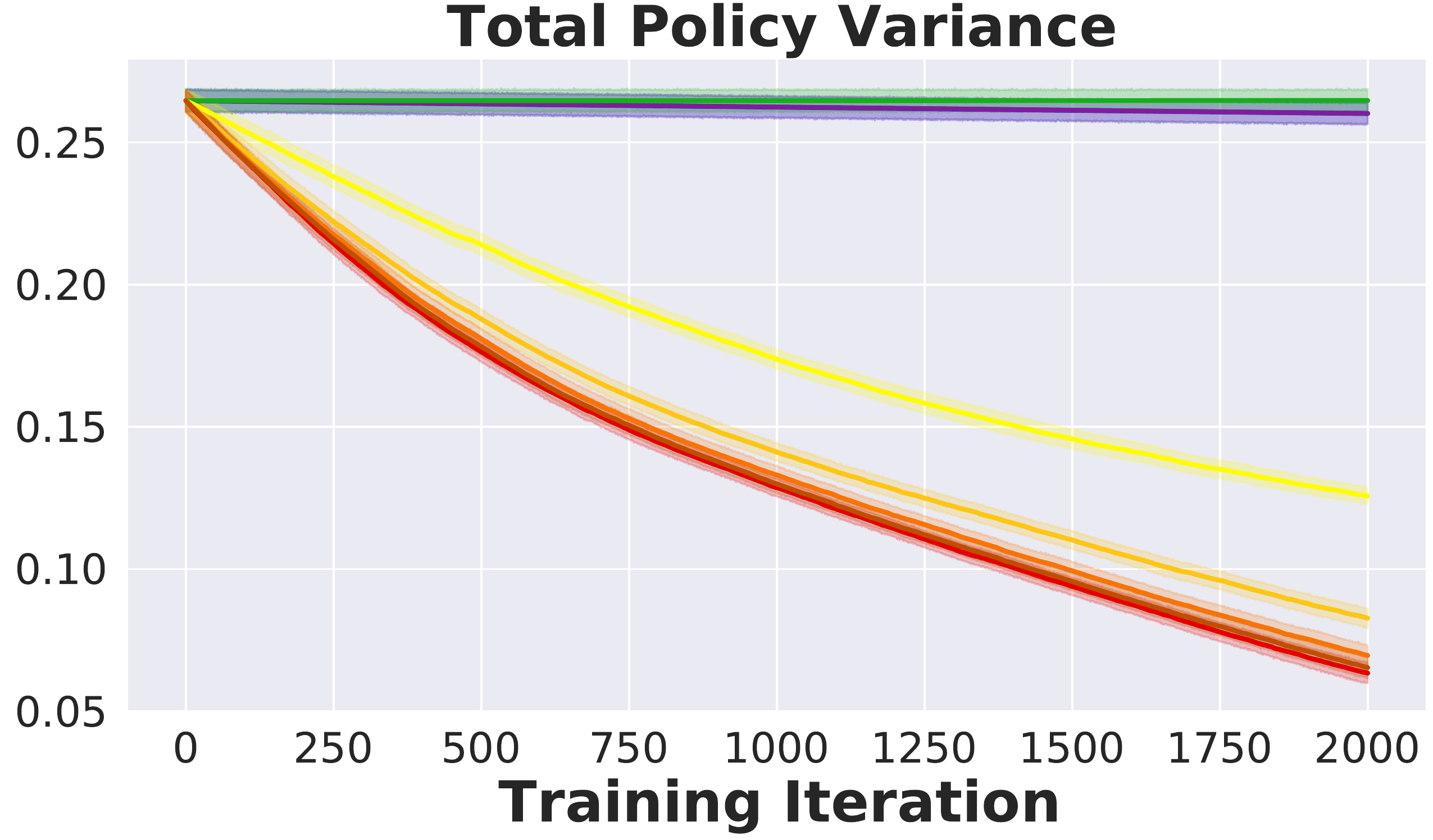}
    \end{subfigure}
    ~
    \begin{subfigure}{0.32\textwidth}
        \centering
        \includegraphics[width=\textwidth]{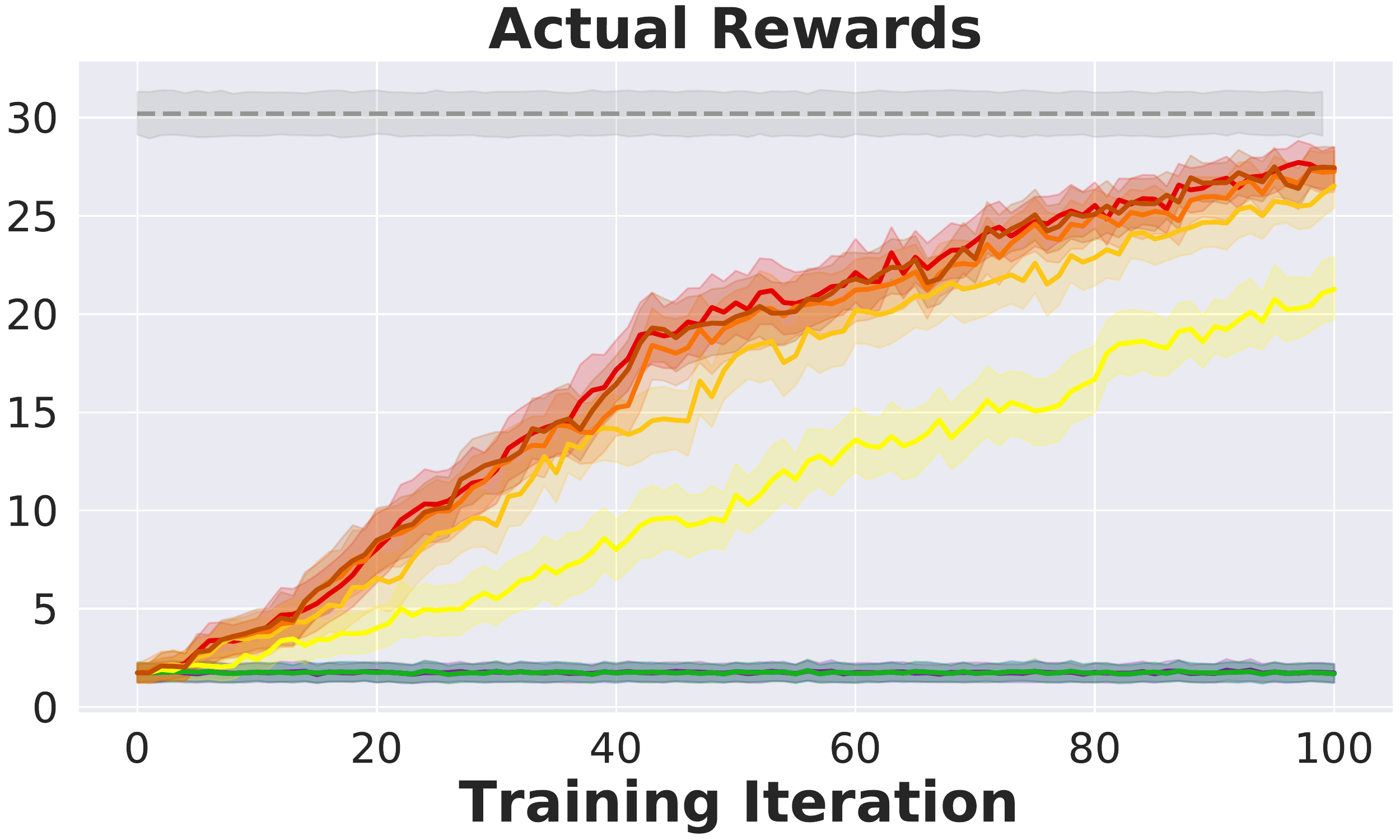}
    \end{subfigure}
    \caption{Online inverse reinforcement learning Sparse}
    \label{fig:OIRLE-adv}
    \end{subfigure}
    
    \begin{subfigure}{\textwidth}
        \centering
        \includegraphics[width=\textwidth]{Figures/imitate_legend.pdf}
    \end{subfigure}
    \caption{Adversarial teacher results. With an adversarial teacher, a naive learner can no longer learn effectively. ITAL still learns efficiently. SGD and batch learning are included for comparison.}
    \label{fig:adv}
\end{figure*}

We further test the robustness of our algorithm with an adversarial teacher, who, instead of choosing the most helpful data, chooses the least helpful one. She replace the $\argmax$ in equation~\eqref{eq:teaching-volume} in the main text with $\argmin$. In this scenario, a naive learner can barely learn, but the teacher-aware learner still shows steady improvement. See table~\ref{sup:tab:adv-beta} for the $\beta$ used in these experiments.

\subsection{Human Teacher}\label{sup:sec:exp-human}
\begin{figure}
    \centering
    \includegraphics[width=\textwidth]{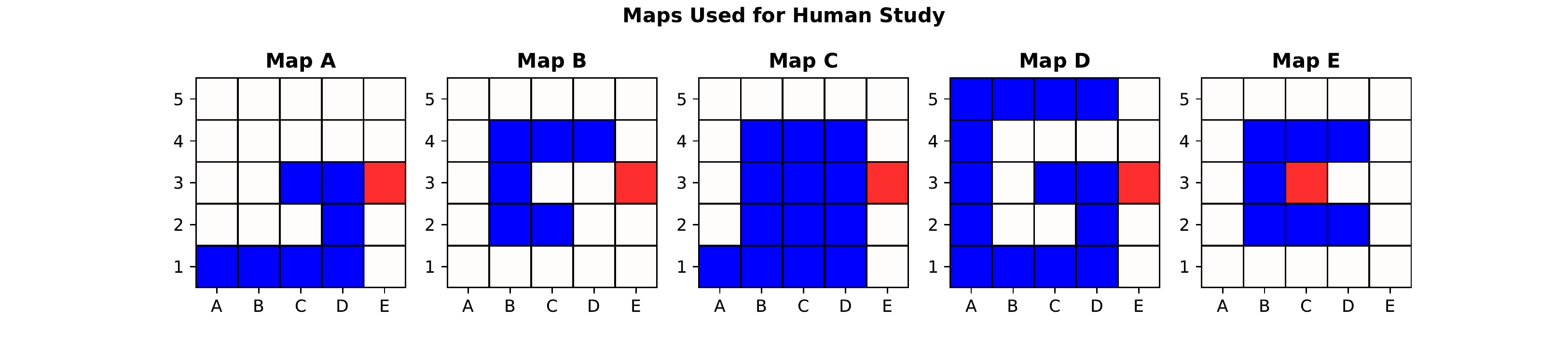}
    \caption{Map configurations}
    \label{sup:fig:map_configuration}
\end{figure}
\begin{figure}
    \centering
    \includegraphics[width=\textwidth]{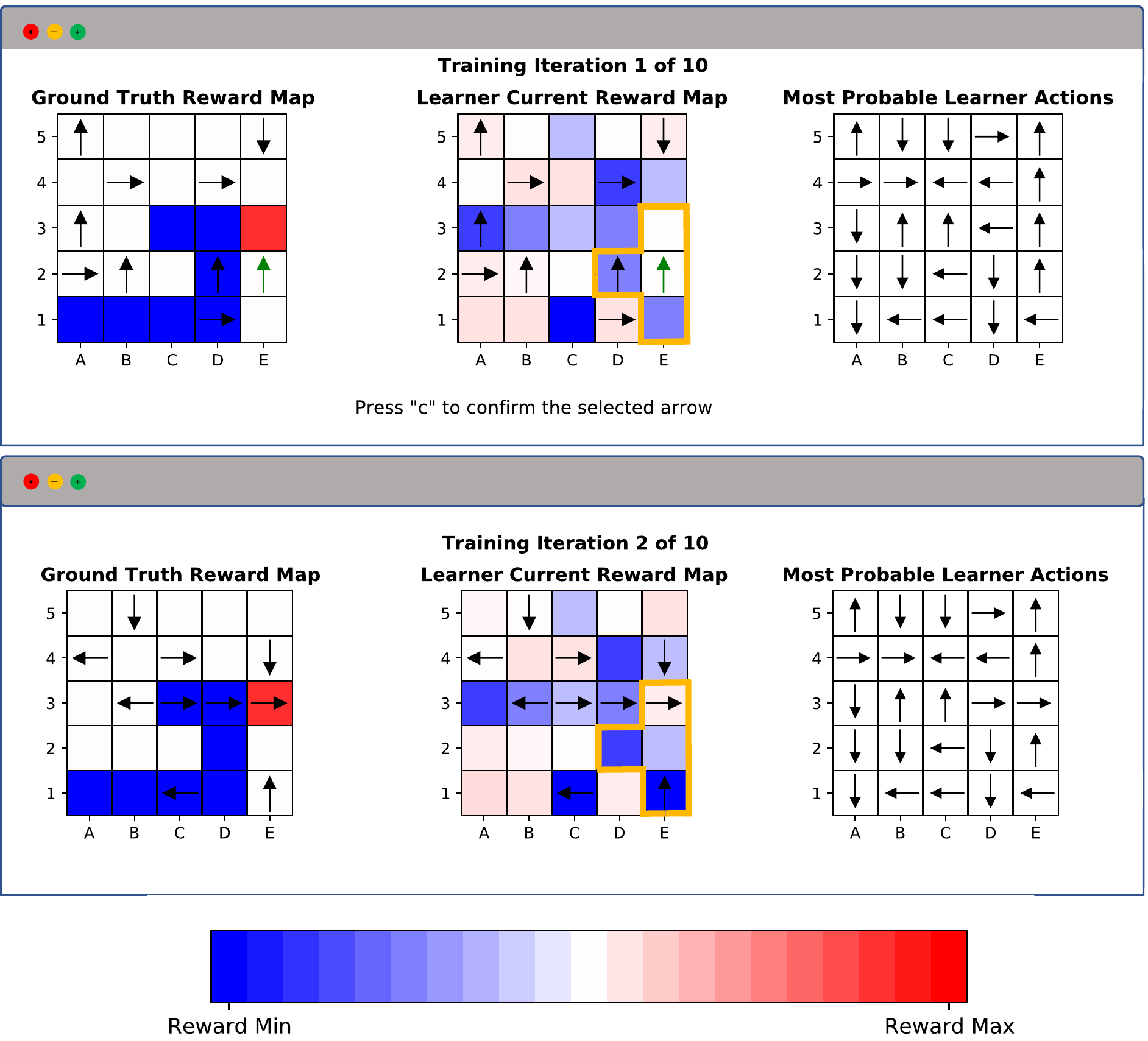}
    \caption{Actual interface used in the human study. The reward spectrum will be shown to the subjects at the experiment introduction. In this example, the subject chose the green arrow as the example at the 1-st iteration. Then, as we annotated with the orange T-shaped boxes, the estimated reward of the target tile of the green arrow increased, while rewards of the arrow source and the surrounding neighbors decreased. The orange boxes were not included in the human study.}
    \label{sup:fig:human_interface}
\end{figure}
\begin{figure}
    
    \begin{subfigure}{0.49\textwidth}
        \centering
        \includegraphics[width=\textwidth]{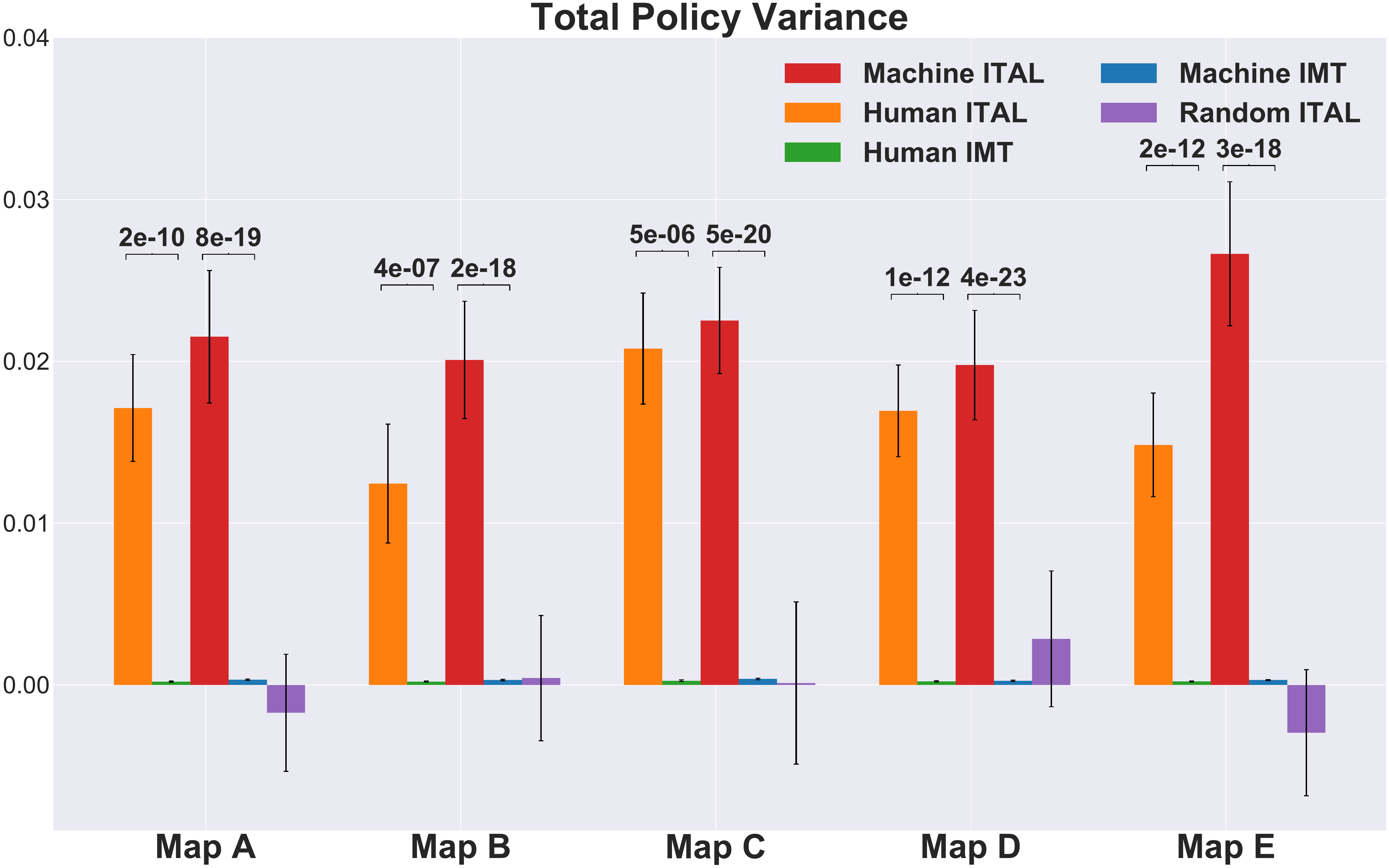}
    \end{subfigure}
    ~
    \begin{subfigure}{0.49\textwidth}
        \centering
        \includegraphics[width=\textwidth]{Figures/HumanStudy/human_rewards.pdf}
    \end{subfigure}
    \caption{Human study results. All the p-values are calculated with paired t-test.}
    \label{sup:fig:human_study_results}
\end{figure}
\begin{figure}
    \begin{subfigure}{\textwidth}
        \begin{subfigure}{0.32\textwidth}
            \centering
            \includegraphics[width=\textwidth]{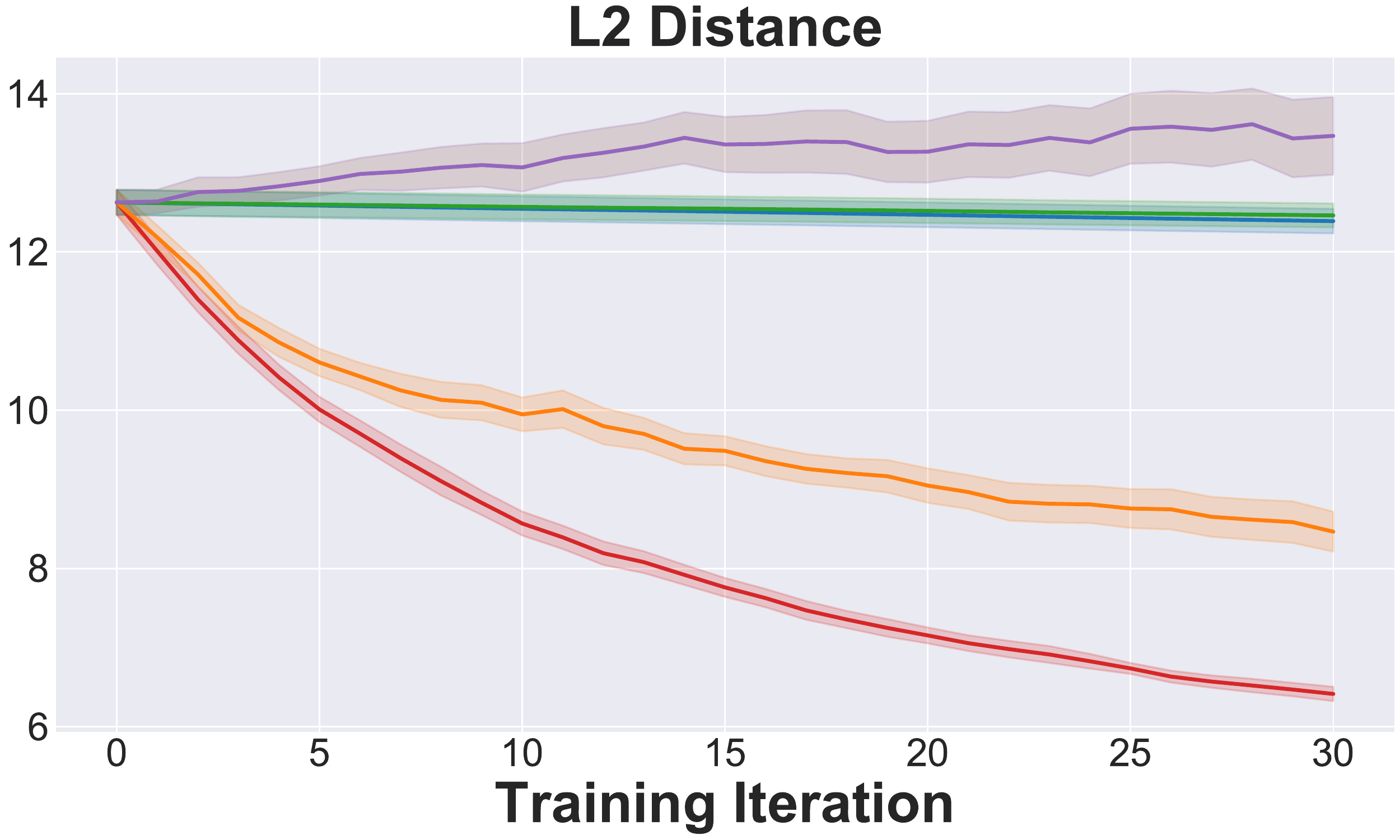}
        \end{subfigure}
        ~
        \begin{subfigure}{0.32\textwidth}
            \centering
            \includegraphics[width=\textwidth]{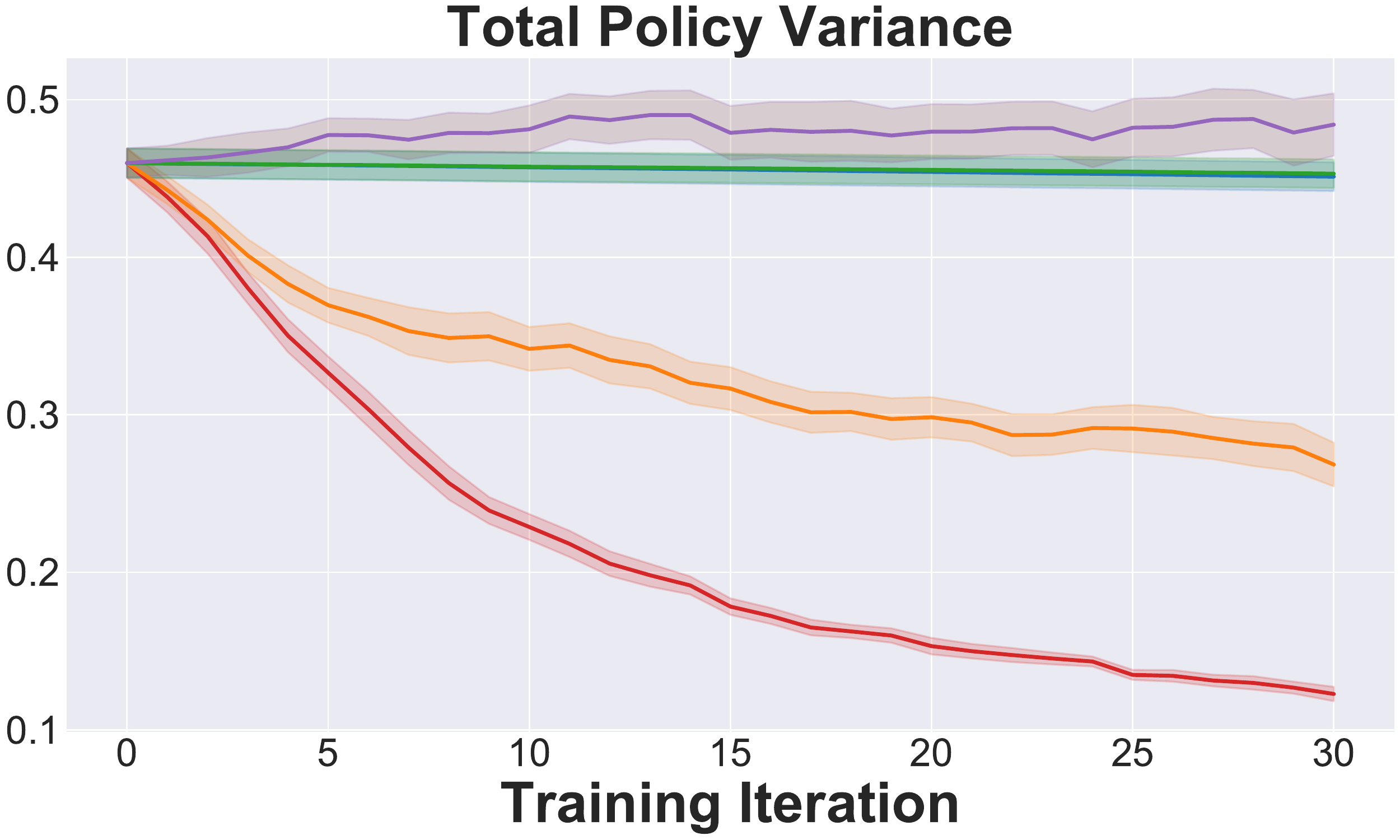}
        \end{subfigure}
        ~
        \begin{subfigure}{0.32\textwidth}
            \centering
            \includegraphics[width=\textwidth]{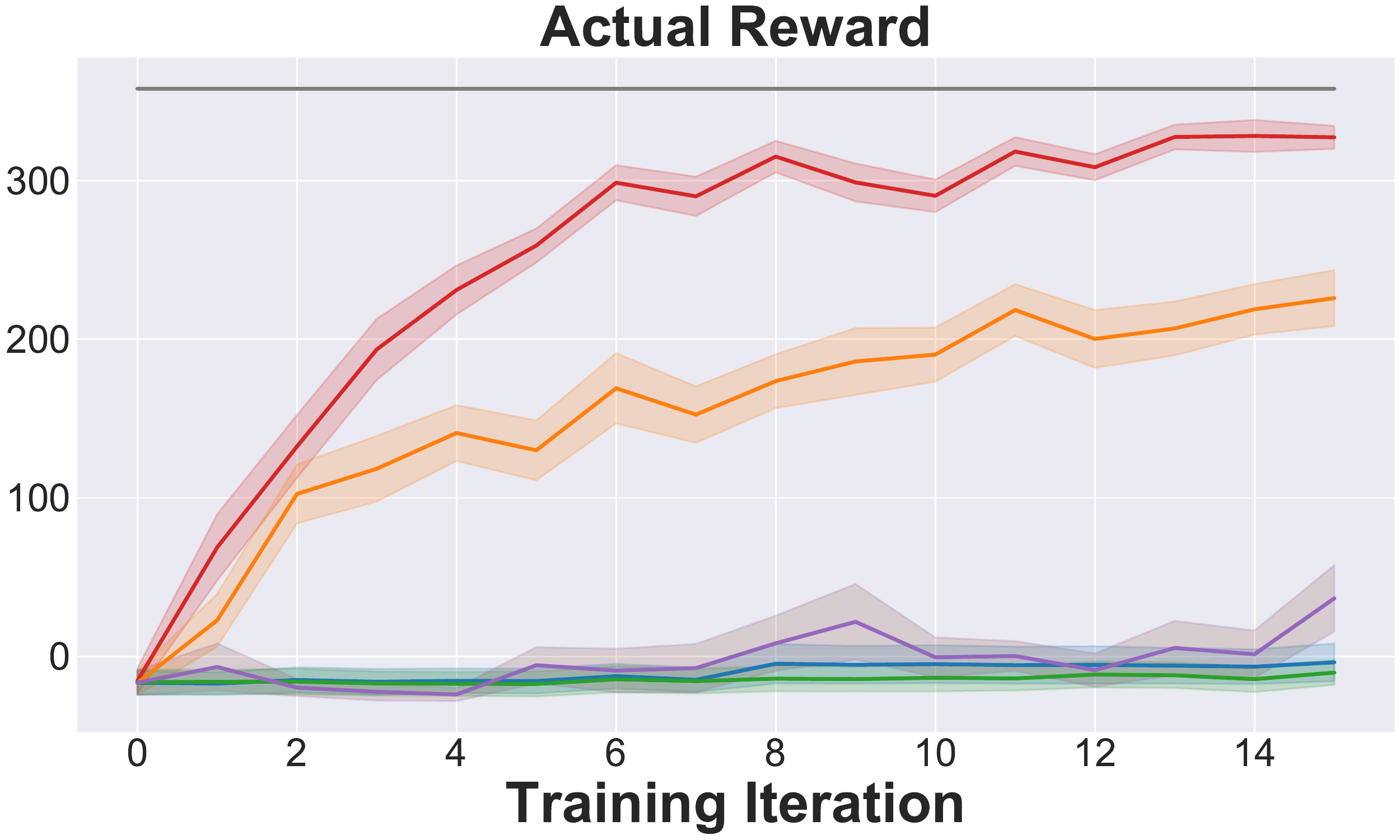}
        \end{subfigure}
        \caption{Map A}
    \end{subfigure}
    
    \begin{subfigure}{\textwidth}
        \begin{subfigure}{0.32\textwidth}
            \centering
            \includegraphics[width=\textwidth]{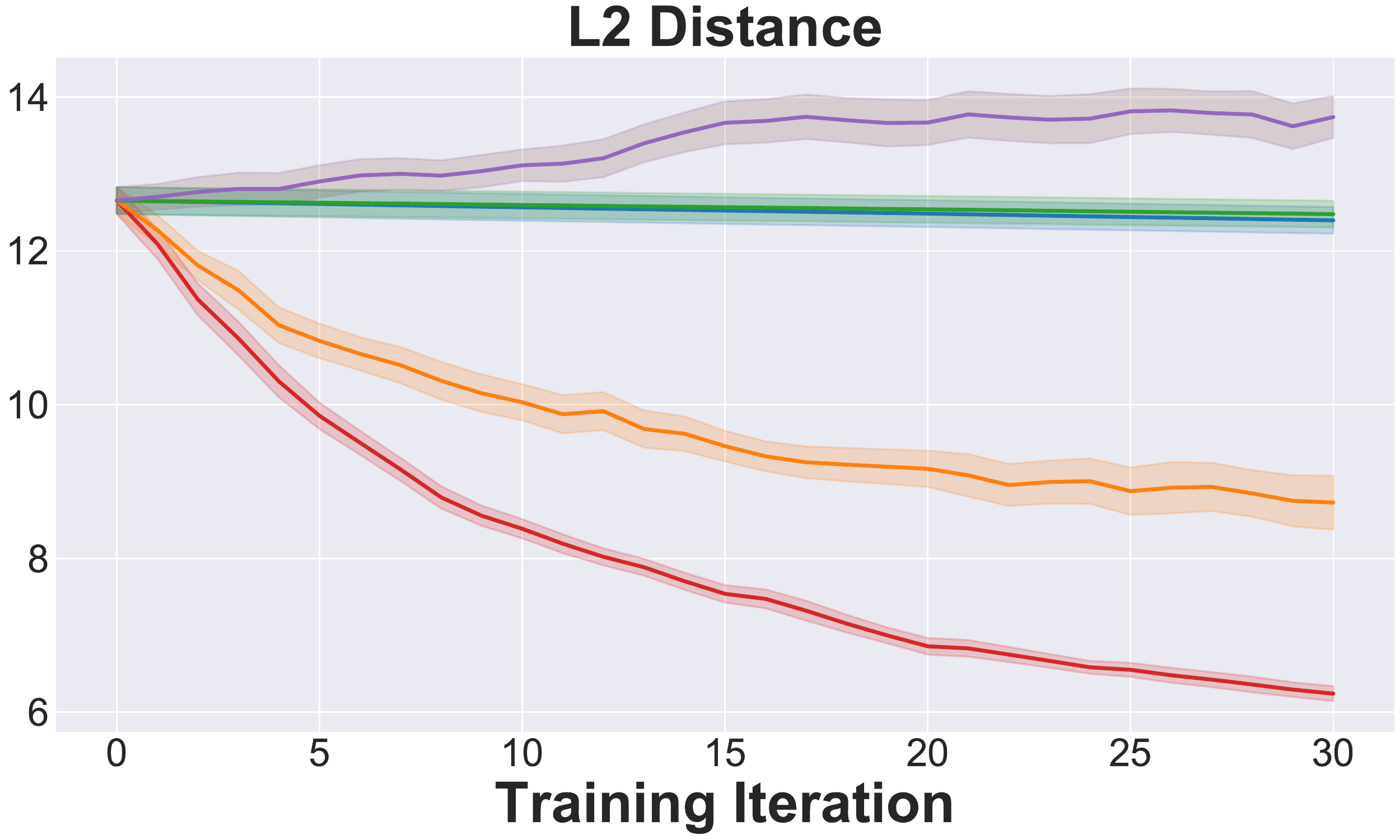}
        \end{subfigure}
        ~
        \begin{subfigure}{0.32\textwidth}
            \centering
            \includegraphics[width=\textwidth]{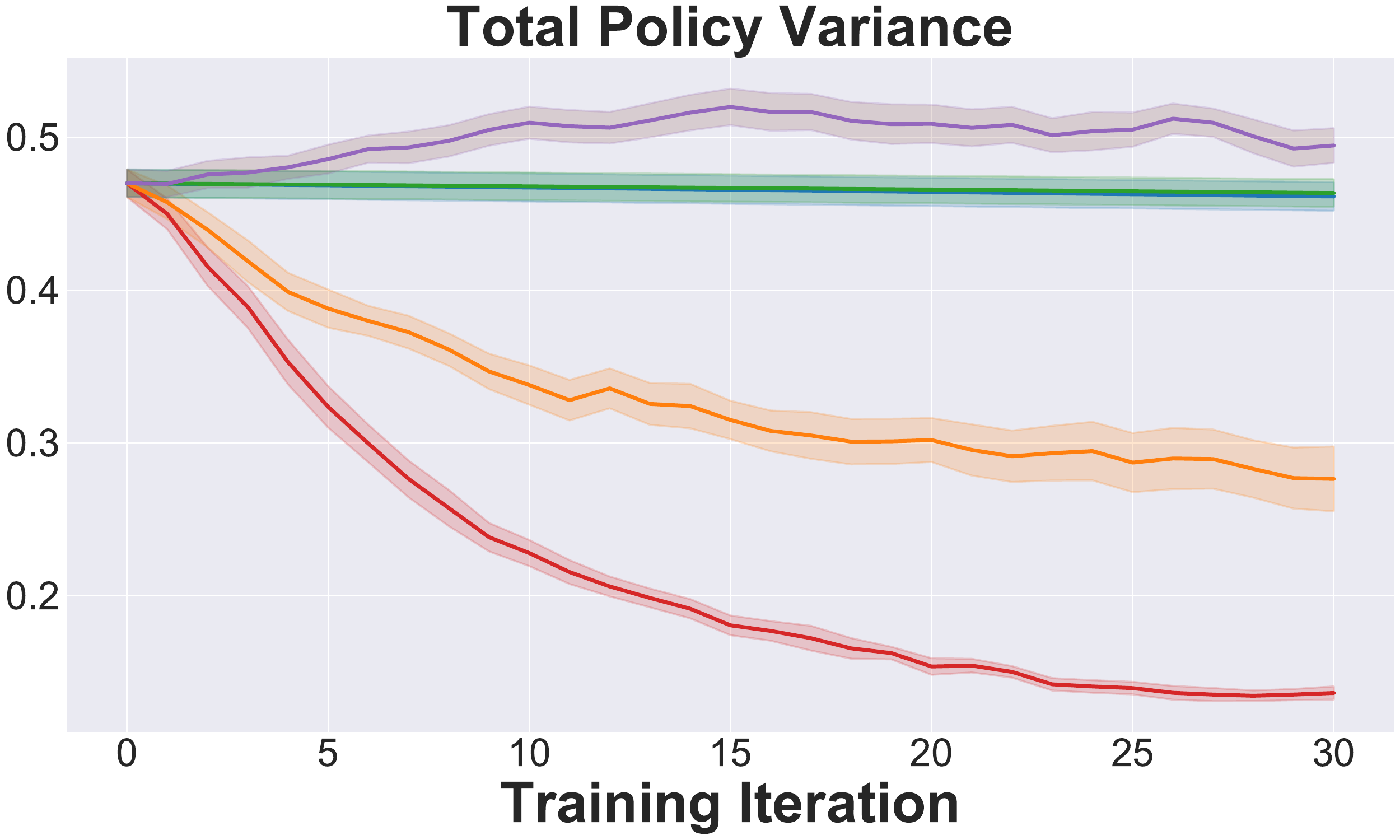}
        \end{subfigure}
        ~
        \begin{subfigure}{0.32\textwidth}
            \centering
            \includegraphics[width=\textwidth]{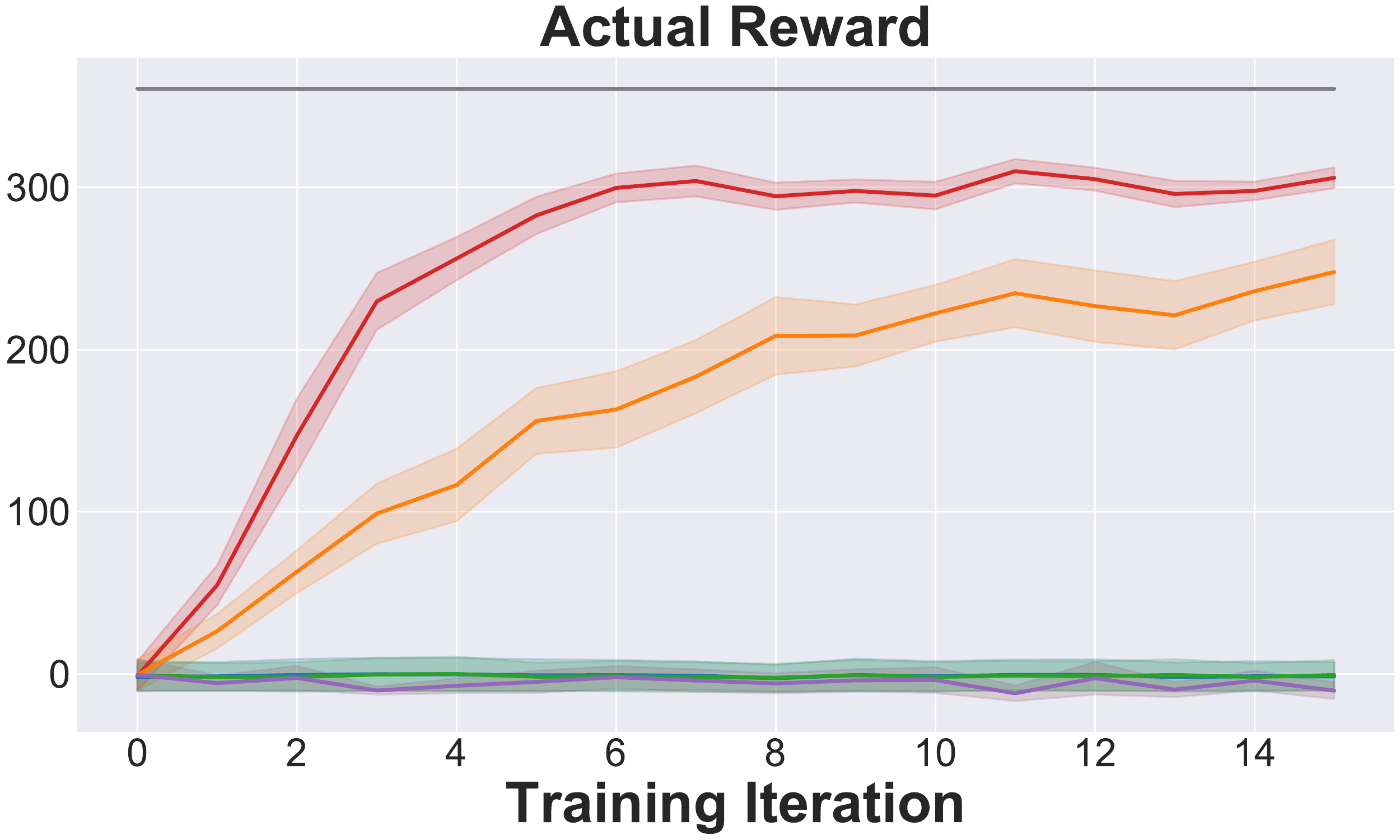}
        \end{subfigure}
    \caption{Map B}
    \end{subfigure}
    
    \begin{subfigure}{\textwidth}
    \begin{subfigure}{0.32\textwidth}
        \centering
        \includegraphics[width=\textwidth]{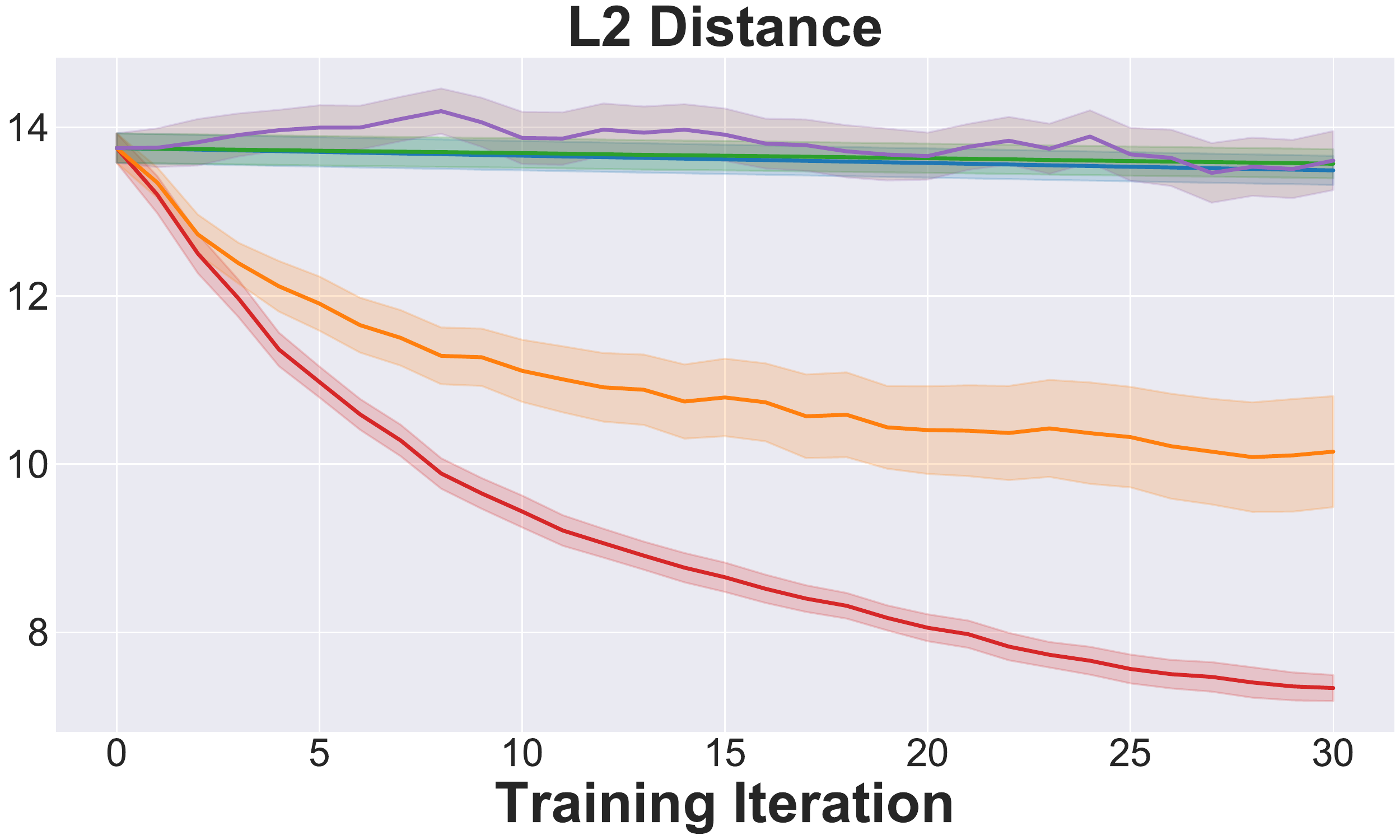}
    \end{subfigure}
    ~
    \begin{subfigure}{0.32\textwidth}
        \centering
        \includegraphics[width=\textwidth]{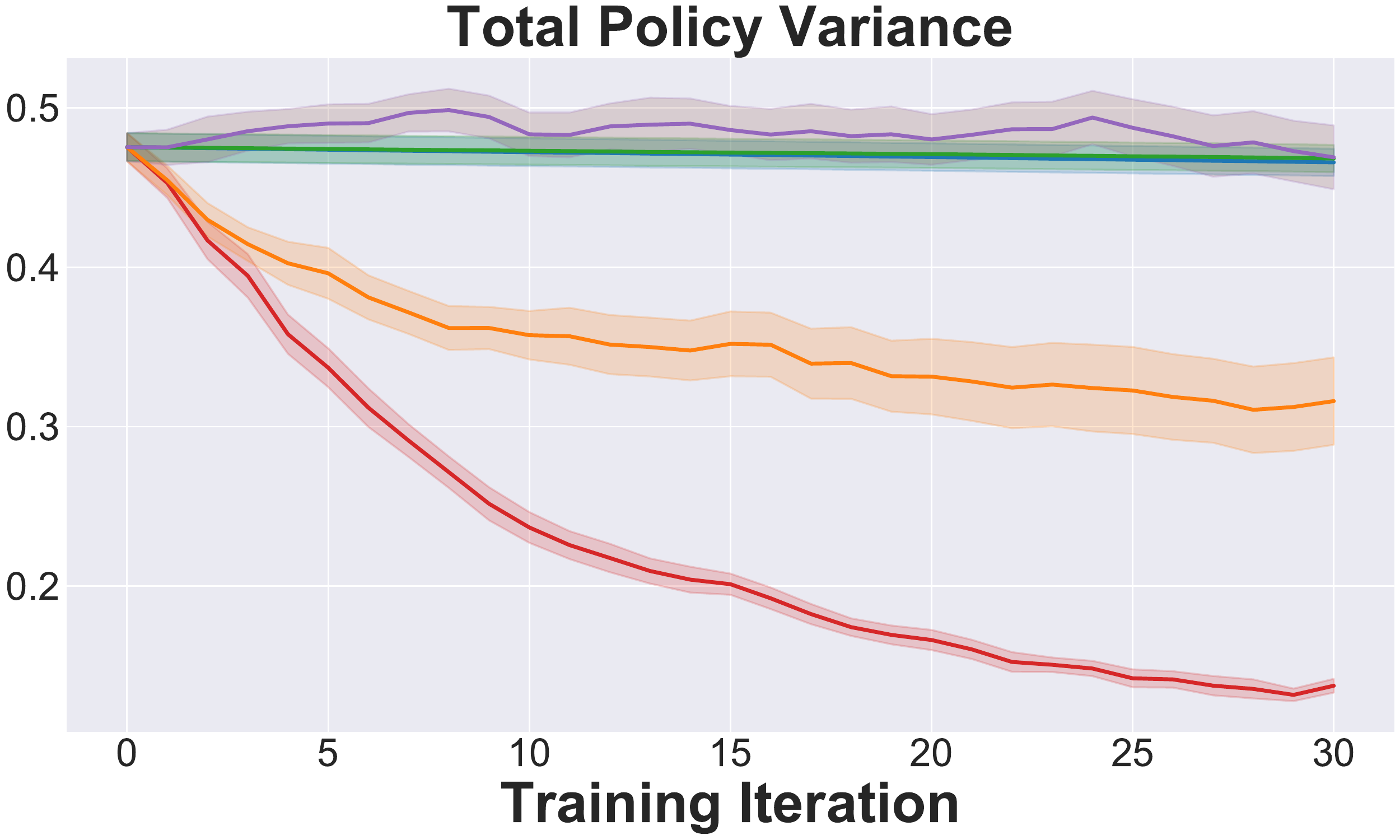}
    \end{subfigure}
    ~
    \begin{subfigure}{0.32\textwidth}
        \centering
        \includegraphics[width=\textwidth]{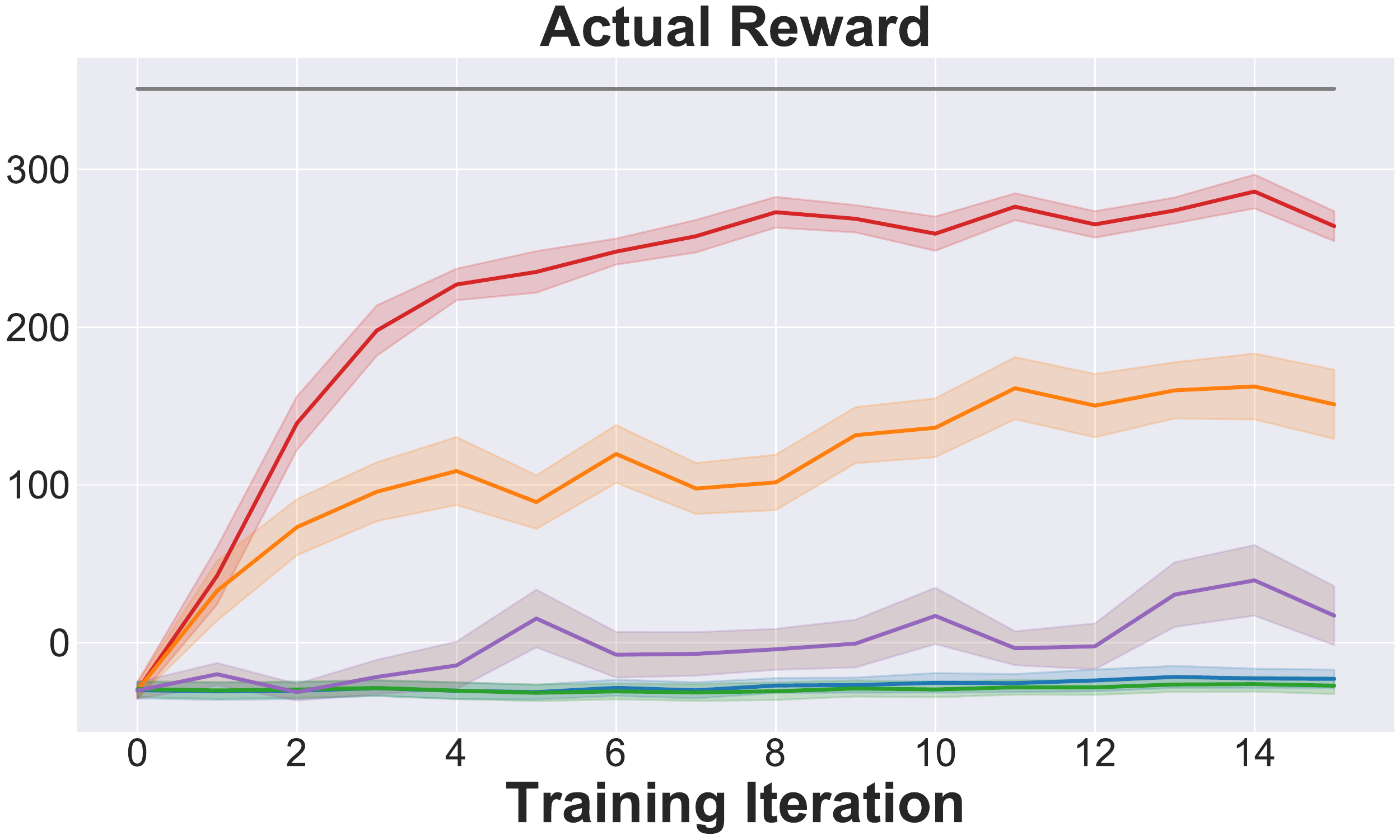}
    \end{subfigure}
    \caption{Map C}
    \end{subfigure}
    
    \begin{subfigure}{\textwidth}
    \begin{subfigure}{0.32\textwidth}
        \centering
        \includegraphics[width=\textwidth]{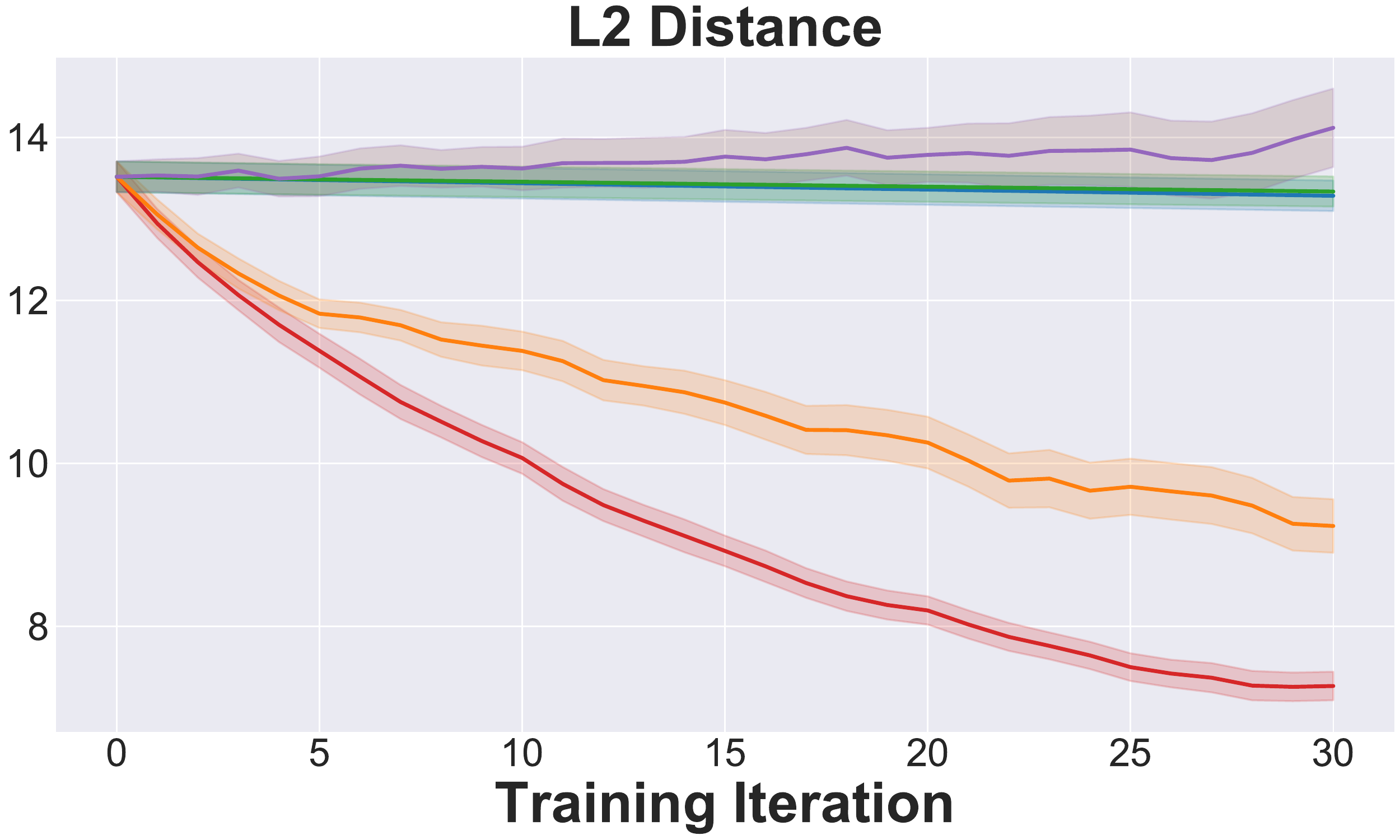}
    \end{subfigure}
    ~
    \begin{subfigure}{0.32\textwidth}
        \centering
        \includegraphics[width=\textwidth]{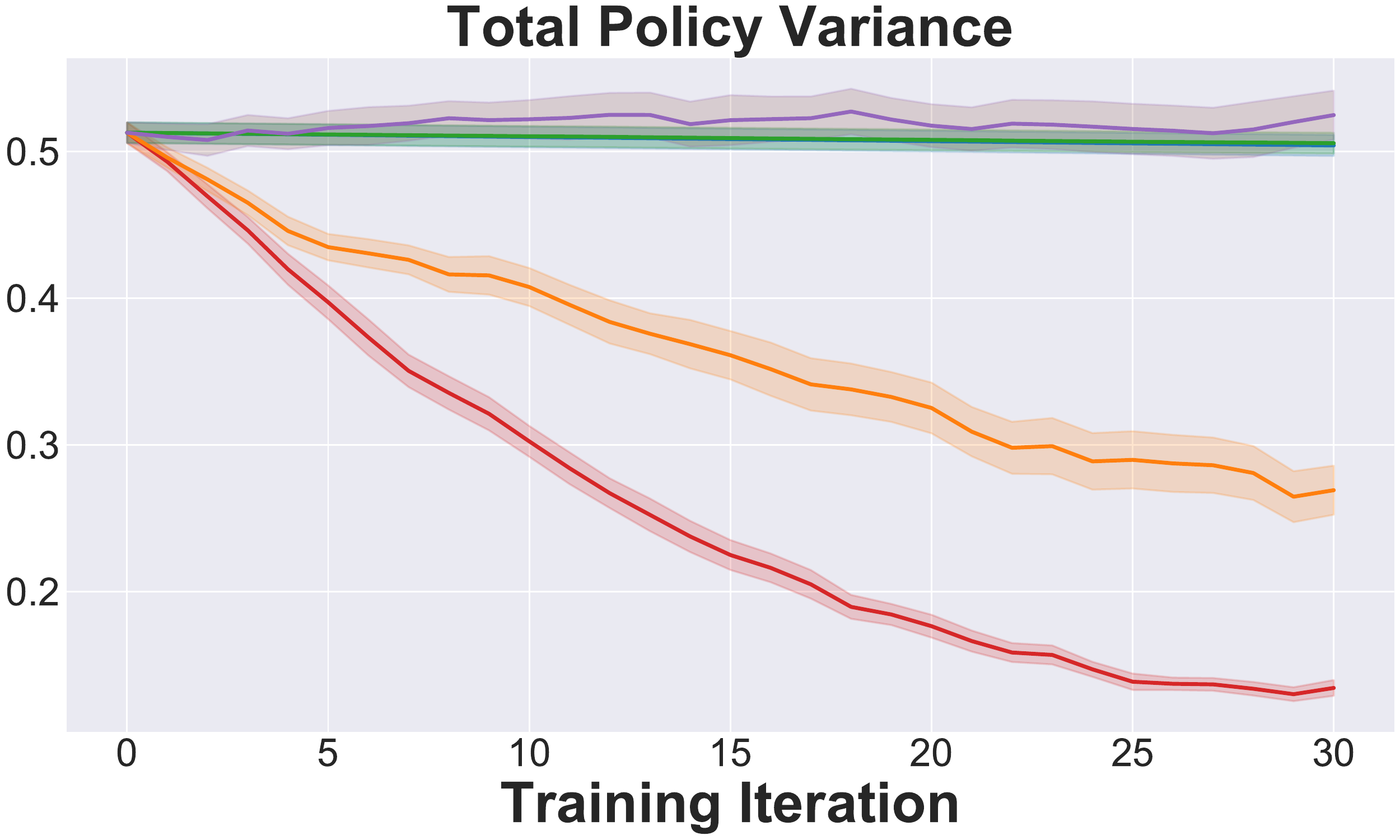}
    \end{subfigure}
    ~
    \begin{subfigure}{0.32\textwidth}
        \centering
        \includegraphics[width=\textwidth]{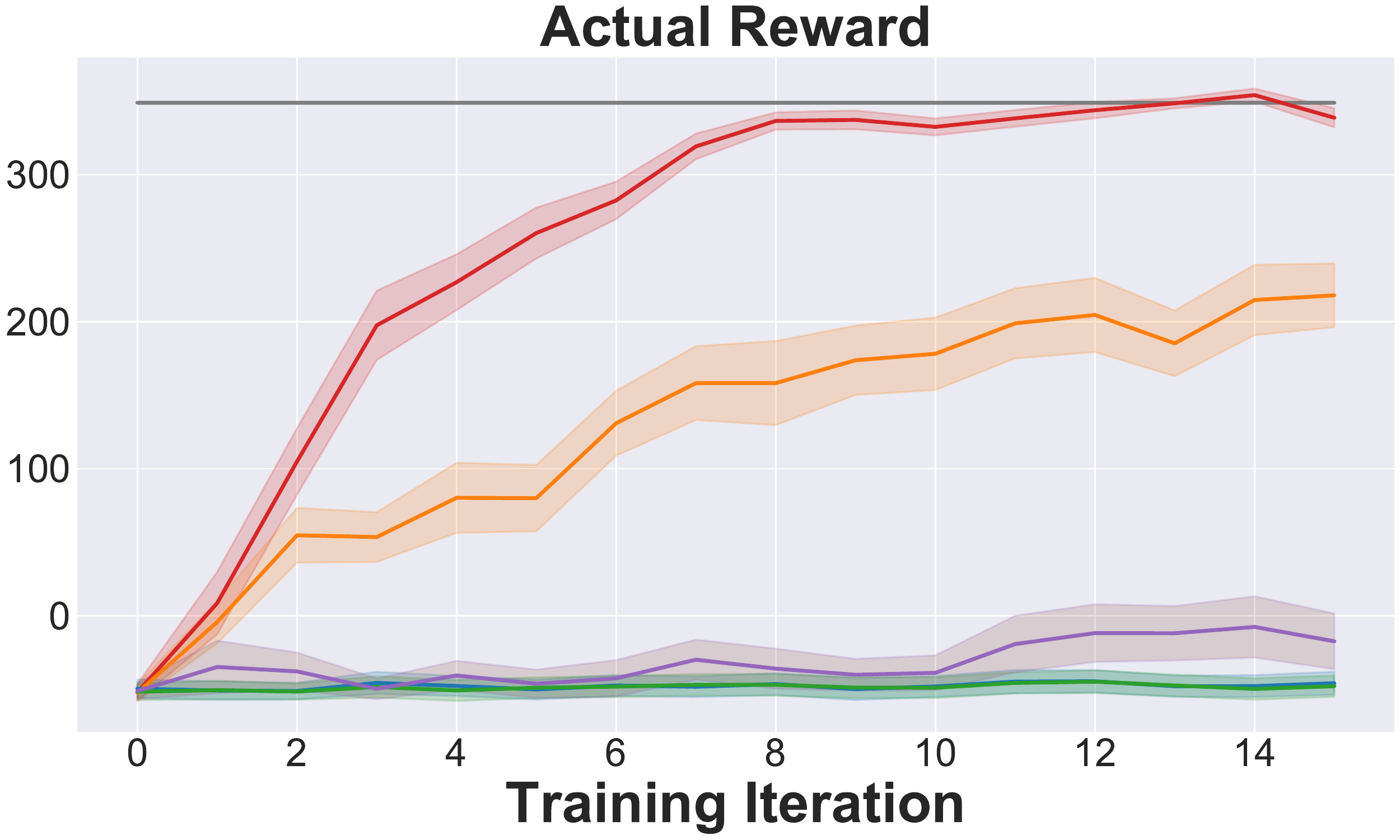}
    \end{subfigure}
    \caption{Map D}
    \end{subfigure}
    
    \begin{subfigure}{\textwidth}
    \begin{subfigure}{0.32\textwidth}
        \centering
        \includegraphics[width=\textwidth]{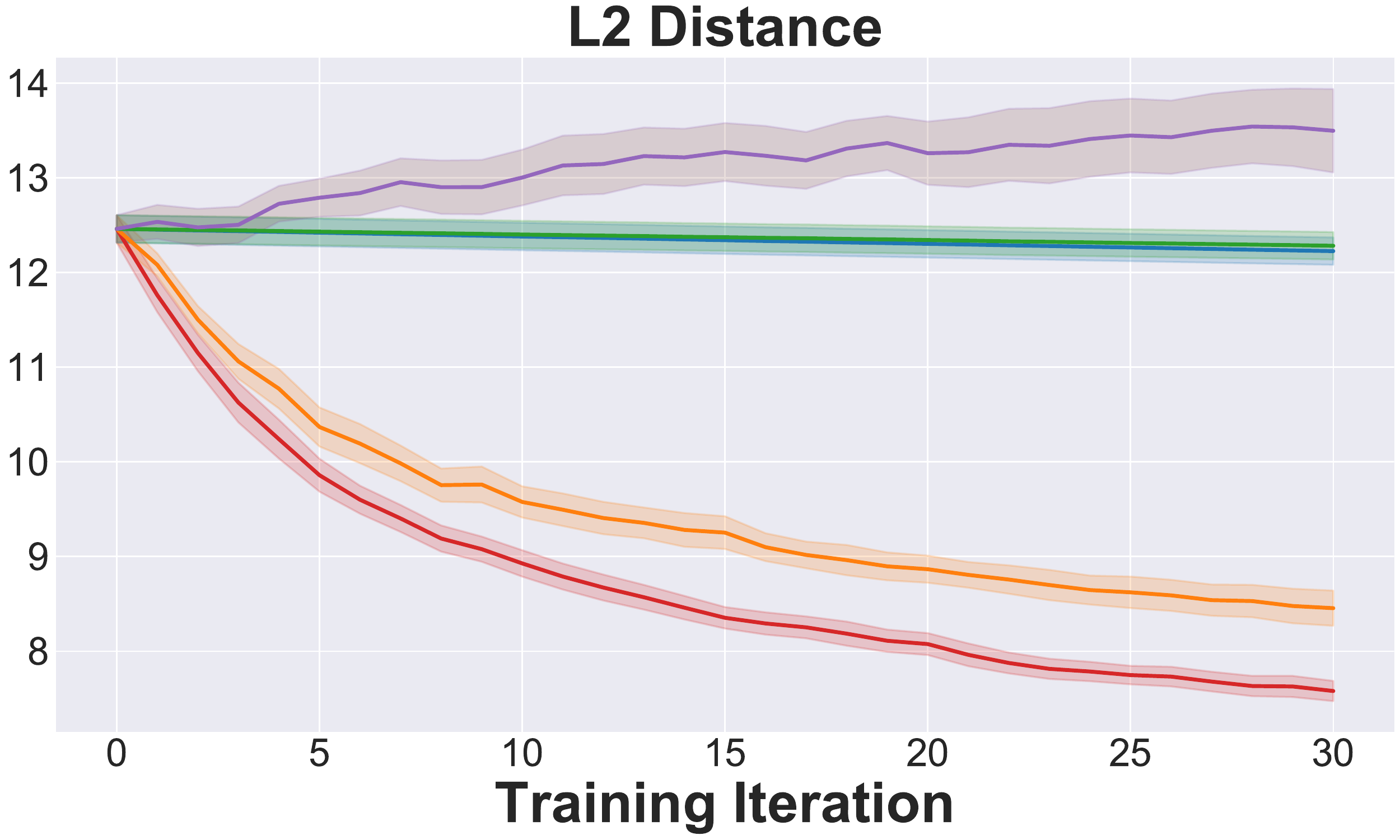}
    \end{subfigure}
    ~
    \begin{subfigure}{0.32\textwidth}
        \centering
        \includegraphics[width=\textwidth]{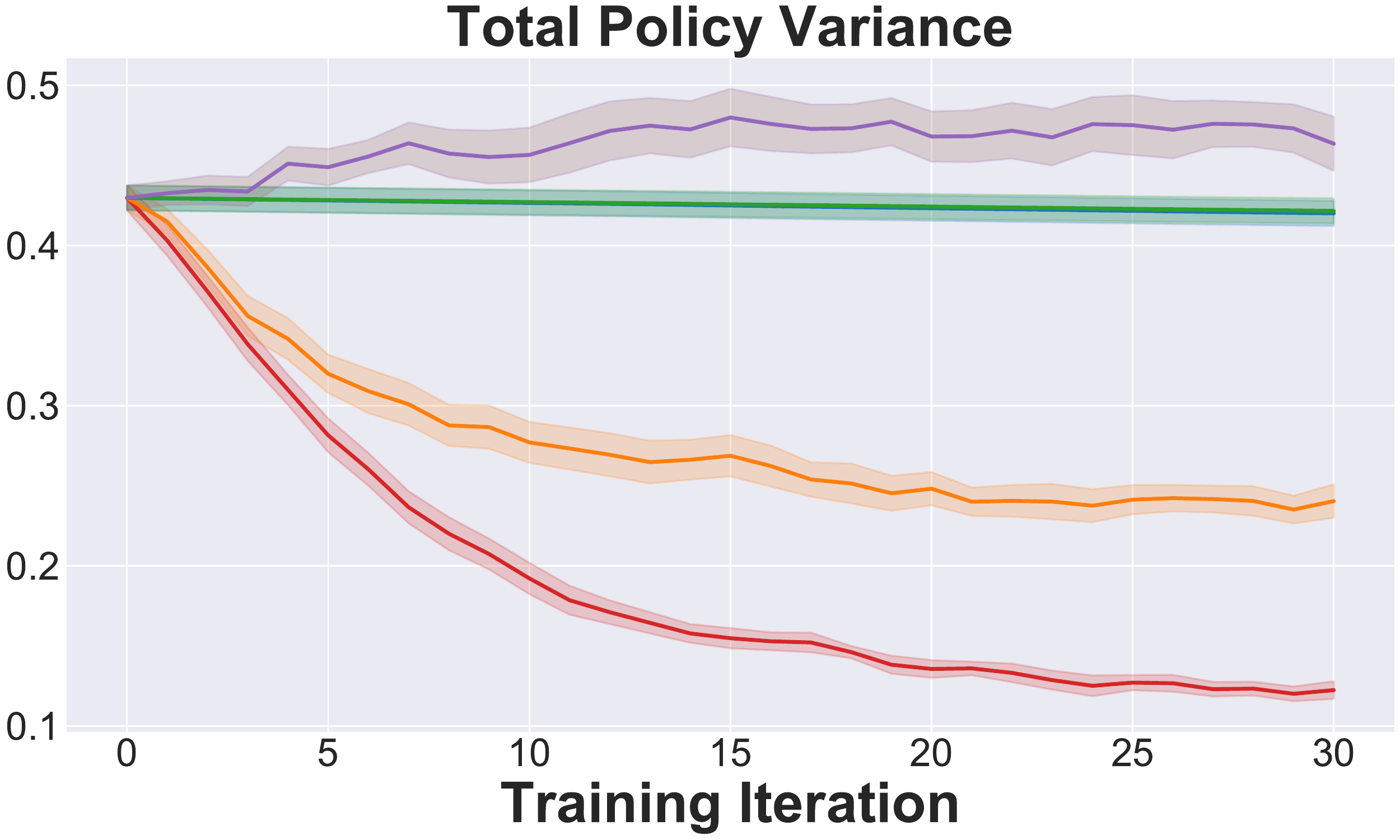}
    \end{subfigure}
    ~
    \begin{subfigure}{0.32\textwidth}
        \centering
        \includegraphics[width=\textwidth]{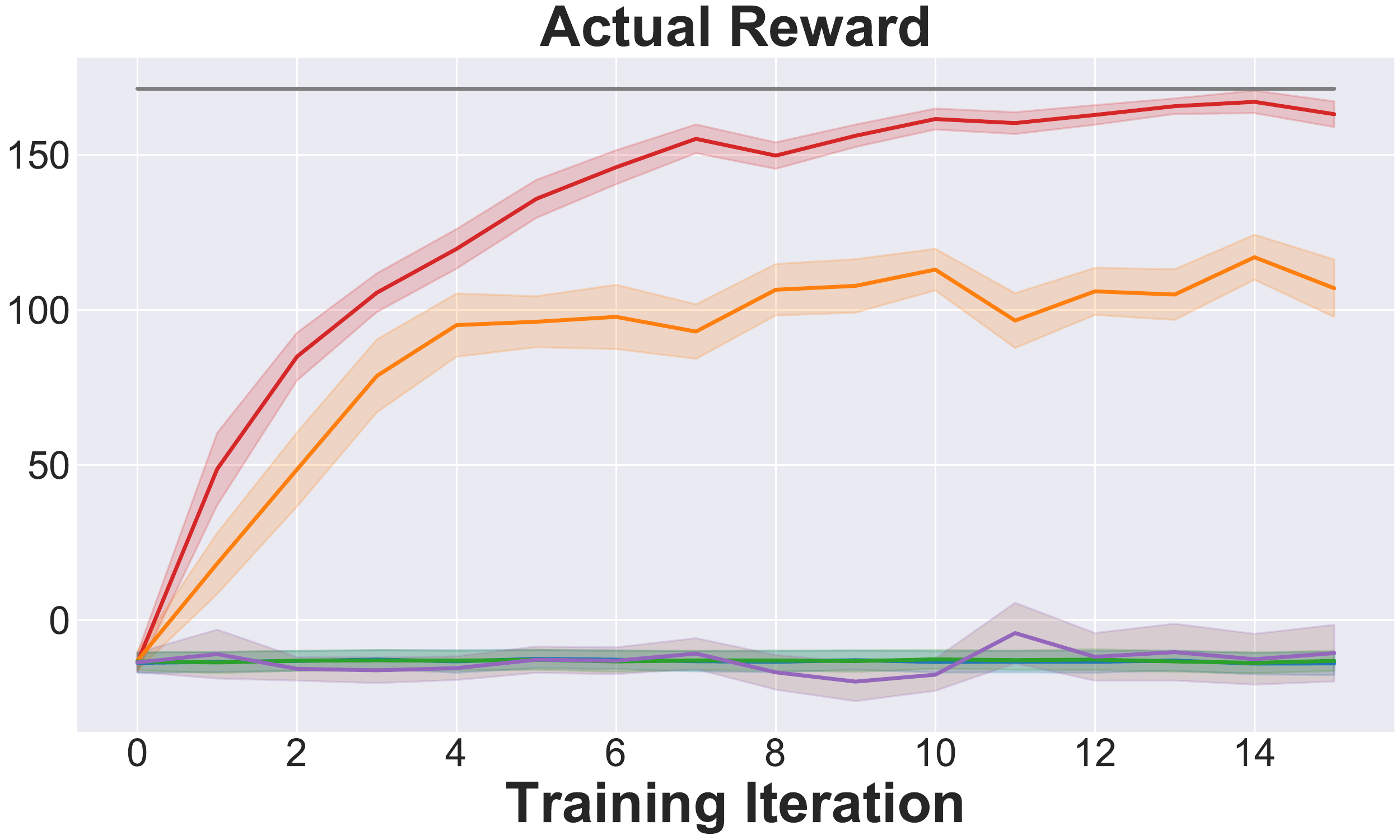}
    \end{subfigure}
    \caption{Map E}
    \end{subfigure}
    
    \begin{subfigure}{\textwidth}
        \centering
        \includegraphics[width=\textwidth]{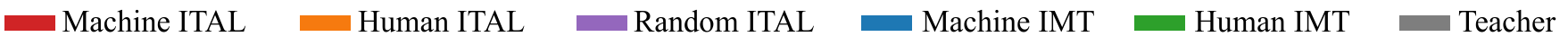}
    \end{subfigure}
    \caption{Learning curves for each map.}
    \label{sup:fig:human_curves}
\end{figure}
We conducted a proof-of-concept human study on 20 university students, 10 females and 10 males. We want to validate that our teacher-aware learner can also outperform naive learners given a human teacher. In other words, our teacher model can be applied to human teachers, despite of their potentially different pedagogical patterns. The goal of the experiment is for the participant to teach the reward of a ground-truth reward map to a learner. To reduce human subjects' cognitive burden, we use three types of tiles (red, blue and white) on the map to represent bad, good and neutral grids. We used 5 different map configurations shown in figure~\ref{sup:fig:map_configuration}. The learner's current reward map is shown to the human participants during the entire teaching session. As the reward values are continuous at the learner's side, we used the color pallet in figure~\ref{sup:fig:human_interface} to render the grids in the learner's map. We also included a map indicating the most probable actions the learner will take given his current reward map, so that the human teacher can tell which grid the learner attaches a higher reward if some neighboring grids have similar colors. The directions of these arrows are calculated with value iteration using the learner's current reward parameters. An example human interface was shown in figure~\ref{sup:fig:human_interface}. In each time step, ten arrows will be drawn on ten randomly sampled grids. Selecting one of the arrows tells the learner that he should follow the arrow's direction if he was at this grid. Then the learner will update his reward parameters based on this instruction using the same equation~\eqref{sup:eq:irl_loss} as in section~\ref{sup:sec:exp-oirl}. 

We hold the experiment as a Jupyter Notebook~\citep{Kluyver:2016aa} and launch it via Binder~\citep{jupyter2018binder}. We first introduce the experiment logic to the human subjects and include a short warm-up phase for the subjects to get familiar with the learner's update process. Then, we let the subjects to teach the maps, starting from Map A to Map E. Every subject needs to teach both a teacher-aware learner and a naive learner, whose order is randomly determined. For every map, the initialization of the two learners are the same for the same human subject and different across subjects. Like the inverse RL experiment in section~\ref{sup:sec:exp-oirl}, we evaluate the learning results in terms of L2-distance between the learners' reward parameters and the ground-truth parameters, the total variance between the learners' policy and the policy derived with the ground-truth reward and the actual accumulated reward acquired by the learner after the learning completes.

The results are presented in figure~\ref{sup:fig:human_study_results}. The advantages of the teacher-aware learners are significant ($p$-value $< 0.01$) on all measurements, computed with a paired t-test. We also did an ablative study, in which the human teacher was replaced by a random teacher. As shown in the figure~\ref{sup:fig:human_study_results} and~\ref{sup:fig:human_curves}, when paired with a random teacher, the teacher-aware learner doesn't show any advantage and has much larger variance. That is to say, the teacher model only benefits the learning when it matches with the actual teacher data selection process. Otherwise, the teacher-aware learner will over-interpret the data he receives. Figure~\ref{sup:fig:human_curves} shows the learning curves of all the map configurations.

\bibliography{reference}
\bibliographystyle{IEEEtranSN}

\makeatletter\@input{xx.tex}\makeatother